\documentclass[11pt,oneside,letterpaper]{article}
\usepackage[mode=buildnew]{standalone}
\usepackage{eurosym}
\usepackage{graphicx,amsmath,amsfonts}
\usepackage{setspace}
\usepackage{amssymb}
\usepackage{dsfont}
\usepackage{amsmath}
\usepackage{amsfonts}
\usepackage{ifthen}
\usepackage{mathtools}
\usepackage{graphicx}
\usepackage{enumitem}
\usepackage{color}
\usepackage{framed}
\usepackage[margin=1in]{geometry}
\usepackage{natbib}
\usepackage[dvipsnames]{xcolor}
\usepackage[colorlinks=true, urlcolor=Mahogany, linkcolor=Mahogany, citecolor=Mahogany]{hyperref}
\usepackage{fp}
\usepackage{amsthm}
\usepackage{caption}
\usepackage{subcaption}
\usepackage{palatino}
\usepackage{cuted,balance}
\usepackage{amssymb}
\usepackage{soul}
\usepackage{graphicx,amsmath,amsfonts}
\usepackage{float}
\usepackage{bbm}

\usepackage{changes}

\usepackage{pgfplots}
\pgfplotsset{compat=1.18}

\usepackage{macros}

\usepackage[ruled]{algorithm2e}

\usepackage{bm}

\usepackage{tikz}
\usetikzlibrary{positioning}
\usepackage{pgfplots}
\usetikzlibrary{intersections, pgfplots.fillbetween}

\begin{document}

\title{The Statistical Fairness-Accuracy Frontier}
\author{Alireza Fallah\thanks{Department of Computer Science and Ken Kennedy Institute, Rice University} \and Michael I. Jordan\thanks{Departments of Electrical Engineering and Computer Sciences and Statistics, University of California, Berkeley; Inria and l'Ecole Normale Supérieure, Paris, France} \and Annie Ulichney\thanks{Department of Statistics, University of California, Berkeley}}

\date{October 2026 \\ \footnotesize{First version: August 2025}}

\maketitle

\sloppy

\begin{abstract}
We study fairness–accuracy tradeoffs when a single predictive model must serve multiple demographic groups. A useful tool for understanding this tradeoff is the fairness–accuracy (FA) Pareto frontier, which characterizes the set of models that cannot be improved in either fairness or accuracy without worsening the other. While characterizing the FA frontier requires full knowledge of the data distribution, we focus on the finite-sample regime, quantifying how well a designer can approximate any point on the frontier from limited data and bounding the worst-case gap. In particular, we derive worst-case-optimal estimators that depend on the designer’s knowledge of the covariate distribution. For each estimator, we characterize how finite-sample effects asymmetrically impact each group’s welfare and identify optimal sample allocation strategies. Finally, we provide uniform finite-sample bounds along the entire FA frontier, quantifying finite-sample deviations across alternative fairness–accuracy tradeoffs.    
\end{abstract}

\clearpage 
\setcounter{tocdepth}{2}
\tableofcontents
\newpage

\section{Introduction}
Across domains where predictions guide high-stakes decisions---from lending and hiring to healthcare and education---policymakers increasingly rely on algorithms that must serve heterogeneous populations. A central, practically relevant challenge arises when a \emph{single predictive model} must serve multiple groups, whether due to legal, logistical, or normative constraints. For instance, in the United States, the Equal Credit Opportunity Act prohibits lending decisions based on protected attributes such as race or sex and is commonly interpreted to disallow the use of different predictive models for different groups. Similarly, in hiring, automated screening systems typically apply a single predictive model to all applicants, since Title VII of the Civil Rights Act discourages the use of group-specific models.\footnote{See \citet{raghavan2020mitigating} for a discussion of U.S. employment discrimination law and the guidelines on ``disparate treatment''.}

In such cases, two objectives collide. On the one hand, an efficiency-minded social planner wants to minimize prediction error, since errors translate into misallocated loans, mistargeted treatments, or inefficient school placements \citep{kleinberg2015prediction, kleinberg2018human}. On the other hand, an equity-minded planner cares about how those errors are distributed across groups: large systematic gaps in error rates are perceived as unfair and may violate anti-discrimination law \citep{bartlett2022consumer, fuster2022predictably, dastin2022amazon, obermeyer2019dissecting}. In lending, for example, the relationship between applicants' characteristics and default risk can differ across groups, in part because credit scores are a noisier signal of default for historically underserved groups \citep{blattner2021costly}; a model that predicts well for one group may therefore predict poorly for another. 

These competing forces give rise to a \emph{fairness--accuracy} tradeoff, a special kind of \emph{equity--efficiency} tradeoff, in which the model's overall predictive performance is in tension with its allocation of errors across groups.\footnote{See \cite{liang2026algorithm} for further discussion relating fairness--accuracy tradeoffs and equity--efficiency tradeoffs.} The resulting \emph{fairness--accuracy (FA) frontier} summarizes the tradeoff. This object---studied in detail by \citet{liang2026algorithm}---has a natural ordinal interpretation: it is the set of models that are not dominated by any other model that simultaneously delivers (weakly) lower errors for each group and (weakly) lower disparity in error across groups, with at least one strict improvement. This ordering is compatible with a broad range of planners’ preferences, from a utilitarian perspective that minimizes a (possibly weighted) average of the two groups’ errors to a more egalitarian perspective that places particular weight on reducing the absolute difference between them.  

Existing characterizations of the FA frontier, including \citet{liang2026algorithm}, take as a primitive a planner who knows the joint population distribution over outcomes, covariates, and group identity. In that idealized benchmark, the frontier is a purely technological object, one that could be computed exactly only with infinite data: given a fully known data-generating process and a feasible set of prediction rules, one can trace out the attainable combinations of groupwise errors and make a clean, ex ante welfare comparison across policies. But in practice a policymaker never observes the population distribution. They have access to only a finite number of observations---often imbalanced across groups---and must choose how to construct an empirical estimator to make predictions from those data. A lender may observe characteristics for many applicants but repayment outcomes for relatively few borrowers, with fewer outcome observations for some groups than others. As a result, the realized policy is not the population-optimal point on the FA frontier, but an empirical approximation. This raises a set of questions that are invisible in the full-information benchmark: How large is the \textit{welfare loss} from replacing the population-optimal predictor with an empirically estimated one? How should a planner \textit{allocate limited sampling resources} across groups to minimize that loss? And how does sampling error distort the \textit{realized distribution of errors across groups}, relative to the planner’s stated fairness preferences?

We address these questions in a linear model with two groups and establish three key findings. First, we find that the cost of estimation is borne unevenly between the groups. Even when estimation raises the planner's overall loss only slightly, it can increase one group's prediction error while reducing the other's. The estimated model can therefore depart from the planner's intended fairness--accuracy tradeoff. Second, we derive estimators that attain the smallest estimation cost in the worst case, enabling the planner to implement her desired tradeoff as effectively as possible with the information available to her. Importantly, we establish that the planner should estimate the model differently depending on what she already knows about group features. For example, a lender that knows the distributions of applicants' characteristics should fit a different model than a lender that does not. Third, we identify how the planner should, in turn, collect data to reduce these costs. We find that, like the choice of estimator, the optimal sampling procedure depends on her information about group features as well as the heterogeneity between groups. The remainder of the introduction describes these results in more detail, beginning with our model.

Formally, we work in a stylized yet rich linear regression environment with two demographic groups, ``red” and ``blue.”\footnote{We discuss extensions to several groups in \Cref{sec:several_groups}.} Outcomes are generated by group-specific linear models where the covariate distributions and regression coefficients differ across groups. However, a single linear predictor must be used for everyone, reflecting the aforementioned institutional constraints that prohibit group-specific models in many settings. For any candidate predictor, we define the \textit{prediction error} for each group as the expected squared loss, and take fairness to be the absolute difference in group errors (known as the \textit{equalized loss} in the fairness literature; see, e.g., \citet{zhang2019group, khalili2023loss}). As illustrated in \Cref{fig:Frontier_Intro_0}, the FA frontier can be visualized in a two-dimensional plot. The risks for the red and blue groups are represented on the two axes, while the population FA frontier traces the lower boundary of the feasible set of group error pairs. Under mild conditions on the difference in noise variance across groups---satisfied, in particular, in the homoskedastic case---the FA frontier is  \textit{group-balanced} in the sense of \cite{liang2026algorithm}: the most fair point lies between the two group-optimal endpoints. Hence, as displayed in \Cref{fig:Frontier_Intro_0}, the two endpoints of the FA frontier lie at the group-$r$ optimum $\beta_r$ and the group-$b$ optimum $\beta_b$.  

\begin{figure}[t]
\centering
\begin{subfigure}{.48\textwidth}
  \centering
  \begin{tikzpicture}[scale=13, line cap=round, line join=round]

\pgfmathsetmacro{\cx}{0.43}
\pgfmathsetmacro{\cy}{0.42}
\pgfmathsetmacro{\a}{0.41}
\pgfmathsetmacro{\b}{0.28}
\pgfmathsetmacro{\theta}{30}
\pgfmathsetmacro{\ct}{cos(\theta)}
\pgfmathsetmacro{\st}{sin(\theta)}
\pgfmathsetmacro{\pointSize}{0.007}
\pgfmathsetmacro{\lineWidth}{0.35mm}

\pgfmathsetmacro{\tYraw}{atan2(\b*\ct,\a*\st)}
\pgfmathsetmacro{\tY}{\tYraw + 180}
\pgfmathsetmacro{\betaBx}{\cx + \a*cos(\tY)*\ct - \b*sin(\tY)*\st}
\pgfmathsetmacro{\betaBy}{\cy + \a*cos(\tY)*\st + \b*sin(\tY)*\ct}

\pgfmathsetmacro{\tXraw}{atan2(-\b*\st,\a*\ct)}
\pgfmathsetmacro{\tX}{\tXraw + 180}
\pgfmathsetmacro{\betaRx}{\cx + \a*cos(\tX)*\ct - \b*sin(\tX)*\st}
\pgfmathsetmacro{\betaRy}{\cy + \a*cos(\tX)*\st + \b*sin(\tX)*\ct}

\pgfmathsetmacro{\posOnFrontier}{0.44}

\pgfmathsetmacro{\tStart}{\tX}
\pgfmathsetmacro{\tEnd}{\tY}
\pgfmathsetmacro{\tEndAdjusted}{\tY < \tX ? \tY + 360 : \tY}

\pgfmathsetmacro{\tLambda}{\tStart + \posOnFrontier*(\tEndAdjusted - \tStart)}
\pgfmathsetmacro{\betaLambdax}{\cx + \a*cos(\tLambda)*\ct - \b*sin(\tLambda)*\st}
\pgfmathsetmacro{\betaLambday}{\cy + \a*cos(\tLambda)*\st + \b*sin(\tLambda)*\ct}

\pgfmathsetmacro{\dxdt}{-\a*sin(\tLambda)*\ct - \b*cos(\tLambda)*\st}
\pgfmathsetmacro{\dydt}{-\a*sin(\tLambda)*\st + \b*cos(\tLambda)*\ct}
\pgfmathsetmacro{\mtan}{\dydt/\dxdt}
\pgfmathsetmacro{\xinter}{\betaLambdax - \betaLambday/\mtan} 
\pgfmathsetmacro{\yinter}{\betaLambday - \mtan*\betaLambdax} 

\pgfmathsetmacro{\Acoef}{\a*(\ct - \st)}
\pgfmathsetmacro{\Bcoef}{\b*(-\st - \ct)}
\pgfmathsetmacro{\Delta}{\cy - \cx}
\pgfmathsetmacro{\Rcoef}{sqrt(\Acoef*\Acoef + \Bcoef*\Bcoef)}
\pgfmathsetmacro{\phi}{atan2(\Bcoef,\Acoef)}
\pgfmathsetmacro{\ratio}{min(1,max(-1,\Delta/\Rcoef))}
\pgfmathsetmacro{\alpha}{acos(\ratio)}
\pgfmathsetmacro{\tFairA}{\phi + \alpha}
\pgfmathsetmacro{\tFairB}{\phi - \alpha}
\pgfmathsetmacro{\xA}{\cx + \a*cos(\tFairA)*\ct - \b*sin(\tFairA)*\st}
\pgfmathsetmacro{\yA}{\cy + \a*cos(\tFairA)*\st + \b*sin(\tFairA)*\ct}
\pgfmathsetmacro{\xB}{\cx + \a*cos(\tFairB)*\ct - \b*sin(\tFairB)*\st}
\pgfmathsetmacro{\yB}{\cy + \a*cos(\tFairB)*\st + \b*sin(\tFairB)*\ct}
\pgfmathsetmacro{\xfair}{(\xA<=\cx && \yA<=\cy) ? \xA : \xB}
\pgfmathsetmacro{\yfair}{(\xA<=\cx && \yA<=\cy) ? \yA : \yB}

\pgfmathsetmacro{\xmax}{max(\betaBx,\betaRx) + 0.10}
\pgfmathsetmacro{\ymax}{max(\betaBy,\betaRy) + 0.10}
\pgfmathsetmacro{\diag}{min(\xmax,\ymax)}

\draw[->] (0,0) -- (\xmax,0);
\draw[->] (0,0) -- (0,\ymax);

\node[anchor=north, BrickRed] at ({\xmax/2},0) {Red group error};
\node[rotate=90, anchor=south, Blue] at (0,{\ymax/2}) {Blue group error};

\draw[dashed, ForestGreen!70, line width=\lineWidth] (0,0) -- (\diag,\diag);

\draw[ForestGreen!70] (0.08,0) arc[start angle=0, end angle=45, radius=0.08];

\pgfmathsetmacro{\angr}{0.06}
\node[ForestGreen!70] at ({\angr*cos(22.5)},{\angr*sin(22.5)}) {\scriptsize $45^\circ$};

\node[ForestGreen!80!black, rotate=45, anchor=south east]
  at (\diag,\diag) {\scriptsize $x = y$};

\begin{scope}
  \clip (0,0) rectangle (\xmax,\ymax);

  \begin{scope}[shift={(\cx,\cy)}, rotate=\theta]
    \fill[CornflowerBlue, opacity=0.15] (0,0) ellipse [x radius=\a, y radius=\b];
    \draw[CornflowerBlue!60] (0,0) ellipse [x radius=\a, y radius=\b];
  \end{scope}


  \begin{scope}[shift={(\cx,\cy)}, rotate=\theta]
    \draw[ultra thick, Thistle]
      ({\a*cos(\tStart)},{\b*sin(\tStart)})
        arc[start angle=\tStart, end angle=\tEndAdjusted, x radius=\a, y radius=\b];
  \end{scope}
\end{scope}

\fill[Blue]     (\betaBx,\betaBy) circle[radius=\pointSize] node[below, Blue] {$\beta_b$};
\fill[BrickRed] (\betaRx,\betaRy) circle[radius=\pointSize] node[left,  BrickRed] {$\beta_r$};

\fill[ForestGreen] (\xfair,\yfair) circle[radius=\pointSize];
\node[ForestGreen] at ({\xfair},{\yfair-0.035}) {$\beta_{\text{fair}}$};


\end{tikzpicture}
  \caption{}
  \label{fig:Frontier_Intro_0}
\end{subfigure}
\hfill
\begin{subfigure}{.48\textwidth}
  \centering
  \begin{tikzpicture}[scale=13, line cap=round, line join=round]

\pgfmathsetmacro{\cx}{0.43}
\pgfmathsetmacro{\cy}{0.42}
\pgfmathsetmacro{\a}{0.41}
\pgfmathsetmacro{\b}{0.28}
\pgfmathsetmacro{\theta}{30}
\pgfmathsetmacro{\ct}{cos(\theta)}
\pgfmathsetmacro{\st}{sin(\theta)}
\pgfmathsetmacro{\pointSize}{0.007}
\pgfmathsetmacro{\lineWidth}{0.35mm}

\pgfmathsetmacro{\tYraw}{atan2(\b*\ct,\a*\st)}
\pgfmathsetmacro{\tY}{\tYraw + 180}
\pgfmathsetmacro{\betaBx}{\cx + \a*cos(\tY)*\ct - \b*sin(\tY)*\st}
\pgfmathsetmacro{\betaBy}{\cy + \a*cos(\tY)*\st + \b*sin(\tY)*\ct}

\pgfmathsetmacro{\tXraw}{atan2(-\b*\st,\a*\ct)}
\pgfmathsetmacro{\tX}{\tXraw + 180}
\pgfmathsetmacro{\betaRx}{\cx + \a*cos(\tX)*\ct - \b*sin(\tX)*\st}
\pgfmathsetmacro{\betaRy}{\cy + \a*cos(\tX)*\st + \b*sin(\tX)*\ct}

\pgfmathsetmacro{\posOnFrontier}{0.44}

\pgfmathsetmacro{\tStart}{\tX}
\pgfmathsetmacro{\tEnd}{\tY}
\pgfmathsetmacro{\tEndAdjusted}{\tY < \tX ? \tY + 360 : \tY}

\pgfmathsetmacro{\tLambda}{\tStart + \posOnFrontier*(\tEndAdjusted - \tStart)}
\pgfmathsetmacro{\betaLambdax}{\cx + \a*cos(\tLambda)*\ct - \b*sin(\tLambda)*\st}
\pgfmathsetmacro{\betaLambday}{\cy + \a*cos(\tLambda)*\st + \b*sin(\tLambda)*\ct}

\pgfmathsetmacro{\dxdt}{-\a*sin(\tLambda)*\ct - \b*cos(\tLambda)*\st}
\pgfmathsetmacro{\dydt}{-\a*sin(\tLambda)*\st + \b*cos(\tLambda)*\ct}
\pgfmathsetmacro{\mtan}{\dydt/\dxdt}
\pgfmathsetmacro{\xinter}{\betaLambdax - \betaLambday/\mtan} 
\pgfmathsetmacro{\yinter}{\betaLambday - \mtan*\betaLambdax} 

\pgfmathsetmacro{\Acoef}{\a*(\ct - \st)}
\pgfmathsetmacro{\Bcoef}{\b*(-\st - \ct)}
\pgfmathsetmacro{\Delta}{\cy - \cx}
\pgfmathsetmacro{\Rcoef}{sqrt(\Acoef*\Acoef + \Bcoef*\Bcoef)}
\pgfmathsetmacro{\phi}{atan2(\Bcoef,\Acoef)}
\pgfmathsetmacro{\ratio}{min(1,max(-1,\Delta/\Rcoef))}
\pgfmathsetmacro{\alpha}{acos(\ratio)}
\pgfmathsetmacro{\tFairA}{\phi + \alpha}
\pgfmathsetmacro{\tFairB}{\phi - \alpha}
\pgfmathsetmacro{\xA}{\cx + \a*cos(\tFairA)*\ct - \b*sin(\tFairA)*\st}
\pgfmathsetmacro{\yA}{\cy + \a*cos(\tFairA)*\st + \b*sin(\tFairA)*\ct}
\pgfmathsetmacro{\xB}{\cx + \a*cos(\tFairB)*\ct - \b*sin(\tFairB)*\st}
\pgfmathsetmacro{\yB}{\cy + \a*cos(\tFairB)*\st + \b*sin(\tFairB)*\ct}
\pgfmathsetmacro{\xfair}{(\xA<=\cx && \yA<=\cy) ? \xA : \xB}
\pgfmathsetmacro{\yfair}{(\xA<=\cx && \yA<=\cy) ? \yA : \yB}

\pgfmathsetmacro{\xmax}{max(\betaBx,\betaRx) + 0.10}
\pgfmathsetmacro{\ymax}{max(\betaBy,\betaRy) + 0.10}
\pgfmathsetmacro{\diag}{min(\xmax,\ymax)}

\draw[->] (0,0) -- (\xmax,0);
\draw[->] (0,0) -- (0,\ymax);

\node[anchor=north, BrickRed] at ({\xmax/2},0) {Red group error};
\node[rotate=90, anchor=south, Blue] at (0,{\ymax/2}) {Blue group error};





\begin{scope}
  \clip (0,0) rectangle (\xmax,\ymax);

  \begin{scope}[shift={(\cx,\cy)}, rotate=\theta]
    \fill[CornflowerBlue, opacity=0.15] (0,0) ellipse [x radius=\a, y radius=\b];
    \draw[CornflowerBlue!60] (0,0) ellipse [x radius=\a, y radius=\b];
  \end{scope}

  \draw[dashed, Fuchsia!70, line width=\lineWidth] (0,\yinter) -- (\xinter,0);

  \begin{scope}[shift={(\cx,\cy)}, rotate=\theta]
    \draw[ultra thick, Thistle]
      ({\a*cos(\tStart)},{\b*sin(\tStart)})
        arc[start angle=\tStart, end angle=\tEndAdjusted, x radius=\a, y radius=\b];
  \end{scope}
\end{scope}

\fill[Blue]     (\betaBx,\betaBy) circle[radius=\pointSize] node[below, Blue] {$\beta_b$};
\fill[Blue]     (\betaBx,\betaBy) circle[radius=\pointSize] node[above, Blue] {\tiny($\lambda = 0$)};
\fill[BrickRed] (\betaRx,\betaRy) circle[radius=\pointSize] node[left,  BrickRed] {$\beta_r$};
\fill[BrickRed] (\betaRx,\betaRy) circle[radius=\pointSize] node[right, BrickRed] {\tiny($\lambda = 1$)};
\fill[Fuchsia]  (\betaLambdax,\betaLambday) circle[radius=\pointSize] node[below left, Fuchsia] {$\beta_\lambda$};


\node[text=Fuchsia, anchor=south east]
  at ({\xinter-0.05}, {0.02}) {\footnotesize{$\lambda x + (1 - \lambda) y = c$}};

\end{tikzpicture}
  \caption{}
  \label{figs:Frontier_Intro_1}
\end{subfigure}
\caption{(a) Population FA frontier: The blue shaded area represents the set of error pairs that are achievable by some linear model, and $\beta_r$ and $\beta_b$ are the error-minimizing models for groups $r$ and $b$, respectively. $\beta_\text{fair}$ is the error-equalizing model, which falls between $\beta_r$ and $\beta_b$ in the group-balanced case. The FA frontier corresponds to the purple region along the boundary of the feasible set of error pairs. (b) For any $\lambda \in [0,1]$, $\beta_\lambda$ is the first point of tangency between the line $\lambda x + (1-\lambda) y=c$ and the FA frontier as we increase $c>0$. This point  moves along the frontier from $\beta_r$ to $\beta_b$ as $\lambda$ ranges from $1 \to 0$. }
\label{fig:FA_Intro_1}
\end{figure}


In the group-balanced case, we parametrize this frontier by a weight $\lambda \in [0,1]$ on the red group’s error in a utilitarian social objective: for each $\lambda$, the policy $\beta_\lambda$ minimizes a weighted average of the two group errors and therefore represents the planner’s \textit{target} point on the fairness–accuracy frontier. Consequently, $\beta_\lambda$ is the first point of tangency between the frontier and a line of the form $\lambda x + (1 - \lambda) y = c$ as $c > 0$ increases, as illustrated in \Cref{figs:Frontier_Intro_1}. Varying $\lambda$ from 0 to 1 traces the entire FA frontier under group-balancedness.

From a welfare perspective, the vector of group errors is a pair of per-capita losses, and the planner’s choice of $\lambda$ encodes the relative social marginal cost of errors across groups, in the spirit of distribution-sensitive welfare weights used in public economics \citep[see, e.g.,][]{saez2016generalized}. To better highlight this, suppose 90\% of the population belongs to the red group. A natural choice would then be to set $\lambda = 0.9$. However, if we are training a predictive model for patients or for hiring across different demographics, we might want to ensure that the smaller group is not overshadowed. 
In that case, we could increase its weight and even assign equal importance to both groups, i.e., set $\lambda = 0.5$. This approach, often referred to as ``\textit{weighting by inverse class frequency},'' is common in the fairness literature \citep{cui2019class, steininger2021density}. There are also smoothed variants that interpolate between these extremes, offering a context-dependent balance between accuracy and fairness \citep{cui2019class, mikolov2013distributed}.

Once we leave the population idealization, every empirically implementable predictor sits \textit{inside} the feasible region: because $\beta_\lambda$ is defined as the minimizer of the population objective, any estimator based on a finite number of observations must incur \textit{welfare loss from estimation}, that is, a higher weighted error than $\beta_\lambda$. \Cref{figs/Frontier_Intro_2} depicts this as the outward shift of the planner’s iso-welfare line: instead of achieving her desired frontier point corresponding to $\beta_\lambda$, the policymaker implements an estimator $\widehat{\beta}_\lambda$ whose group error vector lies on a worse indifference curve. We interpret this shift as a \textit{welfare cost of statistical uncertainty}: 
while the planner targets a fixed choice of $\lambda$, limited data constrain how closely she can implement its population optimum.
Moreover, as \Cref{figs/Frontier_Intro_3} highlights, finite-sample error generally has \textit{asymmetric groupwise effects}: the displacement from $\beta_\lambda$ to $\widehat{\beta}_\lambda$ has both a horizontal and a vertical component, so one group’s error may increase substantially more than the other’s. Thus the realized outcome of a ``fair” design problem, when implemented empirically, can tilt the distribution of errors toward one group, implicitly implementing a different set of welfare weights than those the planner intended.

\begin{figure}[t]
\centering
\begin{subfigure}{.48\textwidth}
  \centering
  \begin{tikzpicture}[scale=13, line cap=round, line join=round]

\pgfmathsetmacro{\cx}{0.43}
\pgfmathsetmacro{\cy}{0.42}
\pgfmathsetmacro{\a}{0.41}
\pgfmathsetmacro{\b}{0.28}
\pgfmathsetmacro{\theta}{30}
\pgfmathsetmacro{\ct}{cos(\theta)}
\pgfmathsetmacro{\st}{sin(\theta)}
\pgfmathsetmacro{\pointSize}{0.007}
\pgfmathsetmacro{\lineWidth}{0.35mm}

\pgfmathsetmacro{\tYraw}{atan2(\b*\ct,\a*\st)}
\pgfmathsetmacro{\tY}{\tYraw + 180}
\pgfmathsetmacro{\betaBx}{\cx + \a*cos(\tY)*\ct - \b*sin(\tY)*\st}
\pgfmathsetmacro{\betaBy}{\cy + \a*cos(\tY)*\st + \b*sin(\tY)*\ct}

\pgfmathsetmacro{\tXraw}{atan2(-\b*\st,\a*\ct)}
\pgfmathsetmacro{\tX}{\tXraw + 180}
\pgfmathsetmacro{\betaRx}{\cx + \a*cos(\tX)*\ct - \b*sin(\tX)*\st}
\pgfmathsetmacro{\betaRy}{\cy + \a*cos(\tX)*\st + \b*sin(\tX)*\ct}

\pgfmathsetmacro{\posOnFrontier}{0.44}

\pgfmathsetmacro{\tStart}{\tX}
\pgfmathsetmacro{\tEnd}{\tY}
\pgfmathsetmacro{\tEndAdjusted}{\tY < \tX ? \tY + 360 : \tY}

\pgfmathsetmacro{\tLambda}{\tStart + \posOnFrontier*(\tEndAdjusted - \tStart)}
\pgfmathsetmacro{\betaLambdax}{\cx + \a*cos(\tLambda)*\ct - \b*sin(\tLambda)*\st}
\pgfmathsetmacro{\betaLambday}{\cy + \a*cos(\tLambda)*\st + \b*sin(\tLambda)*\ct}

\pgfmathsetmacro{\dxdt}{-\a*sin(\tLambda)*\ct - \b*cos(\tLambda)*\st}
\pgfmathsetmacro{\dydt}{-\a*sin(\tLambda)*\st + \b*cos(\tLambda)*\ct}
\pgfmathsetmacro{\mtan}{\dydt/\dxdt}
\pgfmathsetmacro{\xinter}{\betaLambdax - \betaLambday/\mtan} 
\pgfmathsetmacro{\yinter}{\betaLambday - \mtan*\betaLambdax} 

\pgfmathsetmacro{\xmax}{max(\betaBx,\betaRx) + 0.10}
\pgfmathsetmacro{\ymax}{max(\betaBy,\betaRy) + 0.10}

\draw[->] (0,0) -- (\xmax,0);
\draw[->] (0,0) -- (0,\ymax);
\node[anchor=north, BrickRed] at ({\xmax/2},0) {Red group error};
\node[rotate=90, anchor=south, Blue] at (0,{\ymax/2}) {Blue group error};

\begin{scope}
  \clip (0,0) rectangle (\xmax,\ymax);

  \begin{scope}[shift={(\cx,\cy)}, rotate=\theta]
    \fill[CornflowerBlue, opacity=0.15] (0,0) ellipse [x radius=\a, y radius=\b];
    \draw[CornflowerBlue!60] (0,0) ellipse [x radius=\a, y radius=\b];
  \end{scope}

  \draw[dashed, Fuchsia!70, line width=\lineWidth] (0,\yinter) -- (\xinter,0);

  \begin{scope}[shift={(\cx,\cy)}, rotate=\theta]
    \draw[ultra thick, Thistle]
      ({\a*cos(\tStart)},{\b*sin(\tStart)})
        arc[start angle=\tStart, end angle=\tEndAdjusted, x radius=\a, y radius=\b];
  \end{scope}
\end{scope}

\pgfmathsetmacro{\kshift}{1.5}                 
\pgfmathsetmacro{\xinterUp}{\kshift*\xinter}
\pgfmathsetmacro{\yinterUp}{\kshift*\yinter}

\pgfmathsetmacro{\db}{\yinterUp - \yinter}
\pgfmathsetmacro{\den}{\mtan*\mtan + 1}
\pgfmathsetmacro{\xhat}{\betaLambdax + (-\mtan)*\db/\den}
\pgfmathsetmacro{\yhat}{\betaLambday + \db/\den}

\pgfmathsetmacro{\along}{-0.08*\ymax}
\pgfmathsetmacro{\normLine}{sqrt(1+\mtan*\mtan)} 
\pgfmathsetmacro{\xhat}{\xhat + \along/\normLine}
\pgfmathsetmacro{\yhat}{\yhat + \mtan*\along/\normLine}

\draw[dashed, Orange, line width=\lineWidth] (0,\yinterUp) -- (\xinterUp,0);

\draw[-stealth, Orange!90!black, line width=\lineWidth]
  (\betaLambdax,\betaLambday) -- (\xhat-0.005,\yhat-0.005);

\node[text=Orange!90!black, anchor=west]
  at ({\xhat-0.08}, {\yhat+0.09}) {\tiny{$\lambda x + (1 - \lambda) y = c'$}};

\fill[Fuchsia]  (\betaLambdax,\betaLambday) circle[radius=\pointSize]
  node[below left, Fuchsia] {$\beta_\lambda$};
\fill[Blue]     (\betaBx,\betaBy) circle[radius=\pointSize]
  node[below, Blue] {$\beta_b$};
\fill[BrickRed] (\betaRx,\betaRy) circle[radius=\pointSize]
  node[left, BrickRed] {$\beta_r$};
\fill[Orange!90!black] (\xhat,\yhat) circle[radius=\pointSize]
  node[right, Orange!90!black] {$\widehat{\beta}_\lambda$};

\pgfmathsetmacro{\perpOff}{0.28*\ymax}
\pgfmathsetmacro{\xzero}{\betaLambdax}
\pgfmathsetmacro{\yzero}{\betaLambday - \perpOff}

\pgfmathsetmacro{\xPerpP}{(\xzero + \mtan*(\yzero - \yinter))/\den}
\pgfmathsetmacro{\yPerpP}{\mtan*\xPerpP + \yinter}
\pgfmathsetmacro{\xPerpO}{(\xzero + \mtan*(\yzero - \yinterUp))/\den}
\pgfmathsetmacro{\yPerpO}{\mtan*\xPerpO + \yinterUp}

\draw[stealth-stealth, line width=\lineWidth]
  (\xPerpP,\yPerpP) -- (\xPerpO,\yPerpO)
  node[midway, below, sloped, yshift=1pt] {\tiny{welfare loss}};

\end{tikzpicture}
  \caption{}
  \label{figs/Frontier_Intro_2}
\end{subfigure}
\hfill
\begin{subfigure}{.48\textwidth}
  \centering
  \begin{tikzpicture}[scale=13, line cap=round, line join=round]

\pgfmathsetmacro{\cx}{0.43}
\pgfmathsetmacro{\cy}{0.42}
\pgfmathsetmacro{\a}{0.41}
\pgfmathsetmacro{\b}{0.28}
\pgfmathsetmacro{\theta}{30}
\pgfmathsetmacro{\ct}{cos(\theta)}
\pgfmathsetmacro{\st}{sin(\theta)}
\pgfmathsetmacro{\pointSize}{0.007}
\pgfmathsetmacro{\lineWidth}{0.35mm}

\pgfmathsetmacro{\tYraw}{atan2(\b*\ct,\a*\st)}
\pgfmathsetmacro{\tY}{\tYraw + 180}
\pgfmathsetmacro{\betaBx}{\cx + \a*cos(\tY)*\ct - \b*sin(\tY)*\st}
\pgfmathsetmacro{\betaBy}{\cy + \a*cos(\tY)*\st + \b*sin(\tY)*\ct}

\pgfmathsetmacro{\tXraw}{atan2(-\b*\st,\a*\ct)}
\pgfmathsetmacro{\tX}{\tXraw + 180}
\pgfmathsetmacro{\betaRx}{\cx + \a*cos(\tX)*\ct - \b*sin(\tX)*\st}
\pgfmathsetmacro{\betaRy}{\cy + \a*cos(\tX)*\st + \b*sin(\tX)*\ct}

\pgfmathsetmacro{\posOnFrontier}{0.44}

\pgfmathsetmacro{\tStart}{\tX}
\pgfmathsetmacro{\tEnd}{\tY}
\pgfmathsetmacro{\tEndAdjusted}{\tY < \tX ? \tY + 360 : \tY}

\pgfmathsetmacro{\tLambda}{\tStart + \posOnFrontier*(\tEndAdjusted - \tStart)}
\pgfmathsetmacro{\betaLambdax}{\cx + \a*cos(\tLambda)*\ct - \b*sin(\tLambda)*\st}
\pgfmathsetmacro{\betaLambday}{\cy + \a*cos(\tLambda)*\st + \b*sin(\tLambda)*\ct}

\pgfmathsetmacro{\dxdt}{-\a*sin(\tLambda)*\ct - \b*cos(\tLambda)*\st}
\pgfmathsetmacro{\dydt}{-\a*sin(\tLambda)*\st + \b*cos(\tLambda)*\ct}
\pgfmathsetmacro{\mtan}{\dydt/\dxdt}
\pgfmathsetmacro{\xinter}{\betaLambdax - \betaLambday/\mtan} 
\pgfmathsetmacro{\yinter}{\betaLambday - \mtan*\betaLambdax} 

\pgfmathsetmacro{\xmax}{max(\betaBx,\betaRx) + 0.10}
\pgfmathsetmacro{\ymax}{max(\betaBy,\betaRy) + 0.10}

\draw[->] (0,0) -- (\xmax,0);
\draw[->] (0,0) -- (0,\ymax);
\node[anchor=north, BrickRed] at ({\xmax/2},0) {Red group error};
\node[rotate=90, anchor=south, Blue] at (0,{\ymax/2}) {Blue group error};

\begin{scope}
  \clip (0,0) rectangle (\xmax,\ymax);

  \begin{scope}[shift={(\cx,\cy)}, rotate=\theta]
    \fill[CornflowerBlue, opacity=0.15] (0,0) ellipse [x radius=\a, y radius=\b];
    \draw[CornflowerBlue!60] (0,0) ellipse [x radius=\a, y radius=\b];
  \end{scope}

  \draw[dashed, Fuchsia!70, line width=\lineWidth] (0,\yinter) -- (\xinter,0);

  \begin{scope}[shift={(\cx,\cy)}, rotate=\theta]
    \draw[ultra thick, Thistle]
      ({\a*cos(\tStart)},{\b*sin(\tStart)})
        arc[start angle=\tStart, end angle=\tEndAdjusted, x radius=\a, y radius=\b];
  \end{scope}
\end{scope}

\pgfmathsetmacro{\kshift}{1.5}                  
\pgfmathsetmacro{\xinterUp}{\kshift*\xinter}
\pgfmathsetmacro{\yinterUp}{\kshift*\yinter}

\pgfmathsetmacro{\db}{\yinterUp - \yinter}
\pgfmathsetmacro{\den}{\mtan*\mtan + 1}
\pgfmathsetmacro{\xhat}{\betaLambdax + (-\mtan)*\db/\den}
\pgfmathsetmacro{\yhat}{\betaLambday + \db/\den}

\pgfmathsetmacro{\along}{-0.08*\ymax}
\pgfmathsetmacro{\normLine}{sqrt(1+\mtan*\mtan)} 
\pgfmathsetmacro{\xhat}{\xhat + \along/\normLine}
\pgfmathsetmacro{\yhat}{\yhat + \mtan*\along/\normLine}

\draw[dashed, Orange, line width=\lineWidth] (0,\yinterUp) -- (\xinterUp,0);

\draw[-stealth, Orange!90!black, line width=\lineWidth]
  (\betaLambdax,\betaLambday) -- (\xhat-0.005,\yhat-0.005);

\node[text=Orange!90!black, anchor=west]
  at ({\xhat-0.08}, {\yhat+0.09}) {\tiny{$\lambda x + (1 - \lambda) y = c'$}};

\draw[-stealth, Blue, line width=\lineWidth]
  (\betaLambdax,\betaLambday) -- (\betaLambdax,\yhat); 
\draw[-stealth, BrickRed, line width=\lineWidth]
  (\betaLambdax,\betaLambday) -- (\xhat,\betaLambday); 

\draw[dashed, BrickRed, line width=\lineWidth]
  (\betaLambdax,\yhat) -- (\xhat,\yhat);  
\draw[dashed, Blue, line width=\lineWidth]
  (\xhat,\betaLambday) -- (\xhat,\yhat);  

\fill[Fuchsia]  (\betaLambdax,\betaLambday) circle[radius=\pointSize]
  node[below left, Fuchsia] {$\beta_\lambda$};
\fill[Blue]     (\betaBx,\betaBy) circle[radius=\pointSize]
  node[below, Blue] {$\beta_b$};
\fill[BrickRed] (\betaRx,\betaRy) circle[radius=\pointSize]
  node[left, BrickRed] {$\beta_r$};
\fill[Orange!90!black] (\xhat,\yhat) circle[radius=\pointSize]
  node[right, Orange!90!black] {$\widehat{\beta}_\lambda$};

\end{tikzpicture}
  \caption{}
  \label{figs/Frontier_Intro_3}
\end{subfigure}
\caption{(a) Finite-sample estimation: The error pair corresponding to the empirical estimator $\widehat{\beta}_\lambda$ lies on the line $\lambda x + (1 - \lambda) y = c'$, where $c' - c$ is the welfare loss from estimation. (b) Asymmetric groupwise impact: The displacement from $\beta_\lambda$ to $\widehat{\beta}_\lambda$ decomposes into a vertical change in the blue-group error and a horizontal change in the red-group error; their unequal magnitudes show that the estimator affects the two groups asymmetrically.}
\label{fig:FA_Intro_2}
\end{figure}


Building on this welfare interpretation, we study the FA frontier as a design problem under resource and information constraints. We first fix a target $\lambda$ and ask: among all estimators that a planner might construct from $n_r$ and $n_b$ observations from the two groups, which estimator minimizes the worst-case welfare loss over a broad class of linear data-generating processes? And given a total sampling budget $n_r + n_b$, how should the planner choose $n_r$ and $n_b$ to minimize that worst-case welfare loss? We tackle these questions in two distinct informational environments that are natural in applications. In the first, the planner knows the distribution of the covariates in each group (or just their covariance matrices) but not the outcome model; for instance, she may have access to large administrative data on covariates but only a small subsample with observed outcomes, like a lender that observes the characteristics of the entire applicant pool but repayment outcomes for only a limited number. In the second, more demanding case, neither the outcome model nor the covariate distribution is known.

Our first main contribution is to characterize the minimax-optimal estimator, that is, the estimator with the smallest worst-case expected welfare loss, when group-specific covariance matrices are known. We show that in this environment the optimal finite-sample estimator preserves the structure of the population-optimal policy $\beta_\lambda$ but replaces the unknown group coefficients with their least-squares estimates. We derive matching upper and lower bounds on its worst-case expected welfare loss, showing that this estimator achieves the optimal rate under broad distributional assumptions. In particular, we establish that the welfare loss decomposes into a weighted sum of groupwise error terms that depend on the covariance matrices and the sample sizes. 
With equally costly observations, the leading term of this bound yields a sampling rule that accounts for welfare weights, outcome noise, and the relevance of each group's covariates to the objective.
This rule is a direct analogue of \textit{Neyman allocation} in stratified sampling \citep{neyman1934representative, cochran1977sampling}: the planner invests more observations in groups that both matter more in her social objective and exhibit more variable covariates. This result tells a policymaker which parts of the population to oversample in order to move as close as possible to the desired point on the FA frontier.

Our second contribution is to analyze the case in which covariances must themselves be estimated. Here, we take the ordinary least squares estimator that minimizes the sample analogue of the weighted loss. We show that the welfare cost of estimation has two distinct components. The first is a \textit{variance} term that reflects the uncertainty about the group-specific regression coefficients and coincides with the error from the known-covariance case. The second is a \textit{bias} term that arises purely from uncertainty about the covariances: even if both groups’ regression coefficients ($\beta_r$ and $\beta_b$) were known, mis-estimating the covariates’ second moments leads the planner away from the desired convex combination encoded by $\beta_\lambda$. We derive nonasymptotic upper bounds for both terms under certain distributional assumptions, and show via matching lower bounds that no estimator can improve on ordinary least squares by more than a constant factor. 

Interestingly, the bias term is governed by the \textit{heterogeneity} of the groups in their regression parameters, i.e., $\|\beta_r-\beta_b\|$. When groups are similar, the variance term dominates and the optimal sampling rule mirrors the known-covariance case; when groups are very different, the bias term dominates and the welfare-efficient sampling design shifts toward balancing $n_r$ and $n_b$ in order to reduce covariance-estimation bias. Thus the optimal allocation of data across groups is itself state-dependent: a policymaker facing very heterogeneous groups should collect data more evenly, even if her welfare weights $\lambda$ would otherwise favor one group, because asymmetric sampling can amplify misalignment between her intended and realized fairness–accuracy tradeoff. This means that, when characteristics that predict default differ markedly across groups, a lender that cannot learn the characteristics of its applicants from other sources should collect repayment data on a similar number of borrowers from each group.

Third, we move from a pointwise analysis at a fixed $\lambda$ to a \textit{frontier-wide} perspective. In many applications, the FA frontier is valuable not because the planner has committed to a single set of welfare weights ex ante, but because it summarizes the menu of feasible equity–efficiency tradeoffs. However, when the same dataset is used to estimate the entire frontier, estimation errors at different $\lambda$’s are highly correlated. We therefore derive a uniform, high-probability bound on the deviation between the population error pair achieved by the estimator $\widehat{\beta}_\lambda$ and the population error pair associated with its target $\beta_\lambda$, simultaneously over all $\lambda \in [0, 1]$. The gap shrinks at a rate governed by the inverse square roots of the group sample sizes, quantifying how much additional data are needed to make reliable welfare comparisons between alternative points on the frontier.

Fourth, we study the group-specific impact of moving from $\beta_\lambda$ to its empirical counterpart and show that, even for the minimax-optimal estimator, the welfare loss can be markedly uneven across groups, driven by differences in covariate dispersion and by asymmetric bias terms. These results provide a formal language for describing the welfare cost borne by each group when policymakers implement fairness-aware algorithms estimated from finite data.

Finally, we show that these distortions are not unique to linear prediction under squared loss in an extension in \Cref{sec:glm}. For canonical generalized linear models, including logistic regression, we derive a leading-order expansion of the expected groupwise and $\lambda$-weighted errors under the empirical M-estimator. We find that the same qualitative mechanisms appear as in the unknown covariance analysis: outcome noise gives rise to a variance term and groupwise heterogeneity generates an additional bias term controlled by $\norm{\beta_r - \beta_b}^2$. Moreover, the groupwise errors contain a signed redistribution term that cancels in the planner's objective but can affect the error of each individual group asymmetrically. This extension complements the sharp linear minimax analysis by showing that the bias-variance decomposition and asymmetric frontier distortion generated by estimation errors are a more general phenomenon.

Taken together, our results reinterpret the fairness–accuracy frontier as a welfare object that is shaped not only by the structural relationship between covariates and outcomes, but also by the \textit{statistical environment} in which algorithms are designed and deployed. They show how sampling design, group heterogeneity, and the gap between population-optimal and estimated predictors translate into predictable patterns of welfare loss across groups. Hence, our work connects the FA frontier to familiar tools from welfare analysis and optimal sampling, and provides finite-sample guarantees that are directly relevant for regulators and policymakers who must decide which data to collect and which prediction rules to trust.

\subsection{Related work}
 We build on \citet{liang2026algorithm}, who formalize the FA Pareto frontier for assessing fairness--accuracy tradeoffs under the assumption of full knowledge of the covariate and outcome distributions.  \citet{liu2024inference} relax the full-knowledge assumption, representing the feasible error set and Pareto frontier through its support functions, and develop debiased machine learning estimators and asymptotic inference for these objects. Our $\lambda$-weighted error parameterization of the frontier is equivalent to their support function parameterization; see \Cref{sec:support_function}. Our analysis of these objects differs by studying finite-sample estimation of predictors for frontier points, including the corresponding minimax welfare loss, optimal sample allocation, groupwise distortions, and uniform frontier guarantees.  \citet{auerbach2024testing} develop asymptotically valid tests for fairness--accuracy Pareto improvements, while \citet{liu2026identification} propose a framework for asymptotically valid inference over the frontier when outcomes are selectively observed. 

 Our work also relates to the literature on statistical policy design, which uses finite-sample bounds on welfare loss to study the consequences of choosing policies from data \citep{kitagawa2018should, kitagawa2021equality, mbakop2021model, viviano2024fair}, building on the decision-theoretic analysis of statistical treatment rules \citep{manski2004statistical, stoye2009minimax, hirano2009asymptotics}. We characterize these costs for a predictor shared across heterogeneous groups, showing how group sample sizes and characteristics interact with the planner's preferences to affect the implementation of a planner's preferred fairness--accuracy tradeoff.

Our results on sample allocation connect to the classical problem of optimal allocation in stratified sampling \citep{neyman1934representative, cochran1977sampling} and to its use in the design of experiments \citep{hahn2011adaptive, tabord2023stratification}. Our analysis of the known covariance case relates to the econometric literature on combining a main sample with auxiliary covariate information \citep{imbens1994combining, hellerstein1999imposing, ridder2007econometrics}.

Our work contributes to the literature on fairness in regression, a topic that has traditionally received less attention than fair classification \citep[see, e.g.,][]{donini2018empirical}. The most closely related work is that of \citet{chzhen2022minimax}, who develop a minimax-optimal fair regression estimator for an explicit fairness constraint based on the Wasserstein distance between the loss distributions across groups. Beyond this, several other works have introduced various notions of fairness for the regression setting and have proposed corresponding optimal procedures for estimation \citep{calders2013controlling, berk2017convex, fitzsimons2019general, agarwal2019fair, chzhen2020fair, chzhen2020fairb, oneto2020general} and for estimating the minimal utility cost associated with achieving fairness in regression \citep{zhao2021costs}.

A growing body of work bridges perspectives on fairness--accuracy tradeoffs in computer science and statistics with analyses of social welfare \citep{hu2020fair, liang2024algorithmic, gans2024algorithmic, rosenfeld2025machine}. 
Our results on the FA frontier can also be interpreted in this context if we view each individual’s disutility as the prediction error associated with their group. More broadly, in the algorithmic fairness literature, it is widely recognized that no single fairness criterion can capture the diverse preferences that decision-makers may have over these tradeoffs \citep{casacuberta2023augmenting, verma2018fairness, kleinberg2016inherent}. Economists make this point in a similar way by documenting the heterogeneity in how individuals weigh equity against efficiency \citep{capozza2024should, lockwood2015gustibus}. Our framework builds on this connection by adopting a flexible, weight-based formulation that accommodates this kind of heterogeneity \citep{saez2016generalized, hendren2020measuring, kasy2021fairness}. 

Finally, much attention has been devoted to the fundamental question of when fairness and accuracy are in conflict and when they are aligned \citep{corbett2017algorithmic, wick2019unlocking, blum2019recovering, liang2026algorithm, andersen2020big, dutta2020there, cooper2021emergent}.  
In settings where these objectives are at odds, a variety of methods have been developed to identify \emph{Pareto improvements}, where either fairness or accuracy can be improved while keeping the other one unchanged. One line of work designs algorithms for computing fairer alternatives to a given model for specific applications \citep{fish2016confidence, coston2021characterizing, wei2022fairness, chohlas2024learning, dehdashtian2024utility}. A closely related line of work draws on the legal notion of a \emph{Less Discriminatory Alternative}, or a fairer model that maintains the original model's accuracy \citep{gillis2024operationalizing, laufer2025constitutes, samad2025selecting}. Others take the reverse approach, focusing on accuracy improvements subject to fairness guarantees \citep{agarwal2018reductions, wang2023aleatoric}, or, more generally, methods for balancing multiple competing objectives in learning \citep{rolf2020balancing, martinez2020minimax}. 

The remainder of the paper is organized as follows. \Cref{sec:formulation} introduces the framework and the planner's estimation problem, and \Cref{sec:linearmodel} specializes it to linear regression. \Cref{sec:knowncov} studies the case in which the groups' covariance matrices are known, \Cref{sec:unknown_covariance} the case in which they must be estimated, and \Cref{sec:uniform_bound} derives a bound that holds uniformly over the frontier. \Cref{sec:experiments} reports simulations, and \Cref{sec:conclusion} concludes. Proofs are in \Cref{appA}, \Cref{appB} contains additional discussion, including the extension to several groups, and \Cref{sec:glm} extends the analysis to generalized linear models. 
\section{Framework}\label{sec:formulation}
We consider a population where each individual is represented by a covariate vector $X \in \mathcal{X} \subset \mathbb{R}^d$, a response (outcome) variable $Y \in \mathbb{R}$, and a group label $g \in \mathcal{G} \coloneq \{r, b\}$ (e.g., red, blue). For each group $g$, we observe a sample $\mathcal{S}_g\coloneq \{(X_i^{(g)}, Y_i^{(g)})\}_{i=1}^{n_g}$ consisting of $n_g$ independent and identically distributed draws from $P(X, Y | G = g)$; draws are also independent across groups.

Let $\mathcal{H}$ be a class of prediction functions $f: \mathcal{X} \to \RR$. We quantify the error of a predictor $f(X)$ in estimating the output variable $Y$ using a nonnegative loss function, $\ell: \mathbb{R} \times \mathbb{R} \to \mathbb{R}_{\geq 0}$. Specifically, the loss incurred on a single example is given by $\ell(f(X), Y)$.
The groupwise population risk and its empirical counterpart for prediction function $f \in \mathcal{H}$ are given, respectively, by
\begin{equation*}
    \mathcal{R}_g(f) \coloneq \EE_{(X, Y) \sim P(\cdot|G = g)}[\ell(f(X), Y)], \quad \widehat{\mathcal{R}}_g(f) \coloneq \frac{1}{n_g} \sum_{i=1}^{n_g} \ell\left(f\left(X_i^{(g)}\right), Y_i^{(g)}\right).
\end{equation*}
Each prediction function $f \in \mathcal{H}$ induces a pair of risks $(\mathcal{R}_r(f), \mathcal{R}_b(f))$. Here, $\mathcal{R}_r(f)$ and $\mathcal{R}_b(f)$ represent the prediction accuracy of $f$ over groups $r$ and $b$, respectively. Throughout, we use the term \textit{risk}, standard in statistical learning, for the expected loss of a fixed prediction rule: under squared loss, $\mathcal{R}_g(f)$ is the mean squared prediction error of $f$ in group $g$, which the introduction refers to as group $g$'s prediction error. This differs from the decision-theoretic use of the term, in which the risk of a statistical procedure is its expected loss over repeated draws of the data \citep{chamberlain2000econometrics, manski2004statistical}. The second notion enters our analysis through the expected excess risk defined in \Cref{sec:emp_fa_frontier}.

Following \cite{liang2026algorithm}, we measure the \textit{fairness} of a prediction function $f$ using the risk disparity $|\mathcal{R}_r(f) - \mathcal{R}_b(f)|$, which quantifies the difference in performance of the model between the two groups. For regression tasks with a continuous outcome variable $Y$---the primary focus of this paper---this notion coincides with the definition of \textit{Equalized Loss} proposed by \cite{zhang2019group}.\footnote{For binary classifiers, i.e., when the outcome variable $Y$ is binary, \citet[Appendix O.1]{liang2026algorithm} provide a detailed discussion on how this definition relates to well-known fairness notions in the binary setting such as Equalized Odds \citep{hardt2016equality}.}
\subsection{Fairness--accuracy frontier}
To investigate the fairness–accuracy tradeoff, we begin by examining how one can maneuver between group accuracies and fairness by changing the prediction function. More formally, we start by defining the set of achievable population risk pairs over the class $\mathcal{H}$, given by
\begin{equation*}
    \begin{aligned}
        \mathcal{E}(\mathcal{H}) \coloneq \{ (\mathcal{R}_r(f), \mathcal{R}_b(f)): f \in \mathcal{H}\}.
    \end{aligned}
\end{equation*}
We next adopt the definition of the \emph{Fairness--Accuracy (FA) Pareto frontier} from \cite{liang2026algorithm}. To do so, we first revisit the definition of \textit{FA dominance}.
\begin{definition}\label{def:fa_dominance}
We say that a function $f' \in \mathcal{H}$ \emph{FA-dominates} a function $f \in \mathcal{H}$, denoted by $f' \succ f$, if $\mathcal{R}_r(f') \leq \mathcal{R}_r(f)$, $\mathcal{R}_b(f') \leq \mathcal{R}_b(f)$, and $|\mathcal{R}_r(f') - \mathcal{R}_b(f')| \leq |\mathcal{R}_r(f) - \mathcal{R}_b(f)|$, with at least one of these three inequalities being strict.     
\end{definition}
In other words, a function $f'$ FA-dominates a function $f$ if it achieves no higher risk on either group and no greater risk disparity between groups, with strict improvement for at least one. The FA frontier is then defined as the subset of achievable risk pairs that are not FA-dominated by any other point.
\begin{definition}\label{def:fa_frontier}
The FA frontier, denoted by $\mathcal{F}(\mathcal{H})$, is defined as: 
\begin{equation*}
\mathcal{F}(\mathcal{H}) \coloneq  \left \{ \left(\mathcal{R}_r(f), \mathcal{R}_b(f) \right) \in \mathcal{E}(\mathcal{H}) : \nexists f' \in \mathcal{H}: f' \succ f \right \}.
\end{equation*}
\end{definition}
Insight into the shape of the fairness--accuracy (FA) frontier can be obtained by focusing on three key points: the two points corresponding to the prediction functions that achieve the best accuracy for each group, and the point that achieves the best fairness guarantee. More formally, for any $g \in \mathcal{G}$, let $f_g$ be defined as
\begin{equation}
f_g \in \arg\min_{f \in \mathcal{H}} \mathcal{R}_g(f),
\end{equation}
and define $f_{\textrm{fair}}$ as
\begin{equation}
f_{\textrm{fair}} \in \arg\min_{f \in \mathcal{H}: (\mathcal{R}_r(f), \mathcal{R}_b(f)) \in \mathcal{F}(\mathcal{H})} |\mathcal{R}_r(f) - \mathcal{R}_b(f)|.
\end{equation}

It is straightforward to verify that the pairs of risks corresponding to all three functions lie on the FA frontier. In fact, the endpoints of the FA frontier are two of the three points associated with $\{f_r, f_b, f_{\textrm{fair}}\}$. \cite{liang2026algorithm} define a distribution $P$ as \textit{group-skewed} if the point corresponding to $f_{\textrm{fair}}$ lies at one of the endpoints of the frontier. In contrast, $P$ is said to be \textit{group-balanced} if the endpoints correspond to $f_r$ and $f_b$, and the fairest point $f_{\textrm{fair}}$ lies between them. An illustration of group-skewed and group-balanced distributions is provided in \Cref{fig:FA_Liang}.

\begin{figure}
\centering
\begin{subfigure}{.48\textwidth}
  \centering
  \resizebox{0.65\linewidth}{!}{\begin{tikzpicture}[scale=6]

  \draw[->] (0,0) -- (1,0) node[anchor=north west] {$\mathcal{R}_r$};
  \draw[->] (0,0) -- (0,1) node[anchor=south east] {$\mathcal{R}_b$};

  \draw[dashed, gray] (0,0) -- (1,1) 
    node[anchor=south east, rotate=45] {\small $\: \: \mathcal{R}_b = \mathcal{R}_r$};
    
\pgfmathsetmacro{\pointSize}{0.01}

  \pgfmathsetmacro{\cx}{0.5}
  \pgfmathsetmacro{\cy}{0.5}
  \pgfmathsetmacro{\a}{0.4}
  \pgfmathsetmacro{\b}{0.25}
  \pgfmathsetmacro{\theta}{30}
  \pgfmathsetmacro{\contractionScale}{0.6}

  \pgfmathsetmacro{\ct}{cos(\theta)}
  \pgfmathsetmacro{\st}{sin(\theta)}

  \pgfmathsetmacro{\tYraw}{atan2(\b*\ct,\a*\st)} 
  \pgfmathsetmacro{\tY}{\tYraw + 180} 
  \pgfmathsetmacro{\betaBx}{\cx + \a*cos(\tY)*\ct - \b*sin(\tY)*\st}
  \pgfmathsetmacro{\betaBy}{\cy + \a*cos(\tY)*\st + \b*sin(\tY)*\ct}

  \pgfmathsetmacro{\tXraw}{atan2(-\b*\st,\a*\ct)}
  \pgfmathsetmacro{\tX}{\tXraw + 180}
  \pgfmathsetmacro{\betaRx}{\cx + \a*cos(\tX)*\ct - \b*sin(\tX)*\st}
  \pgfmathsetmacro{\betaRy}{\cy + \a*cos(\tX)*\st + \b*sin(\tX)*\ct}

  \begin{scope}
    \clip (0,0) rectangle (0.9,0.9);
    \fill[CornflowerBlue, opacity=0.3, rotate around={\theta:(\cx,\cy)}]
      (\cx,\cy) ellipse [x radius=\a, y radius=\b];
  \end{scope}

  \node at (0.7,0.7) {$\mathcal{E}(\mathcal{H})$};

  \pgfmathsetmacro{\tStart}{\tX}
  \pgfmathsetmacro{\tEnd}{\tY}
  \pgfmathsetmacro{\tEndAdjusted}{\tY < \tX ? \tY + 360 : \tY}

  \draw[ultra thick, Thistle, rotate around={\theta:(\cx,\cy)}, opacity=0.8, line width=2pt]
    (\cx,\cy) ++(\tStart:\a cm and \b cm) 
    arc[start angle=\tStart, end angle=\tEndAdjusted, x radius=\a, y radius=\b];

  \pgfmathsetmacro{\fairx}{\cx - 1/sqrt((cos(\theta)+sin(\theta))^2/\a^2+(sin(\theta)-cos(\theta))^2/\b^2)}
  \pgfmathsetmacro{\fairy}{\cy - 1/sqrt((cos(\theta)+sin(\theta))^2/\a^2+(sin(\theta)-cos(\theta))^2/\b^2)}

  \fill[Blue] (\betaBx,\betaBy) circle[radius=\pointSize];
  \node[Blue, below] at (\betaBx,\betaBy) {$f_b$};
  \fill[BrickRed] (\betaRx,\betaRy) circle[radius=\pointSize];
  \node[BrickRed, left] at (\betaRx,\betaRy) {$f_r$};
  \fill[Fuchsia] (\fairx,\fairy) circle[radius=\pointSize];
  \node[Fuchsia, left] at (\fairx,\fairy) {$f_{\textrm{fair}}$};

\end{tikzpicture}}
  \caption{}
  \label{fig:Group_Balanced}
\end{subfigure}
\hfill
\begin{subfigure}{.48\textwidth}
  \centering
  \resizebox{0.65\linewidth}{!}{\begin{tikzpicture}[scale=6]

  \draw[->] (0,0) -- (1,0) node[anchor=north west] {$\mathcal{R}_r$};
  \draw[->] (0,0) -- (0,1) node[anchor=south east] {$\mathcal{R}_b$};

  \draw[dashed, gray] (0,0) -- (1,1) 
    node[anchor=south east, rotate=45] {\small $\: \: \mathcal{R}_b = \mathcal{R}_r$};
    
\pgfmathsetmacro{\pointSize}{0.01}

  \pgfmathsetmacro{\cx}{0.3}
  \pgfmathsetmacro{\cy}{0.6}
  \pgfmathsetmacro{\a}{0.25}
  \pgfmathsetmacro{\b}{0.15}
  \pgfmathsetmacro{\theta}{30}
  \pgfmathsetmacro{\contractionScale}{0.6}

  \pgfmathsetmacro{\ct}{cos(\theta)}
  \pgfmathsetmacro{\st}{sin(\theta)}

  \pgfmathsetmacro{\tYraw}{atan2(\b*\ct,\a*\st)} 
  \pgfmathsetmacro{\tY}{\tYraw + 180} 
  \pgfmathsetmacro{\betaBx}{\cx + \a*cos(\tY)*\ct - \b*sin(\tY)*\st}
  \pgfmathsetmacro{\betaBy}{\cy + \a*cos(\tY)*\st + \b*sin(\tY)*\ct}

  \pgfmathsetmacro{\tXraw}{atan2(-\b*\st,\a*\ct)}
  \pgfmathsetmacro{\tX}{\tXraw + 180}
  \pgfmathsetmacro{\betaRx}{\cx + \a*cos(\tX)*\ct - \b*sin(\tX)*\st}
  \pgfmathsetmacro{\betaRy}{\cy + \a*cos(\tX)*\st + \b*sin(\tX)*\ct}

  \begin{scope}
    \clip (0,0) rectangle (0.9,0.9);
    \fill[CornflowerBlue, opacity=0.3, rotate around={\theta:(\cx,\cy)}]
      (\cx,\cy) ellipse [x radius=\a, y radius=\b];
  \end{scope}

  \node at (0.7,0.7) {$\mathcal{E}(\mathcal{H})$};

  \pgfmathsetmacro{\tStart}{\tX}
  \pgfmathsetmacro{\tEnd}{\tY}
  \pgfmathsetmacro{\tEndAdjusted}{\tY < \tX ? \tY + 360 : \tY}

\pgfmathsetmacro{\tLambda}{\tStart + 2*(\tEndAdjusted - \tStart)}
\pgfmathsetmacro{\fairx}{\cx + \a*cos(\tLambda)*\ct - \b*sin(\tLambda)*\st}
\pgfmathsetmacro{\fairy}{\cy + \a*cos(\tLambda)*\st + \b*sin(\tLambda)*\ct}

\draw[ultra thick, Thistle, rotate around={\theta:(\cx,\cy)}, opacity=0.8, line width=2pt]
(\cx,\cy) ++(\tStart:\a cm and \b cm) 
arc[start angle=\tStart, end angle=\tStart + 2*(\tEndAdjusted - \tStart), x radius=\a, y radius=\b];

  \fill[Blue] (\betaBx,\betaBy) circle[radius=\pointSize];
  \node[Blue, below] at (\betaBx,\betaBy) {$f_b$};
  \fill[BrickRed] (\betaRx,\betaRy) circle[radius=\pointSize];
  \node[BrickRed, left] at (\betaRx,\betaRy) {$f_r$};
  \fill[Fuchsia] (\fairx,\fairy) circle[radius=\pointSize];
  \node[Fuchsia, left] at (\fairx,\fairy) {$f_\textrm{fair}$};

\end{tikzpicture}}
  \caption{}
  \label{fig:Group_Skewed}
\end{subfigure}
\caption{Group-Balanced vs. Group-Skewed. To avoid cluttering the figure, we label points by their corresponding predictors; for example, $f_r$ denotes the point $(\mathcal{R}_r(f_r), \mathcal{R}_b(f_r))$.}
\label{fig:FA_Liang}
\end{figure}

In the group-skewed scenario, the points associated with the risk pairs of both $f_r$ and $f_b$ lie on one side of the $\mathcal{R}_b = \mathcal{R}_r$ line. For instance, in \Cref{fig:Group_Skewed}, $\mathcal{R}_r(f_b)$ is lower than $\mathcal{R}_b(f_b)$ even though $f_b$ is the optimal predictor for group $b$.
In contrast, in the group-balanced case, the fair predictor $f_{\textrm{fair}}$ falls between $f_b$ and $f_r$, indicating that each group achieves a lower risk at its respective optimal predictor. As a result, achieving fairness may require trading off some accuracy between the two groups. In the next section, we give primitive conditions under which the linear regression model studied in this paper is group balanced. This holds, for instance, when the two groups have equal noise variance.

For the group-balanced setting, the FA frontier can be traversed by varying a parameter $\lambda \in [0, 1]$ which represents the designer's preference in trading off between fairness and accuracy for each group. In particular, given a weight $\lambda$, the decision-maker seeks a decision policy $f_\lambda$ that achieves the desired tradeoff in errors between groups. This corresponds to solving the weighted risk minimization:
\begin{equation}\label{eqn:f_lambda_R_lambda_defs}
    \begin{aligned}
        f_\lambda {\in} \arg\min_{f \in \mathcal{H}} \mathcal{R}_\lambda(f) \quad \text{with} \quad
        \mathcal{R}_\lambda(f)\coloneq \lambda \mathcal{R}_r(f) + (1 - \lambda) \mathcal{R}_b(f).
    \end{aligned}
\end{equation}
Here, $\lambda$ encodes the relative cost of errors affecting group $r$ to those affecting group $b$. The extreme cases $\lambda = 1$ and $\lambda = 0$ recover the group-optimal decision policies $f_r$ and $f_b$, respectively. In general the risk pair associated with $f_\lambda$ is the first point in $\mathcal{E}(\mathcal{H})$ that intersects the line $\lambda \mathcal{R}_r + (1-\lambda) \mathcal{R}_b = c$ as $c$ increases---that is, as the line shifts upward and to the right. See \Cref{fig:Frontier-Lambda} for an illustration. 
This formulation thus allows us to target a point on the FA frontier and trade off between fairness and accuracy.\footnote{Our parameterization of the frontier by $\lambda$ is equivalent to the support-function construction used by \citet{liu2024inference}; see \Cref{sec:support_function}.} 

It is worth highlighting that the welfare weight $\lambda$ and the sample share $n_r/(n_r+n_b)$ play different roles. The former specifies how much weight the planner places on the error of the red group relative to the blue group, while the latter determines the relative amount of information available on the red group (compared with the blue group) for estimating the common predictor. A group can therefore be oversampled without receiving greater normative weight. Moreover, $\lambda$ itself is not a monotone measure of fairness: moving toward the risk-equalizing point reduces disparity, while moving past it can increase disparity.

\begin{figure}
\centering
\begin{subfigure}{.48\textwidth}
  \centering
  \resizebox{0.65\linewidth}{!}{\begin{tikzpicture}[scale=6]

  \draw[->] (0,0) -- (1,0) node[anchor=north west] {$\mathcal{R}_r$};
  \draw[->] (0,0) -- (0,1) node[anchor=south east] {$\mathcal{R}_b$};
  \definecolor{DarkerGreen}{rgb}{0, 0.4, 0}

    
\pgfmathsetmacro{\pointSize}{0.01}

  \pgfmathsetmacro{\cx}{0.5}
  \pgfmathsetmacro{\cy}{0.5}
  \pgfmathsetmacro{\a}{0.4}
  \pgfmathsetmacro{\b}{0.25}
  \pgfmathsetmacro{\theta}{30}
  \pgfmathsetmacro{\contractionScale}{0.6}

  \pgfmathsetmacro{\ct}{cos(\theta)}
  \pgfmathsetmacro{\st}{sin(\theta)}

  \pgfmathsetmacro{\tYraw}{atan2(\b*\ct,\a*\st)} 
  \pgfmathsetmacro{\tY}{\tYraw + 180} 
  \pgfmathsetmacro{\betaBx}{\cx + \a*cos(\tY)*\ct - \b*sin(\tY)*\st}
  \pgfmathsetmacro{\betaBy}{\cy + \a*cos(\tY)*\st + \b*sin(\tY)*\ct}

  \pgfmathsetmacro{\tXraw}{atan2(-\b*\st,\a*\ct)}
  \pgfmathsetmacro{\tX}{\tXraw + 180}
  \pgfmathsetmacro{\betaRx}{\cx + \a*cos(\tX)*\ct - \b*sin(\tX)*\st}
  \pgfmathsetmacro{\betaRy}{\cy + \a*cos(\tX)*\st + \b*sin(\tX)*\ct}

  \begin{scope}
    \clip (0,0) rectangle (0.9,0.9);
    \fill[CornflowerBlue, opacity=0.3, rotate around={\theta:(\cx,\cy)}]
      (\cx,\cy) ellipse [x radius=\a, y radius=\b];
  \end{scope}

  \node at (0.7,0.7) {$\mathcal{E}(\mathcal{H})$};

  \pgfmathsetmacro{\tStart}{\tX}
  \pgfmathsetmacro{\tEnd}{\tY}
  \pgfmathsetmacro{\tEndAdjusted}{\tY < \tX ? \tY + 360 : \tY}

  \draw[ultra thick, Thistle, rotate around={\theta:(\cx,\cy)}, opacity=0.8, line width=2pt]
    (\cx,\cy) ++(\tStart:\a cm and \b cm) 
    arc[start angle=\tStart, end angle=\tEndAdjusted, x radius=\a, y radius=\b];

  \pgfmathsetmacro{\tLambda}{\tStart + 0.5*(\tEndAdjusted - \tStart)}
  \pgfmathsetmacro{\betaLambdax}{\cx + \a*cos(\tLambda)*\ct - \b*sin(\tLambda)*\st}
  \pgfmathsetmacro{\betaLambday}{\cy + \a*cos(\tLambda)*\st + \b*sin(\tLambda)*\ct}

  \fill[Blue] (\betaBx,\betaBy) circle[radius=\pointSize];
  \node[Blue, below] at (\betaBx,\betaBy) {$f_b$};
  \fill[BrickRed] (\betaRx,\betaRy) circle[radius=\pointSize];
  \node[BrickRed, left] at (\betaRx,\betaRy) {$f_r$};
  \fill[Fuchsia] (\betaLambdax,\betaLambday) circle[radius=\pointSize];
  \node[Fuchsia, left] at (\betaLambdax,\betaLambday) {$f_\lambda$};

  \draw[dashed, Fuchsia] (0,0.4) -- (0.5,0) 
    node[anchor=south west] {\footnotesize $\lambda \mathcal{R}_r + (1-\lambda) \mathcal{R}_b = c$};

\end{tikzpicture}}
  \caption{Illustration of $f_\lambda$}
  \label{fig:Frontier-Lambda}
\end{subfigure}
\hfill
\begin{subfigure}{.48\textwidth}
  \centering
  \resizebox{0.65\linewidth}{!}{\begin{tikzpicture}[scale=6]

  \draw[->] (0,0) -- (1,0) node[anchor=north west] {$\mathcal{R}_r$};
  \draw[->] (0,0) -- (0,1) node[anchor=south east] {$\mathcal{R}_b$};

    
\pgfmathsetmacro{\pointSize}{0.01}

  \pgfmathsetmacro{\cx}{0.5}
  \pgfmathsetmacro{\cy}{0.5}
  \pgfmathsetmacro{\a}{0.4}
  \pgfmathsetmacro{\b}{0.25}
  \pgfmathsetmacro{\theta}{30}
  \pgfmathsetmacro{\contractionScale}{0.6}

  \pgfmathsetmacro{\ct}{cos(\theta)}
  \pgfmathsetmacro{\st}{sin(\theta)}

  \pgfmathsetmacro{\tYraw}{atan2(\b*\ct,\a*\st)} 
  \pgfmathsetmacro{\tY}{\tYraw + 180} 
  \pgfmathsetmacro{\betaBx}{\cx + \a*cos(\tY)*\ct - \b*sin(\tY)*\st}
  \pgfmathsetmacro{\betaBy}{\cy + \a*cos(\tY)*\st + \b*sin(\tY)*\ct}

  \pgfmathsetmacro{\tXraw}{atan2(-\b*\st,\a*\ct)}
  \pgfmathsetmacro{\tX}{\tXraw + 180}
  \pgfmathsetmacro{\betaRx}{\cx + \a*cos(\tX)*\ct - \b*sin(\tX)*\st}
  \pgfmathsetmacro{\betaRy}{\cy + \a*cos(\tX)*\st + \b*sin(\tX)*\ct}

  \begin{scope}
    \clip (0,0) rectangle (0.9,0.9);
    \fill[CornflowerBlue, opacity=0.3, rotate around={\theta:(\cx,\cy)}]
      (\cx,\cy) ellipse [x radius=\a, y radius=\b];
  \end{scope}

  \node at (0.7,0.7) {$\mathcal{E}(\mathcal{H})$};

  \pgfmathsetmacro{\tStart}{\tX}
  \pgfmathsetmacro{\tEnd}{\tY}
  \pgfmathsetmacro{\tEndAdjusted}{\tY < \tX ? \tY + 360 : \tY}

  \draw[ultra thick, Thistle, rotate around={\theta:(\cx,\cy)}, opacity=0.8, line width=2pt]
    (\cx,\cy) ++(\tStart:\a cm and \b cm) 
    arc[start angle=\tStart, end angle=\tEndAdjusted, x radius=\a, y radius=\b];

  \draw[ultra thick, Thistle, rotate around={\theta:(\cx,\cy)}, 
        opacity=0.8, line width=2pt, dashed]
    (\cx,\cy) ++(\tStart:\a*\contractionScale cm and \b*\contractionScale cm) 
    arc[start angle=\tStart, end angle=\tEndAdjusted, 
       x radius=\a*\contractionScale, y radius=\b*\contractionScale];

  \pgfmathsetmacro{\tLambda}{\tStart + 0.5*(\tEndAdjusted - \tStart)}
  \pgfmathsetmacro{\betaLambdax}{\cx + \a*cos(\tLambda)*\ct - \b*sin(\tLambda)*\st}
  \pgfmathsetmacro{\betaLambday}{\cy + \a*cos(\tLambda)*\st + \b*sin(\tLambda)*\ct}

  \pgfmathsetmacro{\hatBetaBx}{\cx + \contractionScale*(\betaBx - \cx)}
  \pgfmathsetmacro{\hatBetaBy}{\cy + \contractionScale*(\betaBy - \cy)}
  \pgfmathsetmacro{\hatBetaRx}{\cx + \contractionScale*(\betaRx - \cx)}
  \pgfmathsetmacro{\hatBetaRy}{\cy + \contractionScale*(\betaRy - \cy)}
  \pgfmathsetmacro{\hatBetaLambdax}{\cx + \contractionScale*(\betaLambdax - \cx)}
  \pgfmathsetmacro{\hatBetaLambday}{\cy + \contractionScale*(\betaLambday - \cy)}

  \pgfmathsetmacro{\rightAngleX}{\hatBetaLambdax}
  \pgfmathsetmacro{\rightAngleY}{\betaLambday}

  \fill[Blue] (\betaBx,\betaBy) circle[radius=\pointSize];
  \node[Blue, below] at (\betaBx,\betaBy) {$f_b$};
  \fill[BrickRed] (\betaRx,\betaRy) circle[radius=\pointSize];
  \node[BrickRed, left] at (\betaRx,\betaRy) {$f_r$};
  \fill[Fuchsia] (\betaLambdax,\betaLambday) circle[radius=\pointSize];
  \node[Fuchsia, left] at (\betaLambdax,\betaLambday) {$f_\lambda$};

  \fill[Blue] (\hatBetaBx,\hatBetaBy) circle[radius=\pointSize];
  \node[Blue, below] at (\hatBetaBx,\hatBetaBy) {$\widehat{f}_b$};
  \fill[BrickRed] (\hatBetaRx,\hatBetaRy) circle[radius=\pointSize];
  \node[BrickRed, left] at (\hatBetaRx,\hatBetaRy) {$\widehat{f}_r$};
  \fill[Fuchsia] (\hatBetaLambdax,\hatBetaLambday) circle[radius=\pointSize];
  \node[Fuchsia, above right] at (\hatBetaLambdax,\hatBetaLambday) {$\widehat{f}_\lambda$};

  \draw[Fuchsia, thin, solid] 
    (\betaLambdax,\betaLambday) -- 
    (\rightAngleX,\rightAngleY) -- 
    (\hatBetaLambdax,\hatBetaLambday) -- cycle;

\end{tikzpicture}}
  \caption{Empirical FA Frontier}
  \label{fig:EmpiricalFrontier}
\end{subfigure}
\caption{Illustration of FA frontier, parametrized by $\lambda$, and its empirical version.}
\label{fig:FA_Parametrized}
\end{figure}

\subsection{Empirical fairness--accuracy frontier}\label{sec:emp_fa_frontier}
In practice, the true distribution $P$ is unknown, so $f_\lambda$ is not available. Instead, we begin by obtaining an empirical estimate $\widehat{f}_\lambda \in \mathcal{H}$ based on the dataset $\{\mathcal{S}_r, \mathcal{S}_b\}$. The error associated with this estimate, known as the \textit{excess risk}, is given by $\mathcal{R}_\lambda(\widehat{f}_\lambda) - \mathcal{R}_\lambda(f_\lambda)$. It is nonnegative, since $f_\lambda$ minimizes $\mathcal{R}_\lambda(\cdot)$. This excess risk is the increase in the planner's weighted prediction error that results from estimating $f_\lambda$ from data; this quantity is referred to as welfare loss from estimation in the introduction. Its expectation over the draw of the data captures the decision-theoretic sense of risk over the estimation procedure; in the statistical treatment choice literature, this is known as the planner's expected regret \citep{manski2004statistical, kitagawa2018should}.

Note that the corresponding pair of risks $(\mathcal{R}_r(\widehat{f}_\lambda), \mathcal{R}_b(\widehat{f}_\lambda))$ will always lie in $\mathcal{E}(\mathcal{H})$, and may fall strictly in its interior. Moreover, for any estimator $\widehat{f}_\lambda$, the corresponding excess risk quantifies how much the frontier is pushed inward in the direction orthogonal to the line $\lambda \mathcal{R}_r + (1-\lambda) \mathcal{R}_b = \text{constant}$. An illustration of this phenomenon is shown in \Cref{fig:EmpiricalFrontier}.

Now, assuming that the distribution $P$ belongs to a class of distributions $\mathcal{P}$, we can proceed without knowledge of $P$ by finding an estimator $\widehat{f}_\lambda$ that minimizes the worst-case excess risk. More formally, we aim to solve the following problem:
\begin{equation}\label{eqn:minimax_empirical_frontier}
 \inf_{\widehat{f}_\lambda \in \mathcal{H}} \sup_{P \in \mathcal{P}} \EE \left[ \mathcal{R}_\lambda(\widehat{f}_\lambda) - \mathcal{R}_\lambda(f_\lambda) \right],  
\end{equation}
where the expectation is taken with respect to the draw of the dataset $\{\mathcal{S}_r, \mathcal{S}_b\}$.
Note that the comparison holds the planner's preferences and the underlying population fixed: excess risk is the loss from implementing the desired policy with limited information. The minimax criterion asks for an estimation rule with the smallest worst-case expected loss across the specified environments. It is thus a criterion for \textit{robust} implementation, rather than a choice of welfare weights or a claim that one rule is best in every environment. Minimax (maxmin) criteria of this kind have a long tradition in econometrics; see, e.g., \citet{chamberlain2000econometric, chamberlain2000econometrics} for minimax and other decision-theoretic approaches to estimation, and \citet{manski2004statistical, stoye2009minimax}, and \citet{hirano2009asymptotics} for treatment choice.
In the remainder of the paper, we study this minimax problem in the setting of linear regression. We present a lower bound for the worst-case excess risk in \eqref{eqn:minimax_empirical_frontier}, and develop an estimator that matches this lower bound with respect to many of the problem's parameters. We also discuss how the risks for the two groups change when this optimal estimator is used.

\subsection{Notation}
For a vector $v$, denote by $\norm{v}$ its Euclidean norm. For a positive-semidefinite (PSD) matrix $A$, we denote its trace by $\Tr{A}$, its minimum and maximum eigenvalues by $\lambda_{\min}(A)$ and $\lambda_{\max}(A)$, respectively, and its operator (spectral) norm by $\norm{A}$. For a PSD matrix $A$ and vector $v$, we define the Mahalanobis norm $\norm{v}_A$ of $v$ with respect to $A$ through $\norm{v}_A^2 \coloneq v^\top A v$. We write $I_d$ for the $d \times d$ identity matrix and use the Loewner order $A \preceq B$ to mean that $B - A$ is PSD. For two probability measures $P, Q$, we write $\tvdist{P, Q} \coloneq \frac{1}{2} \int |dP - dQ|$ and $\kldiv{P}{Q} \coloneq \int \log \frac{dP}{dQ} dP$ for their total variation and Kullback-Leibler divergences, respectively. For two functions $g(x)$ and $h(x)$, we write $g \lesssim h$ or $g = \mathcal{O}(h)$ if there exists a constant $C$ (independent of $x$) and a threshold $x_0$ such that $g(x) \leq C h(x)$ for all $x \geq x_0$. Finally, $d_h$ denotes the Hamming distance on the hypercube and $\Xi^d \coloneq \{-1, 1\}^d$ is the $d$-dimensional signed hypercube.
 
\section{Linear Model Class} \label{sec:linearmodel}
To study theoretical bounds on the empirical FA frontier, we restrict attention to the class of linear models:
\begin{definition}\label{def:linear_model}
The linear model class $\mathcal{P}_{\text{linear}}(\sigma_r^2, \sigma_b^2)$ consists of all distributions in which the response variable $Y$, conditional on covariate $X$ and group $g \in \mathcal{G}$, satisfies
\begin{equation*}
Y = X^\top \beta_g + \varepsilon_g,
\end{equation*}
for some parameter $\beta_g \in \RR^d$ and a random variable $\varepsilon_g$ such that $\EE[\varepsilon_g | X] = 0$ and $\EE[\varepsilon_g^2 | X] = \sigma_g^2>0$. If we further assume that the distribution of the covariate $X$ given group identity $g$ is known and equal to $P_X^g$, then we denote this subclass of $\mathcal{P}_{\text{linear}}(\sigma_r^2, \sigma_b^2)$ by $\mathcal{P}_{\text{linear}}(P_X^r, P_X^b, \sigma_r^2, \sigma_b^2)$.
\end{definition}
A canonical example of a distribution in $\mathcal{P}_{\text{linear}}(\sigma_r^2, \sigma_b^2)$ is the well-specified linear model with Gaussian noise that is homoskedastic within each group: $\varepsilon_g \sim \NN(0, \sigma_g^2)$, independently of $X$ conditional on $G = g$ where $\sigma_r^2$ and $\sigma_b^2$ may differ. 

In what follows, we take $\mathcal{H}$ as the class of linear models, i.e., $\{f_\beta(x) = x^\top \beta : \beta \in \RR^d\}$. In addition, we consider the squared loss function $\ell(f_\beta(X),Y) = (\beta^\top X - Y)^2$, and also simplify the denotation of the risks $\mathcal{R}_g(f_\beta)$ and $\mathcal{R}_\lambda(f_\beta)$ to $\mathcal{R}_g(\beta)$ and $\mathcal{R}_\lambda(\beta)$, respectively.\footnote{For the linear model class with squared loss, the feasible set $\mathcal{E}(\mathcal{H})$ is closed and convex under \Cref{assumption:invertible_cov} and $d \geq 2$; see \Cref{lemma:feasible_set_regression_convex}.  }
We write $\PP_g$, $\EE_g$, and $\operatorname{Var}_g$  to denote probability, expectation and variance, respectively, under the conditional distribution $P(X, Y | G = g)$. 

The group-$g$ population covariance matrix  is defined as $\Sigma_g \coloneq \EE_g[X X^\top]$,\footnote{We adopt the common shorthand `covariance matrix' to refer to the second-moment matrix, even when the covariates are uncentered.} and the corresponding sample covariance is 
 \begin{equation*}
     \widehat{\Sigma}_g \coloneq \frac{1}{n_g} \sum_{i = 1}^{n_g} X_i^{(g)} X_i^{(g)\top}. 
 \end{equation*}
For the subsequent analysis, we also find it helpful to define the cross-moment between $X$ and $Y$ for each group, along with its empirical counterpart:
\begin{equation}
\nu_g \coloneq \EE_g\left[XY \right] = \Sigma_g \beta_g,
\quad 
\widehat{\nu}_g \coloneq \frac{1}{n_g} \sum_{i= 1}^{n_g} Y_i^{(g)} X_i^{(g)}.
\end{equation}
The following result establishes when the model falls under the group-balanced scenario described in \Cref{sec:formulation}. The proof is provided in \Cref{proof:linear_group_balanced}.
\begin{lemma}\label{lemma:linear_model_group_balanced}
Suppose that
\begin{equation*}
    \sigma_r^2 - \sigma_b^2 \leq \norm{\beta_r-\beta_b}_{\Sigma_b}^2,\quad \text{ and }\quad  \sigma_b^2 - \sigma_r^2 \leq \norm{\beta_r-\beta_b}_{\Sigma_r}^2.
\end{equation*}
Then, the linear model given in \Cref{def:linear_model} exhibits a \emph{group-balanced} structure. That is, each group's risk-minimizing predictor achieves (weakly) lower prediction error on its own group than on the other:
\begin{equation*}
\mathcal{R}_r(\beta_r) \leq \mathcal{R}_b(\beta_r), \quad \mathcal{R}_r(\beta_b) \geq \mathcal{R}_b(\beta_b).
\end{equation*}
In particular, the model is always group-balanced when $\sigma_r^2 = \sigma_b^2$. 
\end{lemma}
The condition in \Cref{lemma:linear_model_group_balanced} requires that, for each group, the difference in the noise level across groups is no larger than the loss from using the other group's risk-minimizing predictor. The statistical results that follow hold over the full linear model class and do not require group-balancedness. Group-balancedness is needed only to interpret the weighted risk path as the FA frontier, as discussed below.

\subsection{Assumptions}
We introduce the following assumptions to ensure that population-level quantities and their empirical estimators are well-defined and admit tractable risk bounds.  Assumptions~\ref{assumption:invertible_cov} and \ref{assumption:invertible_emp_cov} are maintained throughout our analysis and guarantee identifiability and consistency of the least-squares estimator. Assumptions~\ref{assumption:small_ball} and \ref{assumption:subgaussian} provide control over the tails of the sample covariance spectrum, which we use to further simplify our risk bounds.

\begin{assumption}\label{assumption:invertible_cov}
For each group $g \in \mathcal{G}$, the population covariance matrix $\Sigma_g$ is invertible.
\end{assumption}
This assumption ensures that, for each group $g \in \mathcal{G}$, $\beta_g$ is the unique minimizer of $\mathcal{R}_g(\cdot)$.
\begin{assumption}\label{assumption:invertible_emp_cov}
For each group $g \in \mathcal{G}$, the empirical covariance matrix $$\widehat{\Sigma}_g \coloneq \frac{1}{n_g} \sum_{i = 1}^{n_g} X_i^{(g)} X_i^{(g)\top} $$ is invertible almost surely. 
\end{assumption}
This assumption requires that $n_g \geq d$. Given this condition, and as highlighted by \citet[Fact 1]{mourtada2022exact}, \Cref{assumption:invertible_emp_cov} is equivalent to either of the following two properties:
(i) the distribution $P(X \mid G = g)$ does not assign positive probability to any linear hyperplane; and
(ii) the ordinary least squares (OLS) estimator for each group, defined as
\begin{equation}\label{eqn:ols_estimator}
\widehat{\beta}_g \coloneq \argmin_{\beta \in \RR^d} \widehat{\mathcal{R}}_g(\beta)
= \argmin_{\beta \in \RR^d}
\sum_{i = 1}^{n_g} (X_i^{(g)\top} \beta - Y_i^{(g)})^2,
\end{equation}
is uniquely defined almost surely, and is given by $\widehat{\Sigma}_g^{-1} \widehat{\nu}_g.$
\begin{assumption}\label{assumption:small_ball}
We say that the covariate of group $g$ satisfies the \emph{small-ball condition} with parameters $C_g \geq 1$ and $\alpha_g \in (0, 1]$ if, for all nonzero $\theta \in \RR^d$ and all $t > 0$, 
    \begin{equation*}
        \PP_g\left( | \theta^\top X | \leq t \norm{\theta}_{\Sigma_g} \right) \leq \left( C_g t \right)^{\alpha_g}. 
    \end{equation*}
\end{assumption}
This condition provides lower-tail control of the sample covariance spectrum, ensuring that the empirical covariance does not degenerate in directions of low variance of the covariates. We adopt the variant of the small-ball condition used in, e.g., \citet{mourtada2022exact, fukuchi2023demographic}, which requires anti-concentration uniformly over all thresholds $t>0$, rather than only for a fixed threshold as in, e.g., \citet{koltchinskii2015bounding}.\footnote{When $n_g \geq d$, \Cref{assumption:small_ball} implies \Cref{assumption:invertible_emp_cov}. Nevertheless, we state \Cref{assumption:invertible_emp_cov} separately, since some of our results require only this weaker assumption.} This assumption holds for multivariate Gaussian distributions with $\alpha_g = 1$, since in that case $\theta^\top X$ is Gaussian with variance $\|\theta\|_{\Sigma_g}^2$. Furthermore, \cite{rudelson2015small} show that this condition holds with $\alpha_g=1$ for covariates with independent coordinates and bounded density.

\begin{assumption}\label{assumption:subgaussian}
We say that the covariates of group $g$ satisfy the sub-Gaussian assumption if there exists a constant $K_g \geq 1$ such that, for any $u \in \RR^d\setminus\{0\}$ and every $t \geq 0$,
\begin{equation*}
    \PP_g\left( \left| X^\top u\right| \geq t \norm{u}_{\Sigma_g} \right) \leq 2 \exp\left( \frac{-t^2}{K_g^2} \right).
\end{equation*}
\end{assumption}
This condition ensures that the covariates have light tails. It is satisfied, for instance, by multivariate Gaussian distributions or bounded distributions.

\begin{remark}
    Assumptions~\ref{assumption:small_ball} and \ref{assumption:subgaussian} are complementary and provide control over the spectrum of the sample covariance matrix. The sub-Gaussian condition bounds the upper tail, limiting the growth of the largest eigenvalue and thereby the variance of our estimators. The small-ball condition controls the lower tail, preventing the collapse of the smallest eigenvalue such that inverse covariance terms are bounded. As discussed above, the two assumptions hold simultaneously for a broad class of distributions, including multivariate Gaussian distributions and covariates with independent sub-Gaussian coordinates and bounded density.
\end{remark}

\subsection{Empirical estimators}
Let $\beta_\lambda$ denote the optimal predictor corresponding to the weighted loss $\mathcal{R}_\lambda(\beta)$:
\begin{equation} \label{eqn:beta_lambda}
\beta_\lambda \coloneq \argmin_{\beta \in \RR^d} \mathcal{R}_\lambda(\beta).
\end{equation}
Under \Cref{assumption:invertible_cov}, it is straightforward to verify that $\beta_\lambda$ is unique and given by $\beta_\lambda = \Sigma_\lambda^{-1} \nu_\lambda$,
where  
\begin{equation} \label{eqn:Sigma_lambda}
\Sigma_\lambda \coloneq \lambda \Sigma_r + (1 - \lambda) \Sigma_b \quad \text{and} \quad 
\nu_\lambda\coloneq \lambda \nu_r + (1 - \lambda) \nu_b. 
\end{equation} 
Moreover, when $d \geq 2$, under group-balancedness, the risk pairs $\{(\mathcal{R}_r(\beta_\lambda), \mathcal{R}_b(\beta_\lambda)):\lambda \in [0,1]\}$ trace the entire FA frontier. If $\beta_r \neq \beta_b$, this frontier is strictly convex; see \Cref{lemma:frontier_regression_convex}.
Notice that $\beta_\lambda$ can also be written as
\begin{equation} \label{eqn:optimal_weighted_matrices}
\beta_\lambda
= (\lambda \Sigma_r + (1-\lambda) \Sigma_b)^{-1}
  (\lambda \Sigma_r \beta_r + (1-\lambda) \Sigma_b \beta_b).
\end{equation}
Thus, the population-optimal predictor depends on both the welfare weight $\lambda$ and the groups' covariate distributions. Even with equal welfare weights, $\lambda = 1/2$, simply averaging the two group-optimal predictors need not minimize the weighted loss. The reason is that equal deviations from the two groups' optimal coefficients can entail different prediction losses: a coefficient error in direction $v$ is more costly for a group with a larger second moment $\mathbb{E}_g[(X^\top v)^2] = v^\top \Sigma_g v$. The covariance weighting in \eqref{eqn:optimal_weighted_matrices} accounts for these differences. Equal concern for the two groups therefore does not imply equal weight on their optimal coefficients in the common predictor. This observation relates to results in the causal inference literature characterizing how pooled regression combines coefficients across heterogeneous groups: each group's contribution depends on both its population share and the variation in explanatory variables within each group \citep{angrist1995estimating, solon2015we}.


The next proposition allows us to relate the excess risk of any estimator $\widehat{\beta}$ to that of $\beta_\lambda$, as well as to its estimation error on both groups. For details, see \Cref{proof:proposition:excess_risk_to_matrix_norm}. 
\begin{proposition} \label{proposition:excess_risk_to_matrix_norm}
For a given $\widehat{\beta} \in \RR^d$, $\lambda \in [0,1]$, and $g \in \mathcal{G}$, the following identities hold:
\begin{subequations} \label{eqn:error_to_matrix_norm}
\begin{align} 
\mathcal{R}_\lambda(\widehat{\beta}) - \mathcal{R}_\lambda(\beta_\lambda) &= 
\lambda \: \|\widehat{\beta}-\beta_\lambda\|_{\Sigma_r}^2 + 
(1-\lambda) \|\widehat{\beta}-\beta_\lambda\|_{\Sigma_b}^2, \label{eqn:excess_to_matrix_norm} \\
\mathcal{R}_g(\widehat{\beta}) - \mathcal{R}_g(\beta_\lambda) &= \|\widehat{\beta}-\beta_\lambda\|_{\Sigma_g}^2 + 2 (\widehat{\beta}-\beta_\lambda)^\top \Sigma_g (\beta_\lambda-\beta_g). \label{eqn:estimation_to_matrix_norm}
\end{align}
\end{subequations}
\end{proposition}

An immediate consequence is that estimation can affect aggregate welfare and its distribution across groups at different orders of magnitude. The first identity shows that the increase in the planner's weighted loss is quadratic in the estimation error $\widehat{\beta}-\beta_\lambda$: because $\beta_\lambda$ minimizes the weighted objective, the first-order changes in the two groups' losses cancel after applying the welfare weights. By contrast, the second identity shows that each group's loss generally retains a signed first-order term. Consequently, a small estimation error can benefit one group and harm the other to first order, even though its aggregate welfare cost is only second order.

Next, recall that our goal is to minimize the worst case expected excess risk, i.e.,
\begin{equation} \label{eqn:minimax_regression}
\inf_{\widehat{\beta}} \sup_{P \in \mathcal{P}} \EE_{\mathcal{S}_r, \mathcal{S}_b} \left[ \mathcal{R}_\lambda (\widehat{\beta}) - \mathcal{R}_\lambda(\beta_\lambda) \right]
= 
\inf_{\widehat{\beta}} \sup_{P \in \mathcal{P}} \EE_{\mathcal{S}_r, \mathcal{S}_b} \left[ \lambda \|\widehat{\beta}-\beta_\lambda\|_{\Sigma_r}^2 + 
(1-\lambda) \|\widehat{\beta}-\beta_\lambda\|_{\Sigma_b}^2 \right],     
\end{equation}
where $\widehat{\beta}$ is a function of the datasets $\{\mathcal{S}_r, \mathcal{S}_b\}$.\footnote{Moving forward, we omit the dependence of the expectation on the datasets when it is clear from the context.} 
For the cases $\lambda = 0$ and $\lambda = 1$, which reduce to the single-group setting (say, group $g$), \cite{mourtada2022exact} shows that the OLS estimator $\widehat{\beta}_g$, defined in \eqref{eqn:ols_estimator}, is the optimal worst-case estimator with respect to the minimax criterion in \eqref{eqn:minimax_regression}, when $\mathcal{P}$ is taken to be the linear model class $\mathcal{P}_{\text{linear}}(\sigma_r^2, \sigma_b^2)$. Interestingly, he also shows that knowledge of the covariate distribution $P_X^g$ does not affect this optimality: the OLS estimator remains minimax optimal even when the supremum is taken over the more restricted class $\mathcal{P}_{\text{linear}}(P_X^r, P_X^b, \sigma_r^2, \sigma_b^2)$.


This naturally leads us to consider the empirical version of $\mathcal{R}_\lambda(\cdot)$, denoted by $\widehat{\mathcal{R}}_\lambda(\beta)$, which is given by   
\begin{equation}\label{eqn:empirical_lambda_risk}
\widehat{\mathcal{R}}_\lambda(\beta) \coloneq \lambda \widehat{\mathcal{R}}_r(\beta) + ( 1 - \lambda) \widehat{\mathcal{R}}_b(\beta).
\end{equation}

Under \Cref{assumption:invertible_emp_cov}, and via first-order optimality conditions, it is straightforward to see the minimizer of \eqref{eqn:empirical_lambda_risk} is given by
\begin{equation} \label{eqn:OLS_estimator_lambda}
\widehat{\beta}_\lambda \coloneq \widehat{\Sigma}_\lambda^{-1} \widehat{\nu}_\lambda 
\quad \text{with} \quad 
\widehat{\Sigma}_\lambda = \lambda \widehat{\Sigma}_r + (1-\lambda) \widehat{\Sigma}_b
\, \, \text{ and } \, \,
\widehat{\nu}_\lambda \coloneq \lambda \widehat{\nu}_r + (1-\lambda) \widehat{\nu}_b.
\end{equation}
As we will see in \Cref{sec:unknown_covariance}, this estimator is nearly minimax optimal when the distribution class $\mathcal{P}$ in \eqref{eqn:minimax_regression} is taken to be the linear model class $\mathcal{P}_{\text{linear}}(\sigma_r^2, \sigma_b^2)$. However, unlike in the one-group case, under additional knowledge of the covariate distribution---or even just the covariance matrices $\Sigma_r$ and $\Sigma_b$---a different estimator becomes minimax optimal. This marks a sharp departure from the classical linear regression setting, in which the OLS estimator is optimal regardless of knowledge of the covariate distribution.

\section{Estimation with Known Covariance} \label{sec:knowncov}
We first consider the setting in which the population covariance matrices $\Sigma_r$ and $\Sigma_b$ are known. 
This benchmark separates knowledge of the population's characteristics from knowledge of how those characteristics predict outcomes. In many settings, observations of the covariates alone are much cheaper to obtain than observations with measured outcomes; for example, administrative records or application files often contain individual characteristics but not the outcome of interest. A large sample of covariates can therefore provide precise estimates of each group's covariance matrix even when data with observed outcomes remain limited. Incorporating auxiliary covariate information is in the spirit of the econometric literature on combining data sources \citep{imbens1994combining, hellerstein1999imposing, ridder2007econometrics}. As we show, this information changes the optimal estimation procedure and can reduce the welfare cost of learning a common predictor from limited outcome data.

More specifically, given observations from both groups, we propose the following estimator:
\begin{equation} \label{eqn:optimal_estimator_known_cov}
\tilde{\beta}_\lambda = (\Sigma_\lambda)^{-1} (\lambda \Sigma_r \widehat{\beta}_r + (1-\lambda) \Sigma_b \widehat{\beta}_b),
\end{equation}
with $\Sigma_\lambda = \lambda \Sigma_r + (1 - \lambda) \Sigma_b$, and $\widehat{\beta}_g = \widehat{\Sigma}_g^{-1} \widehat{\nu}_g$, as defined in \eqref{eqn:Sigma_lambda} and \eqref{eqn:ols_estimator}, respectively. This estimator preserves the structure of the optimal predictor $\beta_\lambda$, but replaces the true cross-moment $\nu_g = \Sigma_g \beta_g$ with the empirical quantity $\Sigma_g \widehat{\beta}_g$ for each group $g$.\footnote{Another natural estimator in the known-covariance regime is $\dbtilde{\beta}_\lambda \coloneq \Sigma_\lambda^{-1}(\lambda \widehat{\nu}_r + (1-\lambda)\widehat{\nu}_b)$. While both estimators are unbiased for $\beta_\lambda$, the estimator $\dbtilde{\beta}_\lambda$ incurs additional error from the empirical covariance fluctuations. Under the assumptions of our analysis and $d \geq 2$, these fluctuations imply that $\dbtilde{\beta}_\lambda$ has infinite worst-case excess risk over the unrestricted linear model class, and is therefore minimax suboptimal; see \Cref{sec:alternative_estimator_known_covariance}.}

In what follows, we first establish that this estimator is minimax optimal when the covariance matrices (or even the full distribution of the covariates) are known. We then characterize its estimation error for each group as a function of the distribution’s parameters. Finally, we discuss the implications of these results for algorithm design and optimal sampling strategies.

\subsection{The optimal estimator}
As \Cref{proposition:excess_risk_to_matrix_norm} shows, the excess risk of an estimator can be decomposed into the sum of its squared distances from the optimal predictor $\beta_\lambda$ under the Mahalanobis norms induced by $\Sigma_r$ and $\Sigma_b$. The following result shows that $\tilde{\beta}_\lambda$ minimizes the worst-case distance $\|\beta - \beta_\lambda\|_{\Sigma_g}^2$ for each group $g$, thereby implying its minimax optimality.
\begin{theorem}\label{theorem:minimax_optimal_known_cov}
Suppose Assumptions~\ref{assumption:invertible_cov} and \ref{assumption:invertible_emp_cov} hold. Then, for any group $g \in \mathcal{G}$, we have
\begin{align}
& \inf_{\widehat{\beta}} \sup_{P \in \mathcal{P}_{\text{linear}} (P_X^r, P_X^b, \sigma_r^2, \sigma_b^2)} \EE \left[ \norm{\widehat{\beta} - \beta_\lambda }_{\Sigma_g}^2\right] \\ 
&\quad \quad = \lambda^2 \frac{\sigma_r^2}{n_r} \EE \left[ \Tr{ \Sigma_g \Sigma_\lambda^{-1} \Sigma_r\widehat{\Sigma}_r^{-1}  \Sigma_r\Sigma_\lambda^{-1}} \right]  + (1 - \lambda)^2 \frac{\sigma_b^2}{n_b} \EE \left[ \Tr{   \Sigma_g \Sigma_\lambda^{-1} \Sigma_b \widehat{\Sigma}_b^{-1}\Sigma_b \Sigma_\lambda^{-1}} \right],  \nonumber 
\end{align}
and the infimum is achieved by setting $\widehat{\beta}$ equal to $\tilde{\beta}_\lambda$, given by \eqref{eqn:optimal_estimator_known_cov}. 
\end{theorem}

The minimax result identifies an information constraint on implementing the planner's target: no alternative estimator based on the same observations can lower the worst-case estimation cost.
Notice also that the expression does not depend on $\beta_r-\beta_b$. One might expect the cost of having access to only finite data to increase as the two groups' ideal predictors move further apart. After all, the planner must then combine estimates of two increasingly different predictors. Yet, when covariances are known, this intuition does not hold: holding the covariate distributions and noise levels fixed, the optimal estimator's expected excess risk does not increase with the distance between $\beta_r$ and $\beta_b$. Greater disagreement changes the population-optimal predictor, but does not make that predictor more costly to estimate. As we will see in the next section, this conclusion changes when the covariances must also be estimated: heterogeneity between the groups' ideal predictors then creates an additional source of estimation error.

The proof is given in \Cref{proof:theorem:minimax_optimal_known_cov} and proceeds in two parts. We first derive the exact finite-sample risk of $\tilde{\beta}_\lambda$ for any distribution in $\mathcal{P}_{\text{linear}}(P_X^r, P_X^b, \sigma_r^2, \sigma_b^2)$. We then show that no estimator can uniformly improve on this risk by proving a matching minimax lower bound. The lower bound follows from the fact that minimax risk is at least as large as the Bayes risk under any prior \citep{lehmann1998theory}. We therefore restrict attention to the Gaussian-noise submodel, place independent Gaussian priors on the group-specific predictors $\beta_r$ and $\beta_b$, and evaluate the resulting Bayes risk. Taking a sequence of increasingly diffuse Gaussian priors and applying monotone convergence yields the stated lower bound. \citet{chamberlain2000econometric} uses the same connection between minimax estimation and Bayes risk under a least-favorable prior.

\begin{remark}\label{remark:only_cov}
Notice that, while the minimax problem is posed over the class $\mathcal{P}_{\text{linear}}(P_X^r, P_X^b, \sigma_r^2, \sigma_b^2)$, which fixes the full covariate distribution in each group, the optimal estimator $\tilde{\beta}_\lambda$ only uses the covariance matrices. In other words, knowledge of features of the covariate distributions beyond second moments cannot be used to design an estimator with lower minimax risk.  
\end{remark}

As stated in the discussion before \Cref{theorem:minimax_optimal_known_cov}, this result, together with \Cref{proposition:excess_risk_to_matrix_norm}, yields the following corollary establishing the minimax optimality of $\tilde{\beta}_\lambda$ with respect to the excess risk of $\mathcal{R}_\lambda(\cdot)$.
\begin{corollary} \label{corollary:minimax_optimal_known_cov}
Suppose Assumptions~\ref{assumption:invertible_cov} and \ref{assumption:invertible_emp_cov} hold. Then, $\tilde{\beta}_\lambda$ is the minimizer of the worst-case excess risk \eqref{eqn:minimax_regression} when $\mathcal{P}$ is set as $\mathcal{P}_{\text{linear}} (P_X^r, P_X^b, \sigma_r^2, \sigma_b^2)$.
\end{corollary}

\subsection{Bounding the excess risk}

In this subsection, we impose additional distributional assumptions to illustrate the implications of \Cref{theorem:minimax_optimal_known_cov}. These finite-sample bounds isolate how group sample sizes, noise levels, covariance structure, and dimension affect the statistical cost of implementing the planner's preferred predictor. For example, this dependence shows how an imbalance in the number of observations across the two groups---a common concern in the fairness literature---interacts with welfare weights and group characteristics to impact the final expected error. When the upper and lower bounds match up to constants, the dependence reflects that the costs are unavoidable in the sense that no alternative estimator can improve its order in the worst case.

\begin{corollary}\label{cor:known_small_ball}
Suppose that \Cref{assumption:invertible_cov} and \Cref{assumption:invertible_emp_cov} hold and that the small-ball condition (\Cref{assumption:small_ball}) holds for each group $g \in \mathcal{G}$ with constants $(C_g, \alpha_g)$. Also suppose that $n_g \geq 6 d/ \alpha_g$ and that $d \geq 2$. Then we have
\begin{equation}\label{eq:known_general_small_ball_upper}
    \EE  \left[ \norm{ \tilde{\beta}_\lambda - \beta_\lambda }_{\Sigma_g}^2\right] \leq 2 \lambda^2 C_r' \frac{\sigma_r^2}{n_r} \operatorname{Tr}\left( \Sigma_g \Sigma_{\lambda}^{-1} \Sigma_r \Sigma_{\lambda}^{-1}\right) + 2 (1-\lambda)^2 C_b' \frac{\sigma_b^2}{n_b} \operatorname{Tr}\left( \Sigma_g \Sigma_{\lambda}^{-1} \Sigma_b \Sigma_{\lambda}^{-1}\right),
\end{equation}
where $C_g' = 3 C_g^4 \exp(1 + 9/\alpha_g).$ 

Furthermore, when the covariates satisfy \Cref{assumption:subgaussian} with parameter $K_g$ and $n_g \geq \max\left\{6 \alpha_g^{-1} d, 12 \alpha_g^{-1} \log(12 \alpha_g^{-1})\right\}$, we have
\begin{equation}\label{eq:known_general_small_ball_subgaussian}
    \EE \left[ \norm{ \tilde{\beta}_\lambda - \beta_\lambda }_{\Sigma_g}^2\right] \leq \lambda^2 \frac{\sigma_r^2}{n_r}  \norm{W_{g, r}} \left( d + \frac{32 C_r' K_r^4 d^2}{n_r}\right) + (1-\lambda)^2 \frac{\sigma_b^2}{n_b} \norm{W_{g, b}} \left( d + \frac{32 C_b' K_b^4 d^2}{n_b}\right)
\end{equation}
where $W_{g, h} \coloneq \Sigma_{h}^{1/2}\Sigma_\lambda^{-1}\Sigma_g\Sigma_{\lambda}^{-1}\Sigma_h^{1/2}$.
\end{corollary}
See \Cref{sec:proof_cor:known_small_ball} for a proof. 
\begin{remark} \label{remark:constants}
The constant $C'_g$ arises from the bound on the minimum eigenvalue of the sample covariance, which is used here to bound the trace of each $\widehat{\Sigma}_g^{-1}$, as in \citet[Theorem 4]{mourtada2022exact}. However, as highlighted in \citet[Remark 6]{mourtada2022exact}—based on results from \citet{wu2012optimal} and \citet{edelman1988eigenvalues}—this constant can be significantly reduced for Gaussian distributions. In particular, if the covariate distributions are Gaussian, $C'_g$ can be replaced by a constant whose limit, when ${d}/{n_g} \to h \in (0, 1)$, is upper bounded by
\begin{equation*}
\left( \frac{1}{h} \right)^{3h} \left( \frac{\sqrt{e}}{1 - h} \right)^{3(1-h)},
\end{equation*}
which, in turn, can be shown to be upper bounded by $(1+\sqrt{e})^3$.
\end{remark}
The following corollary combines the differences in matrix norms to explicitly characterize the excess risk bound for the optimal estimator $\tilde{\beta}_\lambda$.
\begin{corollary} \label{corollary:known_cov_excess_final}
Under the premise of \Cref{cor:known_small_ball}, we have
\begin{equation}
\EE \left[ \mathcal{R}_\lambda (\tilde{\beta}_\lambda) - \mathcal{R}_\lambda(\beta_\lambda) \right]
\leq 2 \lambda^2 C_r' \frac{\sigma_r^2}{n_r} \operatorname{Tr}(\Sigma_r \Sigma_{\lambda}^{-1}) + 2(1-\lambda)^2 C_b' \frac{\sigma_b^2}{n_b}\operatorname{Tr}\left(\Sigma_b \Sigma_{\lambda}^{-1}\right).  
\end{equation}
\end{corollary}
We conclude this section with a few remarks on the insights we draw from these results.
\paragraph{Optimal allocation of the sampling budget:} In many applications, the total sampling budget is fixed. It is therefore natural to ask: how should we allocate our sampling budget across the two groups? A simple Cauchy--Schwarz argument shows that the excess risk upper bound in \Cref{corollary:known_cov_excess_final} is minimized over allocations satisfying $n_r + n_b = n$ by the allocation rule
\begin{equation}\label{eqn:optimal_sampling_known_cov}
\frac{n_r}{n_b} \asymp \frac{\lambda \sigma_r}{(1-\lambda) \sigma_b } \cdot \left(\frac{{\operatorname{Tr}\left( \Sigma_r \Sigma_{\lambda}^{-1} \right)}}{{\operatorname{Tr}\left( \Sigma_b \Sigma_{\lambda}^{-1} \right)}}\right)^{1/2}.
\end{equation}

The dependence on $\lambda$ is intuitive: groups that receive greater weight in the objective should receive a larger share of the sampling budget. The dependence on $\sigma_g$ has a similarly natural interpretation: when the outcome noise is larger in group $g$, each observation is less informative, so more data are needed to achieve the same reduction in estimation error.
The trace term $\operatorname{Tr}\left( \Sigma_g \Sigma_{\lambda}^{-1} \right)$ can be viewed as an effective measure of covariate dispersion for group $g$, adjusted for relevance to the planner's objective. All else equal, groups with larger such covariate dispersion require more observations to achieve the same reduction in estimation error.\footnote{For example, in the spherical model with $\Sigma_g = \rho_g^2 I_d$ for known $\rho_g>0$, the second term in \eqref{eqn:optimal_sampling_known_cov} reduces to $ \rho_r / \rho_b$. } Taken together, \eqref{eqn:optimal_sampling_known_cov} implies that, under an efficient sampling scheme, more of the sampling budget is allocated to groups that are weighted more heavily in the social planner's objective and to groups for which estimation is noisier. Equivalently, the rule can be viewed as equalizing the weighted marginal reduction in excess risk generated by an additional observation in each group. The rule parallels Neyman allocation in stratified sampling \citep{neyman1934representative, cochran1977sampling}, which allocates a fixed sampling budget across strata in proportion to the population share times the standard deviation of the outcome within each stratum. Here, the welfare weight plays the role of the population share, and the outcome noise scaled by the covariate dispersion is analogous to the standard deviation scaling.
\paragraph{Groupwise estimation error---finite-sample effects can shift fairness–accuracy tradeoffs:}
Note that, for each group $g \in \mathcal{G}$, the estimator $\widehat{\beta}_g$ is unbiased for $\beta_g$, since
\begin{equation} \label{eqn:beta_hat_g}
\widehat{\beta}_g = \widehat{\Sigma}_g^{-1} \widehat{\nu}_g = \beta_g +  \widehat{\Sigma}_g^{-1} \frac{1}{n_g} \sum_{i = 1}^{n_g} X_i^{(g)}  \varepsilon_i^{(g)}
\end{equation}
Consequently, the optimal estimator $\tilde{\beta}_\lambda$ is also an unbiased estimator of $\beta_\lambda$. As a result, and by \Cref{proposition:excess_risk_to_matrix_norm}, the estimation error of this estimator on group $g$ is given by
\begin{equation}
\EE \left[ \mathcal{R}_g(\tilde{\beta}_\lambda) - \mathcal{R}_g(\beta_\lambda) \right] = \EE \left[  \|\tilde{\beta}_\lambda-\beta_\lambda\|_{\Sigma_g}^2 \right]= \Tr{\Sigma_g M_\lambda},
\end{equation}
where $M_\lambda \coloneq \EE\left[ (\tilde{\beta}_\lambda - \beta_\lambda)(\tilde{\beta}_\lambda - \beta_\lambda)^\top\right]$. This decomposition highlights the difference across groups in the error incurred by using the empirical estimator instead of the true parameter. The matrix $M_\lambda$ captures the common finite sample error in the estimator. Notably, differences in sample size $n_r, n_b$ and noise $\sigma_r, \sigma_b$ affect both groups symmetrically through $M_\lambda$. The asymmetry in each group's resulting excess risk arises (only) through how this shared estimation error is amplified by $\Sigma_r$ versus $\Sigma_b$.

Therefore, equal representation in the training data, i.e., $n_r=n_b$, and equal exposure to estimation losses are distinct objectives. Data collected from either group can benefit both groups because the fitted predictor is shared. However, because the groups have different covariate distributions, the same estimation error in the shared predictor can increase one group's expected risk more than the other's. In fact, when $\Sigma_r=\Sigma_b$, the expected risk increments are equal even if sample sizes differ.

An important implication of this discrepancy is that finite-sample estimation can distort the realized fairness--accuracy tradeoff away from the planner's intended tradeoff encoded by $\lambda$. Thus, when comparing different choices of $\lambda$, the planner should consider not only the population-optimal risks, but also the additional losses that each group is expected to bear when the predictor is estimated from finite data.

\paragraph{Expectation vs. one-instance realization:}
It is worth noting that the results above are computed in expectation over the random draw of datasets $\{\mathcal{S}_r, \mathcal{S}_b\}$, and are not for a particular realization. For the excess risk with respect to the weighted loss, i.e., $\mathcal{R}_\lambda(\tilde{\beta}_\lambda) - \mathcal{R}_\lambda(\beta_\lambda)$, however, each realization is also nonnegative. This means that the empirical estimator always increases the loss function $\mathcal{R}_\lambda(\cdot)$, simply because $\beta_\lambda$ is its minimizer. By contrast, this monotonicity does not hold for groupwise risks: one group’s estimation error $\mathcal{R}_g(\cdot)$ might \textit{decrease} under the empirical estimator relative to $\beta_\lambda$, even though the other group's risk increases enough to make the weighted objective worse. The following result formalizes this distinction.
\begin{proposition} \label{proposition:expectation-vs-realization}
Suppose Assumptions \ref{assumption:invertible_cov}–\ref{assumption:subgaussian} hold, the covariance matrix $\Sigma_g$ is equal to $\rho_g^2 I_d$ for some $\rho_g > 0$,  $\beta_r \neq \beta_b$, $\lambda \in (0,1)$, and $n_g \geq \alpha_g^{-1} \max\{6d, 48\}$. Then, for any $\delta > 0$, there exists $\bar{n}$ such that, for any $n_r, n_b \geq \bar{n}$, with probability at least $1/2 - \delta$ over the draw of the data, we have
\begin{equation*}
\mathcal{R}_g(\tilde{\beta}_\lambda) < \mathcal{R}_g(\beta_\lambda).    
\end{equation*}
\end{proposition}
The proof is given in \Cref{proof:proposition:expectation-vs-realization}. This result follows from the decomposition established in \Cref{proposition:excess_risk_to_matrix_norm}: 
\begin{equation}
\mathcal{R}_g(\beta) - \mathcal{R}_g(\beta_\lambda) = \|\beta-\beta_\lambda\|_{\Sigma_g}^2 + 2~(\beta-\beta_\lambda)^\top \Sigma_g (\beta_\lambda-\beta_g). 
\end{equation}
The first term is nonnegative and, with high probability, bounded by $C_1/n$ for some constant $C_1$. The second term, when multiplied by $\sqrt{n}$, converges to a mean-zero Gaussian with constant variance (up to an error of order $\mathcal{O}(1/\sqrt{n})$). Therefore, for large enough $n$, the leading term is a mean-zero term of order $\mathcal{O}(1/\sqrt{n})$.

Thus, although the expected groupwise estimation error increases when using an empirical estimator, for large enough $n$ it increases with a probability close to $\frac{1}{2}$. This reflects a \emph{redistribution of risk} between groups due to statistical uncertainty. On the other hand,  since the excess risk---defined as the $\lambda$-weighted sum of the two groups’ estimation errors---is always nonnegative under an empirical estimator, in any given realization at most one group can benefit from estimation error; any such gain must be offset by a larger loss for the other group. As a result, a small aggregate welfare loss can conceal a larger change in who bears prediction errors.
\section{Estimation with Unknown Covariance}\label{sec:unknown_covariance}
We next consider the setting in which the  distributions of the covariates are unknown and must be estimated from data. Here, our proposed estimator is the OLS estimator \eqref{eqn:OLS_estimator_lambda}, which, as we recall, is given by:
\begin{equation}\label{eqn:optimal_estimator_unknown_cov}
\widehat{\beta}_\lambda = \widehat{\Sigma}_\lambda^{-1} \widehat{\nu}_\lambda = (\lambda \widehat{\Sigma}_r + (1-\lambda) \widehat{\Sigma}_b)^{-1} (\lambda \widehat{\nu}_r + (1-\lambda) \widehat{\nu}_b).
\end{equation}
In this section, we first present a result that characterizes the excess risk of this estimator and decomposes it into bias and variance terms. We then provide upper and lower bounds under additional distributional assumptions (namely, sub-Gaussian and small-ball conditions), thereby establishing the near-optimality of the estimator under these assumptions.
\subsection{The bias--variance decomposition}
Recall from \Cref{proposition:excess_risk_to_matrix_norm} that the excess risk of $\widehat{\beta}_{\lambda}$ can be expressed as the weighted sum of the two distances $\|  \widehat{\beta}_{\lambda} - \beta_\lambda \|_{\Sigma_r}^2$ and $\| \widehat{\beta}_{\lambda} - \beta_\lambda \|_{\Sigma_b}^2$. Thus, as in the previous section, we begin by characterizing these two distances. 
\begin{proposition}\label{prop:unknown_cov_upper_bound}
Suppose Assumptions~\ref{assumption:invertible_cov} and \ref{assumption:invertible_emp_cov} hold. Then, for any group $g \in \mathcal{G}$, we have:
\begin{equation}
\EE \left[ \norm{ \widehat{\beta}_{\lambda} - \beta_\lambda }_{\Sigma_g}^2\right] = \mathcal{V}_g(\lambda) + \mathcal{B}_g(\lambda),
\end{equation}
where the variance term is given by
\begin{equation}
\mathcal{V}_g(\lambda) \coloneq \lambda^2 \frac{\sigma_r^2}{n_r } \EE \Tr{ \widehat{\Sigma}_{\lambda}^{-1} \Sigma_g \widehat{\Sigma}_{\lambda}^{-1} \widehat{\Sigma}_r} + (1-\lambda)^2 \frac{\sigma_b^2}{n_b} \EE \Tr{\widehat{\Sigma}_{\lambda}^{-1} \Sigma_g \widehat{\Sigma}_{\lambda}^{-1} \widehat{\Sigma}_b},
\end{equation}
and the bias term is given by
\begin{align}
\mathcal{B}_g(\lambda)
&\coloneq \EE \left[
    \norm{
        \Sigma_g^{1/2} \widehat{\Sigma}_\lambda^{-1}
        \left(
            \lambda \widehat{\Sigma}_r (\beta_r - \beta_\lambda)
            + (1-\lambda) \widehat{\Sigma}_b (\beta_b - \beta_\lambda)
        \right)
    }^2
\right]
\label{eq:unknown_bias_1}
\\
&= \EE \left[
    \norm{
        \Sigma_g^{1/2}
        \left[
            \left(
                I_d + \frac{1-\lambda}{\lambda}
                \widehat{\Sigma}_r^{-1} \widehat{\Sigma}_b
            \right)^{-1}
            -
            \left(
                I_d + \frac{1-\lambda}{\lambda}
                \Sigma_r^{-1} \Sigma_b
            \right)^{-1}
        \right] (\beta_r - \beta_b)
    }^2
\right]. 
\label{eq:unknown_bias_2}
\end{align}
\end{proposition}
See \Cref{sec:proof_prop:unknown_cov_upper_bound} for the proof.
The variance term closely resembles the error rate in the known-covariance case: it captures the irreducible sampling error within each group, as it is scaled by the squared fairness weights and the inverse sample sizes of each group. The bias term, however, represents a key departure from the known-covariance setting. This term can be viewed as a systematic welfare distortion caused by informational asymmetries; the planner misallocates welfare weights across groups due to the fact that they cannot perfectly observe the true structure of the heterogeneity across groups. In particular, it arises from the underlying difference in the true coefficients $\beta_r$ and $\beta_b$, it depends explicitly on $\beta_r - \beta_b$, and it vanishes when $\beta_r = \beta_b$.

To interpret this bias term more precisely, note that, even if $\beta_r$ and $\beta_b$ were known exactly, the target parameter $\beta_\lambda$ cannot be recovered without also knowing the covariance matrices $\Sigma_r$, $\Sigma_b$. In fact, $\beta_\lambda$ is a covariance-weighted average of $\beta_r$ and $\beta_b$; it coincides with both when the group-specific parameters are equal, and the discrepancy increases as the  parameters diverge. The bias term quantifies the additional error introduced by replacing the true covariance structure with empirical estimates, which explains why it vanishes in the known-covariance case.
\subsection{Upper and lower bounds for the bias and variance terms}
In this section, we investigate the optimality of the estimator $\widehat{\beta}_\lambda$ under additional assumptions. For our upper bounds, we assume the distribution of covariates is sub-Gaussian (\Cref{assumption:subgaussian}) and also satisfies the small-ball condition (\Cref{assumption:small_ball}). We further assume that, for any group $g \in \mathcal{G}$,
\begin{equation} \label{eqn:condition_cov}
\frac{1}{2} \rho_g^2 I_d \preceq \Sigma_g  \preceq \frac{3}{2} \rho_g^2 I_d.
\end{equation}
The constants $1/2$ and $3/2$ are chosen for simplicity; as we show in the appendix, our upper bounds extend to more general conditions on the eigenvalues of the covariance matrix $\Sigma_g$, but we adopt these specific values in the main text for clarity. {We also introduce the notation $\rho_\lambda^2 \coloneq \lambda \rho_r^2 + (1-\lambda) \rho_b^2$.}

For the lower bounds, we must specify the class of distributions $\mathcal{P}$ in \eqref{eqn:minimax_regression} over which the worst-case excess risk is taken. To match our upper bounds, we impose the same assumptions as above. In addition, we have to assume the boundedness of the group-specific predictors $\beta_r$ and $\beta_b$, since \Cref{prop:unknown_cov_upper_bound} (and later our lower bound) shows that the bias term grows with $\|\beta_r - \beta_b\|$ and would diverge if the groups' predictors were unbounded.

Accordingly, we consider the following subclass of $\mathcal{P}_{\text{linear}}(\sigma_r^2, \sigma_b^2)$, consisting of Gaussian covariates satisfying the assumptions above:
\begin{equation}
\begin{aligned}
\mathcal{P}_{\text{Gauss}}(\sigma_r^2, \sigma_b^2, \rho_r^2, \rho_b^2, B)\coloneq & \left \{ P \in \mathcal{P}_{\text{linear}}(\sigma_r^2, \sigma_b^2) ~\big|~ 
\forall g \in \mathcal{G}: \norm{\beta_g}_2 \leq B \text{ and } P(X|G=g) \text{ is} \right. \\
& \left. \text{ Gaussian with }~ \frac{1}{2} \rho_g^2 I_d \preceq \Sigma_g  \preceq \frac{3}{2} \rho_g^2 I_d
\right\}.    
\end{aligned}
\end{equation}
We treat the variance and bias terms separately. For the upper bound, we bound the quantities $\mathcal{V}_g(\lambda)$ and $\mathcal{B}_g(\lambda)$ from \Cref{prop:unknown_cov_upper_bound}. For the lower bound, we decompose the worst-case error into bias and variance components by considering two complementary scenarios. Specifically, for every estimator $\widehat{\beta}$, we lower bound its worst-case error
\begin{equation}
\sup_{P \in \mathcal{P}_{\text{Gauss}}(\sigma_r^2, \sigma_b^2, \rho_r^2, \rho_b^2, B)} \EE \left[ \norm{ \widehat{\beta} - \beta_\lambda }_{\Sigma_g}^2\right]    
\end{equation}
by the maximum of two restricted subproblems: (1) the case where the covariance matrices $\Sigma_r$ and $\Sigma_b$ are known but the group predictors $\beta_r$ and $\beta_b$ are unknown (corresponding to the variance term), and (2) the case where the group predictors are known but the covariance matrices are unknown (corresponding to the bias term). Formally, consider
\begin{equation}\label{eqn:assouad_max}
\max \left \{
\inf_{\widehat{\beta}}\sup_{\substack{P \in \mathcal{P}_{\text{Gauss}}(\sigma_r^2, \sigma_b^2, \rho_r^2, \rho_b^2, B) \\ \Sigma_r, \Sigma_b \text{ are known.}}} \EE \left[ \norm{ \widehat{\beta} - \beta_\lambda }_{\Sigma_g}^2\right] ,
\inf_{\widehat{\beta}}\sup_{\substack{P \in \mathcal{P}_{\text{Gauss}}(\sigma_r^2, \sigma_b^2, \rho_r^2, \rho_b^2, B) \\ \beta_r, \beta_b \text{ are known.}}} \EE \left[ \norm{ \widehat{\beta} - \beta_\lambda }_{\Sigma_g}^2\right]    
\right \},
\end{equation}
which can be further lower bounded by the average of these two terms. Our next two results establish upper and lower bounds for the variance and bias terms, respectively.
\begin{theorem} \label{theorem:variance_unknown_cov}
Suppose Assumptions \ref{assumption:invertible_cov}–\ref{assumption:subgaussian} hold, the covariance matrices satisfy \eqref{eqn:condition_cov}, and, for every $g \in \mathcal{G}$, $n_g \geq \max\{6d/\alpha_g, 48/\alpha_g, K_g^4 d\}$. Then, we have:
\begin{equation}
\mathcal{V}_g(\lambda) \lesssim 
\frac{\rho_g^2 d}{\left ({\lambda}\rho_r^2/{{C}_r'}  + (1-\lambda) \rho_b^2/{{C}_b'} \right)^2}
\left(   \frac{\lambda^2 \sigma_r^2 \rho_r^2}{n_r}  +  \frac{(1-\lambda)^2 \sigma_b^2 \rho_b^2}{n_b} \right), 
\end{equation}
where $\mathcal{V}_g(\lambda)$ is the variance term, as defined in \Cref{prop:unknown_cov_upper_bound}, and $C'_g$ is given in \Cref{cor:known_small_ball} and \Cref{remark:constants}. Moreover, assuming $n_g \geq \sigma_g^2d/(4 B^2 \rho_g^2)$ for $g \in \mathcal{G}$, we have
\begin{equation}
\inf_{\widehat{\beta}} \sup_{\substack{P \in \mathcal{P}_{\text{Gauss}}(\sigma_r^2,\sigma_b^2, \rho_r^2, \rho_b^2, B) \\ \Sigma_r, \Sigma_b \text{ are known.}}} \EE \left[ \norm{ \widehat{\beta} - \beta_\lambda }_{\Sigma_g}^2\right]
\gtrsim \frac{ \rho_g^2  d}{\rho_\lambda^4} \left(\frac{\lambda^2 \sigma_r^2 \rho_r^2}{n_r} + \frac{(1 - \lambda)^2 \sigma_b^2 \rho_b^2}{n_b} \right),
\end{equation}
where the infimum is taken over any estimator $\widehat{\beta}$ as a function of the datasets $(\mathcal{S}_r, \mathcal{S}_b)$.
\end{theorem}
See \Cref{proof:theorem:variance_unknown_cov} for the proof. The upper and lower bounds together show that, up to constant factors, the variance term indeed captures the error arising from not knowing the true predictors $\beta_r$ and $\beta_b$, even in the known-covariance case. Moreover, the estimator $\widehat{\beta}_\lambda$ is minimax-optimal, again up to constant factors.

It is worth noting that, while the lower bound is for the known-covariance case, the result in \Cref{theorem:minimax_optimal_known_cov} from the previous section is not applicable here. This is because we now operate under the additional assumption that the group predictors $\beta_r$ and $\beta_b$ are bounded, whereas in that earlier result, the lower bound was constructed by defining a prior over a family of distributions with unbounded parameters.

The proof of the lower bound applies Assouad’s method. Specifically, we construct a finite subclass of $\mathcal{P}_{\text{Gauss}}(\sigma_r^2, \sigma_b^2, \rho_r^2, \rho_b^2, B)$ of size $2^{2d}$, such that the concatenated vector $[\beta_r^\top, \beta_b^\top]^\top$ lies on the vertices of a $2d$-dimensional hypercube with a carefully chosen side length. The minimax risk is then bounded below by the worst-case error over this subclass. Assouad’s lemma, as recalled in the appendix, reduces this problem to one of testing multiple hypotheses, from which the stated bound follows.

We next derive upper and lower bounds for the bias term.
\begin{theorem} \label{theorem:bias_unknown_cov}
Suppose Assumptions \ref{assumption:invertible_cov}–\ref{assumption:subgaussian} hold, the covariance matrices satisfy \eqref{eqn:condition_cov}, and, for $g \in \mathcal{G}$, {$n_g \geq \alpha_g^{-1} \max\{6d, 48\}$}. Then, we have:
\begin{equation}
\mathcal{B}_g(\lambda) \lesssim 
\frac{ \lambda^2 ( 1 - \lambda)^2 ~ \rho_g^2 ~ \rho_r^4 ~ \rho_b^4 ~ d}{\left ({\lambda}\rho_r^2/{{C}_r'}  + (1-\lambda) \rho_b^2/{{C}_b'} \right)^2 ~ \rho_\lambda^4 } ~ \left(\frac{K_r^4}{n_r} +  \frac{K_b^4}{n_b} \right) \norm{\beta_r - \beta_b}^2, 
\end{equation}
where $\mathcal{B}_g(\lambda)$ is the bias term, as defined in \Cref{prop:unknown_cov_upper_bound}, and $C'_g$ is given in \Cref{cor:known_small_ball} and \Cref{remark:constants}. Moreover, if $n_g \geq 16d^2$ for $g \in \mathcal{G}$, we have
\begin{equation}
\inf_{\widehat{\beta}} \sup_{\substack{P \in \mathcal{P}_{\text{Gauss}}(\sigma_r^2, \sigma_b^2, \rho_r^2, \rho_b^2, B) \\ \beta_r, \beta_b \text{ are known.}}} \EE \left[ \norm{ \widehat{\beta} - \beta_\lambda }_{\Sigma_g}^2\right]
\gtrsim 
\frac{ \lambda^2 ( 1 - \lambda)^2 ~ \rho_g^2 ~ \rho_r^4 ~ \rho_b^4 ~ d}{\rho_\lambda^8} 
~ \left( \frac{1}{n_r} + \frac{1}{n_b} \right) \norm{\beta_r - \beta_b}^2,
\end{equation}
where the infimum is taken over any estimator $\widehat{\beta}$ as a function of the datasets $(\mathcal{S}_r, \mathcal{S}_b)$.
\end{theorem}
See \Cref{proof:theorem:bias_unknown_cov} for the proof. Recall that the constants $K_r$ and $K_b$, as defined in \Cref{assumption:subgaussian}, are invariant to scaling and hence are independent of $\rho_r$ and $\rho_b$. As a result, the upper and lower bounds above match up to constant factors, which again shows that the bias term truly captures the error arising from not knowing the covariance matrices, even when the group predictors $\beta_r$ and $\beta_b$ are known. Moreover, \Cref{theorem:bias_unknown_cov}, together with the result of \Cref{theorem:variance_unknown_cov}, highlights that the OLS estimator $\widehat{\beta}_\lambda$ achieves the minimax excess risk, as given in \eqref{eqn:minimax_regression} with $\mathcal{P} = \mathcal{P}_{\text{Gauss}}(\sigma_r^2, \sigma_b^2, \rho_r^2, \rho_b^2, B)$, up to constant factors.

The proof of the bias lower bound also uses Assouad’s method, but over a delicately constructed family of perturbed covariance matrices. To this end, define $\{u_i u_i^\top \}_{i=1}^d$, each of which is rank-one and positive semi-definite, where, for $\beta_r \neq \beta_b$, $u_i$ is defined as
\begin{align}
u_i &=
\begin{cases}
\frac{e_i + v}{\left\|e_i + v\right\|_2}, & \text{if } e_i^\top v \ge 0,\\
\frac{e_i - v}{\left\|e_i - v\right\|_2}, & \text{if } e_i^\top v < 0,
\end{cases}
\quad \text{ with } \quad v = \frac{\beta_r-\beta_b}{\|\beta_r-\beta_b\|}.
\end{align}
For each group $g$, we perturb the baseline covariance matrix $\rho_g^2 I_d$ in $2^d$ different ways by either adding or subtracting $h_g u_i u_i^\top$ for each $i$ and for some scalar $h_g$.
This yields $2^d$ distinct covariance matrices, which form the family of distributions over which we take the worst-case error and apply Assouad’s method.

The choice of the perturbations $\{u_i\}$ is key here. Since the bias term has the form $M (\beta_r - \beta_b)$ for some matrix $M$, its magnitude is maximized when $\beta_r - \beta_b$ is closely aligned with an eigenvector of $M$ corresponding to its largest eigenvalue. Hence we choose the perturbations so that this alignment holds, i.e., the construction of $u_i$ ensures that it has a positive and constant inner product with both $\beta_r - \beta_b$ and the canonical basis vector $e_i$. This alignment guarantees that we perturb the covariance matrix along all coordinate directions while maintaining proximity to $\beta_r - \beta_b$.

The next corollary combines the bounds on the variance and bias terms across both groups, in accordance with \Cref{proposition:excess_risk_to_matrix_norm}, to present the excess risk bound for the estimator $\widehat{\beta}_\lambda$.
\begin{corollary} \label{corollary:excess-risk-final-unknown-cov}
Suppose Assumptions \ref{assumption:invertible_cov}–\ref{assumption:subgaussian} hold, the covariance matrices satisfy \eqref{eqn:condition_cov}, and, for $g \in \mathcal{G}$, and $n_g \geq \max\{6d/\alpha_g, 48/\alpha_g, K_g^4 d\}$. Then, we have:
\begin{align*}
& \EE \left[ \mathcal{R}_\lambda (\widehat{\beta}_\lambda) - \mathcal{R}_\lambda(\beta_\lambda) \right]
\lesssim   d ~ \frac{ \rho_\lambda^2}{\left ({\lambda}\rho_r^2/{{C}_r'}  + (1-\lambda) \rho_b^2/{{C}_b'} \right)^2}
\left(   \frac{\lambda^2 \sigma_r^2 \rho_r^2}{n_r}  +  \frac{(1-\lambda)^2 \sigma_b^2 \rho_b^2}{n_b} \right) 
\\
& 
+ \frac{ \lambda^2 ( 1 - \lambda)^2 ~ \rho_r^4 ~ \rho_b^4 ~ d}{\left ({\lambda}\rho_r^2/{{C}_r'}  + (1-\lambda) \rho_b^2/{{C}_b'} \right)^2 ~ \rho_\lambda^2} ~ \left(\frac{K_r^4}{n_r} +  \frac{K_b^4}{n_b} \right) \norm{\beta_r - \beta_b}^2.
\end{align*}
\end{corollary}

\paragraph{Optimal allocation of the sampling budget:} We can pose the same question as in the previous section: given a fixed budget on $n_r + n_b$, what is the optimal way to choose these two parameters to minimize the excess risk?
\Cref{theorem:variance_unknown_cov} suggests that, for minimizing the variance term, the answer is similar to the case of known covariances. In fact, similar to \eqref{eqn:optimal_sampling_known_cov}, we should choose ${n_r}/{n_b} = (\lambda \sigma_r \rho_r)/((1-\lambda)  \sigma_b \rho_b)$ to minimize the variance term. In contrast, \Cref{theorem:bias_unknown_cov} shows that minimizing the bias term requires a balanced design, namely, $n_r = n_b$ (under the assumption of $K_r=K_b$).

Thus, the optimal sampling allocation is more nuanced and depends on whether the variance term or the bias term dominates the error, which in turn depends on the heterogeneity between the groups, i.e., $\|\beta_r - \beta_b\|$. In fact, the more the two groups differ, the larger the bias term becomes, and the more we would prefer $n_r$ and $n_b$ to be closer.

For a fixed tradeoff $\lambda$ and equal covariate tail parameters, this tendency toward balanced sampling is an informational response to heterogeneity, even when the planner's welfare weights are unequal. With a large disagreement between group optima, imprecise knowledge of either population can make the common compromise costly. The allocation comparison concerns the bounds, assumes equal observation costs, and is subject to their minimum sample-size requirements.
\subsection{Groupwise estimation errors}
We next ask: what is the impact of using the OLS estimator $\widehat{\beta}_\lambda$ instead of the population target $\beta_\lambda$ on each group’s risk? Recall from \Cref{proposition:excess_risk_to_matrix_norm} that the excess risk for group $g \in \mathcal{G}$ can be written as
\begin{equation} \label{eqn:per-group-error}
\mathcal{R}_g(\widehat{\beta}_\lambda) - \mathcal{R}_g(\beta_\lambda) = \|\widehat{\beta}_\lambda-\beta_\lambda\|_{\Sigma_g}^2 + 2(\widehat{\beta}_\lambda-\beta_\lambda)^\top \Sigma_g (\beta_\lambda-\beta_g).    
\end{equation}
In the known-covariance case (see \Cref{sec:knowncov}), the expectation of the second term on the right-hand side is zero, so our bounds on $\|\widehat{\beta}_\lambda-\beta_\lambda\|_{\Sigma_g}^2$ translate directly to the groupwise (expected) estimation error. However, that is not the case here, as $\widehat{\beta}_\lambda$ is not an unbiased estimate of $\beta_\lambda$. Therefore, we must also bound the second term.

So far, our results have provided bounds on the expectation of the first term on the right-hand side of \eqref{eqn:per-group-error}, showing that its dependence on the sample size takes the form $\mathcal{O}(1/n_r + 1/n_b)$. Our next result characterizes the second term in the expectation of the right-hand side of \eqref{eqn:per-group-error} and shows that it also admits a similar bound. 
\begin{proposition} \label{proposition:cross-term-unknown-cov}
Suppose Assumptions \ref{assumption:invertible_cov}–\ref{assumption:subgaussian} hold, the covariance matrices satisfy \eqref{eqn:condition_cov}, and, for $g \in \mathcal{G}$, {$n_g \geq \alpha_g^{-1} \max\{6d, 48\}$}. Then, we have:
\begin{align*}
& \left | \EE \left [(\widehat{\beta}_\lambda-\beta_\lambda)^\top \Sigma_r (\beta_\lambda-\beta_r) \right] \right | \lesssim \\
&
\frac{\lambda (1-\lambda)^2 d \rho_r^4 \rho_b^4~
~\|\beta_r - \beta_b\|^2
}{\rho_\lambda^6 \left ({\lambda}\rho_r^2/{{C}_r'}  + (1-\lambda) \rho_b^2/{{C}_b'} \right)}~
\left(
\frac{K_r^2}{\sqrt{n_r}} + \frac{K_b^2}{\sqrt{n_b}}
\right )
\left(
\lambda \frac{K_r^2 \rho_r^2}{\sqrt{n_r}} + (1-\lambda)  \frac{K_b^2 \rho_b^2}{\sqrt{n_b}}
\right ),
\end{align*}
where $C'_g$ is given in \Cref{cor:known_small_ball} and \Cref{remark:constants}.
\end{proposition}
The proof is provided in \Cref{proof:proposition:cross-term-unknown-cov}. A similar result can also be stated for group~$b$, with the only difference being that $\lambda (1-\lambda)^2$ in the bound is replaced by $\lambda^2 (1-\lambda)$.

\Cref{proposition:cross-term-unknown-cov}, together with the bounds from \Cref{theorem:variance_unknown_cov} and \Cref{theorem:bias_unknown_cov}, shows that, much like in the known-covariance setting, there is an inherent asymmetry in how the empirical estimator affects the risks of the two groups. Specifically, if we examine the groupwise estimation errors in \eqref{eqn:per-group-error}, we see that the first term on the right-hand side differs between group~$r$ and group~$b$ only through the substitution of $\rho_r^2$ with $\rho_b^2$. This mirrors exactly what was observed in the known-covariance setting. 

The second term, $(\widehat{\beta}_\lambda-\beta_\lambda)^\top \Sigma_g (\beta_\lambda-\beta_g)$, captures a new phenomenon arising from the bias term. First, note that the bound in \Cref{proposition:cross-term-unknown-cov} controls its absolute value, but the term itself may be positive or negative. Moreover, it is straightforward to verify that if we sum this term across the two groups with weights $\lambda$ and $1-\lambda$, the result is zero. Thus, its contribution to one group's risk is always offset by the other's, ensuring that the two groups experience it with opposite signs. Even if the absolute-value bound were symmetric, the term would still be a source of disparity between the two groups’ risks.

However, the bound is in fact not symmetric. As stated after the proposition, it changes from $\lambda (1-\lambda)^2$ for group~$r$ to $\lambda^2 (1-\lambda)$ for group~$b$. This bias effect is larger for group~$r$ when $\lambda$ is close to zero and for group~$b$ when $\lambda$ is close to one, meaning that, near those endpoints, the term is larger in absolute value for the group that is not prioritized.


These observations show why a small loss in the planner's weighted objective need not imply small changes in each group's risk. The signed term can increase one group's risk and decrease the other's, with these changes canceling in the weighted objective. Looking only at aggregate excess risk can therefore hide changes in who bears the prediction errors. Collecting the same number of observations from each group does not resolve this issue: differences in covariate distributions can still make estimation error affect the two groups differently. The planner should therefore evaluate a sampling design by considering both its aggregate excess risk and its effects on each group's risk.

\section{Uniform Frontier Bounds} \label{sec:uniform_bound}
So far, our analysis has provided finite-sample guarantees for a single fixed fairness–accuracy preference $\lambda$. In practice, however, a designer may wish to compare fitted models across many preferences $\lambda$ before selecting a data-dependent weight $\widehat{\lambda}$. In this section, we seek a \emph{uniform} guarantee that holds simultaneously over $\lambda \in [0, 1]$ rather than only at a prespecified weight $\lambda$. We obtain such a guarantee by controlling the groupwise estimation errors that are common components of the risk for each choice $\lambda \in [0, 1]$. 

Define the population and estimator-induced error pairs across groups
\begin{equation*}
    \mathbf{R}_\lambda \coloneq \left( \mathcal{R}_r(\beta_\lambda), \mathcal{R}_b(\beta_\lambda) \right), \qquad  \widehat{\mathbf{R}}_{\lambda} \coloneq \left( {\mathcal{R}}_r(\widehat{\beta}_\lambda),  {\mathcal{R}}_b(\widehat{\beta}_\lambda) \right),
\end{equation*}
and write their difference as  
\begin{equation*}
    \Delta_\lambda \coloneq \widehat{\mathbf{R}}_\lambda - \mathbf{R}_\lambda \in \RR^2, \quad \lambda \in [0, 1].
\end{equation*}

We will assume that the noise in the linear model given by~\Cref{def:linear_model} is sub-Gaussian. 

\begin{assumption}\label{assumption:sg_noise}
For each group $g\in \mathcal{G}$, the noise $\varepsilon_g$ is independent of $X$, has mean zero, and is $\sigma_{\varepsilon, g}$-sub-Gaussian. 
\end{assumption}

The following theorem establishes a high-probability uniform bound on the deviation of the empirical fairness--accuracy frontier from its population counterpart across all possible fairness--accuracy preferences $\lambda \in [0, 1]$ for the unknown-covariance estimator $\widehat{\beta}_\lambda$ defined in \eqref{eqn:optimal_estimator_unknown_cov}. 

\begin{theorem}\label{thrm:uniform_control}
    Suppose Assumptions~\ref{assumption:invertible_cov}-\ref{assumption:sg_noise} hold. Fix $\delta \in (0, 1)$ and suppose that, for each $g \in \mathcal{G}$, 
    \begin{equation}
        n_g \geq \max\left\{ \frac{6d}{\alpha_g}, 4\left( d + \log(8/\delta)\right)\right\}. 
    \end{equation}
    Then, there exists a constant $C_U>0$ such that, with probability at least $1 - \delta - \eta$, 
    \begin{equation}
        \sup_{\lambda \in [0, 1]} \norm{\Delta_{\lambda}}_2 \leq C_U \sum_{g \in \mathcal{G}}\left\{ \sqrt{\frac{d + \log(8/\delta)}{n_g}}  + \frac{d + \log(8/\delta)}{n_g}\right\}, 
    \end{equation}
    where 
    \begin{equation}
        \eta \leq 2^{-\alpha_r n_r /6} +  2^{-\alpha_b n_b /6}.
    \end{equation}
    The constant $C_U$ depends only on the covariance matrices, the sub-Gaussian and small-ball parameters, the noise parameters, and the heterogeneity $\norm{\beta_r - \beta_b}_2$, but it does not depend on $n_r, n_b, d$ or $\delta$.
\end{theorem}
See \Cref{proof:thrm:uniform_control} for a more detailed theorem statement and proof. \Cref{thrm:uniform_control} controls the population risks of the fitted linear models simultaneously over all $\lambda \in [0, 1]$. For fixed groupwise covariate distributions $P_X^g$ and parameters $\beta_g$, the bound is of the order $\sqrt{d + \log(8/\delta)}(n_r^{-1/2} + n_b^{-1/2})$, which implies that, as in the fixed $\lambda$ analysis, limited data for either group constrains the accuracy the planner can achieve uniformly over the frontier. 

Since the bound given by \Cref{thrm:uniform_control} holds \emph{simultaneously} for all $\lambda \in [0,1]$, a designer may select $\widehat{\lambda}$ adaptively after inspecting the empirical frontier while enjoying the same high-probability guarantee. This is particularly useful when the planner collects data before deciding how much weight to place on each group. The bound remains valid even if the planner chooses $\lambda$ after examining various fitted models. It also explains why estimating the entire frontier requires sufficient data from both groups: at $\lambda=1$, the fitted predictor relies only on the red group's data, while at $\lambda=0$, it relies only on the blue group's data.

The proof proceeds by controlling both groups' empirical covariance and noise score errors on a high probability event. The small-ball condition provides high probability control from below for the empirical covariance matrices. We then bound $\widehat{\beta}_\lambda - \beta_\lambda$ on the intersection of these events in terms of groupwise disagreement with the population target $\beta_\lambda - \beta_g$. We obtain the bound by controlling these groupwise disagreements with $\beta_\lambda$ in terms of the groupwise heterogeneity $\norm{\beta_r - \beta_b}$. 

\section{Experiments}\label{sec:experiments}

We conduct synthetic experiments to study how well the estimators $\tilde{\beta}_\lambda$ and $\widehat{\beta}_\lambda$ perform relative to the population quantity across the FA frontier, for the known and unknown covariance settings, respectively.

We construct a two-group Gaussian linear model in $d$ dimensions with spherical covariances $\Sigma_g = \rho_g^2 I_d$. For each group $g \in \mathcal{G}$, features $X_g \sim \NN(0, \Sigma_g)$, and outcomes $Y_g = X_g^\top \beta_g + \varepsilon_g$ with $\varepsilon_g \sim \NN(0, \sigma_g^2)$. The pair $(\beta_r, \beta_b)$ is constructed via a dense random perturbation scaled so that the heterogeneity $\norm{\beta_b - \beta_r}_2$ is fixed across experimental configurations.  We evaluate (i) the deterministic population frontier from $\beta_\lambda$ using the true $\Sigma_g$ and $\beta_g$, (ii) the empirical frontier in the known $\Sigma_g$ case from $\tilde{\beta}_\lambda$ given in \eqref{eqn:optimal_estimator_known_cov}, and (iii) the empirical frontier in the unknown $\Sigma_g$ case from $\widehat{\beta}_\lambda$ given in \eqref{eqn:optimal_estimator_unknown_cov}. 

\begin{figure}[t]
    \centering
    \begin{subfigure}[t]{0.49\textwidth}
        \centering
        \resizebox{\linewidth}{!}{\input{figs/simulations/lambda_gap_same_n}}
        \caption{}
        \label{fig:simulations_same_n}
    \end{subfigure}
    \hfill
    \begin{subfigure}[t]{0.49\textwidth}
        \centering
        \resizebox{\linewidth}{!}{\input{figs/simulations/lambda_gap_same_cov}}
        \caption{}
        \label{fig:simulations_same_cov}
    \end{subfigure}
    \caption{\textbf{Finite sample contraction of the FA frontier.} We plot the population frontier parameterized by $\lambda$, sweeping $\lambda$ on a uniform grid of $50$ points in $[0,1]$. For each $\lambda$, we run 100 Monte Carlo repetitions; in each iteration, we draw data $\mathcal{S}_g = \{(X_g^{(i)}, Y_g^{(i)})\}_{i = 1}^{n_g}$ from the group-$g$ Gaussian linear model with $\sigma_r^2 = \sigma_b^2 = 1$, build the estimators $\tilde{\beta}_\lambda, \widehat{\beta}_\lambda$, and calculate the corresponding risk pair $(\mathcal{R}_r, \mathcal{R}_b)$ for each estimator. We plot the mean empirical frontiers across repetitions and depict the empirical contraction for the choice of $\lambda = 0.5$ for two regimes: (a) \emph{Equal sample sizes:} $n_r = n_b = 30, d = 15, \rho_r = 1, \rho_b = 2.5$, (b) \emph{Equal covariance:} $n_r = 50, n_b = 30, d = 15, \rho_r = \rho_b = 1.75$.}
    \label{fig:simulations}
\end{figure}

\Cref{fig:simulations} isolates the effect of heterogeneous covariance and sample size.  \Cref{fig:simulations_same_n} illustrates the unequal covariance, equal sample size regime ($\rho_b > \rho_r$, $n_r = n_b$). Here, we observe that the segment connecting the population point $\left( \mathcal{R}_r(\beta_{1/2}), \mathcal{R}_b(\beta_{1/2})\right)$ to the corresponding empirical points $\left( \mathcal{R}_r(\tilde{\beta}_{1/2}), \mathcal{R}_b(\tilde{\beta}_{1/2})\right)$ and $\left( \mathcal{R}_r(\widehat{\beta}_{1/2}), \mathcal{R}_b(\widehat{\beta}_{1/2})\right)$ is not parallel to the $45^\circ$ line. Instead, we see that it is tilted towards group $b$, reflecting a larger inflation in $\mathcal{R}_b$ than $\mathcal{R}_r$, on average. This asymmetry is explained by the results of \Cref{cor:known_small_ball} and \Cref{proposition:cross-term-unknown-cov}, which show that asymmetry in the covariance across groups drives asymmetric contraction of the empirical frontier estimated in a finite sample. 

In contrast, in the unequal sample size, equal covariance regime ($n_r > n_b$, $\rho_r = \rho_b$) illustrated in \Cref{fig:simulations_same_cov}, the same segment is approximately parallel to the  $45^\circ$ line. In other words, with equal covariance, we see that the inflation in $\mathcal{R}_b$ and $\mathcal{R}_r$ is roughly symmetric for both the known and unknown covariance estimation regimes. Beyond $\lambda = 1/2$, we observe that the same pattern holds across the frontier. 

Moreover, for both figures, we observe that the two empirical FA frontiers for known and unknown covariances approach each other for $\lambda$ near zero or one, and coincide at the endpoints $\lambda \in \{0, 1\}$. The latter is expected, as the endpoints represent the groupwise cases in which the optimal estimator does not depend on the knowledge of the covariance. The former observation, that the difference between the two frontiers shrinks around the endpoints, directly reflects our result in \Cref{theorem:bias_unknown_cov}, which shows that the bias term of the unknown-covariance estimator includes a factor of $\lambda^2 (1 - \lambda)^2$, and therefore vanishes quadratically at the endpoints.

In \Cref{fig:lambda_effective}, we compute the \emph{effective} value of $\lambda$ that results from estimation error. This value identifies the point on the population frontier that is closest, with respect to Euclidean distance, to the risk pair achieved by the empirical estimator. These experiments capture a central implication of our theoretical results: an estimate constructed from a finite sample can correspond to a different FA tradeoff than the one intended by the planner's choice of $\lambda$.

In particular, the figure shows that the map from the chosen $\lambda$ to the effective $\lambda$ is not the identity. Instead, estimation errors compress the implemented tradeoff toward the interior of the frontier. For small values of $\lambda$, the effective value satisfies $\lambda_{\textrm{eff}} > \lambda$, while, for large values of $\lambda$, the opposite holds, though the resulting gap is not symmetric about the risk-equalizing point $\lambda = 1/2$.  The known- and unknown-covariance estimators exhibit qualitatively similar effective $\lambda$ behavior, but the unknown-covariance estimators show slightly larger deviations. This is consistent with the theory, which establishes that the unknown-covariance estimator incurs an additional bias error from estimating the groupwise covariance matrices relative to the known-covariance estimator. The narrow Monte Carlo  confidence bands indicate that this distortion is systematic.

Finally, in \Cref{fig:known_allocation} and \Cref{fig:unknown_allocation}, we vary the share of observations across groups to examine the optimal sampling allocations implied by the known- and unknown-covariance minimax rates that we established in \eqref{eqn:optimal_sampling_known_cov} and \Cref{theorem:variance_unknown_cov}, respectively. In each, we fix a total sampling budget $n = n_r + n_b$ and a tradeoff $\lambda$, and we vary the proportion of red-group observations $p_r = n_r/n$. The results show that the optimal allocation depends on the information regime. 

In the known-covariance setting, we observe that the empirical excess risk curve closely matches the theoretical rate. The empirical minimizer is realized at $p_r \approx 0.71$ while the theoretical rate is minimized at $p_r \approx 0.73$; in other words, this validates our finding in \Cref{sec:knowncov} that, when the covariance matrices are known, the optimal allocation is attained at $p_r = \frac{\lambda \rho_r}{\lambda \rho_r + (1-\lambda) \rho_b}$.  The unknown-covariance setting exhibits a different structure. As the groupwise heterogeneity $\norm{\beta_r - \beta_b}$ increases, the excess risk curve rises uniformly, and the empirical minimizer shifts from near $p_r = \lambda$ toward the equal sample share setting $p_r \approx 1/2$. This is consistent with the additional bias term in the unknown-covariance minimax rate which increases with the groupwise heterogeneity. 

Taken together, these sample allocation experiments underscore that the variance component of the excess risk favors allocating more observations to the group receiving larger weight in the planner's objective, while the bias term favors sampling enough across both groups to control the covariance estimation errors. 

\begin{figure}[t]
    \centering
    \begin{subfigure}[t]{0.32\textwidth}
        \centering
        \resizebox{\linewidth}{!}{\input{figs/simulations/lambda_effective}}
        \caption{Intended vs. effective $\lambda$}
        \label{fig:lambda_effective}
    \end{subfigure}
    \hfill
    \begin{subfigure}[t]{0.32\textwidth}
        \centering
        \resizebox{\linewidth}{!}{
\begin{tikzpicture}

\definecolor{darkgray176}{RGB}{176,176,176}
\definecolor{darkviolet}{RGB}{148,0,211}
\definecolor{lightgray204}{RGB}{204,204,204}

\begin{axis}[
legend cell align={left},
legend style={
  fill opacity=0.8,
  draw opacity=1,
  text opacity=1,
  at={(0.5,0.91)},
  anchor=north,
  draw=lightgray204
},
tick align=outside,
tick pos=left,
title style={align=center},
x grid style={darkgray176},
xlabel={Red sample share \(\displaystyle p_r = n_r/(n_r+n_b)\)},
xmajorgrids,
xmin=0.115, xmax=0.885,
xtick style={color=black},
y grid style={darkgray176},
ylabel={Excess risk},
ymajorgrids,
ymin=0.142338476653413, ymax=1.72043343035532,
ytick style={color=black}
]
\path [draw=darkviolet, thick]
(axis cs:0.15,1.5857499946832)
--(axis cs:0.15,1.64870184155069);

\path [draw=darkviolet, thick]
(axis cs:0.174137931034483,1.11729103962331)
--(axis cs:0.174137931034483,1.15628233323166);

\path [draw=darkviolet, thick]
(axis cs:0.198275862068966,0.85234890843462)
--(axis cs:0.198275862068966,0.879948226504573);

\path [draw=darkviolet, thick]
(axis cs:0.222413793103448,0.725148697951623)
--(axis cs:0.222413793103448,0.746700082949244);

\path [draw=darkviolet, thick]
(axis cs:0.246551724137931,0.620374796239525)
--(axis cs:0.246551724137931,0.638663742651617);

\path [draw=darkviolet, thick]
(axis cs:0.270689655172414,0.527765308652155)
--(axis cs:0.270689655172414,0.542431314670519);

\path [draw=darkviolet, thick]
(axis cs:0.294827586206897,0.480242363748532)
--(axis cs:0.294827586206897,0.493366886059542);

\path [draw=darkviolet, thick]
(axis cs:0.318965517241379,0.425970579221175)
--(axis cs:0.318965517241379,0.437622648572193);

\path [draw=darkviolet, thick]
(axis cs:0.343103448275862,0.396388053799915)
--(axis cs:0.343103448275862,0.406824520977159);

\path [draw=darkviolet, thick]
(axis cs:0.367241379310345,0.364985696234503)
--(axis cs:0.367241379310345,0.374819330452454);

\path [draw=darkviolet, thick]
(axis cs:0.391379310344828,0.335027098280312)
--(axis cs:0.391379310344828,0.343653781228765);

\path [draw=darkviolet, thick]
(axis cs:0.41551724137931,0.317691414807972)
--(axis cs:0.41551724137931,0.325852581710648);

\path [draw=darkviolet, thick]
(axis cs:0.439655172413793,0.299837126401339)
--(axis cs:0.439655172413793,0.307539001124469);

\path [draw=darkviolet, thick]
(axis cs:0.463793103448276,0.279211338408824)
--(axis cs:0.463793103448276,0.286382276597603);

\path [draw=darkviolet, thick]
(axis cs:0.487931034482759,0.267638742748342)
--(axis cs:0.487931034482759,0.274344496588838);

\path [draw=darkviolet, thick]
(axis cs:0.512068965517241,0.255205879449767)
--(axis cs:0.512068965517241,0.261435987892771);

\path [draw=darkviolet, thick]
(axis cs:0.536206896551724,0.246992092345604)
--(axis cs:0.536206896551724,0.253050589045422);

\path [draw=darkviolet, thick]
(axis cs:0.560344827586207,0.23542549485232)
--(axis cs:0.560344827586207,0.241227898757227);

\path [draw=darkviolet, thick]
(axis cs:0.58448275862069,0.227264517102374)
--(axis cs:0.58448275862069,0.232754977329758);

\path [draw=darkviolet, thick]
(axis cs:0.608620689655172,0.22782985982141)
--(axis cs:0.608620689655172,0.233439192740297);

\path [draw=darkviolet, thick]
(axis cs:0.632758620689655,0.220113081251149)
--(axis cs:0.632758620689655,0.225391819683037);

\path [draw=darkviolet, thick]
(axis cs:0.656896551724138,0.216867805614658)
--(axis cs:0.656896551724138,0.222057062811113);

\path [draw=darkviolet, thick]
(axis cs:0.681034482758621,0.215855976742896)
--(axis cs:0.681034482758621,0.221089088340275);

\path [draw=darkviolet, thick]
(axis cs:0.705172413793103,0.214070065458045)
--(axis cs:0.705172413793103,0.219183963982141);

\path [draw=darkviolet, thick]
(axis cs:0.729310344827586,0.219363680263038)
--(axis cs:0.729310344827586,0.224709320121128);

\path [draw=darkviolet, thick]
(axis cs:0.753448275862069,0.219931289005142)
--(axis cs:0.753448275862069,0.225236317067108);

\path [draw=darkviolet, thick]
(axis cs:0.777586206896552,0.232965438116321)
--(axis cs:0.777586206896552,0.238768157859289);

\path [draw=darkviolet, thick]
(axis cs:0.801724137931034,0.245779131892866)
--(axis cs:0.801724137931034,0.252057093184548);

\path [draw=darkviolet, thick]
(axis cs:0.825862068965517,0.277158085335885)
--(axis cs:0.825862068965517,0.284744917757142);

\path [draw=darkviolet, thick]
(axis cs:0.85,0.341196756352473)
--(axis cs:0.85,0.351844192891754);

\addplot [darkgray, solid, opacity = 0.5, very thick]
table {%
0.15 1.10405604603235
0.174137931034483 0.93880445579346
0.198275862068966 0.81788591814698
0.222413793103448 0.746607643153947
0.246551724137931 0.66977657260665
0.270689655172414 0.608194135135246
0.294827586206897 0.569528815353275
0.318965517241379 0.52572837717702
0.343103448275862 0.497569605772816
0.367241379310345 0.465055058257984
0.391379310344828 0.437250217777794
0.41551724137931 0.418962467767089
0.439655172413793 0.39746445959384
0.463793103448276 0.37879048282844
0.487931034482759 0.366392013650156
0.512068965517241 0.351747598039235
0.536206896551724 0.342042152404786
0.560344827586207 0.330649313369014
0.58448275862069 0.320895075727424
0.608620689655172 0.314596077119964
0.632758620689655 0.307511976633129
0.656896551724138 0.30192787042169
0.681034482758621 0.298753395148565
0.705172413793103 0.295964914772382
0.729310344827586 0.295064481402248
0.753448275862069 0.295701245452681
0.777586206896552 0.298889056485023
0.801724137931034 0.303409292215815
0.825862068965517 0.313262156274178
0.85 0.329417944488268
};
\addlegendentry{Theoretical rate}
\addplot [very thick, darkviolet, opacity=0.75, dotted]
table {%
0.705172413793103 0.142338476653413
0.705172413793103 1.72043343035532
};
\addlegendentry{Empirical min $p_r=0.71$}
\addplot [very thick, darkgray, opacity=0.75, dotted]
table {%
0.729310344827586 0.142338476653413
0.729310344827586 1.72043343035532
};
\addlegendentry{Theoretical min $p_r=0.73$}
\addplot [thick, darkviolet, mark=*, mark size=1, mark options={solid}]
table {%
0.15 1.61722591811695
0.174137931034483 1.13678668642748
0.198275862068966 0.866148567469596
0.222413793103448 0.735924390450433
0.246551724137931 0.629519269445571
0.270689655172414 0.535098311661337
0.294827586206897 0.486804624904037
0.318965517241379 0.431796613896684
0.343103448275862 0.401606287388537
0.367241379310345 0.369902513343478
0.391379310344828 0.339340439754539
0.41551724137931 0.32177199825931
0.439655172413793 0.303688063762904
0.463793103448276 0.282796807503213
0.487931034482759 0.27099161966859
0.512068965517241 0.258320933671269
0.536206896551724 0.250021340695513
0.560344827586207 0.238326696804774
0.58448275862069 0.230009747216066
0.608620689655172 0.230634526280853
0.632758620689655 0.222752450467093
0.656896551724138 0.219462434212886
0.681034482758621 0.218472532541585
0.705172413793103 0.216627014720093
0.729310344827586 0.222036500192083
0.753448275862069 0.222583803036125
0.777586206896552 0.235866797987805
0.801724137931034 0.248918112538707
0.825862068965517 0.280951501546513
0.85 0.346520474622113
};
\addlegendentry{Known-$\Sigma_g$}
\end{axis}

\end{tikzpicture}}
        \caption{Known covariance allocation}
        \label{fig:known_allocation}
    \end{subfigure}
    \hfill
    \begin{subfigure}[t]{0.32\textwidth}
        \centering
        \resizebox{\linewidth}{!}{
\begin{tikzpicture}

\definecolor{crimson2143940}{RGB}{214,39,40}
\definecolor{darkgray176}{RGB}{176,176,176}
\definecolor{darkorange25512714}{RGB}{255,127,14}
\definecolor{forestgreen4416044}{RGB}{44,160,44}
\definecolor{gray}{RGB}{128,128,128}
\definecolor{lightgray204}{RGB}{204,204,204}
\definecolor{steelblue31119180}{RGB}{31,119,180}

\begin{axis}[
legend cell align={left},
legend style={fill opacity=0.8, draw opacity=1, text opacity=1, draw=lightgray204},
tick align=outside,
tick pos=left,
title style={align=center},
x grid style={darkgray176},
xlabel={Red sample share \(\displaystyle p_r=n_r/(n_r+n_b)\)},
xmajorgrids,
xmin=0.115, xmax=0.885,
xtick style={color=black},
y grid style={darkgray176},
ylabel={Excess Risk},
ymajorgrids,
ymin=0.0672139144587588, ymax=2.84057776062095,
ytick style={color=black}
]
\addplot [very thick, black, opacity=0.65, dashed]
table {%
0.75 0.0672139144587588
0.75 2.84057776062095
};
\addlegendentry{$p_r = \lambda$}
\addplot [very thick, gray, opacity=0.75, dotted]
table {%
0.5 0.0672139144587588
0.5 2.84057776062095
};
\addlegendentry{$p_r=1/2$}
\addplot [line width=0.76pt, steelblue31119180, mark=*, mark size=1, mark options={solid}]
table {%
0.15 0.691773388471987
0.174137931034483 0.596199203513737
0.198275862068966 0.526423370629004
0.222413793103448 0.485419526125605
0.246551724137931 0.433665839898833
0.270689655172414 0.399275861207525
0.294827586206897 0.371974624503371
0.318965517241379 0.346836415641047
0.343103448275862 0.329649721220236
0.367241379310345 0.302585722455108
0.391379310344828 0.287784156034677
0.41551724137931 0.273883769931441
0.439655172413793 0.263671574158057
0.463793103448276 0.251391705099288
0.487931034482759 0.242358947880214
0.512068965517241 0.232910254469838
0.536206896551724 0.227857656057113
0.560344827586207 0.218738144678275
0.58448275862069 0.211146533075862
0.608620689655172 0.208198525707832
0.632758620689655 0.201550982976357
0.656896551724138 0.201204350537708
0.681034482758621 0.199150289929102
0.705172413793103 0.198554580498453
0.729310344827586 0.195590667799704
0.753448275862069 0.197553618552759
0.777586206896552 0.199841530548648
0.801724137931034 0.202077028267704
0.825862068965517 0.208752359534417
0.85 0.215183425112254
};
\addlegendentry{$\|\beta_r-\beta_b\| \approx 0.0$}
\addplot [line width=0.76pt, darkorange25512714, mark=*, mark size=1, mark options={solid}]
table {%
0.15 0.908885092723756
0.174137931034483 0.774127175205062
0.198275862068966 0.671648278495483
0.222413793103448 0.609183826211529
0.246551724137931 0.554065693056775
0.270689655172414 0.49512678980424
0.294827586206897 0.467478941183275
0.318965517241379 0.429428085592338
0.343103448275862 0.409906782082593
0.367241379310345 0.382648500089588
0.391379310344828 0.361704198003449
0.41551724137931 0.347890867748214
0.439655172413793 0.333721770748237
0.463793103448276 0.317378695569422
0.487931034482759 0.30835215157659
0.512068965517241 0.299000647652093
0.536206896551724 0.293046839535935
0.560344827586207 0.279707385171504
0.58448275862069 0.276118616654421
0.608620689655172 0.277511524682258
0.632758620689655 0.270145984133679
0.656896551724138 0.270845345033276
0.681034482758621 0.268819034643523
0.705172413793103 0.26805960806572
0.729310344827586 0.273365796257333
0.753448275862069 0.271302135103652
0.777586206896552 0.288259362537651
0.801724137931034 0.292892687624319
0.825862068965517 0.308305103332984
0.85 0.332986206979321
};
\addlegendentry{$\|\beta_r-\beta_b\| \approx 2.0$}
\addplot [line width=0.76pt, forestgreen4416044, mark=*, mark size=1, mark options={solid}]
table {%
0.15 1.57837844419715
0.174137931034483 1.29828431475218
0.198275862068966 1.11009512361619
0.222413793103448 1.00454484511558
0.246551724137931 0.890987794006069
0.270689655172414 0.810275250288148
0.294827586206897 0.745389661752412
0.318965517241379 0.692946677444593
0.343103448275862 0.660068953716569
0.367241379310345 0.619348762348631
0.391379310344828 0.580950358815121
0.41551724137931 0.563080312252575
0.439655172413793 0.534929507654582
0.463793103448276 0.516792994457551
0.487931034482759 0.50520229027219
0.512068965517241 0.491520588016024
0.536206896551724 0.482164249808599
0.560344827586207 0.47980447822645
0.58448275862069 0.471467828996964
0.608620689655172 0.472527379272266
0.632758620689655 0.468977959823418
0.656896551724138 0.471990766602389
0.681034482758621 0.475480027569241
0.705172413793103 0.483210187104963
0.729310344827586 0.498575390078484
0.753448275862069 0.514023082902402
0.777586206896552 0.540908313441157
0.801724137931034 0.565615117825845
0.825862068965517 0.604681684944583
0.85 0.657761299475304
};
\addlegendentry{$\|\beta_r-\beta_b\| \approx 4.0$}
\addplot [line width=0.76pt, crimson2143940, mark=*, mark size=1, mark options={solid}]
table {%
0.15 2.67603142184234
0.174137931034483 2.18278288872614
0.198275862068966 1.8247100071683
0.222413793103448 1.63510882346309
0.246551724137931 1.44950784494988
0.270689655172414 1.32025925615034
0.294827586206897 1.2386992603644
0.318965517241379 1.1202827538933
0.343103448275862 1.08315874430126
0.367241379310345 0.997443522315987
0.391379310344828 0.947681337734143
0.41551724137931 0.910486892604548
0.439655172413793 0.880635491076827
0.463793103448276 0.853221239393714
0.487931034482759 0.829804972354571
0.512068965517241 0.815868353883731
0.536206896551724 0.801397465850862
0.560344827586207 0.795419834004307
0.58448275862069 0.790794423273568
0.608620689655172 0.798079364093558
0.632758620689655 0.805592757442546
0.656896551724138 0.818047810266875
0.681034482758621 0.825005034021799
0.705172413793103 0.849844639119605
0.729310344827586 0.861598540009204
0.753448275862069 0.904680850588289
0.777586206896552 0.947267718776609
0.801724137931034 1.01133698406393
0.825862068965517 1.09988281877054
0.85 1.19967120121119
};
\addlegendentry{$\|\beta_r-\beta_b\| \approx 6.0$}
\addplot [only marks, steelblue31119180, mark=star, mark size=3.2, mark options={solid, fill=steelblue31119180, draw=steelblue31119180}, forget plot]
table {%
0.729310344827586 0.195590667799704
};
\addplot [only marks, darkorange25512714, mark=star, mark size=3.2, mark options={solid, fill=darkorange25512714, draw=darkorange25512714}, forget plot]
table {%
0.705172413793103 0.26805960806572
};
\addplot [only marks, forestgreen4416044, mark=star, mark size=3.2, mark options={solid, fill=forestgreen4416044, draw=forestgreen4416044}, forget plot]
table {%
0.632758620689655 0.468977959823418
};
\addplot [only marks, crimson2143940, mark=star, mark size=3.2, mark options={solid, fill=crimson2143940, draw=crimson2143940}, forget plot]
table {%
0.58448275862069 0.790794423273568
};
\end{axis}

\end{tikzpicture}}
        \caption{Unknown covariance allocation}
        \label{fig:unknown_allocation}
    \end{subfigure}
    \caption{\textbf{Effective frontier position and groupwise sample allocation.} Panel (a) plots the effective $\lambda$ induced by the known- and unknown-covariance estimators. The gray solid line denotes the identity map $\lambda_{\textrm{eff}} = \lambda$. Finite sample estimation errors compress the implemented tradeoff toward the interior of the population frontier. Panel (b) studies groupwise sample allocation in the known-covariance case. The horizontal axis is the red-group sample share $p_r = n_r/(n_r + n_b)$ and the vertical axis is the $\lambda$-weighted excess risk. The empirical risk curve closely aligns with the theoretical rate. Panel (c) studies the unknown-covariance setting. As groupwise heterogeneity $\norm{\beta_r - \beta_b}$ increases, excess risk rises and the empirical minimizer shifts toward a more balanced allocation. This reflects the impact of the additional bias term which results from the estimation of the groupwise covariance matrices. }
\end{figure}
\section{Conclusion}\label{sec:conclusion}
Our analysis demonstrates that operationalizing the FA frontier with finite data poses unique challenges. Notably, departing from the population ideal introduces asymmetries across groups. These asymmetries grow with group heterogeneity, and such settings are precisely those in which fairness considerations are often most pressing. 
More specifically, we show that, when the covariance structure is known, the optimal sampling allocation prioritizes groups with greater variance or higher weight in the objective. When the covariance structure is unknown, however, a new challenge emerges in the form of a bias term, which motivates balancing sample sizes across groups. In both settings, failing to account for these asymmetries can yield solutions that deviate  meaningfully from the intended population-level tradeoff.

The central welfare implication is that choosing a desired tradeoff and acquiring the information to implement it are distinct policy decisions. A planner can hold social priorities fixed and still improve implementation by changing the sampling design. Evaluating the resulting policy also requires attention to both total excess loss and its incidence across groups.

The extension to generalized linear models (GLMs) in \Cref{sec:glm} establishes that these distortion mechanisms also arise in nonlinear models, including logistic regression. To leading order, the planner's welfare loss admits the same bias--variance decomposition and the groupwise risks include a redistribution term that cancels in the weighted objective. This generalization suggests that the systematic distortion between the population and statistical FA frontiers is a broader implication of implementing fairness--accuracy tradeoffs from data.

These results open several avenues for further work. A first direction is to extend our finite-sample analysis and sharpen the analysis for GLMs and other broader model classes or for non-parametric models. Such extensions would provide insight into the dynamics of statistical distortions in implementing tradeoffs under groupwise heterogeneity from data. A second direction is to broaden the group structure, for instance by considering multiple, overlapping, partially observed, or inferred group labels. Finally, studying the statistical FA frontier in dynamic settings, for instance where models themselves affect future populations, reporting behavior, and data availability, remains an important open problem. Overall, our results take a step toward understanding how data limitations and statistical uncertainty shape our ability to effectively balance fairness and accuracy.

\section*{Acknowledgments}
We are grateful to Amanda Coston, Federico Echenique, Nika Haghtalab, Samory Kpotufe, Annie Liang, Yiqi Liu, Francesca Molinari, and Ryan Tibshirani for helpful comments and discussion. Annie Ulichney's work is supported by the National Science Foundation Graduate Research Fellowship Program under Grant No. DGE 2146752 and the Cooperative AI Foundation. We also acknowledge funding from the European Union (ERC-2022-SYG-OCEAN-101071601).
Views and opinions expressed are however those of the author(s) only and do not necessarily reflect those of the National Science Foundation, the European Union or the European Research Council
Executive Agency. 

\bibliography{references}
\newpage
\appendix
\addtocontents{toc}{\protect\setcounter{tocdepth}{1}}
\section{Proofs}\label{appA}
\subsection{Auxiliary results}
\begin{lemma}\label{lemma:unknown_covariance_exact_risk}
Under Definition \ref{def:linear_model}, suppose that Assumptions \ref{assumption:invertible_cov}, \ref{assumption:invertible_emp_cov} hold. For each $g \in \mathcal{G}$ define $\Tilde{\Sigma}_g \coloneq \Sigma_g^{-1/2}\widehat{\Sigma}_\lambda \Sigma_g^{-1/2}$ and for each $h \in \mathcal{G}, i\in[n_h]$ $\Tilde{X}_{i, g}^{(h)} \coloneq \Sigma_g^{-1/2}X_i^{(h)}$.  Then, for each group $g \in \mathcal{G}$, the groupwise error decomposes as
    \begin{equation*}
    \begin{aligned}
        \EE \left[ \norm{ \widehat{\beta}_{\lambda} - \beta_\lambda }_{\Sigma_g}^2\right] &=  \frac{\lambda^2 \sigma_r^2}{n_r^2} \EE \left[ \sum_{i =1}^{n_r}  \norm{\Tilde{X}_{i, g}^{(r)}}_{\Tilde{\Sigma}_g^{-2}}^2 \right]  +  \frac{(1-\lambda)^2 \sigma_b^2}{n_b^2} \EE \left[ \sum_{i =1}^{n_b} \norm{\Tilde{X}_{i, g}^{(b)}}_{\Tilde{\Sigma}_g^{-2}}^2 \right] \\
        & + \quad  \EE \left[ \norm{\Sigma_g^{-1/2}\left( \lambda \widehat{\Sigma}_r (\beta_r - \beta_\lambda) + ( 1 - \lambda) \widehat{\Sigma}_b (\beta_b - \beta_\lambda) \right)}_{\Tilde{\Sigma}_g^{-2}}^2 \right]. 
    \end{aligned}
    \end{equation*}
\end{lemma}
\begin{proof}[Proof of \Cref{lemma:unknown_covariance_exact_risk}]
    Substitute the decomposition \eqref{eqn:known_covariance_upper_bound_nu_g_decomposition} into $\widehat{\nu}_\lambda$ to obtain $\widehat{\nu}_\lambda  = \lambda \widehat{\Sigma}_r \beta_r  + (1 - \lambda) \widehat{\Sigma}_b \beta_b + \lambda \widehat{Z}_r + (1 - \lambda) \widehat{Z}_b$. Define 
\begin{equation*}
    B \coloneq \lambda \widehat{\Sigma}_r (\beta_r - \beta_\lambda) + ( 1 - \lambda) \widehat{\Sigma}_b (\beta_b - \beta_\lambda), \quad \widehat{Z}_\lambda \coloneq \lambda \widehat{Z}_r + (1-\lambda)\widehat{Z}_b. 
\end{equation*}
We can then express $\widehat{\beta}_\lambda - \beta_\lambda = \widehat{\Sigma}_\lambda^{-1}(\widehat{Z}_\lambda +  B)$. Hence, defining the symmetric matrix $W_g \coloneq \widehat{\Sigma}_{\lambda}^{-1}\Sigma_g \widehat{\Sigma}_{\lambda}^{-1}$, we obtain
\begin{equation*}
    \begin{aligned}
        \EE \left[ \norm{ \widehat{\beta}_{\lambda} - \beta_\lambda }_{\Sigma_g}^2\right] = \EE \left[ \norm{   \widehat{Z}_\lambda  + B }_{W_g}^2  \right] = \EE \left[ \norm{\widehat{Z}_\lambda }_{W_g}^2 \right]   + \EE \left[ \norm{B}_{W_g}^2 \right]   + 2 \EE \left[ 
        \langle \widehat{Z}_\lambda, B \rangle_{W_g} \right].
    \end{aligned}
\end{equation*}
Using that the noise is mean-zero conditional on the design, for each group $g$, $\EE\left[ \widehat{Z}_g \mid \{X_i^{(g)}\}_{i=1}^{n_g}\right]=0$, hence
\begin{equation*}
    \EE\left[ \langle \widehat{Z}_\lambda, B \rangle_{W_g} \right] = \EE\left[ \EE\left[ \widehat{Z}_\lambda \mid \{X_i^{(r)}\}_{i=1}^{n_r}, \{X_i^{(b)}\}_{i=1}^{n_b}\right]^\top W_g B \right] = 0.
\end{equation*}
Using the same observation to cancel cross terms, 
\begin{equation*}
    \EE \left[ \norm{\widehat{Z}_\lambda }_{W_g}^2 \right] = \lambda^2 \EE \left[ \norm{\widehat{Z}_r }_{W_g}^2 \right] + (1-\lambda)^2 \EE \left[ \norm{\widehat{Z}_b }_{W_g}^2 \right].
\end{equation*}
For each $h \in \mathcal{G}$, since the noise is independent and mean-zero conditional on the design,
\begin{equation*}
    \EE \left[ \norm{\widehat{Z}_h }_{W_g}^2 \right] = \EE \left[ \frac{1}{n_h^2} \sum_{i=1}^{n_h} \sum_{j=1}^{n_h} \varepsilon_i^{(h)} \varepsilon_j^{(h)} {X_j^{(h)}}^\top W_g {X_i^{(h)}} \right] = \frac{\sigma_h^2}{n_h^2}\EE\left[ \sum_{i=1}^{n_h} \norm{{X_i^{(h)}}}_{W_g}^2 \right].
\end{equation*}
Finally, recalling $\Tilde{\Sigma}_g = \Sigma_g^{-1/2}\widehat{\Sigma}_\lambda\Sigma_g^{-1/2}$, note that, for any vector $x$,
\begin{equation*}
    \norm{x}_{W_g}^2 = x^\top \Sigma_g^{-1/2}\Tilde{\Sigma}_g^{-2}\Sigma_g^{-1/2}x = \norm{\Sigma_g^{-1/2}x}_{\Tilde{\Sigma}_g^{-2}}^2.
\end{equation*}
Applying this identity to each $X_i^{(h)}$ and combining terms yields the claim.
\end{proof}
\begin{lemma}[\cite{mourtada2022exact} Corollary 4 generalized to spherical covariance]\label{lemma:mourtada_small_ball_min_eval_qth_moment}
    Suppose that $X^{(g)}$ satisfies \Cref{assumption:small_ball} with parameters $C_g$ and $\alpha_g$ and that $\rho_g^2 I_d \preceq \Sigma_g$. Define $C'_g \coloneq 3 C_g^4 \exp(1 + 9/\alpha_g)$. Then, for every $q \geq 1$, if $n_g \geq \alpha_g^{-1} \max\{ 6d , 12q\}$, 
    \begin{equation}\label{eq:min_eval_neg_4}
        \left(\EE\left[\lambda_{\min}(\widehat{\Sigma}_g)^{-q} \right]\right)^{1/q}\leq 2^{1/q} {{C}_g'} \rho_g^{-2}.
    \end{equation}
    
\end{lemma}
\begin{proof}[Proof of \Cref{lemma:mourtada_small_ball_min_eval_qth_moment}]
    Let $\tilde{X}_i^{(g)}\coloneq \Sigma_{g}^{-1/2} X_i^{(g)}$ and $\bar{\Sigma}_g \coloneq \Sigma_g^{-1/2}\widehat{\Sigma}_g \Sigma_g^{-1/2}$. The whitened covariates are isotropic and satisfy the same small-ball condition. Since  $\rho_g^2 I_d \preceq \Sigma_g$, we have $\lambda_{\min}(\widehat{\Sigma}_g) \geq \rho_g^2 \lambda_{\min}(\bar{\Sigma}_g)$.  Therefore, by \citet[Corollary 4]{mourtada2022exact},
    \begin{equation*}
        \EE\left[\lambda_{\min}(\widehat{\Sigma}_g)^{-q} \right] \leq \rho_g^{-2q} \EE\left[\lambda_{\min}(\bar{\Sigma}_g)^{-q} \right] \leq 2 C_g'^q \rho_g^{-2q}. 
    \end{equation*} 
\end{proof}
\begin{lemma}\label{lemma:sub_gaussian_covariance_bound}
Suppose Assumption \ref{assumption:subgaussian} holds and $n_g \geq d$. Then, for every $2 \leq q \leq d$, there exists a universal constant $C > 0$ such that
    \begin{equation*}
         \left( \EE\left[ \norm{\widehat{\Sigma}_g - \Sigma_g}^q \right]\right)^{1/q} \leq C K_g^2 \norm{\Sigma_g} \left( \sqrt{\frac{d}{n_g}}  + \frac{d}{n_g} \right). 
    \end{equation*}
\end{lemma}
\begin{proof}[Proof of \Cref{lemma:sub_gaussian_covariance_bound}]
    Let $Z \coloneq \norm{\widehat{\Sigma}_g - \Sigma_g}$ and $B_g \coloneq K_g^2 \norm{\Sigma_g}$. By standard sub-Gaussian matrix 
    concentration \citep[see, e.g.,][Remark 4.7.3]{vershynin2018high}, for all $u \geq 0$, there exists an absolute constant $C_0 > 0$ such that: 
    \begin{equation*}
        \PP\left( Z \geq C_0 B_g \left( \sqrt{\frac{d + u}{n_g}} + \frac{d + u}{n_g}\right)  \right) \leq 2 e^{-u}.
    \end{equation*}
    To invert this tail bound, we seek $f(s)$ such that $\PP(Z \geq s) \leq 2 \exp(-f(s))$. In other words, using $2 \max\{a, b\} \geq a + b$, we want 
    \begin{equation}\label{eq:condition}
        s \leq C B_g \left( \sqrt{\frac{d + u}{n_g}} + \frac{d + u}{n_g}\right) \leq 2 C B_g  \max\left\{ \sqrt{\frac{d + u}{n_g}}, \frac{d + u}{n_g} \right\} 
    \end{equation}
    We split this into two cases:
    \begin{enumerate}
        \item If $\sqrt{\frac{d + u}{n_g}}\geq \frac{d + u}{n_g},$ \eqref{eq:condition} is satisfied if $s \leq 2 C B_g  \sqrt{\frac{d + u}{n_g}} $ which occurs when $u \geq \frac{s^2 n_g}{4 C^2 B_g^2 } - d$.
        \item If $\sqrt{\frac{d + u}{n_g}}< \frac{d + u}{n_g},$ \eqref{eq:condition} is satisfied if $s \leq 2 C B_g {\frac{d + u}{n_g}}$ which occurs when $u \geq \frac{s n_g}{2 C B_g } - d$.
    \end{enumerate}
    
    Combining, for a constant $C'>0$, we bound
    \begin{equation*}
        \PP\left( Z \geq s \right) \leq 2 \exp\left( d - C' n_g \min \left\{\frac{s^2 }{B_g^2}, \frac{s }{ B_g} \right\}\right).
    \end{equation*}
    
    Define $s_1 \coloneq B_g \sqrt{\frac{d}{C' n_g}}$ and $s_2 \coloneq B_g \frac{d}{C' n_g}$, the critical values of $s$ that make the argument of the exponent positive for cases (1) and (2) defined above, respectively. 
  
    By the tail integration formula,
    \begin{equation*}
        \EE\left[ Z^q \right] = q \int_{0}^\infty \PP\left(Z \geq s\right) s^{q-1} ds
    \end{equation*}
    
    Define the critical point $s_d \coloneq \max \{s_1, s_2\}$ and split the integral 
    \begin{equation*}
        \begin{aligned}
            q \int_{0}^\infty \PP\left(Z \geq s\right) s^{q-1} ds &= q \int_{0}^{s_d} \PP\left(Z \geq s\right) s^{q-1} ds + q \int_{s_d}^\infty \PP\left(Z \geq s\right) s^{q-1} ds
        \end{aligned}
    \end{equation*}
    For the first term, use the trivial upper bound $\PP\left(Z \geq s\right)\leq 1$ to obtain
    \begin{equation*}
        q \int_{0}^{s_d} \PP\left(Z \geq s\right) s^{q-1} ds\leq s_d^q\leq C^q B_g^q \left( \left(\frac{d}{n_g}\right)^{q/2} + \left(\frac{d}{n_g}\right)^{q}\right). 
    \end{equation*}
    For the second term, we consider two cases of $s_d=s_1$ and $s_d=s_2$. First  take $s_d=s_1$: 
    \begin{equation*}
        q \int_{s_1}^\infty \PP\left(Z \geq s\right) s^{q-1} ds \leq  2 e^d q
        \int_{s_1}^\infty s^{q-1} \exp\left(-\frac{C' n_g s^2}{B_g^2}  \right) ds = 2 e^d q \left(\frac{B_g}{\sqrt{C' n_g}} \right)^q \int_{\sqrt{d}}^\infty u^{q-1} \exp(-u^2) du. 
    \end{equation*}
    Then, for $d \geq q$ and $q \geq 2$, by a change of variables and using the incomplete gamma bound result  \citet[Proposition 2.7]{pinelis2020exact}, we obtain
    \begin{equation*}
        2 e^d q \left(\frac{B_g}{\sqrt{C' n_g}} \right)^q  \int_{\sqrt{d}}^\infty u^{q-1} \exp(-u^2) du \leq 2 e^d q \left(\frac{B_g}{\sqrt{C' n_g}} \right)^q d^{q/2-1}e^{-d} = 2 \left(\frac{B_g}{\sqrt{C' n_g}} \right)^q d^{q/2}. 
    \end{equation*}
    The second case $s_d=s_2$ is more direct: 
    \begin{equation*}
        q \int_{s_2}^\infty \PP\left(Z \geq s\right) s^{q-1} ds \leq  2 e^d q \int_{s_2}^\infty s^{q-1} \exp\left(-\frac{C' n_g s}{B_g}  \right) ds  = 2 e^d q \left(\frac{B_g}{{C' n_g}} \right)^q \Gamma(q, d).
    \end{equation*}
    Another application of the incomplete gamma bound \citet[Proposition 2.7]{pinelis2020exact} yields 
    \begin{equation*}
        2 e^d q \left(\frac{B_g}{{C' n_g}} \right)^q \Gamma(q, d) \leq 4 q \left(\frac{B_g}{{C' n_g}} \right)^q d^{q-1}\leq 4  \left(\frac{B_g}{{C' n_g}} \right)^q d^{q}. 
    \end{equation*}
    Combining, for an absolute constant $C>0$ we bound 
    \begin{equation*}
        \left( \EE\left[ Z^q \right]\right)^{1/q} \leq C K_g^2 \norm{\Sigma_g} \left( \sqrt{\frac{d}{n_g}}  + \frac{d}{n_g} \right)
    \end{equation*}
    as desired.
\end{proof}
\begin{lemma}\label{lemma:subgaussian_sample_covariance_squared_trace}
    Under \Cref{assumption:subgaussian}, there exists a constant $C > 0$ such that 
    \begin{equation*}
        \EE\left[ \Tr{ \widehat{\Sigma}_g^{2}} \right] \leq \Tr{\Sigma_g^2} + \frac{C K_g^4}{n_g}\Tr{\Sigma_g}^2.
    \end{equation*}
\end{lemma}
\begin{proof}[Proof of \Cref{lemma:subgaussian_sample_covariance_squared_trace}]
    For notational convenience, we write $X_i$ in place of $X_i^{(g)} \sim X | G = g$ throughout the proof. Applying linearity of the trace and expectation and expanding the empirical covariance matrix product, we have:
\begin{equation*}
    \EE\left[ \Tr{ \widehat{\Sigma}_g^{2}} \right] = \frac{1}{n_g^2} \sum_{i = 1}^{n_g} \Tr{\EE[X_i X_i^\top X_i X_i^\top]} + \frac{1}{n_g^2} \sum_{i \neq j} \Tr{\EE[X_i X_i^\top X_j X_j^\top]}
\end{equation*}
For the cross-terms, by independence, for $i \neq j$,
\begin{equation*}
    \EE\left[X_i{X_i}^\top X_j{X_j}^\top\right] = \EE\left[X_i{X_i}^\top\right] \EE \left[ X_j{X_j}^\top\right] = \Sigma_g^2.
\end{equation*}
Thus, 
\begin{equation*}
    \sum_{i \neq j} \Tr{\EE[X_i X_i^\top X_j X_j^\top]} = n_g (n_g - 1) \Tr{\Sigma_g^2}.
\end{equation*}
 For the diagonal terms, interchanging the trace and expectation and applying invariance under cyclic permutations of the trace, we have $$\Tr{\EE\left[X_i {X_i}^\top X_i
{X_i}^\top\right]} = \EE\left[ \norm{X_i}_2^4 \right].$$  By our sub-Gaussianity assumption (Assumption \ref{assumption:subgaussian}) and standard moment bounds for sub-Gaussian vectors \citep[see, e.g.,][]{vershynin2018high}, we have, for some absolute constant $C > 0$:
\begin{equation*}
    \EE\left[ \norm{X_1}^4 \right] \leq C K_g^4 \Tr{\Sigma_g}^2.
\end{equation*}
Putting these results together yields the claim.
\end{proof}
\begin{lemma}[Assouad's Lemma ]\label{lemma:assouad}
    Fix an integer $m \geq 1$ and define the $m$-dimensional hypercube $\Xi^m \coloneq \{-1, 1\}^m$.  Let $\mathcal{P}= \left\{ \PP_{\xi}: \xi \in \Xi^m \right\}$ be a family of probability measures such that for each $\xi \in \Xi^m$ there is an associated parameter $\theta(\PP_\xi)$ that is the target of interest. Also suppose that $|\mathcal{P}| = 2^m$. Take the loss function to be the squared $\ell_2$-distance. For $j \in [m]$, let $\xi^{(j)}$ denote the vector obtained by flipping the $j$-th coordinate of vector $\xi$, i.e., $\xi_{-j} = \xi^{(j)}_{-j}$, $\xi_{j} = -\xi^{(j)}_{j}$. Define the per-coordinate separation as
    \begin{equation*}
        a \coloneq \inf_{j \in [m], \xi \in \Xi^m }  \left\langle e_j, \theta(\PP_\xi) - \theta\left(\PP_{\xi^{(j)}}\right) \right\rangle^2.
    \end{equation*}
    Then, we have the bound
    \begin{equation*}
        \inf_{\widehat{\theta}}\sup_{\PP_{\xi} \in \mathcal{P}} \EE_{\PP_{\xi}} \left[ \norm{\widehat{\theta} - \theta(\PP_{\xi}) }_2^2 \right] \geq \frac{a m}{4} \cdot \left( 1 - \sup_{\substack{\xi, \xi' \in \Xi^m \\ d_h(\xi, \xi')=1}} \tvdist{\PP_\xi, \PP_{\xi'}}\right). 
    \end{equation*}
\end{lemma}
\subsection{Proofs from
  \texorpdfstring{\Cref{sec:linearmodel}}{Section \ref{sec:linearmodel}}}
\begin{lemma}\label{lemma:feasible_set_regression_convex}
  Suppose $d \geq 2$ and $\Sigma_r, \Sigma_b \succ 0$. Then, the set of feasible risk pairs for regression
  \begin{equation}\label{eqn:feasible_set_regression}
      \mathcal{E}(\mathcal{H}) \coloneq \left\{\left( \norm{\beta - \beta_r}_{\Sigma_r}^2 + \sigma_r^2,  \norm{\beta - \beta_b}_{\Sigma_b}^2 + \sigma_b^2\right): \beta \in \RR^d \right\}
  \end{equation}
  is closed and convex.  
\end{lemma}
\begin{proof}[Proof of \Cref{lemma:feasible_set_regression_convex}]
    Translation by $(\sigma_r^2, \sigma_b^2)$ preserves convexity, so it suffices to consider 
\begin{equation*}
    \mathcal{E}_0 \coloneq \left\{\left( \norm{\beta - \beta_r}_{\Sigma_r}^2,  \norm{\beta - \beta_b}_{\Sigma_b}^2 \right): \beta \in \RR^d \right\}.
\end{equation*}
Perform a change of variables $u = \Sigma_r^{1/2}(\beta - \beta_r)$ and define $V \coloneq \Sigma_r^{-1/2}\Sigma_b\Sigma_r^{-1/2} \succ 0$, $a \coloneq \Sigma_r^{-1/2}\Sigma_b(\beta_b-\beta_r)$, $c \coloneq (\beta_b - \beta_r)^\top \Sigma_b (\beta_b - \beta_r)$ to obtain $\mathcal{E}_0 = \{(u^\top u, u^\top V u - 2a^\top u + c):u \in \RR^d\}$. Dropping the constant $c$ in the second coordinate again preserves convexity.  By \citet[Theorem 2.2]{polyak1998convexity} a sufficient condition for closedness and convexity of the joint range of two nonhomogeneous quadratic functions is the existence of a definite linear combination, i.e., that if $d \geq 2$, there exist scalars $\mu_1, \mu_2 \in \RR$ such that $\mu_1 I_d + \mu_2V\succ 0$. Taking $(\mu_1, \mu_2) = (1,0)$ yields the claim that $\mathcal{E}(\mathcal{H})$ is closed and convex.
\end{proof}
\begin{lemma}\label{lemma:frontier_regression_convex}
    Suppose $d \geq 2$ and  $\Sigma_r, \Sigma_b \succ 0$. Then, for every $\lambda \in [0, 1]$, the population minimizer $\beta_\lambda$ defined in \eqref{eqn:beta_lambda} is unique. If additionally the model is group-balanced, the risk pairs generated by $\{\beta_\lambda: \lambda \in [0, 1]\}$ constitute the entire FA frontier $\mathcal{F}(\mathcal{H})$. If, moreover, $\beta_r \neq \beta_b$, then the FA frontier is strictly convex. 
\end{lemma}
\begin{proof}[Proof of \Cref{lemma:frontier_regression_convex}]
    The Hessian of the weighted risk $\mathcal{R}_\lambda(\beta)$ is given by $\nabla^2 \mathcal{R}_\lambda(\beta) = 2 \Sigma_\lambda$. Hence, under  $\Sigma_r, \Sigma_b \succ 0$, $\nabla^2 \mathcal{R}_\lambda(\beta) \succ 0$, and we conclude that $\mathcal{R}_\lambda(\beta)$ is strongly convex with a unique minimizer $\beta_\lambda$. 
    
    Next, for $\lambda \in (0, 1)$, the risk pair $(\mathcal{R}_r(\beta_\lambda), \mathcal{R}_b(\beta_\lambda))$ is Pareto efficient; indeed if some $\beta$ satisfies $\mathcal{R}_r(\beta)\leq \mathcal{R}_r(\beta_\lambda)$ and $\mathcal{R}_b(\beta)\leq \mathcal{R}_b(\beta_\lambda)$ with at least one inequality strict, then $\mathcal{R}_\lambda(\beta) < \mathcal{R}_\lambda(\beta_\lambda)$, contradicting the optimality of $\beta_\lambda$. The endpoint cases $\lambda = 0, 1$ follow from the uniqueness of $\beta_r, \beta_b$. 

    Conversely, let $e(\beta)\coloneq (\mathcal{R}_r(\beta), \mathcal{R}_b(\beta))$. Fix any Pareto-efficient point $e^\star = (e_r^\star, e_b^\star) = (\mathcal{R}_r(\beta^\star), \mathcal{R}_b(\beta^\star)) \in \mathcal{E}(\mathcal{H})$. Since $e^\star$ is Pareto-efficient, $\mathcal{E}(\mathcal{H})\cap (e^\star - \RR_{> 0}^2) = \emptyset$. The two sets are convex, and $\mathcal{E}(\mathcal{H})$ is closed. Thus, by the separating hyperplane theorem, there exists nonzero $a = (a_r, a_b)$ with $a_r, a_b \geq 0$, $a_r + a_b = 1$ such that 
    \begin{equation*}
        a_r x + a_b y \geq a_r e_r^\star + a_b e_b^\star
    \end{equation*}
    for every $(x, y) \in \mathcal{E}(\mathcal{H})$. Since $a$ is nonzero with nonnegative coordinates, we can normalize such that $a_r + a_b = 1$. Set $a_r = \lambda$ such that $a_b = 1-\lambda$. The separating hyperplane inequality then gives 
    \begin{equation*}
        \lambda \mathcal{R}_r(\beta) + (1-\lambda) \mathcal{R}_b(\beta) \geq \lambda \mathcal{R}_r(\beta^\star) + (1-\lambda) \mathcal{R}_b(\beta^\star)
    \end{equation*}
    for every $\beta \in \RR^d$. Equivalently, $\mathcal{R}_\lambda(\beta)\geq \mathcal{R}_\lambda(\beta^\star)$ for all $\beta$ and we conclude that $\beta^\star$ minimizes $\mathcal{R}_\lambda$. By uniqueness, $\beta^\star = \beta_\lambda$. Thus, the Pareto frontier is exactly $\{e(\beta_\lambda): \lambda \in [0, 1]\}$. 

    It remains to relate the groupwise risk Pareto efficiency to FA efficiency. First, every Pareto-efficient pair is FA-efficient by definition. We now show the converse under group-balancedness. If $\beta_r = \beta_b$, group-balancedness implies that the common group-optimal predictor has equal group risks.  Thus, it FA dominates every other predictor and the claim follows. We therefore consider the case $\beta_r \neq \beta_b$ for the remainder of the proof. 

    The first order condition for $\beta_\lambda$ is 
    \begin{equation}\label{eqn:frontier_lemma_first_order}
        \lambda \Sigma_r (\beta_\lambda - \beta_r) + (1-\lambda) \Sigma_b (\beta_\lambda - \beta_b) = 0.
    \end{equation}
     Let $\beta_\lambda'$ denote the derivative of $\beta_\lambda$ with respect to $\lambda$. Differentiating \eqref{eqn:frontier_lemma_first_order} with respect to $\lambda$ yields the condition 
    \begin{equation*}
        \beta_\lambda' = - \Sigma_\lambda^{-1}(\Sigma_r(\beta_\lambda - \beta_r) - \Sigma_b(\beta_\lambda - \beta_b))  = -\frac{1}{1-\lambda}\Sigma_\lambda^{-1} \Sigma_r (\beta_\lambda - \beta_r)
    \end{equation*}
    where the second equality uses that, for $\lambda \in (0, 1)$, 
    \begin{equation}\label{eqn:frontier_lemma_sigma_b_diff}
        \Sigma_b(\beta_\lambda - \beta_b) = -\frac{\lambda}{1-\lambda}\Sigma_r(\beta_\lambda - \beta_r).
    \end{equation}
    Moreover, $\Sigma_r (\beta_\lambda - \beta_r) \neq 0$, for $\lambda \in (0, 1)$, otherwise, under \Cref{assumption:invertible_cov}, this condition would contradict $\beta_r \neq \beta_b$. It follows that $(\beta_\lambda - \beta_r)^\top \Sigma_r\Sigma_\lambda^{-1} \Sigma_r (\beta_\lambda - \beta_r) > 0$ and, thus, that 
    \begin{equation}
        \frac{d}{d \lambda} \mathcal{R}_r(\beta_\lambda) = -\frac{2}{1-\lambda}  (\beta_\lambda - \beta_r)^\top \Sigma_r \Sigma_\lambda^{-1} \Sigma_r (\beta_\lambda - \beta_r) < 0.
    \end{equation}
    Likewise, 
    \begin{equation}
        \frac{d}{d \lambda} \mathcal{R}_b(\beta_\lambda) =  2 (\beta_\lambda - \beta_b)^\top \Sigma_b \beta_\lambda' =   \frac{2\lambda}{(1-\lambda)^2}(\beta_\lambda - \beta_r)^\top \Sigma_r  \Sigma_\lambda^{-1} \Sigma_r (\beta_\lambda - \beta_r) > 0
    \end{equation}
    where the second equality uses \eqref{eqn:frontier_lemma_sigma_b_diff}.
    Thus, $\mathcal{R}_r(\beta_\lambda)$ strictly decreases in $\lambda$ and $\mathcal{R}_b(\beta_\lambda)$ strictly increases in $\lambda$ along the Pareto frontier. 

    By group-balancedness, $\mathcal{R}_r(\beta_1) - \mathcal{R}_b(\beta_1) \leq 0$ and $\mathcal{R}_r(\beta_0) - \mathcal{R}_b(\beta_0) \geq 0$. Hence, there exists $\lambda_{\textrm{fair}} \in [0,1]$ such that $\mathcal{R}_r(\beta_{\lambda_{\textrm{fair}}}) = \mathcal{R}_b(\beta_{\lambda_{\textrm{fair}}}) \eqcolon \mathcal{R}_{\textrm{fair}}$.

    Next, we show that every feasible risk pair outside the Pareto frontier is FA dominated. For $(x,y) \in \mathcal{E}(\mathcal{H})$ not Pareto efficient, we will consider the cases (i) $x \leq y, x \geq \mathcal{R}_{\textrm{fair}}$, (ii) $x \leq y, x < \mathcal{R}_{\textrm{fair}}$, (iii) $x \geq y, y < \mathcal{R}_{\textrm{fair}}$, and (iv) $x \geq y, y \geq \mathcal{R}_{\textrm{fair}}$; these cases are illustrated in \Cref{fig:Convexity_Proof}. 

    \begin{figure}[t]
    \centering
    \begin{subfigure}{.98\textwidth}
      \centering
      \resizebox{0.65\linewidth}{!}{

\begin{tikzpicture}[scale=13,line cap=round,line join=round]

\pgfmathsetmacro{\cx}{0.43}
\pgfmathsetmacro{\cy}{0.42}
\pgfmathsetmacro{\a}{0.41}
\pgfmathsetmacro{\b}{0.28}
\pgfmathsetmacro{\theta}{30}
\pgfmathsetmacro{\ct}{cos(\theta)}
\pgfmathsetmacro{\st}{sin(\theta)}
\pgfmathsetmacro{\pointSize}{0.007}
\pgfmathsetmacro{\lineWidth}{0.35mm}

\pgfmathsetmacro{\tYraw}{atan2(\b*\ct,\a*\st)}
\pgfmathsetmacro{\tY}{\tYraw+180}
\pgfmathsetmacro{\betaBx}{
  \cx+\a*cos(\tY)*\ct-\b*sin(\tY)*\st
}
\pgfmathsetmacro{\betaBy}{
  \cy+\a*cos(\tY)*\st+\b*sin(\tY)*\ct
}

\pgfmathsetmacro{\tXraw}{atan2(-\b*\st,\a*\ct)}
\pgfmathsetmacro{\tX}{\tXraw+180}
\pgfmathsetmacro{\betaRx}{
  \cx+\a*cos(\tX)*\ct-\b*sin(\tX)*\st
}
\pgfmathsetmacro{\betaRy}{
  \cy+\a*cos(\tX)*\st+\b*sin(\tX)*\ct
}

\pgfmathsetmacro{\Aeq}{\a*(\ct-\st)}
\pgfmathsetmacro{\Beq}{-\b*(\st+\ct)}
\pgfmathsetmacro{\Req}{sqrt(\Aeq*\Aeq+\Beq*\Beq)}
\pgfmathsetmacro{\deltaeq}{atan2(\Beq,\Aeq)}
\pgfmathsetmacro{\tFairRaw}{
  \deltaeq-acos(-(\cx-\cy)/\Req)
}
\pgfmathsetmacro{\tFair}{
  \tFairRaw<0 ? \tFairRaw+360 : \tFairRaw
}
\pgfmathsetmacro{\Rfair}{
  \cx+\a*cos(\tFair)*\ct-\b*sin(\tFair)*\st
}

\pgfmathsetmacro{\xmax}{max(\betaBx,\betaRx)+0.10}
\pgfmathsetmacro{\ymax}{max(\betaBy,\betaRy)+0.10}

\draw[->] (0,0)--(\xmax,0);
\draw[->] (0,0)--(0,\ymax);
\node[anchor=north,BrickRed] at ({\xmax/2},0) {$\mathcal{R}_r$};
\node[rotate=0,anchor=east,Blue] at (0,{\ymax/2}) {$\mathcal{R}_b$};

\begin{scope}
  \clip (0,0) rectangle (\xmax,\ymax);

  \begin{scope}[shift={(\cx,\cy)},rotate=\theta]
    \fill[Rhodamine,opacity=0.10]
      (0,0) ellipse [x radius=\a,y radius=\b];
  \end{scope}

  \begin{scope}
    \clip[shift={(\cx,\cy)},rotate=\theta]
      (0,0) ellipse [x radius=\a,y radius=\b];

    \fill[BurntOrange,opacity=0.34]
      (\Rfair,\Rfair)--(\Rfair,\ymax)--
      (\xmax,\ymax)--(\xmax,\xmax)--cycle;

    \fill[Goldenrod,opacity=0.48]
      (0,0)--(0,\ymax)--
      (\Rfair,\ymax)--(\Rfair,\Rfair)--cycle;

    \fill[ForestGreen,opacity=0.28]
      (\Rfair,\Rfair)--(\xmax,\Rfair)--
      (\xmax,\xmax)--cycle;

    \fill[Rhodamine,opacity=0.40]
      (0,0)--(\xmax,0)--
      (\xmax,\Rfair)--(\Rfair,\Rfair)--cycle;
  \end{scope}

  \begin{scope}[shift={(\cx,\cy)},rotate=\theta]
    \draw[Rhodamine!65]
      (0,0) ellipse [x radius=\a,y radius=\b];
  \end{scope}

  \draw[densely dashed,gray!85,line width=0.24mm]
    (0,0)--(\xmax,\xmax);
  \draw[densely dashed,gray!85,line width=0.24mm]
    (\Rfair,\Rfair)--(\Rfair,\ymax);
  \draw[densely dashed,gray!85,line width=0.24mm]
    (\Rfair,\Rfair)--(\xmax,\Rfair);

  \begin{scope}[shift={(\cx,\cy)},rotate=\theta]
    \draw[ultra thick,Thistle!55!Fuchsia]
      ({\a*cos(\tX)},{\b*sin(\tX)})
      arc[start angle=\tX,end angle=\tY,
          x radius=\a,y radius=\b];
  \end{scope}
\end{scope}

\node[text=gray!70!black,anchor=south east,font=\scriptsize]
  at ({0.91*\xmax},{0.91*\xmax}) {$x=y$};
\node[text=gray!70!black,anchor=west,font=\scriptsize]
  at ({\Rfair+0.005},{0.92*\ymax})
  {$\mathcal{R}_r=\mathcal{R}_{\textrm{fair}}$};
\node[text=gray!70!black,anchor=south,font=\scriptsize]
  at ({0.78*\xmax},{\Rfair+0.003})
  {$\mathcal{R}_b=\mathcal{R}_{\textrm{fair}}$};

\node[text=black,inner sep=0pt,font=\scriptsize]
  at (0.225,0.305) {(i)};
\node[text=black,inner sep=0pt,font=\scriptsize]
  at (0.085,0.245) {(ii)};
\node[text=black,inner sep=0pt,font=\scriptsize]
  at (0.315,0.215) {(iv)};
\node[text=black,inner sep=0pt,font=\scriptsize]
  at (0.255,0.122) {(iii)};

\fill[Fuchsia] (\Rfair,\Rfair) circle[radius=\pointSize]
  node[below left,Fuchsia,font=\scriptsize]
  {$(\mathcal{R}_{\textrm{fair}},\mathcal{R}_{\textrm{fair}})$};
\fill[Blue] (\betaBx,\betaBy) circle[radius=\pointSize]
  node[below,Blue, font=\scriptsize] {$(\mathcal{R}_r(\beta_0), \mathcal{R}_b(\beta_0))$};
\fill[BrickRed] (\betaRx,\betaRy) circle[radius=\pointSize]
  node[left,BrickRed, font=\scriptsize] {$(\mathcal{R}_r(\beta_1), \mathcal{R}_b(\beta_1))$};

\pgfmathsetmacro{\legendx}{\xmax+0.075}
\pgfmathsetmacro{\swatchw}{0.025}
\pgfmathsetmacro{\swatchh}{0.012}


\fill[BurntOrange,opacity=0.34]
  (\legendx,{\ymax-0.055-\swatchh})
  rectangle ({\legendx+\swatchw},{\ymax-0.055+\swatchh});
\node[anchor=west,font=\small]
  at ({\legendx+\swatchw+0.012},{\ymax-0.055})
  {(i) $x\leq y,\ x\geq\mathcal{R}_{\textrm{fair}}$};

\fill[Goldenrod,opacity=0.48]
  (\legendx,{\ymax-0.135-\swatchh})
  rectangle ({\legendx+\swatchw},{\ymax-0.135+\swatchh});
\node[anchor=west,font=\small]
  at ({\legendx+\swatchw+0.012},{\ymax-0.135})
  {(ii) $x\leq y,\ x<\mathcal{R}_{\textrm{fair}}$};

\fill[Rhodamine,opacity=0.40]
  (\legendx,{\ymax-0.215-\swatchh})
  rectangle ({\legendx+\swatchw},{\ymax-0.215+\swatchh});
\node[anchor=west,font=\small]
  at ({\legendx+\swatchw+0.012},{\ymax-0.215})
  {(iii) $x\geq y,\ y<\mathcal{R}_{\textrm{fair}}$};

\fill[ForestGreen,opacity=0.28]
  (\legendx,{\ymax-0.295-\swatchh})
  rectangle ({\legendx+\swatchw},{\ymax-0.295+\swatchh});
\node[anchor=west,font=\small]
  at ({\legendx+\swatchw+0.012},{\ymax-0.295})
  {(iv) $x\geq y,\ y\geq\mathcal{R}_{\textrm{fair}}$};

\end{tikzpicture}}
      \label{fig:uniqueness_proof_fig}
    \end{subfigure}
    \caption{Illustration of the four cases used to establish that every feasible risk pair outside of the Pareto frontier is FA dominated, as defined in the proof of \Cref{lemma:frontier_regression_convex}.}
    \label{fig:Convexity_Proof}
    \end{figure}

    For case (i), we note that for $\mathcal{R}_{\textrm{fair}} \leq x \leq y$, the risk pair $(\mathcal{R}_{\textrm{fair}}, \mathcal{R}_{\textrm{fair}})$ dominates $(x, y)$ and has zero disparity: $|\mathcal{R}_{\textrm{fair}}-\mathcal{R}_{\textrm{fair}}|=0$.  For case (ii), first we note that $\mathcal{R}_r(\beta_1) \leq x < \mathcal{R}_{\textrm{fair}}$. Since $\mathcal{R}_r(\beta_\lambda)$ is continuous and, as we showed earlier, decreasing in $\lambda$, there exists $\lambda > \lambda_{\textrm{fair}}$ such that $\mathcal{R}_r(\beta_\lambda) = x$. Earlier, we also established that its corresponding risk pair $(\mathcal{R}_r(\beta_\lambda), \mathcal{R}_b(\beta_\lambda))$ is Pareto efficient. This implies that $\mathcal{R}_b(\beta_\lambda) \leq y$, otherwise $(x, y)$ would Pareto dominate $(\mathcal{R}_r(\beta_\lambda), \mathcal{R}_b(\beta_\lambda))$. In fact, we require strict inequality $\mathcal{R}_b(\beta_\lambda) < y$, otherwise $(x,y)$ would equal the Pareto-efficient pair $(\mathcal{R}_r(\beta_\lambda), \mathcal{R}_b(\beta_\lambda))$. Moreover, $\lambda > \lambda_{\textrm{fair}}$ implies $\mathcal{R}_r(\beta_\lambda) \leq \mathcal{R}_b(\beta_\lambda)$. It follows that $|\mathcal{R}_r(\beta_\lambda) - \mathcal{R}_b(\beta_\lambda)|=\mathcal{R}_b(\beta_\lambda) - \mathcal{R}_r(\beta_\lambda) \leq y-x= |x-y|$ and we conclude that $(\mathcal{R}_r(\beta_\lambda), \mathcal{R}_b(\beta_\lambda))$ FA-dominates $(x, y)$. Case (iii) is symmetric to case (ii). Case (iv) is symmetric to case (i). 
    
    Altogether, we obtain that, across all cases, every feasible point $(x,y)$ outside the Pareto frontier is FA-dominated. Thus, under group-balancedness, the FA frontier coincides with the Pareto frontier and is generated by $\{\beta_\lambda: \lambda \in [0, 1]\}$. 

    Since $\mathcal{R}_r(\beta_\lambda)$ is strictly decreasing in $\lambda$, every frontier red group risk corresponds to exactly one $\lambda$ and, hence, one $\mathcal{R}_b(\beta_\lambda)$. We can therefore write the frontier as $\mathcal{R}_b(\beta_\lambda) = f(\mathcal{R}_r(\beta_\lambda))$ for $\lambda \in [0,1]$. The derivatives above give $f'(\mathcal{R}_r(\beta_\lambda)) = -\lambda/(1-\lambda)$, $f''(\mathcal{R}_r(\beta_\lambda)) = (2(1-\lambda)(\beta_\lambda - \beta_r)^\top \Sigma_r\Sigma_\lambda^{-1} \Sigma_r (\beta_\lambda - \beta_r))^{-1}>0$ for $\lambda \in (0,1)$ and the strict convexity claim follows. 
\end{proof}
\subsubsection{Proof of \Cref{lemma:linear_model_group_balanced}} \label{proof:linear_group_balanced}
For any group $g \in \mathcal{G}$, the risk of the group-optimal predictor $\beta_g$ is:
\begin{equation*}
\mathcal{R}_g(\beta_g) = \EE_g\left[ (X^\top \beta_g - Y)^2\right] = \EE_g[\varepsilon_g^2] = \sigma_g^2.
\end{equation*}
For the other group $g' \neq g$, the same predictor incurs risk:
\begin{equation*}
\mathcal{R}_{g'}(\beta_g) = \EE_{g'}\left[ (X^\top \beta_g - Y)^2\right] = \norm{\beta_g - \beta_{g'}}^2_{\Sigma_{g'}} + \sigma_{g'}^2.
\end{equation*}
Hence, $\mathcal{R}_g(\beta_g) \leq \mathcal{R}_{g'}(\beta_g)$ if and only if $\sigma_g^2 - \sigma_{g'}^2 \leq \norm{\beta_g - \beta_{g'}}^2_{\Sigma_{g'}}$. The claim follows by applying this with $(g, g') = (r, b)$ and $(g, g') = (b, r)$. 
\subsubsection{Proof of \Cref{proposition:excess_risk_to_matrix_norm}} \label{proof:proposition:excess_risk_to_matrix_norm}
Note that
\begin{equation}
\mathcal{R}_g(\beta) = \EE_{g}\left[ (X^\top \beta - Y)^2\right]  =
\EE_{g}\left[ (X^\top (\beta - \beta_g) - \varepsilon_g )^2\right] =  \|\beta-\beta_g\|_{\Sigma_g}^2 + \sigma_g^2,
\end{equation}
where the last equality follows from the fact that $\EE[\varepsilon_g | X] = 0$. As a result, we have
\begin{equation}
\mathcal{R}_g(\beta) - \mathcal{R}_g(\beta_\lambda) =  \|\beta-\beta_g\|_{\Sigma_g}^2 -   \|\beta_\lambda-\beta_g\|_{\Sigma_g}^2 
= \|\beta-\beta_\lambda\|_{\Sigma_g}^2 + 2  (\beta - \beta_\lambda)^\top \Sigma_g (\beta_\lambda - \beta_g) ,
\end{equation}
which completes the proof of \eqref{eqn:estimation_to_matrix_norm}. To establish \eqref{eqn:excess_to_matrix_norm}, simply sum both sides of \eqref{eqn:estimation_to_matrix_norm} for the red and blue groups, weighted by $\lambda$ and $1 - \lambda$, respectively, and use the definition of $\beta_\lambda$.
\subsection{Proofs from \Cref{sec:knowncov}}
\subsubsection{Proof of \Cref{theorem:minimax_optimal_known_cov}} \label{proof:theorem:minimax_optimal_known_cov}
\subsubsection*{Upper bound}
We first establish that the error given in the statement of the theorem is indeed achieved by setting $\beta = \tilde{\beta}_{\lambda}$. Notice that we have
\begin{equation*}
    \Tilde{\beta}_\lambda = \beta_\lambda - \Sigma_{\lambda}^{-1} \left( \lambda \Sigma_r (\beta_r - \widehat{\beta}_r ) + ( 1 - \lambda) \Sigma_b (\beta_b - \widehat{\beta}_b)\right). 
\end{equation*}
Substituting $\widehat{\beta}_g = \widehat{\Sigma}_g^{-1} \widehat{\nu}_g$ and using the decomposition 
\begin{equation}\label{eqn:known_covariance_upper_bound_nu_g_decomposition}
    \widehat{\nu}_g = \widehat{\Sigma}_g \beta_g + \widehat{Z}_g, \quad \widehat{Z}_g \coloneq \frac{1}{n_g} \sum_{i = 1}^{n_g}X_i^{(g)} \varepsilon_i^{(g)}
\end{equation} 
we obtain
\begin{equation}\label{eqn:known_cov_estimator_beta_lambda_gap}
    \Tilde{\beta}_\lambda - \beta_\lambda = \Sigma_{\lambda}^{-1} \left( \lambda \Sigma_r \widehat{\Sigma}_r^{-1}\widehat{Z}_r + (1 - \lambda) \Sigma_b \widehat{\Sigma}_b^{-1} \widehat{Z}_b \right).
\end{equation}

Let $A \coloneq {\Sigma}_{\lambda}^{-1} \Sigma_g {\Sigma}_{\lambda}^{-1}$. By the mean-zero and independence properties of the noise, the cross-terms cancel in the squared norm, which yields
\begin{equation*}
    \begin{aligned}
        \EE & \left[ \norm{ \tilde{\beta}_{\lambda} - \beta_\lambda }_{\Sigma_g}^2\right] 
        &= \lambda^2 \EE \left[ \norm{\Sigma_r \widehat{\Sigma}_r^{-1}  \widehat{Z}_r}_{A}^2 \right] + (1 - \lambda)^2 \cdot 
         \EE \left[ \norm{ \Sigma_b \widehat{\Sigma}_b^{-1}  \widehat{Z}_b }_{A}^2 \right] 
    \end{aligned}
\end{equation*}
Using the independence and mean-zero property of $\varepsilon_i^{(g)}$ and applying the trace identity for the expectation of quadratic forms, we obtain
\begin{equation*}
    \begin{aligned}
        \EE \left[ \norm{\Sigma_r \widehat{\Sigma}_r^{-1} \widehat{Z}_r}_{A}^2 \right] &= \frac{\sigma_r^2}{n_r} \cdot \EE \left[ \Tr{\widehat{\Sigma}_r^{-1} \Sigma_r A \Sigma_r } \right].
    \end{aligned} 
\end{equation*}
Applying the same procedure for the equivalent term for group $b$ and substituting completes the proof.
\subsubsection*{Lower bound}
Next, we establish the minimax optimality of $\tilde{\beta}_\lambda$. Specifically, we will show 
\begin{align*}
\inf_{\widehat{\beta}} \sup_{P \in \mathcal{P}_{\text{linear}} (P_X^r, P_X^b, \sigma_r^2, \sigma_b^2)} &\EE\left[ \norm{ \widehat{\beta} - \beta_\lambda }_{\Sigma_g}^2 \right]  \\
&\geq  \lambda^2 \frac{\sigma_r^2}{n_r} \EE\left[ \Tr{\Sigma_g \Sigma_\lambda^{-1} \Sigma_r \widehat{\Sigma}_r^{-1} \Sigma_r \Sigma_\lambda^{-1} } \right] +
(1 - \lambda)^2 \frac{\sigma_b^2}{n_b} \EE \left[ \Tr{\Sigma_g \Sigma_\lambda^{-1} \Sigma_b \widehat{\Sigma}_b^{-1} \Sigma_b \Sigma_\lambda^{-1}}\right].
\end{align*}

To prove the lower bound, it suffices to lower bound the risk over the submodel with Gaussian noise since it is a subclass of $\mathcal{P}_{\text{linear}} (P_X^r, P_X^b, \sigma_r^2, \sigma_b^2)$. In particular, we fix the model \Cref{def:linear_model} with independent noise $\varepsilon_g \sim \NN(0, \sigma_g^2)$.  

We compare the minimax risk to the Bayes risk under a Gaussian prior. Fix $\tau>0$, and place independent Gaussian priors on the coefficients of each group, i.e., take $\beta_r \sim \NN(0, \tau^2 I_d),\quad \beta_b \sim \NN(0, \tau^2 I_d).$
Let $\Pi_\tau$ denote the joint prior of $(\beta_r, \beta_b)$. Then, 
\begin{equation}\label{eqn:lower_bound_bayes_minimax}
\inf_{\widehat{\beta}} \sup_{\beta_r, \beta_b} \EE\left[ \norm{\widehat{\beta} - \beta_\lambda }_{\Sigma_g}^2 \right] \geq \inf_{\widehat{\beta}} \EE_{\Pi_\tau} \left[ \norm{\widehat{\beta} - \beta_\lambda }^2_{\Sigma_g}  \right].
\end{equation}

Under the linear model given in \Cref{def:linear_model} and taking $\varepsilon_g \sim \NN(0, \sigma_g^2)$ i.i.d., by Gaussian conjugacy,  
\begin{equation*}
    \beta_g | (\mathcal{S}_r, \mathcal{S}_b) \sim \NN\left( \mu_g(\tau), V_g(\tau)\right)
\end{equation*}
where  
\begin{equation*}
    M_g(\tau) \coloneq  \widehat{\Sigma}_g + \frac{\sigma_g^2}{n_g \tau^2} I_d, \quad \mu_g(\tau) \coloneq {M_g(\tau)}^{-1} \widehat{\nu}_g, \quad V_g(\tau) \coloneq \frac{\sigma_g^2}{n_g}{M_g(\tau)}^{-1}.
\end{equation*}

Since $\beta_\lambda$ is a deterministic linear combination of $\beta_r$ and $\beta_b$, its posterior is also Gaussian:
\begin{equation*}
    \begin{aligned}
        \beta_\lambda \mid (\mathcal{S}_r, \mathcal{S}_b)\sim \NN \left( \mu_\lambda(\tau), V_\lambda(\tau) 
        \right)
    \end{aligned}
\end{equation*}
where
\begin{equation*}
    \mu_\lambda(\tau) \coloneq \Sigma_\lambda^{-1}\left( \lambda \Sigma_r \mu_r(\tau) + (1 - \lambda) \Sigma_b \mu_b(\tau) \right),\quad V_\lambda(\tau) \coloneq \Sigma_\lambda^{-1}\left( \lambda^2 \Sigma_r V_r(\tau) \Sigma_r + (1 - \lambda)^2 \Sigma_b V_b(\tau) \Sigma_b \right)\Sigma_\lambda^{-1}.
\end{equation*}

Under squared $\Sigma_g$-loss, the Bayes estimator $\widehat{\beta}^\star_\lambda$ is the posterior mean, i.e., $\widehat{\beta}^\star_\lambda\coloneq \mu_\lambda(\tau)$, and its Bayes risk is
\begin{equation*}
    \EE \left[ \norm{\widehat{\beta}^\star_\lambda - \beta_\lambda }_{\Sigma_g}^2  \right]  = \EE \left[ \Tr{\Sigma_g V_\lambda(\tau) } \right]
\end{equation*}
To obtain a lower bound, we consider the limit $\tau \to \infty$. In this limit:
\begin{equation*}
    \mu_g(\tau) \to \widehat{\Sigma}_g^{-1} \widehat{\nu}_g, \quad V_g(\tau) \to \frac{\sigma_g^2}{n_g} \widehat{\Sigma}_g^{-1}.
\end{equation*}
Thus, $V_\lambda(\tau)$ increases in the Loewner order as $\tau \to \infty$, and by monotone convergence, 
\begin{equation*}
    \lim_{\tau \to \infty} \EE \left[ \Tr{\Sigma_g V_\lambda(\tau) } \right] = \lambda^2 \frac{\sigma_r^2}{n_r} \EE \Tr{\Sigma_g \Sigma_\lambda^{-1} \Sigma_r \widehat{\Sigma}_r^{-1} \Sigma_r \Sigma_\lambda^{-1} } +
(1 - \lambda)^2 \frac{\sigma_b^2}{n_b} \EE  \Tr{\Sigma_g \Sigma_\lambda^{-1} \Sigma_b \widehat{\Sigma}_b^{-1} \Sigma_b \Sigma_\lambda^{-1}}.
\end{equation*}
Combining this with \eqref{eqn:lower_bound_bayes_minimax} proves the claim.
\subsubsection{Proof of \Cref{cor:known_small_ball}}\label{sec:proof_cor:known_small_ball}
Fix $g, h \in \mathcal{G}$. Define $W_{g, h} \coloneq \Sigma_{h}^{1/2}\Sigma_\lambda^{-1}\Sigma_g\Sigma_{\lambda}^{-1}\Sigma_h^{1/2} \succeq 0$. Also define the whitened covariates and sample covariance
\begin{equation*}
    {\tilde{X}_i}^{(h)} \coloneq \Sigma_h^{-1/2} {X}_i^{(h)}, \quad \bar{\Sigma}_h \coloneq \frac{1}{n_h}\sum_{i=1}^{n_h} \tilde{X}_i^{(h)}{{\tilde{X}_i}^{(h)}}^\top.
\end{equation*}
Then, using $\bar{\Sigma}_h^{-1}\preceq \lambda_{\min}(\bar{\Sigma}_h)^{-1} I_d$ and monotonicity of the trace, we bound 
\begin{equation}\label{eqn:trace_bound_W_gh_known_general}
    \operatorname{Tr}\left( W_{g, h} \bar{\Sigma}_h^{-1}\right) \leq \lambda_{\min}(\bar{\Sigma}_h)^{-1} \operatorname{Tr}\left( W_{g, h}\right). 
\end{equation}
Then, applying the bound \eqref{eqn:trace_bound_W_gh_known_general} to each term in the exact expression from \Cref{theorem:minimax_optimal_known_cov} and using that $\widehat{\Sigma}_h^{-1}= \Sigma_h^{-1/2}\bar{\Sigma}_h^{-1}\Sigma_h^{-1/2}$ with cyclicity of the trace, we obtain the bound 
\begin{equation}
    \EE \left[ \norm{ \tilde{\beta}_\lambda - \beta_\lambda }_{\Sigma_g}^2\right] \leq \lambda^2 \frac{\sigma_r^2}{n_r} \operatorname{Tr}\left( W_{g, r}\right) \EE \left[ \lambda_{\min}( \bar{\Sigma}_r)^{-1}\right] + (1-\lambda)^2 \frac{\sigma_b^2}{n_b} \operatorname{Tr}\left( W_{g, b}\right) \EE \left[ \lambda_{\min}( \bar{\Sigma}_b)^{-1}\right]. 
\end{equation}
By \citet[Corollary 4]{mourtada2022exact}, $\EE \left[ \lambda_{\min}( \bar{\Sigma}_h)^{-1}\right]\leq 2 C_h'$, and \eqref{eq:known_general_small_ball_upper} follows by substitution. 

For the sub-Gaussian refinement, note that $\tilde{X}^{(h)}$ satisfies \Cref{assumption:subgaussian} with parameter $K_h$. Since $\tilde{X}^{(h)}$ is isotropic, each coordinate $\tilde{X}_j^{(h)}$ is a scalar $K_h$-sub-Gaussian random variable. Hence, by integrating the tail, 
\begin{equation*}
    \EE\left[ \left| \tilde{X}_j^{(h)}
     \right|^4 \right] \leq 8 \int_0^\infty  t^3 \exp(-t^2/K_h^2) dt \leq 4 K_h^4. 
\end{equation*}
By expanding the expected norm, we get
\begin{equation*}
    \EE\left[ \norm{\tilde{X}^{(h)}}_2^4 \right] = \sum_{j=1}^d \EE\left[ \left| \tilde{X}_j^{(h)}
     \right|^4 \right] + 2 \sum_{1 \leq j < k \leq d} \EE\left[ \left| \tilde{X}_j^{(h)}
     \right|^2 \left| \tilde{X}_k^{(h)}
     \right|^2 \right].
\end{equation*}
By Cauchy-Schwarz, the cross terms are bounded as
\begin{equation*}
     \EE\left[ \left| \tilde{X}_j^{(h)}
     \right|^2 \left| \tilde{X}_k^{(h)}
     \right|^2 \right] \leq 4 K_h^4.
\end{equation*}
Therefore, 
\begin{equation*}
    \EE\left[ \norm{\tilde{X}^{(h)}}_2^4 \right] \leq 4 d  K_h^4 + 8 \binom{d}{2}  K_h^4 = 4 K_h^4 d^2.
\end{equation*}
Then, this time applying the bound
\begin{equation*}
    \operatorname{Tr}\left( W_{g, h} \bar{\Sigma}_h^{-1}\right)\leq\norm{W_{g,h}}\operatorname{Tr}\left(\bar{\Sigma}_h^{-1}\right)
\end{equation*}
to each term in the exact expression from \Cref{theorem:minimax_optimal_known_cov} and using that $\widehat{\Sigma}_h^{-1}= \Sigma_h^{-1/2}\bar{\Sigma}_h^{-1}\Sigma_h^{-1/2}$ with cyclicity of the trace, we obtain the bound 
\begin{equation}
    \EE \left[ \norm{ \tilde{\beta}_\lambda - \beta_\lambda }_{\Sigma_g}^2\right] \leq \lambda^2 \frac{\sigma_r^2}{n_r}  \norm{W_{g, r}} \EE \left[\operatorname{Tr}\left(\bar{\Sigma}_r^{-1}\right)\right] + (1-\lambda)^2 \frac{\sigma_b^2}{n_b} \norm{W_{g, b}} \EE \left[ \operatorname{Tr}\left(\bar{\Sigma}_b^{-1}\right)\right]. 
\end{equation}
By \citet[Theorem 3]{mourtada2022exact}, 
\begin{equation*}
    \EE \left[ \Tr{ \bar{\Sigma}_h^{-1}  } \right] \leq d + \frac{32 C_h'  K_h^4 d^2}{n_h},
\end{equation*}
and substitution yields \eqref{eq:known_general_small_ball_subgaussian}. $\qed$
\subsubsection{Proof of \Cref{proposition:expectation-vs-realization}} \label{proof:proposition:expectation-vs-realization}
Throughout, let $n = \min\{n_r, n_b\}$. By a variant of the Woodbury identity, $(A-B)^{-1} = A^{-1} + A^{-1}B(A-B)^{-1}$, we obtain from taking $A = \Sigma_g$ and $B = \Sigma_g-\widehat{\Sigma}_g$ that 
\begin{equation} \label{eqn:proof_realization_2} 
\widehat{\Sigma}_g^{-1} = \Sigma_g^{-1} + \Sigma_g^{-1} (\Sigma_g - \widehat{\Sigma}_g) \widehat{\Sigma}_g^{-1}.    
\end{equation}
Substituting \eqref{eqn:proof_realization_2} into the earlier decomposition \eqref{eqn:known_cov_estimator_beta_lambda_gap} yields 
\begin{equation}\label{eqn:proof_realization_3}
    \tilde{\beta}_\lambda - \beta_\lambda = 
\Sigma_\lambda^{-1} \left( 
\lambda \widehat{Z}_r
+ (1-\lambda) \widehat{Z}_b \right)+ \Sigma_\lambda^{-1} \left ( 
\lambda (\Sigma_r - \widehat{\Sigma}_r) \widehat{\Sigma}_r^{-1} \widehat{Z}_r
+ (1-\lambda) (\Sigma_b - \widehat{\Sigma}_b) \widehat{\Sigma}_b^{-1} \widehat{Z}_b
\right). 
\end{equation}
Now, note that the expectation of the norm of the second term in the right-hand side of \eqref{eqn:proof_realization_3} is bounded by $\mathcal{O}(1/{n})$. To see this, note that, for any $g \in \mathcal{G}$, we have
\begin{equation*}
\EE \left [ \left \| 
(\Sigma_g - \widehat{\Sigma}_g) \widehat{\Sigma}_g^{-1} \widehat{Z}_g
\right\| \right]  
 \leq  \sqrt{\EE \left[ \left \| \widehat{Z}_g  \right\|^2 \right]} 
\left( \EE \left[ \left \| (\Sigma_g - \widehat{\Sigma}_g) \right\|^4 \right] \right)^{1/4}
\left( \EE \left[ \left \| \widehat{\Sigma}_g^{-1} \right\|^4 \right] \right)^{1/4},
\end{equation*}
and the first term is bounded by $\mathcal{O}(n_g^{-1/2})$ since it is zero mean, the second term is bounded by $\mathcal{O}(n_g^{-1/2})$ by \Cref{lemma:sub_gaussian_covariance_bound} with $q = 4$, and, finally, the last term is bounded by a constant due to \Cref{lemma:mourtada_small_ball_min_eval_qth_moment} taking $q = 4$.
Now, this, along with \Cref{proposition:excess_risk_to_matrix_norm} and \Cref{cor:known_small_ball} implies that we can write groupwise estimation error as 
\begin{equation}
\mathcal{R}_g(\tilde{\beta}_\lambda) - \mathcal{R}_g(\beta_\lambda)  = \Lambda_1     + \Lambda_2, \quad \Lambda_1\coloneq 2  \left ( 
\lambda \widehat{Z}_r
+ (1-\lambda) \widehat{Z}_b \right )^\top \Sigma_\lambda^{-1} \Sigma_g (\beta_\lambda - \beta_g)
\end{equation}
where $\EE[|\Lambda_2|] \leq \kappa /{n}$ for some constant $\kappa$. By Markov's inequality,
\begin{equation*}
    \PP\left( \Lambda_2 \leq \frac{2\kappa}{n \delta} \right)\geq 1 - \frac{\delta}{2}.
\end{equation*}
By the central limit theorem, $\sqrt{{n}}\Lambda_1$ converges in distribution to a mean-zero Gaussian, thus, there exists a constant $c_1 > 0$ and a threshold $\bar{n}_1$, such that for all ${n} \geq \bar{n}_1$,
\begin{equation*}
    \mathbb{P}\left(\Lambda_1 \leq \frac{-c_1}{\sqrt{{n}}}\right) \geq \frac{1}{2} - \frac{\delta}{2}.
\end{equation*}
Therefore, with probability at least $1/2-\delta$ and for ${n} \geq \bar{n}_1$, we have both $\Lambda_1 \leq \frac{-c_1}{\sqrt{n}}$ and $\Lambda_2 \leq \frac{2\kappa}{n \delta}$.  Finally, taking ${n} = \max \left \{ \bar{n}_1, \frac{4\kappa^2}{c_1^2 \delta^2}\right\}$ completes the proof. $\qed$

\subsection{Proofs from \Cref{sec:unknown_covariance}}
\subsubsection{Proof of Proposition \ref{prop:unknown_cov_upper_bound}}\label{sec:proof_prop:unknown_cov_upper_bound}
Recall $\tilde{X}_{i, g}^{(h)} = \Sigma_g^{-1/2} X_i^{(h)}$ and  $\Tilde{\Sigma}_g = \Sigma_g^{-1/2} \widehat{\Sigma}_\lambda \Sigma_g^{-1/2}$ and the result of \Cref{lemma:unknown_covariance_exact_risk}: 
\begin{equation}\label{eqn:upper_bound_risk_decomp_1}
\begin{aligned}
    \EE \left[ \norm{ \widehat{\beta}_{\lambda} - \beta_\lambda }_{\Sigma_g}^2\right] &=  \frac{\lambda^2 \sigma_r^2}{n_r^2} \EE \left[ \sum_{i =1}^{n_r}  \norm{\Tilde{X}_{i, g}^{(r)}}_{\Tilde{\Sigma}_g^{-2}}^2 \right]  +  \frac{(1-\lambda)^2 \sigma_b^2}{n_b^2} \EE \left[ \sum_{i =1}^{n_b} \norm{\Tilde{X}_{i, g}^{(b)}}_{\Tilde{\Sigma}_g^{-2}}^2 \right] \\
    & + \quad  \EE \left[ \norm{\Sigma_g^{-1/2}\left( \lambda \widehat{\Sigma}_r (\beta_r - \beta_\lambda) + ( 1 - \lambda) \widehat{\Sigma}_b (\beta_b - \beta_\lambda) \right)}_{\Tilde{\Sigma}_g^{-2}}^2 \right]. 
\end{aligned}
\end{equation}
Observe that for each $h \in \mathcal{G}$, with $W_g = \widehat{\Sigma}_\lambda^{-1}\Sigma_g \widehat{\Sigma}_\lambda^{-1}$ 
\begin{equation}\label{eqn:upper_bound_risk_decomp_2}
    \frac{\sigma_h^2}{n_h^2 }\EE \left[ \sum_{i =1}^{n_h}  \norm{{\Tilde{X}_i}^{(h)}}_{\Tilde{\Sigma}_g^{-2}} ^2 \right] = \frac{\sigma_h^2}{n_h }\EE \left[ \operatorname{Tr}\left( W_g \frac{1}{n_h} \sum_{i =1}^{n_h}  {X_i^{(h)}} {X_i^{(h)}}^\top \right)\right] = \frac{\sigma_h^2}{n_h }\EE \left[ \operatorname{Tr}\left( W_g \widehat{\Sigma}_h \right)\right]. 
\end{equation}

Applying \eqref{eqn:upper_bound_risk_decomp_2} to the first two terms of  \eqref{eqn:upper_bound_risk_decomp_1} yields the variance term. To establish the first representation of the bias term, observe that:
\begin{equation*}
    \norm{\lambda \Sigma_g^{-1/2} \widehat{\Sigma}_r (\beta_r - \beta_\lambda) + ( 1 - \lambda) \Sigma_g^{-1/2} \widehat{\Sigma}_b (\beta_b - \beta_\lambda) }_{\Tilde{\Sigma}_g^{-2}}^2 = \norm{\Sigma_g^{1/2} \widehat{\Sigma}_\lambda^{-1} \left( \lambda \widehat{\Sigma}_r (\beta_r - \beta_\lambda) + ( 1 - \lambda) \widehat{\Sigma}_b (\beta_b - \beta_\lambda)\right) }^2.
\end{equation*}

To establish the second equivalent representation, note that
\begin{equation}\label{eqn:upper_bound_risk_decomp_zero_term}
    \lambda \Sigma_r (\beta_r - \beta_\lambda) + (1 - \lambda) \Sigma_b (\beta_b - \beta_\lambda) = 0
\end{equation}
which implies
\begin{equation*}
\lambda \widehat{\Sigma}_r (\beta_r - \beta_\lambda) + ( 1 - \lambda) \widehat{\Sigma}_b (\beta_b - \beta_\lambda) = \lambda \left(  \widehat{\Sigma}_r - \Sigma_r \right)(\beta_r - \beta_\lambda) + (1 - \lambda) \left( \widehat{\Sigma}_b - \Sigma_b \right)(\beta_b - \beta_\lambda).
\end{equation*}
We can then rearrange to see that
\begin{align}
& \widehat{\Sigma}_\lambda^{-1} \left( \lambda \left(  \widehat{\Sigma}_r - \Sigma_r \right)(\beta_r - \beta_\lambda) + (1 - \lambda) \left( \widehat{\Sigma}_b - \Sigma_b \right)(\beta_b - \beta_\lambda)\right)  \\
& \quad = \lambda \widehat{\Sigma}_\lambda^{-1} \widehat{\Sigma}_r \beta_r + (1 - \lambda) \widehat{\Sigma}_\lambda^{-1} \widehat{\Sigma}_b \beta_b - \beta_\lambda \nonumber \\
& \quad = \lambda \left( \widehat{\Sigma}_\lambda^{-1}\widehat{\Sigma}_r - \Sigma_\lambda^{-1} \Sigma_r \right) \beta_r + (1 - \lambda) \left( \widehat{\Sigma}_\lambda^{-1} \widehat{\Sigma}_b - \Sigma_\lambda^{-1} \Sigma_b \right) \beta_b. \label{eqn:cor-unknown-proof-0}
\end{align}
Next, note that, we have
\begin{align}\label{eqn:cor-unknown-proof-1}
    \lambda \left( \widehat{\Sigma}_\lambda^{-1}\widehat{\Sigma}_r - \Sigma_\lambda^{-1} \Sigma_r \right) 
     = \left( I_d + \frac{1 - \lambda}{\lambda} \widehat{\Sigma}_r^{-1} \widehat{\Sigma}_b \right)^{-1}
     - 
     \left( I_d + \frac{1 - \lambda}{\lambda} {\Sigma}_r^{-1} {\Sigma}_b \right)^{-1}.
\end{align}
Similarly, we can write
\begin{equation} \label{eqn:cor-unknown-proof-2}
     (1 - \lambda) \left( \widehat{\Sigma}_\lambda^{-1} \widehat{\Sigma}_b - \Sigma_\lambda^{-1} \Sigma_b \right)
     = \left( I_d + \frac{\lambda}{1 - \lambda} \widehat{\Sigma}_b^{-1} \widehat{\Sigma}_r \right)^{-1}
     - 
     \left( I_d + \frac{\lambda}{1 - \lambda} {\Sigma}_b^{-1} {\Sigma}_r \right)^{-1}.
\end{equation}
Next, using the Woodbury matrix identity, we have $(I_d+X)^{-1} = I_d - (I_d+X^{-1})^{-1}$ for any invertible $d\times d$ matrix $X$. Therefore, we can recast \eqref{eqn:cor-unknown-proof-2} as 
\begin{equation} \label{eqn:cor-unknown-proof-3}
     (1 - \lambda) \left( \widehat{\Sigma}_\lambda^{-1} \widehat{\Sigma}_b - \Sigma_\lambda^{-1} \Sigma_b \right)
     = - \left( I_d + \frac{1 - \lambda}{\lambda} \widehat{\Sigma}_r^{-1} \widehat{\Sigma}_b \right)^{-1}
     + 
     \left( I_d + \frac{1 - \lambda}{\lambda} {\Sigma}_r^{-1} {\Sigma}_b \right)^{-1}.
\end{equation}
Plugging \eqref{eqn:cor-unknown-proof-1} and \eqref{eqn:cor-unknown-proof-3} into \eqref{eqn:cor-unknown-proof-0} completes the proof. $\qed$
\subsubsection{Proof of \Cref{theorem:variance_unknown_cov}} \label{proof:theorem:variance_unknown_cov}
\subsubsection*{Upper bound}
We show a slightly more general result for the upper bound. In particular, we relax the condition \eqref{eqn:condition_cov} to assume that, for each group $g \in \mathcal{G}$, 
\begin{equation}\label{eqn:condition_cov_general}
\rho_g^2 I_d \preceq \Sigma_g  \preceq \Rho_g^2 I_d.  
\end{equation}
We also drop the assumption $n_g \geq K_g^4d$,
and establish the upper bound 
\begin{equation}
\mathcal{V}_g(\lambda) \lesssim 
\frac{\Rho_g^2 }{\left ({\lambda}\rho_r^2/{{C}_r'}  + (1-\lambda) \rho_b^2/{{C}_b'} \right)^2} 
\left[  \lambda^2 \sigma_r^2 \Rho_r^2 \left( \frac{d}{n_r} +  \frac{{K_r^2} {d}^{3/2} }{n_r^{3/2}} \right) + (1-\lambda)^2 \sigma_b^2 \Rho_b^2  \left( \frac{d}{n_b} +  \frac{{K_b^2} {d}^{3/2} }{{n_b}^{3/2}} \right) \right].    
\end{equation}
To do so, first apply \Cref{prop:unknown_cov_upper_bound} under $\Sigma_g \preceq \Rho_g^2 I_d$ to obtain the following bound
\begin{equation*}
\begin{aligned}
\mathcal{V}_g(\lambda) &\leq  \Rho_g^2 \left(
\lambda^2 \frac{\sigma_r^2}{n_r }  \EE \left[\Tr{ \widehat{\Sigma}_{\lambda}^{-2}  \widehat{\Sigma}_r}\right] + (1-\lambda)^2 \frac{\sigma_b^2}{n_b} \EE \left[\Tr{\widehat{\Sigma}_{\lambda}^{-2}\widehat{\Sigma}_b}\right] \right)
\end{aligned}
\end{equation*}
Next, by Cauchy-Schwarz on the trace inner product and again on the expectation,
    \begin{equation*}
        \EE \left[\Tr{ \widehat{\Sigma}_{\lambda}^{-2}  \widehat{\Sigma}_g}\right] \leq \EE\left[ \sqrt{\Tr{\ \widehat{\Sigma}_{\lambda}^{-4}}} \sqrt{\Tr{ \widehat{\Sigma}_g^{2}}} \right] \leq \sqrt{ \EE\left[\Tr{\ \widehat{\Sigma}_{\lambda}^{-4}}\right] \EE\left[ \Tr{ \widehat{\Sigma}_g^{2}} \right]}.
    \end{equation*}

Under \Cref{assumption:subgaussian} that $X|G = g$ is $K_g$-sub-Gaussian, by \Cref{lemma:subgaussian_sample_covariance_squared_trace},
\begin{equation}\label{eq:trace_sample_covariance_neg_2_ub}
    \EE\left[ \Tr{ \widehat{\Sigma}_g^{2}} \right] \leq \Tr{\Sigma_g^2} + \frac{C K_g^4}{n_g}\Tr{\Sigma_g}^2 \lesssim \Rho_g^4\left( d +  \frac{{K_g^4} d^2}{n_g} \right)
\end{equation}
where the second inequality uses that $\Sigma_g \preceq \Rho_g^2 I_d$.
Next, by Weyl's inequality, 
\begin{equation}
\lambda_{\min}\left(\widehat{\Sigma}_\lambda\right)  \geq \lambda\cdot  \lambda_{\min}\left( \widehat{\Sigma}_r  \right) +  (1 - \lambda) \cdot \lambda_{\min}\left(\widehat{\Sigma}_b\right).
\end{equation}
Therefore, we have
\begin{equation}
\EE\left[  \lambda_{\min}\left( \widehat{\Sigma}_\lambda \right)^{-4}\right] \leq 
\min \left \{ 
\lambda^{-4} \EE\left[ \lambda_{\min}\left( \widehat{\Sigma}_r \right)^{-4}\right], (1-\lambda)^{-4} \EE\left[ \lambda_{\min}\left( \widehat{\Sigma}_b \right)^{-4}\right]
\right \}
\end{equation}    
Now, by \Cref{lemma:mourtada_small_ball_min_eval_qth_moment} with $q = 4$, we have
\begin{equation}\label{eq:trace_sigma_lambda_negative_4_ub}
\sqrt{\EE\left[\Tr{\widehat{\Sigma}_{\lambda}^{-4}}\right]}  \leq \sqrt{2d} 
\min \left \{ \lambda^{-2} {C_r'}^2 \rho_r^{-4}, (1 - \lambda)^{-2} {C_b'}^2 \rho_b^{-4} \right \}
\end{equation}
Combining the bounds \eqref{eq:trace_sigma_lambda_negative_4_ub} and \eqref{eq:trace_sample_covariance_neg_2_ub} and using $\min\{1/a,1/b\} \leq 2/(a+b)$, we reach the bound
\begin{align}
\EE \left[\Tr{ \widehat{\Sigma}_{\lambda}^{-2}  \widehat{\Sigma}_g}\right] &
\lesssim \Rho_g^2   
\frac{1}{\frac{\lambda^2}{{C}_r'^2} \rho_r^4 + \frac{(1-\lambda)^2}{{C}_b'^2} \rho_b^4} 
\left( d +  \frac{{K_g^2} d^{3/2}}{\sqrt{n_g}} \right) \\
& \lesssim \Rho_g^2   
\frac{1}{\left ({\lambda}\rho_r^2/{{C}_r'}  + (1-\lambda) \rho_b^2/{{C}_b'} \right)^2} 
\left( d +  \frac{{K_g^2} d^{3/2}}{\sqrt{n_g}} \right).
\end{align}  
As a result, the variance terms satisfy the following bound
\begin{equation*}
\begin{aligned}
&\Rho_g^2 \left(
\lambda^2 \frac{\sigma_r^2}{n_r }  \EE \left[\Tr{ \widehat{\Sigma}_{\lambda}^{-2}  \widehat{\Sigma}_r}\right] + (1-\lambda)^2 \frac{\sigma_b^2}{n_b} \EE \left[\Tr{\widehat{\Sigma}_{\lambda}^{-2}\widehat{\Sigma}_b}\right] \right)\\
&\lesssim 
\frac{\Rho_g^2 }{\left ({\lambda}\rho_r^2/{{C}_r'}  + (1-\lambda) \rho_b^2/{{C}_b'} \right)^2} 
\left[  \lambda^2 \sigma_r^2 \Rho_r^2 \left( \frac{d}{n_r} +  \frac{{K_r^2} {d}^{3/2} }{n_r^{3/2}} \right) + (1-\lambda)^2 \sigma_b^2 \Rho_b^2  \left( \frac{d}{n_b} +  \frac{{K_b^2} {d}^{3/2} }{{n_b}^{3/2}} \right) \right]. 
\end{aligned}
\end{equation*}
Under the assumption $n_g \geq K_g^4 d$, the term $K_g^2 d^{3/2}/n_g^{3/2}$ is bounded by $d/n_g$ which completes the proof of the upper bound. 
\subsubsection*{Lower bound}
We obtain the lower bound via Assouad's Lemma (\Cref{lemma:assouad}). 
\paragraph{Perturbed parameter structure:}
For $g \in \mathcal{G}$, set 
\begin{equation*}
h_g^2 = \frac{\sigma_g^2}{4 n_g \rho_g^2},\quad \beta_g^{(\xi_g)} \coloneq h_g \xi_g 
\end{equation*}
where $\xi_g \in \Xi^d$. The condition on $n_g$ ensures $\|\beta_g^{(\xi_g)}\| \leq B$.

Consider the $2d$-dimensional hypercube $\Xi \coloneq \Xi^d \times \Xi^d$. Let $\PP_g^{(\xi_g)}$ denote the law of a single observation from group $g$ under parameter $\beta_g^{(\xi_g)}$. The full data law is then $\PP^{(\xi)} = (\PP_r^{(\xi_r)})^{\otimes n_r} \otimes (\PP_b^{(\xi_b)})^{\otimes n_b}$.

\paragraph{KL divergence bound:}
The following result provides an upper bound on the KL divergence between the distributions under parameters $\xi, \xi' \in \Xi$ differing in only a single coordinate. 

\begin{claim}\label{claim:unknown_lower_bound_variance_KL}
    Let $\xi$ and $\xi'$ differ only in coordinate $j \in [d]$ (i.e., a coordinate corresponding to group $r$) so that $\xi_{r, j} = - \xi'_{r, j}$ and all other coordinates agree.  Let $\PP_r^{(\xi_r)}$ and $\PP_r^{(\xi_r')}$ denote the laws of a single red group observation $(X^{(r)}, Y^{(r)})$ under red group parameters $\xi_r, \xi_r'$, respectively. Then
    \begin{equation*}
        \kldiv{\PP_r^{(\xi_r)}}{\PP_r^{(\xi_r')}}  \leq \frac{2 h_r^2 \rho_r^2}{\sigma_r^2}
    \end{equation*}
\end{claim}
\begin{proof}[Proof of \Cref{claim:unknown_lower_bound_variance_KL}]
    For a single sample $X^{(r)} = \left( X^{(r)}_1, \dots, X^{(r)}_d\right)$, the two neighboring models satisfy 
    \begin{equation*}
        Y^{(r)}|X^{(r)}, \xi \sim \NN({X^{(r)}}^\top \beta_r^{(\xi_r)}, \sigma_r^2), \quad Y^{(r)}|X^{(r)}, \xi' \sim \NN({X^{(r)}}^\top \beta_r^{(\xi'_r)}, \sigma_r^2).
    \end{equation*}
    Hence, 
    \begin{equation*}
        {X^{(r)}}^\top \left(\beta_r^{(\xi_r)}-\beta_r^{(\xi'_r)}\right) = \pm 2h_r X_j^{(r)}. 
    \end{equation*}
    Recalling well-known results for the KL divergence of  Gaussians~\citep[see, e.g.,][]{zhang2023properties}, it follows that 
    \begin{equation*}
        \kldiv{\PP_r^{(\xi_r)}}{\PP_r^{(\xi_r')} \mid X^{(r)}} = \frac{2 h_r^2 (X_j^{(r)})^2}{\sigma_r^2}.
    \end{equation*}
    Then, for $X_j^{(r)} \sim \NN(0, \rho_r^2)$ with $\EE[(X_j^{(r)})^2] = \rho_r^2$, we observe that
    \begin{equation*}
        \kldiv{\PP_r^{(\xi_r)}}{\PP_r^{(\xi_r')}} = \frac{2 h_r^2 \rho_r^2}{\sigma_r^2}.
    \end{equation*}
\end{proof}

By \Cref{claim:unknown_lower_bound_variance_KL}, over $n_r$ samples, 
\begin{equation*}
    \kldiv{(\PP_r^{(\xi_r)})^{\otimes n_r}}{(\PP_r^{(\xi'_r)})^{\otimes n_r}}  \leq \frac{2 h_r^2 \rho_r^2 n_r}{\sigma_r^2} = \frac{1}{2}.
\end{equation*}
By Pinsker's inequality and since $\xi$ and $\xi'$ differ only in a red-group coordinate, 
\begin{equation*}
    \tvdist{\PP^{(\xi)}, \PP^{(\xi')}} \leq \frac{1}{2}.
\end{equation*}

\paragraph{Parameter separation:} 
Define $\rho_\lambda^2 = \lambda \rho_r^2 + (1-\lambda) \rho_b^2$. For the same $\xi$ and $\xi'$ as defined in \Cref{claim:unknown_lower_bound_variance_KL}, 
\begin{equation}\label{eq:group_r_separation}
    \left| \langle e_j, \beta_\lambda^{(\xi)}  - \beta_\lambda^{(\xi')}  \rangle\right|  = \frac{2 \lambda \rho_r^2 h_r}{\rho_\lambda^2}.
\end{equation}

After an analogous computation for a single-coordinate group $b$ perturbation, 
\begin{equation}\label{eq:group_b_separation}
    \left| \langle e_j, \beta_\lambda^{(\xi)}  - \beta_\lambda^{(\xi')}  \rangle \right| = \frac{2 (1-\lambda) \rho_b^2 h_b}{\rho_\lambda^2}. 
\end{equation}
Substitute the definitions of $h_r$ and $h_b$ into \eqref{eq:group_r_separation} and \eqref{eq:group_b_separation}, respectively, and define 
\begin{equation*}
    a_r^2 \coloneq \frac{\sigma_r^2 \lambda^2 \rho_r^2}{n_r \rho_\lambda^4}, \quad a_b^2 \coloneq \frac{\sigma_b^2 (1-\lambda)^2 \rho_b^2}{n_b \rho_\lambda^4}.
\end{equation*}

\paragraph{Assouad's Lemma:} 
Fix $\xi_b$. An application of Assouad's lemma (as stated in \Cref{lemma:assouad}) to the $d$-dimensional signed hypercube gives a lower bound $d a_r^2/8$ for the minimax risk over $\Xi$. Similarly, fixing $\xi_r$ gives the lower bound $d a_b^2/8$. Using $\max\{a, b\} \geq \frac{1}{2}(a + b)$ yields the bound
\begin{equation*}
\inf_{\widehat{\beta}} \sup_{\xi \in \Xi} \EE_\xi\left[ \norm{\widehat{\beta} - \beta_\lambda^{(\xi)}}_2^2 \right] \geq \frac{1}{16} \frac{ d}{\rho_\lambda^4} \left(\frac{\lambda^2 \sigma_r^2 \rho_r^2}{n_r} + \frac{(1 - \lambda)^2 \sigma_b^2 \rho_b^2}{n_b} \right).
\end{equation*}
and the final bound follows by observing 
\begin{equation*}
    \norm{\widehat{\beta}_\lambda - \beta_\lambda}_{\Sigma_g}^2 \geq \rho_g^2 \norm{\widehat{\beta}_\lambda - \beta_\lambda}_2^2. \quad \qed
\end{equation*}

\subsubsection{Proof of \Cref{theorem:bias_unknown_cov}}\label{proof:theorem:bias_unknown_cov}
\subsubsection*{Upper Bound}
As for the variance term, we prove a more general result for the upper bound. In particular, we once again relax condition \eqref{eqn:condition_cov} to the more general condition \eqref{eqn:condition_cov_general} and establish the upper bound 
\begin{equation}
\mathcal{B}_g(\lambda) \lesssim \frac{ \lambda^2 ( 1 - \lambda)^2 ~\Rho_g^2 ~\Rho_r^4 ~\Rho_b^4}{\left ({\lambda}\rho_r^2/{{C}_r'}  + (1-\lambda) \rho_b^2/{{C}_b'} \right)^2\rho_\lambda^4} \cdot d \cdot \left(\frac{K_r^4}{n_r} +  \frac{K_b^4}{n_b} \right) \norm{\beta_r - \beta_b}^2, 
\end{equation}
with $\rho_\lambda^2 = \lambda \rho_r^2 + (1-\lambda) \rho_b^2$.

To do so, first apply \Cref{prop:unknown_cov_upper_bound} under $\Sigma_g \preceq \Rho_g^2 I_d$ to obtain the following bound: 
\begin{equation}
\mathcal{B}_g(\lambda) \leq 
\Rho_g^2 \EE \left[ \norm{\widehat{\Sigma}_\lambda^{-1} \left( \lambda \widehat{\Sigma}_r (\beta_r - \beta_\lambda) + ( 1 - \lambda) \widehat{\Sigma}_b (\beta_b - \beta_\lambda)\right) }^2  \right].    
\end{equation} 
Next, subtracting the observation \eqref{eqn:upper_bound_risk_decomp_zero_term}, 
from its empirical counterpart therefore yields, 
\begin{equation}\label{eq:bias_define_A}
\tilde{A} \coloneq \lambda \widehat{\Sigma}_r (\beta_r - \beta_\lambda) + ( 1 - \lambda) \widehat{\Sigma}_b (\beta_b - \beta_\lambda) =  \lambda (\widehat{\Sigma}_r - \Sigma_r) (\beta_r - \beta_\lambda) + ( 1 - \lambda) (\widehat{\Sigma}_b - \Sigma_b) (\beta_b - \beta_\lambda).
\end{equation}
By triangle inequality,
\begin{equation}\label{eq:bias_beta_sigma_empirical_sigma}
\norm{\tilde{A}} \leq \lambda \norm{  \widehat{\Sigma}_r - \Sigma_r} \norm{ \beta_r - \beta_\lambda} + ( 1 - \lambda) \norm{  \widehat{\Sigma}_b - \Sigma_b} \norm{ \beta_b - \beta_\lambda }.
\end{equation}
Next, notice that we can express
\begin{subequations}\label{eq:beta-diffs}
\begin{align}
\beta_r - \beta_\lambda
    &= (1-\lambda)\,\Sigma_\lambda^{-1}\Sigma_b(\beta_r - \beta_b)\label{eq:beta-diffs_r}\\
    \beta_b - \beta_\lambda
    &= \lambda\,\Sigma_\lambda^{-1}\Sigma_r(\beta_b - \beta_r).\label{eq:beta-diffs_b}
\end{align}
\end{subequations}
Substituting \eqref{eq:beta-diffs_r} and \eqref{eq:beta-diffs_b} into \eqref{eq:bias_beta_sigma_empirical_sigma} and applying the assumptions $\norm{\Sigma_g}\leq \Rho_g^2$ and $\lambda_{\min}\left({\Sigma_\lambda}\right) \geq \rho_\lambda^2$, we reach
\begin{equation}\label{eq:A_bound_before_raising_fourth}
\begin{aligned}
\norm{\tilde{A}} \leq \frac{\lambda( 1 - \lambda)}{\rho_\lambda^2} \left(  \Rho_b^2  \norm{  \widehat{\Sigma}_r - \Sigma_r} + \Rho_r^2 \norm{  \widehat{\Sigma}_b - \Sigma_b}  \right) \norm{\beta_r - \beta_b}.
\end{aligned}
\end{equation}
Raising \eqref{eq:A_bound_before_raising_fourth} to the fourth power and using the inequality $(a + b)^4 \leq 8(a^4 + b^4)$,
\begin{equation*}
    \begin{aligned}
        \norm{\tilde{A}}^4 \leq \frac{8\lambda^4( 1 - \lambda)^4 }{\rho_\lambda^8} \left( \Rho_b^8  \norm{  \widehat{\Sigma}_r - \Sigma_r}^4 + \Rho_r^8 \norm{  \widehat{\Sigma}_b - \Sigma_b}^4  \right) \norm{\beta_r - \beta_b}^4.
    \end{aligned}
\end{equation*}
Taking an expectation and square root successively then using that $\sqrt{a + b} \leq \sqrt{a} + \sqrt{b},$
\begin{equation*}
    \begin{aligned}
        \sqrt{\EE\left[ \norm{\tilde{A}}^4\right]} \leq \frac{2 \sqrt{2}\lambda^2( 1 - \lambda)^2}{\rho_\lambda^4} \left( \Rho_b^4  \EE \left[ \norm{  \widehat{\Sigma}_r - \Sigma_r}^4\right]^{1/2} + \Rho_r^4 \EE\left[ \norm{  \widehat{\Sigma}_b - \Sigma_b}^4 \right]^{1/2} \right) \norm{\beta_r - \beta_b}^2.
    \end{aligned}
\end{equation*}    
Applying the sub-Gaussian covariance estimation bound \Cref{lemma:sub_gaussian_covariance_bound} to each term $\EE\left[ \norm{  \widehat{\Sigma}_g - \Sigma_g}^4 \right]$, we obtain  
\begin{equation}\label{eq:bias_sqrt_norm_A}
\begin{aligned}
    \sqrt{\EE\left[ \norm{\tilde{A}}^4\right]} \leq \frac{2 \sqrt{2} C \lambda^2( 1 - \lambda)^2}{\rho_\lambda^4} \Rho_r^4 \Rho_b^4 d \left( \frac{K_r^4}{n_r} + \frac{K_b^4}{n_b} \right) \norm{\beta_r - \beta_b}^2.
\end{aligned}
\end{equation}
Next, recall the property 
\begin{equation}\label{eqn:bias_spectral_bound_upper}
         \norm{ \widehat{\Sigma}_\lambda^{-1} \tilde{A}}^2  
            \leq  \lambda_{\min}\left( \widehat{\Sigma}_\lambda \right)^{-2} \norm{\tilde{A}}^2.
\end{equation}
By Cauchy-Schwarz and \eqref{eqn:bias_spectral_bound_upper}, 
\begin{equation*}
    \begin{aligned}
        \EE \left[ \norm{\widehat{\Sigma}_\lambda^{-1} \tilde{A} }^2  \right] \leq \sqrt{\EE \left[ \lambda_{\min}\left( \widehat{\Sigma}_\lambda \right)^{-4}\right]} \sqrt{\EE \norm{\tilde{A}}^4}.
    \end{aligned}
\end{equation*}
By Weyl's inequality, convexity, and~\Cref{lemma:mourtada_small_ball_min_eval_qth_moment} (with $q = 4$) as in the preceding analysis of the variance term, we have
\begin{equation}\label{eq:bias_lambda_min_sigma_lambda}
\sqrt{\EE \left[  \lambda_{\min}\left( \widehat{\Sigma}_\lambda \right)^{-4}\right]} \lesssim \frac{1}{\left ({\lambda}\rho_r^2/{{C}_r'}  + (1-\lambda) \rho_b^2/{{C}_b'} \right)^2}.
\end{equation}
Combining the bounds \eqref{eq:bias_sqrt_norm_A} and \eqref{eq:bias_lambda_min_sigma_lambda}, we conclude the proof of the upper bound on the bias term.
\subsubsection*{Lower bound}

Fix $\rho_r^2, \rho_b^2 > 0$ and $\lambda \in (0, 1)$. Assume $\beta_r \neq \beta_b$; when $\beta_r = \beta_b$, the bias lower bound is simply zero. For each coordinate $i \in [d]$, set 
\begin{subequations}\label{eq:def_uv_bias_lower_bound}
    \begin{align}
    v &\coloneq \frac{\beta_r - \beta_b}{\left\|\beta_r - \beta_b\right\|_2},
    \label{eq:def_uv:v}\\
    u_i &=
    \begin{cases}
    \frac{e_i + v}{\left\|e_i + v\right\|_2}, & \text{if } e_i^\top v \ge 0,\\
    \frac{e_i - v}{\left\|e_i - v\right\|_2}, & \text{if } e_i^\top v < 0.
    \end{cases}
    \label{eq:def_uv:ui}
    \end{align}
\end{subequations}

\paragraph{Perturbed covariance structure for group $r$:} For a Rademacher vector $\xi \in \{-1, 1\}^d$, set 
\begin{equation*}
    \Sigma_r^{(\xi)} \coloneq \rho_r^2 I_d + h_r \sum_{i = 1}^d \xi_i u_i u_i^\top, \quad \Sigma_b \coloneq \rho_b^2 I_d 
\end{equation*}
where the group $r$ perturbation level $h_r$ is given by
\begin{equation}\label{eq:bias_lower_hr_define}
    h_r  = \frac{2\rho_r^{2}}{5\sqrt{n_r}} \leq \frac{\rho_r^2}{10 d},
\end{equation}
where the inequality follows from $n_r \geq 16d^2$.
Hence, for $d \geq 1$, for every $\xi$, 
\begin{equation}\label{eq:bias_lower_bound_loewner_order}
    0.9 \cdot \rho_r^2 I_d \preceq \Sigma_r^{(\xi)} \preceq 1.1 \rho_r^2 I_d. 
\end{equation}

\paragraph{KL divergence bound:} Define $\PP^{(\xi)} \coloneq \NN(0, \Sigma_r^{(\xi)})$. The following result provides an upper bound on the KL divergence between the distributions under parameters $\xi, \xi' \in \Xi^d$ differing in only a single coordinate.
\begin{claim}\label{claim:unknown_lower_bound_bias_KL2}
    Let $\xi$ and $\xi'$ differ only in coordinate $i$ where $\xi_i = - \xi_i'$.  Then
    \begin{equation*}
        \kldiv{\PP^{(\xi)}}{\PP^{(\xi')}}  \leq \frac{25}{16} \frac{h_r^2}{\rho_r^4}.
    \end{equation*}
\end{claim}
\begin{proof}[Proof of \Cref{claim:unknown_lower_bound_bias_KL2}]
    Observe that 
    \begin{equation*}
        \kldiv{\PP^{(\xi)}}{\PP^{(\xi')}} = \frac{1}{2} \left[\log \frac{\det \left( \Sigma_r - h_r u_i u_i^\top\right) }{\det \left(\Sigma_r + h_r u_i u_i^\top\right) }  - d + \Tr{\left( \Sigma_r - h_r u_i u_i^\top \right)^{-1} \left( \Sigma_r + h_r u_i u_i^\top \right) }\right]
    \end{equation*}
    Define $\alpha = h_r u_i^\top \Sigma_r^{-1} u_i$. By the matrix determinant lemma, $\det \left( \Sigma_r - h_r u_i u_i^\top\right) = \left(1 - \alpha \right) \det (\Sigma_r)$. 
    By the Sherman-Morrison lemma, $\left( \Sigma_r - h_r u_i u_i^\top \right)^{-1} = \Sigma_r^{-1} + \frac{h_r}{1 - \alpha} (\Sigma_r^{-1} u_i u_i^\top \Sigma_r^{-1})$. Substituting and simplifying, we reach
    \begin{equation*}
        \kldiv{\PP^{(\xi)}}{\PP^{(\xi')}} = \frac{1}{2} \left[\log \frac{\left(1 - \alpha \right)  }{\left(1 + \alpha \right) }   +  \frac{2 \alpha}{1 - \alpha}\right].
    \end{equation*}
    Observe $\alpha \geq 0$ since $\Sigma_r$ is PSD. By \eqref{eq:bias_lower_bound_loewner_order}, $\Sigma_r^{-1} \preceq \frac{1}{0.9 \rho_r^2} I_d$ thus $\alpha \leq \frac{h_r}{0.9 \rho_r^2}$. By \eqref{eq:bias_lower_hr_define}, it follows that $0 \leq \alpha \leq \frac{1}{9 d} \leq \frac {1}{9} $ for $d \geq 1$. Define $\gamma \coloneq \frac{2 \alpha}{1 - \alpha}$ where $\gamma \in [0, 1/4]$. Using the fact that, for $\gamma \in [0, 1/4]$, $- \log(1 + \gamma) + \gamma \leq \frac{\gamma^2}{2}$ and the fact that, for $\alpha \in (0, 1/9)$, $\frac{\alpha^2}{(1 - \alpha)^2}\leq \frac{81}{64} \alpha^2$,
    \begin{equation*}
        \kldiv{\PP^{(\xi)}}{\PP^{(\xi')}} = \frac{1}{2}\left[ - \log(1 + \gamma) + \gamma \right] \leq \frac{\gamma^2}{4} = \frac{\alpha^2}{(1 - \alpha)^2} \leq \frac{81}{64} \alpha^2 \leq \frac{81}{64} h_r^2 \norm{\Sigma_r^{-1}}_2^2 \leq \frac{25}{16} \frac{h_r^2}{\rho_r^4}
    \end{equation*}
    where the last inequality follows by \eqref{eq:bias_lower_bound_loewner_order}.
\end{proof}

By \Cref{claim:unknown_lower_bound_bias_KL2}, over $n_r$ independent samples from group $r$, 
\begin{equation*}
    \kldiv{\PP^{(\xi) \otimes n_r}}{\PP^{(\xi') \otimes n_r}} \leq \frac{25}{16} \frac{h_r^2 n_r}{\rho_r^4}.
\end{equation*}
By \eqref{eq:bias_lower_hr_define} and an application of Pinsker's inequality, we reach the bound 
\begin{equation}\label{eq:bias_lower_tv}
    \tvdist{\PP^{(\xi) \otimes n_r}, \PP^{(\xi') \otimes n_r}} \leq \frac{1}{2}.
\end{equation}

\paragraph{Parameter separation:} Define ${A^{(\xi)}} \coloneq \left( \lambda \Sigma_r^{(\xi)} + (1 - \lambda) \Sigma_b\right)$. The target parameter under $\xi$ can be expressed as
\begin{equation*}
    \beta_\lambda^{(\xi)} = \left( {A^{(\xi)}}\right)^{-1} \left( \lambda \Sigma_r^{(\xi)} \beta_r + (1-\lambda) \Sigma_b \beta_b\right) = \beta_r + \left( {A^{(\xi)}}\right)^{-1} (1 -\lambda) \Sigma_b(\beta_b - \beta_r).
\end{equation*}

\begin{claim}\label{claim:param_separation_unknown_cov_bias_lower}
    
    Let $\xi$ and $\xi'$ differ only in coordinate $i$ where $\xi_i = - \xi_i'$. Then, 
    \begin{equation*}
        \left | \left \langle e_i, \beta_\lambda^{(\xi)} -  \beta_\lambda^{(\xi')} \right \rangle  \right | \gtrsim \lambda (1 -\lambda) h_r \rho_b^2 \rho_\lambda^{-4} \norm{\beta_r - \beta_b }.  
    \end{equation*}
\end{claim}
\begin{proof}[Proof of \Cref{claim:param_separation_unknown_cov_bias_lower}]
    By the Sherman-Morrison lemma, 
    \begin{equation*}
        {A^{(\xi)}}^{-1} - {A^{(\xi')}}^{-1} = \frac{2 \lambda h_r }{1 + 2 \lambda h_r u_i^\top {A^{(\xi)}}^{-1} u_i} {A^{(\xi)}}^{-1} u_i u_i^\top {A^{(\xi)}}^{-1}.
    \end{equation*}
    Hence, we can express
    \begin{equation*}
        \beta_\lambda^{(\xi)} -  \beta_\lambda^{(\xi')} =  \left(   {A^{(\xi)}}^{-1} -   {A^{(\xi')}}^{-1} \right)(1 -\lambda) \Sigma_b(\beta_b - \beta_r).
    \end{equation*}
    Then, substitution of $v$ defined in \eqref{eq:def_uv:v}, 
    \begin{equation*}
        \langle e_i, \beta_\lambda^{(\xi)} -  \beta_\lambda^{(\xi')} \rangle = 2 \lambda (1 -\lambda) h_r \norm{\Sigma_b} \norm{\beta_r - \beta_b } \frac{e_i^\top {A^{(\xi)}}^{-1} u_i u_i^\top {A^{(\xi)}}^{-1} v}{1 + 2 \lambda h_r u_i^\top {A^{(\xi)}}^{-1} u_i}.
    \end{equation*}
    By the Loewner order relationship \eqref{eq:bias_lower_bound_loewner_order}, 
    \begin{equation*}
        \lambda_{\min}({A^{(\xi)}}) \geq 0.9 \lambda \rho_r^2 + (1 - \lambda) \rho_b^2 \geq  0.9 \lambda \rho_r^2
    \end{equation*}
    The denominator is bounded, therefore, by:
    \begin{equation*}
        | 1 + 2 \lambda h_r u_i^\top {A^{(\xi)}}^{-1} u_i| \leq \left| 1 + 2 \lambda h_r \norm{{A^{(\xi)}}^{-1}} \right| \leq \left| 1 + \frac{2 \lambda h_r }{\left( 0.9 \lambda \rho_r^2 \right)} \right| \leq  \frac{11}{9}
    \end{equation*}
    where the last inequality follows by assuming $h_r \leq \frac{\rho_r^2}{10 d}$ and $d \geq 1$.

    Next, we lower bound the numerator. Fix a coordinate $i \in [d]$, and observe that we can write the decomposition ${A^{(\xi)}}^{-1} e_i = c_1 e_i + c_2 w_i$ 
    where $w_i \in \RR^d$ is orthogonal to $e_i$ and satisfies $\norm{w_i} = 1$. By definition, $c_1 = e_i^\top {A^{(\xi)}}^{-1} e_i$, and, using \eqref{eq:bias_lower_bound_loewner_order},
    \begin{equation}\label{eq:bias_lower_c1_bound}
         \frac{1}{1.1} \rho_\lambda^{-2}  \leq \left( 1.1 \lambda \rho_r^2 + (1 - \lambda) \rho_b^2 \right)^{-1} \leq c_1  \leq \left( 0.9 \lambda \rho_r^2 + (1 - \lambda) \rho_b^2 \right)^{-1} \leq \frac{1}{0.9} \rho_\lambda^{-2}.
    \end{equation}
    Moreover, by construction and again by \eqref{eq:bias_lower_bound_loewner_order}, 
    \begin{equation}\label{eq:bias_lower_c1_c2_bound}
        c_1^2 + c_2^2 = \norm{{A^{(\xi)}}^{-1} e_i}^2 \leq \norm{{A^{(\xi)}}^{-1}}^2 \leq \left( 0.9 \lambda \rho_r^2 + (1 - \lambda) \rho_b^2 \right)^{-2} \leq \frac{1}{0.9^2} \rho_\lambda^{-4}.
    \end{equation}
    Combining \eqref{eq:bias_lower_c1_bound} and \eqref{eq:bias_lower_c1_c2_bound}, we have
    \begin{equation}\label{eq:bias_lower_c2_upper}
        c_2 \leq \left( 0.9^{-2} - 1.1^{-2} \right)^{0.5} \rho_\lambda^{-2}.
    \end{equation}
    Next, noting that the vectors $u_i$ and $v$ satisfy $|u_i^\top e_i|, |u_i^\top v| \geq \frac{1}{\sqrt{2}}$, we have 
    \begin{equation*}
        |e_i^\top {A^{(\xi)}}^{-1} u_i| = |c_1 e_i^\top u_i + c_2 w_i^\top u_i| \geq \frac{c_1}{\sqrt{2}} - |c_2|, \quad |u_i^\top {A^{(\xi)}}^{-1} v| \geq \frac{c_1}{\sqrt{2}} - |c_2|.
    \end{equation*}
    Combining \eqref{eq:bias_lower_c1_bound} and \eqref{eq:bias_lower_c2_upper}, we see that
    \begin{equation*}
        \frac{c_1}{\sqrt{2}} - |c_2| \geq \frac{1}{300} \rho_\lambda^{-2}.
    \end{equation*}
    Altogether, this yields the parameter separation lower bound:
    \begin{equation*}
        \langle e_i, \beta_\lambda^{(\xi)} -  \beta_\lambda^{(\xi')} \rangle^2 \geq \left(\frac{18}{11 \cdot 300^2}\right)^2 \lambda^2 (1 -\lambda)^2 h_r^2 \rho_b^4 \rho_\lambda^{-8} \norm{\beta_r - \beta_b }^2.
    \end{equation*}
\end{proof}

\paragraph{Assouad's Lemma:} Let $\widehat{\beta}$ be any estimator. Assouad's lemma (see \Cref{lemma:assouad}) applied to the $d$-dimensional hypercube $\Xi^d$, given the results \Cref{claim:param_separation_unknown_cov_bias_lower} and \eqref{eq:bias_lower_tv},  yields
\begin{equation}\label{eq:bias_lower_red_perturbation_final_bound}
    \sup_{\xi \in \Xi^d} \EE_{\xi}\left[ \norm{\widehat{\beta} - \beta}_2^2 \right]\geq  \left(\frac{18}{11 \cdot 300^2}\right)^2 \frac{2}{25} \lambda^2 (1 -\lambda)^2 \frac{\rho_r^{4} \cdot \rho_b^4 \cdot d}{\rho_\lambda^{8}  \cdot n_r}  \norm{\beta_r - \beta_b }^2.
\end{equation}

\paragraph{Symmetric perturbation for group $b$:} Repeat the procedure so far but instead perturb the covariance of group $b$. For a Rademacher vector $\zeta \in \{-1, 1\}^d$, set 
\begin{equation*}
    \Sigma_r \coloneq \rho_r^2 I_d, \quad \Sigma_b^{(\zeta)} \coloneq \rho_b^2 I_d + h_b \sum_{i = 1}^d \zeta_i u_i u_i^\top, \quad h_b = \frac{2 \rho_b^2}{5 \sqrt{n_b}}.
\end{equation*}
This construction yields the analogous lower bound
\begin{equation}\label{eq:bias_lower_blue_perturbation_final_bound}
    \sup_{\zeta \in \Xi^d} \EE_{\zeta}\left[ \norm{\widehat{\beta} - \beta}_2^2 \right]\geq  \left(\frac{18}{11 \cdot 300^2}\right)^2 \frac{2}{25} \lambda^2 (1 -\lambda)^2 \frac{\rho_r^{4} \cdot \rho_b^4 \cdot d}{\rho_\lambda^{8}  \cdot n_b}  \norm{\beta_r - \beta_b }^2.
\end{equation}
Since we may perturb either group, we may take the maximal lower bound and use the fact that $\max(a, b) \geq \frac{1}{2}(a + b)$ to obtain the bound:
\begin{equation*}
    \sup_{\xi, \zeta \in \Xi^d}\EE\left[ \norm{\widehat{\beta} - \beta}_2^2 \right] \geq C \lambda^2 (1 -\lambda)^2 \frac{\rho_r^{4} \cdot \rho_b^4 \cdot d}{\rho_\lambda^{8}}  \left( \frac{1}{n_r} + \frac{1}{n_b} \right) \norm{\beta_r - \beta_b}^2
\end{equation*}
where $C = \left(\frac{18}{11 \cdot 300^2}\right)^2 \frac{1}{25}$.

Finally, using that $\norm{\widehat{\beta}_\lambda - \beta_\lambda}^2_{\Sigma_g} \geq \rho_g^2 \norm{\widehat{\beta}_\lambda - \beta_\lambda}_2^2$, we reach the bound
\begin{equation*}
     \EE \left[ \norm{ \widehat{\beta}_{\lambda} - \beta_\lambda }_{\Sigma_g}^2\right]  \geq \lambda^2 (1 -\lambda)^2 \frac{\rho_g^2 \cdot \rho_r^{4} \cdot \rho_b^4 \cdot d}{\rho_\lambda^{8}}  \left( \frac{1}{n_r} + \frac{1}{n_b} \right) \norm{\beta_r - \beta_b}^2. \quad \qed
\end{equation*}
\subsubsection{Proof of \Cref{proposition:cross-term-unknown-cov}} \label{proof:proposition:cross-term-unknown-cov}
First, note that, using \eqref{eq:beta-diffs} from the proof of \Cref{theorem:bias_unknown_cov}, we have 
\begin{equation} \label{eqn:proof_cross_bias_0}
\left | \EE \left [(\widehat{\beta}_\lambda-\beta_\lambda)^\top \Sigma_r (\beta_\lambda-\beta_r) \right] \right | 
\lesssim  \frac{(1-\lambda) \rho_r^2 \rho_b^2}{\lambda \rho_r^2 + (1-\lambda) \rho_b^2}  
~\|\beta_r - \beta_b\|~
\left \| \EE \left [\widehat{\beta}_\lambda-\beta_\lambda \right] \right \|. 
\end{equation}
Next, we recall from the proof of \Cref{lemma:unknown_covariance_exact_risk} that
\begin{equation} \label{eqn:proof_cross_bias_1}
\EE \left [\widehat{\beta}_\lambda-\beta_\lambda  \right]   = \EE \left [
\widehat{\Sigma}_\lambda^{-1} \left (\lambda \widehat{\Sigma}_r \beta_r + (1-\lambda) \widehat{\Sigma}_b \beta_b \right) - \beta_\lambda
\right] , 
\end{equation}
which, as described in the proof of \Cref{prop:unknown_cov_upper_bound}, can be further cast as
\begin{equation} \label{eqn:proof_cross_bias_2}
\EE \left [\widehat{\beta}_\lambda-\beta_\lambda  \right]   = 
\EE \left [
\widehat{\Sigma}_\lambda^{-1} \left (\lambda \left(  \widehat{\Sigma}_r - \Sigma_r \right)(\beta_r - \beta_\lambda) + (1 - \lambda) \left( \widehat{\Sigma}_b - \Sigma_b \right)(\beta_b - \beta_\lambda) \right)
\right].
\end{equation}
By the Woodbury matrix identity, we can write 
\begin{equation} \label{eqn:proof_cross_bias_3}
\widehat{\Sigma}_\lambda^{-1} = \Sigma_\lambda^{-1} - \Sigma_\lambda^{-1} (\widehat{\Sigma}_\lambda-\Sigma_\lambda) \widehat{\Sigma}_\lambda^{-1}.   
\end{equation}
By substituting the above identity into \eqref{eqn:proof_cross_bias_2}, and using the fact that $\EE[\widehat{\Sigma}_g - \Sigma_g] = 0$ for both groups, we obtain
\begin{align} \label{eqn:proof_cross_bias_4}
\EE \left [\widehat{\beta}_\lambda-\beta_\lambda  \right]  =
- \EE \left [
\Sigma_\lambda^{-1} (\widehat{\Sigma}_\lambda-\Sigma_\lambda) \widehat{\Sigma}_\lambda^{-1} \left (\lambda \left(  \widehat{\Sigma}_r - \Sigma_r \right)(\beta_r - \beta_\lambda) + (1 - \lambda) \left( \widehat{\Sigma}_b - \Sigma_b \right)(\beta_b - \beta_\lambda) \right)
\right].
\end{align}
Next, we bound the two terms on the right-hand side separately. First, notice that 
\begin{align} \label{eqn:proof_cross_bias_5}
& \left \| \EE \left [
\Sigma_\lambda^{-1} (\widehat{\Sigma}_\lambda-\Sigma_\lambda) \widehat{\Sigma}_\lambda^{-1} \left(  \widehat{\Sigma}_r - \Sigma_r \right)(\beta_r - \beta_\lambda)
\right ] \right \| \nonumber \\
&~~ \lesssim
\frac{(1-\lambda) \rho_b^2}{\left(\lambda \rho_r^2 + (1-\lambda) \rho_b^2 \right)^2} ~ \|\beta_r - \beta_b\| ~
\left \| \EE \left [
(\widehat{\Sigma}_\lambda-\Sigma_\lambda) \widehat{\Sigma}_\lambda^{-1} \left(  \widehat{\Sigma}_r - \Sigma_r \right)
\right ] \right \| \nonumber \\
& ~~ 
\lesssim
\frac{(1-\lambda) \rho_b^2}{\left(\lambda \rho_r^2 + (1-\lambda) \rho_b^2 \right)^2} ~ \|\beta_r - \beta_b\| ~
\sqrt{\EE \left[ \left \| \widehat{\Sigma}_\lambda^{-1}  \right\|^2 \right]} 
\left( \EE \left[ \left \| \Sigma_r - \widehat{\Sigma}_r \right\|^4 \right] \right)^{1/4}
\left( \EE \left[ \left \| \Sigma_\lambda - \widehat{\Sigma}_\lambda \right\|^4 \right] \right)^{1/4} \nonumber \\
& ~~ \lesssim
\frac{(1-\lambda) \rho_b^2}{\left(\lambda \rho_r^2 + (1-\lambda) \rho_b^2 \right)^2}
~ \|\beta_r - \beta_b\| ~
\frac{K_r^2 \rho_r^2 {\sqrt{d}}/{\sqrt{n_r}}}{{\lambda}\rho_r^2/{{C}_r'}  + (1-\lambda) \rho_b^2/{{C}_b'}} 
\left(
\lambda K_r^2 \rho_r^2 \frac{\sqrt{d}}{\sqrt{n_r}} + (1-\lambda) K_b^2 \rho_b^2 \frac{\sqrt{d}}{\sqrt{n_b}}
\right ), 
\end{align}
where the last inequality follows from \Cref{lemma:mourtada_small_ball_min_eval_qth_moment} (with $q = 4$) and \Cref{lemma:sub_gaussian_covariance_bound}. Similarly, we can show
\begin{align} \label{eqn:proof_cross_bias_6}
& \left \| \EE \left [
\Sigma_\lambda^{-1} (\widehat{\Sigma}_\lambda-\Sigma_\lambda) \widehat{\Sigma}_\lambda^{-1} \left(  \widehat{\Sigma}_b - \Sigma_b \right)(\beta_b - \beta_\lambda)
\right ] \right \| \nonumber \\
& ~~ \lesssim
\frac{\lambda \rho_r^2}{\left(\lambda \rho_r^2 + (1-\lambda) \rho_b^2 \right)^2}
~ \|\beta_r - \beta_b\| ~
\frac{K_b^2 \rho_b^2 {\sqrt{d}}/{\sqrt{n_b}}}{{\lambda}\rho_r^2/{{C}_r'}  + (1-\lambda) \rho_b^2/{{C}_b'}} 
\left(
\lambda K_r^2 \rho_r^2 \frac{\sqrt{d}}{\sqrt{n_r}} + (1-\lambda) K_b^2 \rho_b^2 \frac{\sqrt{d}}{\sqrt{n_b}}
\right ).
\end{align}
Plugging \eqref{eqn:proof_cross_bias_5} and \eqref{eqn:proof_cross_bias_6} into \eqref{eqn:proof_cross_bias_4}, and then substituting the whole term into \eqref{eqn:proof_cross_bias_0} completes the proof. $\qed$.

\subsection{Proofs from \Cref{sec:uniform_bound}}
\subsubsection{Full Statement and Proof of~\Cref{thrm:uniform_control}}\label{proof:thrm:uniform_control}

\begin{theoremprime}{thrm:uniform_control}\label{thrm:uniform_control_full}
    Suppose Assumptions~\ref{assumption:invertible_cov}-\ref{assumption:sg_noise} hold. Fix $\delta \in (0, 1)$ and suppose that, for each $g \in \mathcal{G}$, 
    \begin{equation}
        n_g \geq \max\left\{ \frac{6d}{\alpha_g}, 4\left( d + \log(8/\delta)\right)\right\}. 
    \end{equation}
    For each $g \in \mathcal{G}$, define 
    \begin{equation}
        a_g \coloneq \sqrt{\frac{d + \log(8/\delta)}{n_g}} + \frac{d + \log(8/\delta)}{n_g}
    \end{equation}
    and let $C_g'$ denote the constant associated with the small-ball bound given in \Cref{cor:known_small_ball}. Define the threshold and minimum eigenvalue across groups:
    \begin{equation}
        t_0 \coloneq \min\{1/2, 1/(2C_r'), 1/(2C_b')\}, \quad m_{\Sigma} \coloneq \min_{g \in \mathcal{G}} \lambda_{\min}(\Sigma_g). 
    \end{equation}
    Finally, define the uniform bound on the distance from the target $\beta_\lambda$ to either endpoint by 
    \begin{equation}
        D_\beta \coloneq \max_{g \in \mathcal{G}} \norm{\Sigma_g} m_\Sigma^{-1} \norm{\beta_r - \beta_b} 
    \end{equation}
    and the uniform estimation error radius
    \begin{equation}
        E_\beta(\delta) \coloneq t_0^{-1} m_\Sigma^{-1} \sum_{g \in \mathcal{G}}\left\{ K_g^2 \norm{\Sigma_g} D_\beta  + K_g \sigma_{\varepsilon, g} \norm{\Sigma_g}^{1/2} \right\} a_g.
    \end{equation}
    Then, with probability at least $1 - \delta - \eta$, 
    \begin{equation}\label{eqn:uniform_bound_1}
        \sup_{\lambda \in [0, 1]} \norm{\Delta_{\lambda}}_2 \leq \left\{ \sum_{g \in \mathcal{G}} \left(\norm{\Sigma_g} \right)^2 \right\}^{1/2} \left( C^2 E_\beta(\delta)^2 + 2 C E_\beta(\delta) D_\beta\right) 
    \end{equation}
    for a constant $C > 0$ where 
    \begin{equation}
        \eta \coloneq (C_r' t_0)^{\alpha_r n_r/6} + (C_b' t_0)^{\alpha_b n_b/6} \leq 2^{-\alpha_r n_r /6} +  2^{-\alpha_b n_b /6}.
    \end{equation}
    In particular, there exists a constant $C_U > 0$ independent of $n_r, n_b, d$, and $\delta$ such that 
    \begin{equation}\label{eqn:uniform_simpler}
        \sup_{\lambda \in [0, 1]} \norm{\Delta_{\lambda}}_2 \leq C_U \sum_{g \in \mathcal{G}}a_g. 
    \end{equation}
\end{theoremprime}
\subsubsection{Proof of \Cref{thrm:uniform_control_full}}
We first establish high probability control of the empirical covariance and score matrices simultaneously across groups. Throughout the proof, $C>0$ denotes a universal constant which may vary by another universal constant factor from one display to the next.  

\paragraph{Empirical covariance and score control:}
Define $s_\delta \coloneq \log(8/\delta)$. By \citet[Theorem 4.7.1]{vershynin2018high}, under \Cref{assumption:subgaussian}, for each $g \in \mathcal{G}$, 
\begin{equation}\label{eqn:uniform_hbp_union1}
    \PP\left( \norm{\widehat{\Sigma}_g - \Sigma_g} > C K_g^2 \norm{\Sigma_g} a_g \right) \leq 2 e^{-s_\delta} = \frac{\delta}{4}. 
\end{equation}
Next, we control $\widehat{Z}_g$. Recall the decomposition \eqref{eqn:known_covariance_upper_bound_nu_g_decomposition} and fix $v \in \mathbb{S}^{d-1}$. By \Cref{assumption:subgaussian} and \ref{assumption:sg_noise}, the product $\varepsilon_i^{(g)} v^\top X_i^{(g)}$ is centered and sub-exponential for every $i \in [n_g]$ and each $g \in \mathcal{G}$. In particular, 
\begin{equation}
    \norm{\varepsilon_i^{(g)} v^\top X_i^{(g)}}_{\psi_1} \leq C K_g \sigma_{\varepsilon, g} \norm{\Sigma_g}^{1/2}. 
\end{equation}
It follows by Bernstein's inequality (e.g., \citet[Corollary 2.8.3]{vershynin2018high}) that 
\begin{equation}
    \PP\left( \left|v^\top \widehat{Z}_g\right| \geq C K_g \sigma_{\varepsilon, g} \norm{\Sigma_g}^{1/2}\left( \sqrt{\frac{ s_\delta + d \log 5}{n_g}} + \frac{ s_\delta + d \log 5}{n_g}\right) \right) \leq 2 e^{-( s_\delta + d \log 5)}. 
\end{equation}
By a standard $1/2$-net approximation (see, e.g., \cite[Corollary 4.2.13, Theorem 4.4.5]{vershynin2018high}), we obtain 
\begin{equation}\label{eqn:uniform_hbp_union2}
    \PP\left( \norm{\widehat{Z}_g} \geq C K_g \sigma_{\varepsilon, g} \norm{\Sigma_g}^{1/2} a_g \right) \leq 2 \cdot 5^d \exp(-s_\delta -d \log 5) = \frac{\delta}{4}. 
\end{equation}
By a union bound over the events controlled in \eqref{eqn:uniform_hbp_union1} and \eqref{eqn:uniform_hbp_union2} for each of the two groups, we obtain that, with probability at least $1 - \delta$,
\begin{equation}\label{eqn:uniform_bound_groupwise_4_event_bound}
    \norm{\widehat{\Sigma}_g - \Sigma_g} \leq C K_g^2 \norm{\Sigma_g} a_g, \quad \norm{\widehat{Z}_g} \leq C K_g \sigma_{\varepsilon, g} \norm{\Sigma_g}^{1/2} a_g
\end{equation}
for both groups $g \in \mathcal{G}$. 

\paragraph{Uniform empirical covariance control: } Under \Cref{assumption:small_ball}, by \citet[Theorem 4]{mourtada2022exact} and a union bound over $g \in \mathcal{G}$, we obtain 
\begin{equation}
    \PP\left(  \exists g \in \mathcal{G}: \widehat{\Sigma}_g \not\succeq  t_0 {\Sigma}_g\right) \leq \sum_{g \in \mathcal{G}} (C_g' t_0)^{\alpha_g n_g/6} \leq 2^{-\alpha_r n_r/6} + 2^{-\alpha_b n_b/6}. 
\end{equation}
Thus, for every $\lambda \in [0, 1]$, on the complementary event, 
\begin{equation}
    \widehat{\Sigma}_\lambda \succeq t_0 \Sigma_\lambda \succeq t_0 m_\Sigma I_d.
\end{equation}
It follows that, with probability at least $1- \eta$, 
\begin{equation}\label{eqn:uniform_bound_emp_cov_lambda}
    \sup_{\lambda \in [0,1]}\norm{\widehat{\Sigma}_\lambda^{-1}}\leq t_0^{-1} m_\Sigma^{-1}. 
\end{equation}

For the remainder of the proof, we work on the intersection of the events in \eqref{eqn:uniform_bound_groupwise_4_event_bound} and \eqref{eqn:uniform_bound_emp_cov_lambda}, which we have established occurs with probability at least $1 - \delta - \eta$. 

\paragraph{Uniform control of 
$\norm{\beta_g - \beta_\lambda}$:}
Using $\beta_\lambda = \Sigma_\lambda^{-1}\nu_\lambda$ and $\nu_g = \Sigma_g \beta_g$, we obtain 
\begin{equation}
    \beta_r - \beta_\lambda= (1-\lambda)\Sigma_\lambda^{-1} \Sigma_b(\beta_r - \beta_b), \quad \beta_b - \beta_\lambda = \lambda \Sigma_\lambda^{-1} \Sigma_r (\beta_b-\beta_r). 
\end{equation}
Since $\Sigma_\lambda \succeq m_\Sigma I_d$, 
it follows from the definition of $D_\beta$ that 
\begin{equation}\label{eqn:uniform_bound_beta_lambda_g_sup}
    \sup_{\lambda \in [0,1]}\max_{g \in \mathcal{G}}\norm{\beta_\lambda - \beta_g} \leq D_\beta. 
\end{equation}

\paragraph{Uniform control of $\norm{\widehat{\beta}_\lambda - \beta_\lambda}$:}
Recalling the decomposition \eqref{eqn:known_covariance_upper_bound_nu_g_decomposition} and rearranging yields 
\begin{equation}
    \widehat{\Sigma}_\lambda(\widehat{\beta}_\lambda - \beta_\lambda) = \lambda \widehat{\Sigma}_r (\beta_r - \beta_\lambda) + (1-\lambda) \widehat{\Sigma}_b (\beta_b-\beta_\lambda) + \lambda \widehat{Z}_r + (1-\lambda) \widehat{Z}_b. 
\end{equation}
Recalling the first order condition \eqref{eqn:frontier_lemma_first_order} and subtracting, we obtain 
\begin{equation}\label{eqn:uniform_bound_foc}
    \widehat{\Sigma}_\lambda(\widehat{\beta}_\lambda - \beta_\lambda) = \lambda (\widehat{\Sigma}_r-\Sigma_r)(\beta_r - \beta_\lambda) + (1-\lambda) (\widehat{\Sigma}_b-\Sigma_b) (\beta_b-\beta_\lambda) + \lambda \widehat{Z}_r + (1-\lambda) \widehat{Z}_b. 
\end{equation}

Combining \eqref{eqn:uniform_bound_emp_cov_lambda}, \eqref{eqn:uniform_bound_beta_lambda_g_sup}, and \eqref{eqn:uniform_bound_foc},
\begin{equation}\label{eqn:uniform_bound_beta_lambda_empirical_gap1}
     \norm{\widehat{\beta}_\lambda - \beta_\lambda} \leq  t_0^{-1} m_\Sigma^{-1} \left( \lambda D_\beta \norm{\widehat{\Sigma}_r-\Sigma_r} + (1-\lambda) D_\beta \norm{\widehat{\Sigma}_b-\Sigma_b}  + \lambda \norm{\widehat{Z}_r} + (1-\lambda) \norm{\widehat{Z}_b} \right). 
\end{equation}
We then use $\lambda, 1-\lambda \leq 1$ and the bounds \eqref{eqn:uniform_bound_groupwise_4_event_bound} to obtain that, on the same event,
\begin{equation}\label{eqn:uniform_bound_beta_lambda_empirical_gap2}
    \sup_{\lambda \in [0,1]} \norm{\widehat{\beta}_\lambda - \beta_\lambda} \leq  C E_\beta(\delta).  
\end{equation}

\paragraph{Collecting bounds:}
By \Cref{proposition:excess_risk_to_matrix_norm} and triangle inequality,
\begin{equation}
    \sup_{\lambda \in [0, 1]} \left| \mathcal{R}_g(\widehat{\beta}_\lambda) - \mathcal{R}_g(\beta_\lambda)\right| \leq \norm{\Sigma_g} \left( \sup_{\lambda \in [0, 1]} \|\widehat{\beta}_\lambda -\beta_\lambda\|\right)^2 + 2 \norm{\Sigma_g} \left(\sup_{\lambda \in [0,1]} \norm{\widehat{\beta}_\lambda-\beta_\lambda} \right) \left(\sup_{\lambda \in [0,1]}\norm{\beta_\lambda-\beta_g}\right). 
\end{equation}
Applying \eqref{eqn:uniform_bound_beta_lambda_g_sup} and \eqref{eqn:uniform_bound_beta_lambda_empirical_gap2},  
\begin{equation}
    \sup_{\lambda \in [0, 1]} \left| \mathcal{R}_g(\widehat{\beta}_\lambda) - \mathcal{R}_g(\beta_\lambda)\right| \leq  \norm{\Sigma_g} \left( C^2 E_\beta(\delta)^2 + 2 C E_\beta(\delta) D_\beta\right). 
\end{equation}
Combining the two coordinate-wise bounds therefore gives \eqref{eqn:uniform_bound_1}: 
\begin{equation}\label{eqn:uniform_bound_1_proof}
    \sup_{\lambda \in [0, 1]}  \norm{\Delta_\lambda} \leq \sup_{\lambda \in [0, 1]}\left( \sum_{g \in \mathcal{G}} \left( \sup_{\lambda \in [0, 1]} \left| \mathcal{R}_g(\widehat{\beta}_\lambda) - \mathcal{R}_g(\beta_\lambda)\right|  \right)^2 \right)^{1/2} \leq \left\{ \sum_{g \in \mathcal{G}} \left(\norm{\Sigma_g} \right)^2 \right\}^{1/2} \left( C^2 E_\beta(\delta)^2 + 2 C E_\beta(\delta) D_\beta\right). 
\end{equation}

\paragraph{Reduction to simpler rate: }For the second claim establishing the simpler rate given in the main text, by definition of $E_\beta(\delta)$, we get
\begin{equation}
    E_\beta(\delta) \leq t_0^{-1} m_\Sigma^{-1} \max_{g \in \mathcal{G}} \left\{ K_g^2 \norm{\Sigma_g} D_\beta  + K_g \sigma_{\varepsilon, g} \norm{\Sigma_g}^{1/2} \right\} \sum_{g \in \mathcal{G}}a_g. 
\end{equation}
The sample size condition implies that, for every $g \in \mathcal{G}$, 
\begin{equation}
    \frac{d + \log(8/\delta)}{n_g}\leq \frac{1}{4}. 
\end{equation}
Thus, $a_g \leq \frac{3}{4}$ and $\sum_{g \in \mathcal{G}} a_g \leq \frac{3}{2}$. 
It follows that 
\begin{equation}\label{eqn:uniform_bound_E_delta_sq}
    E_\beta(\delta)^2 \leq \frac{3}{2} t_0^{-2} m_\Sigma^{-2} \left(\max_{g \in \mathcal{G}} \left\{ K_g^2 \norm{\Sigma_g} D_\beta  + K_g \sigma_{\varepsilon, g} \norm{\Sigma_g}^{1/2} \right\}\right)^2 \sum_{g \in \mathcal{G}}a_g. 
\end{equation}
Similarly, 
\begin{equation}\label{eqn:uniform_bound_E_delta_D_beta}
    E_\beta(\delta) D_\beta \leq D_\beta t_0^{-1} m_\Sigma^{-1} \max_{g \in \mathcal{G}} \left\{ K_g^2 \norm{\Sigma_g} D_\beta  + K_g \sigma_{\varepsilon, g} \norm{\Sigma_g}^{1/2} \right\} \sum_{g \in \mathcal{G}}a_g. 
\end{equation}
Substituting \eqref{eqn:uniform_bound_E_delta_sq} and \eqref{eqn:uniform_bound_E_delta_D_beta} into \eqref{eqn:uniform_bound_1_proof} yields \eqref{eqn:uniform_simpler} with 
\begin{equation}\label{eqn:uniform_bound_2_proof}
    \begin{aligned}
        C_U = C \left\{ \sum_{g \in \mathcal{G}} \left(\norm{\Sigma_g} \right)^2 \right\}^{1/2} & \left( t_0^{-2} m_\Sigma^{-2} \left(\max_{g \in \mathcal{G}} \left\{ K_g^2 \norm{\Sigma_g} D_\beta  + K_g \sigma_{\varepsilon, g} \norm{\Sigma_g}^{1/2} \right\}\right)^2 \right.\\
        & \quad \quad \quad + \left. D_\beta t_0^{-1} m_\Sigma^{-1} \max_{g \in \mathcal{G}} \left\{ K_g^2 \norm{\Sigma_g} D_\beta  + K_g \sigma_{\varepsilon, g} \norm{\Sigma_g}^{1/2} \right\} \right)
    \end{aligned} 
\end{equation}
where $C$ consolidates constant factors.\qed 
\section{Additional Discussion}\label{appB}
\subsection{Connection to the \cite{liu2024inference} support function parameterization}\label{sec:support_function}
The $\lambda$-weighted risk parametrization that we use to analyze estimators for frontier points has a direct connection to the support function parameterization of  \citet[Sections 3.4, 5.3]{liu2024inference}. Let $e(\beta) = (\mathcal{R}_r(\beta), \mathcal{R}_b(\beta))$, $\mathcal{E}(\mathcal{H}) = \{e(\beta):\beta \in \RR^d\}$. Define third quadrant support directions $\mathcal{Q} \coloneq \{q \in \RR^2 : \norm{q}_2 \leq 1, q_r, q_b \leq 0\}$. For $q \in \mathcal{Q}$, define support function and support sets, respectively, as in \cite{liu2024inference}:
\begin{equation*}
    h_\mathcal{E}(q) \coloneq \max_{e \in \mathcal{E}(\mathcal{H})} q^\top e, \quad \mathcal{S}_\mathcal{E}(q) \coloneq \mathcal{E}(\mathcal{H}) \cap \{z \in \RR^2 : q^\top z = h_\mathcal{E}(q)\}. 
\end{equation*}
For fixed $\lambda \in [0,1]$, let 
\begin{equation*}
    w_\lambda \coloneq (\lambda, 1-\lambda), \quad q_\lambda \coloneq \frac{-w_\lambda}{\norm{w_\lambda}_2}. 
\end{equation*}
Then, when $d \geq 2$ and  $\Sigma_r, \Sigma_b \succ 0$, recalling \Cref{lemma:frontier_regression_convex}, we obtain 
\begin{equation*}
    \begin{aligned}
        \argmax_{\beta \in \RR^d}q_\lambda^\top e(\beta) &= \argmax_{\beta \in \RR^d} \left\{ \frac{-1}{\norm{w_\lambda}_2} w_\lambda^\top e(\beta) \right\}
        &= \argmin_{\beta \in \RR^d} \mathcal{R}_\lambda(\beta)
        &= \{\beta_\lambda\}. 
    \end{aligned}
\end{equation*}
Consequently, $\mathcal{S}_\mathcal{E}(q_\lambda) = \{e(\beta_\lambda)\}$. Conversely, $q = (q_r, q_b) \in \mathcal{Q}$ corresponds to our $\lambda$-weighted construction for a unique choice $\lambda \in [0, 1]$: 
\begin{equation*}
    \lambda(q) \coloneq \frac{q_r}{q_r + q_b} \in [0, 1]. 
\end{equation*}
Indeed, 
\begin{equation*}
    q = (q_r + q_b) w_{\lambda(q)}, \quad q^\top e(\beta) = (q_r + q_b) \mathcal{R}_{\lambda(q)}(\beta),
\end{equation*}
so, since $q_r + q_b < 0$, maximizing $q^\top e(\beta)$ is equivalent to minimizing $\mathcal{R}_{\lambda(q)}(\beta)$. Under group-balancedness, \Cref{lemma:frontier_regression_convex} shows that this boundary coincides with the FA frontier. Thus, our $\lambda$-weighted risk parameterization and the support-set construction of \cite{liu2024inference} are equivalent.
\subsection{Extension to several groups}\label{sec:several_groups}

Extending our analysis to several groups raises two distinct questions: (i) how to estimate predictors that minimize a weighted combination of several groups' risks and (ii) whether the weighted risk formulation characterizes the entire fairness--accuracy frontier. Let $\mathcal{G} = \{1, \hdots, K\}$ with $K<\infty$ fixed. Suppose that each group $g \in \mathcal{G}$ has weight $\lambda_g\geq 0$ in the planner's objective function: 
\begin{equation}
    \mathcal{R}_\lambda(\beta) \coloneq \sum_{g=1}^K \lambda_g \mathcal{R}_g(\beta), \quad \sum_{g = 1}^K \lambda_g = 1.
\end{equation}

Much of our statistical analysis in \Cref{sec:knowncov} and \Cref{sec:unknown_covariance} extends naturally to the multigroup setting to address the first question without requiring a notion of multigroup fairness. Under analogous groupwise assumptions, one can extend the quadratic risk identities \Cref{proposition:excess_risk_to_matrix_norm} and matching upper and lower bound arguments of the known covariance setting to recover the exact minimax result \Cref{theorem:minimax_optimal_known_cov} under an analogous estimator. 

For the unknown covariance analysis, the exact decomposition \Cref{prop:unknown_cov_upper_bound} with the bias term of the form \eqref{eq:unknown_bias_1} extends to the $K$-group setting. We expect that the lower bound construction \eqref{eqn:assouad_max} based on Assouad's lemma applied to the setting where (a) the covariance matrices $\Sigma_g$ are known and the coefficients are unknown, and (b) the parameters $\beta_g$ are known while the covariance matrices are unknown will similarly extend to the $K$-group setting up to constant factors depending on $K$. 

The second question of interpreting our $\lambda$-weighted risk construction as coinciding with the FA frontier is more delicate. This analysis depends on how one extends a notion of fairness to several groups. In particular, there are multiple natural notions of aggregating disparities for more than two groups, all of which recover absolute difference disparity for two groups, but only some of which satisfy the property that the FA frontier can be recovered via the $\lambda$-weighted risk formulation. For instance, one can consider the maximum groupwise absolute difference, 
\begin{equation}
    \max_{g \in \mathcal{G}} \mathcal{R}_g(\beta) - \min_{g \in \mathcal{G}} \mathcal{R}_g(\beta)
\end{equation}
or the mean difference over pairwise comparisons
\begin{equation}
    \binom{K}{2}^{-1} \sum_{g < h}|\mathcal{R}_g(\beta) - \mathcal{R}_h(\beta)|. 
\end{equation}
Both recover absolute difference for $K  = 2$ but can induce different orderings for general $K$. 

Our statistical methods provide a foundation for estimating weighted risk models for several groups; we highlight the characterization of when the weighted risk model coincides with the FA frontier under different generalizations of the groupwise disparities as an interesting direction for future work.

\subsection{An alternative candidate estimator for the known-covariance setting}\label{sec:alternative_estimator_known_covariance}
In the known-covariance regime, a natural alternative to our estimator $\tilde{\beta}_\lambda\coloneq\Sigma_\lambda^{-1}(\lambda\Sigma_r\widehat{\beta}_r + (1-\lambda)\Sigma_b\widehat{\beta}_b)$ is
\begin{equation*}
    \dbtilde{\beta}_\lambda \coloneq \Sigma_\lambda^{-1}\left( \lambda \widehat{\nu}_r + (1-\lambda)\widehat{\nu}_b\right).
\end{equation*}
This estimator replaces each cross-moment $\nu_g = \Sigma_g \beta_g$ with its empirical counterpart $\widehat{\nu}_g$ (rather than with $\Sigma_g \widehat{\beta}_g$ as in $\tilde{\beta}_\lambda$). Both estimators are unconditionally unbiased for $\beta_\lambda$, but they differ in their risk as a result of how they make use of knowledge of $\Sigma_g$. To see this, first recall from \eqref{eqn:known_covariance_upper_bound_nu_g_decomposition} that for each group $g \in \mathcal{G}$, we have $\widehat{\nu}_g = \widehat{\Sigma}_g\beta_g + \widehat{Z}_g$ with $\widehat{\nu}_g, \widehat{Z}_g$ defined in \eqref{eqn:known_covariance_upper_bound_nu_g_decomposition}. 

Then, the empirical replacements for $\nu_g$ used by $\tilde{\beta}_\lambda$ and $\dbtilde{\beta}_\lambda$, respectively, can be written as:  
\begin{equation}\label{eqn:nu_hat_using_sigma_beta_hat}
    \Sigma_g \widehat{\beta}_g = \Sigma_g \widehat{\Sigma}_g^{-1} \widehat{\nu}_g = \Sigma_g \beta_g + \Sigma_g \widehat{\Sigma}_g^{-1} \widehat{Z}_g, \quad  \widehat{\nu}_g = \Sigma_g \beta_g + (\widehat{\Sigma}_g - \Sigma_g) \beta_g + \widehat{Z}_g. 
\end{equation}
In both cases, the leading term is the target quantity ${\nu}_g$. The key difference is that $\dbtilde{\beta}_\lambda$ inherits an additional additive covariance fluctuation term $(\widehat{\Sigma}_g - \Sigma_g)\beta_g$, whereas $\tilde{\beta}_\lambda$ removes this term and instead reflects the noise through $\Sigma_g \widehat{\Sigma}_g^{-1}$. In other words, $\tilde{\beta}_\lambda$ uses the known population covariance to eliminate the covariance fluctuation as a separate additive error term.

Using these decompositions, we obtain 
\begin{equation}\label{eqn:known_cov_alternative_tilde_beta_bias}
    \tilde{\beta}_\lambda - \beta_\lambda  = \Sigma_\lambda^{-1}\left( \lambda \Sigma_r \widehat{\Sigma}_r^{-1} \widehat{Z}_r + (1 - \lambda) \Sigma_b \widehat{\Sigma}_b ^{-1}\widehat{Z}_b \right),
\end{equation}
\begin{equation}\label{eqn:known_cov_alternative_dbtilde_bias}
    \dbtilde{\beta}_\lambda - \beta_\lambda  = \Sigma_\lambda^{-1}\left( \lambda \widehat{Z}_r + (1 - \lambda) \widehat{Z}_b \right) + \Sigma_\lambda^{-1}\left( \lambda (\widehat{\Sigma}_r - \Sigma_r) \beta_r + (1-\lambda) (\widehat{\Sigma}_b - \Sigma_b ) \beta_b\right).
\end{equation}
The following result provides the exact risk of $\dbtilde{\beta}_\lambda$. 
\begin{proposition}\label{prop:known_covariance_alternative_estimator_exact_risk}
    For the alternative estimator $\dbtilde{\beta}_\lambda$, the excess risk for group $g \in \mathcal{G}$ satisfies
    \small{\begin{equation*}
        \begin{aligned}
            \EE \left[ \norm{\dbtilde{\beta}_\lambda - \beta_\lambda }_{\Sigma_g}^2\right] &= \lambda^2 \frac{\sigma_r^2}{n_r} \EE\left[ \operatorname{Tr}\left( \Sigma_g\Sigma_{\lambda}^{-1} \widehat{\Sigma}_r \Sigma_\lambda^{-1} \right)\right] + (1-\lambda)^2 \frac{\sigma_b^2}{n_b} \EE\left[ \operatorname{Tr}\left( \Sigma_g\Sigma_{\lambda}^{-1} \widehat{\Sigma}_b \Sigma_\lambda^{-1} \right) \right]\\
            & \quad \quad + \EE\left[ \norm{\Sigma_g^{1/2} \Sigma_\lambda^{-1} \left( \lambda (\widehat{\Sigma}_r - \Sigma_r) \beta_r + (1-\lambda) (\widehat{\Sigma}_b - \Sigma_b ) \beta_b\right)}^2 \right]. 
        \end{aligned}
    \end{equation*}}
\end{proposition}
\begin{proof}[Proof of \Cref{prop:known_covariance_alternative_estimator_exact_risk}]
    Define $A \coloneq \Sigma_\lambda^{-1}\Sigma_g \Sigma_\lambda^{-1}$. Using that $\EE[\varepsilon_i^{(g)} | X_i^{(g)}] = 0$ for all $i \in [n_g]$ and for each $g$, we proceed by the same steps as the proof of the upper bound of \Cref{theorem:minimax_optimal_known_cov}, 
    \begin{equation*}
        \EE \left[ \norm{\dbtilde{\beta}_\lambda - \beta_\lambda }_{\Sigma_g}^2\right] = \EE\left[\norm{\lambda \widehat{Z}_r + (1 - \lambda) \widehat{Z}_b }_A^2 \right] + \EE\left[\norm{\lambda (\widehat{\Sigma}_r - \Sigma_r) \beta_r + (1-\lambda) (\widehat{\Sigma}_b - \Sigma_b ) \beta_b}_A^2\right]. 
    \end{equation*}

    Using the independence of samples and that the noise is mean-zero conditional on the design,  
    \begin{equation*}
        \EE\left[\norm{\lambda \widehat{Z}_r + (1 - \lambda) \widehat{Z}_b }_A^2 \right] = \lambda^2 \EE\left[\norm{\widehat{Z}_r}_A^2\right] + (1-\lambda)^2 \EE\left[\norm{\widehat{Z}_b}_A^2\right]
    \end{equation*}
    and, for each $g \in \mathcal{G}$, 
    \begin{equation*}
        \EE\left[ \norm{\widehat{Z}_g}_A^2 \right] = \frac{\sigma_g^2}{n_g}\EE\left[ \operatorname{Tr}\left(A \widehat{\Sigma}_g \right) \right].
    \end{equation*}
    Hence, 
    \begin{equation*}
        \EE \left[ \norm{\dbtilde{\beta}_\lambda - \beta_\lambda }_{\Sigma_g}^2\right] = \lambda^2 \frac{\sigma_r^2}{n_r}\EE\left[ \operatorname{Tr}\left(A \widehat{\Sigma}_r \right) \right] + (1-\lambda)^2 \frac{\sigma_b^2}{n_b}\EE\left[ \operatorname{Tr}\left(A \widehat{\Sigma}_b \right) \right] + \EE\left[\norm{\lambda (\widehat{\Sigma}_r - \Sigma_r) \beta_r + (1-\lambda) (\widehat{\Sigma}_b - \Sigma_b ) \beta_b}_A^2\right].
    \end{equation*}
    The stated equation follows by substituting $A$ and using cyclicity of the trace. 
\end{proof}
\begin{lemma}\label{lemma:alternative_estimator_empirical_cov}
    Fix group $h \in \mathcal{G}$. Suppose that \Cref{assumption:invertible_cov} and \Cref{assumption:subgaussian} hold for group $h$, and that $d \geq 2$. Then $\PP((\widehat{\Sigma}_h - \Sigma_h) \neq 0)>0$. Moreover, for every $g \in \mathcal{G}$, there exists $v_{g, h}\in\RR^d$ such that 
    \begin{equation*}
        \EE\left[ \norm{\Sigma_g^{1/2} \Sigma_\lambda^{-1}(\widehat{\Sigma}_h - \Sigma_h) v_{g, h} }_2^2 \right] > 0.
    \end{equation*}
\end{lemma}
\begin{proof}[Proof of \Cref{lemma:alternative_estimator_empirical_cov}]
  Suppose for the sake of contradiction that $\widehat{\Sigma}_h = \Sigma_h$ almost surely. For $j, k \in [d]$, by \Cref{assumption:subgaussian}, the product $X_j^{(h)}X_k^{(h)}$ has bounded second moment. If $\widehat{\Sigma}_h = \Sigma_h$, 
  \begin{equation*}
      \Var{ (\widehat{\Sigma}_h)_{jk}} = \frac{1}{n_h}\Var{ X_{j}^{(h)}X_{k}^{(h)}} = 0. 
  \end{equation*}
  Since observations are i.i.d., it follows that, for every $j, k\in[d]$, $X_j^{(h)}X_k^{(h)} = (\Sigma_h)_{jk}$ almost surely. Equivalently, $X^{(h)} {X^{(h)}}^\top = \Sigma_h$ almost surely. However, $X^{(h)} {X^{(h)}}^\top$ is rank at most one while $\Sigma_h$ is invertible with rank $d \geq 2$ by assumption. This is a contradiction. Since $\PP((\widehat{\Sigma}_h - \Sigma_h) \neq 0)> 0$, there exists a basis vector $e_i \in \RR^d$ such that $\PP((\widehat{\Sigma}_h - \Sigma_h)e_i \neq 0)>0$. Since $\Sigma_g, \Sigma_\lambda$ are positive definite by assumption, $\Sigma_g^{1/2}\Sigma_\lambda^{-1}$ is invertible, and the second part of the claim follows. 
\end{proof}
\begin{proposition}\label{prop:minimax_suboptimal_alternative_estimator}
    Let $\lambda \in (0, 1)$ and suppose that $d \geq 2$. Suppose that for each group assumptions \Cref{assumption:invertible_cov}, \Cref{assumption:invertible_emp_cov}, \Cref{assumption:subgaussian}, and \Cref{assumption:small_ball} hold. Also suppose that $n_g \geq 6d/\alpha_g$ for each $g \in \mathcal{G}$. Then, for each group 
    \begin{equation*}
        \sup_{P \in \mathcal{P}_{\textrm{linear}}(P_X^r, P_X^b, \sigma_r^2, \sigma_b^2)} \EE \left[ \norm{\dbtilde{\beta}_\lambda - \beta_\lambda }_{\Sigma_g}^2\right] = \infty
    \end{equation*}
    and, consequently, 
    \begin{equation*}
        \sup_{P \in \mathcal{P}_{\textrm{linear}}(P_X^r, P_X^b, \sigma_r^2, \sigma_b^2)} \EE\left[ \mathcal{R}_\lambda(\dbtilde{\beta}_\lambda) - \mathcal{R}_\lambda(\beta_\lambda) \right] = \infty. 
    \end{equation*}
    Hence, $\dbtilde{\beta}_\lambda$ is not minimax optimal. Indeed, its worst-case risk is infinite whereas the minimax value in the known-covariance problem is finite.
\end{proposition}
\begin{proof}[Proof of \Cref{prop:minimax_suboptimal_alternative_estimator}]
    For $t > 0$, take $\beta_r = t v_{g, r}$ where $v_{g, r}$ is  given in \Cref{lemma:alternative_estimator_empirical_cov} and $\beta_b = 0$. Under this choice, from \Cref{prop:known_covariance_alternative_estimator_exact_risk} and noting that the first two terms are non-negative, 
    \begin{equation*}
        \EE \left[ \norm{\dbtilde{\beta}_\lambda - \beta_\lambda }_{\Sigma_g}^2\right] \geq \lambda^2 t^2 \EE\left[ \norm{\Sigma_g^{1/2} \Sigma_\lambda^{-1}   (\widehat{\Sigma}_r - \Sigma_r) v_{g, r} }_2^2 \right]
    \end{equation*}
    Taking $t \to \infty$ and recalling the minimax rate established in \Cref{theorem:minimax_optimal_known_cov} yields the first part of the claim. Applying \Cref{proposition:excess_risk_to_matrix_norm} yields the second. 
\end{proof}

\section{Extension to Generalized Linear Models}\label{sec:glm}
The main text treats squared loss in well-specified linear regression, where the FA frontier and its corresponding estimators admit closed-form representations. The tractability of such a setting allowed us to prove sharp finite-sample upper and lower bounds and to characterize groupwise asymmetries and the optimal allocation of a fixed sampling budget across groups. In this section, we ask whether our finding that estimation from finite data serves to distort a designer's intended FA tradeoff persists beyond the squared loss linear regression setting.

We answer this question affirmatively for canonical generalized linear models (GLMs) with unknown groupwise covariate distributions. The analysis covers, for instance, logistic regression. In this case, the population target $\beta_\lambda$ no longer admits a closed-form representation, the empirical estimator is nonlinear, and the excess risk is only locally quadratic. Our treatment of this setting is therefore less sharp than that of the linear model. However, it establishes the same predominant structure. For the designer's weighted risk, finite-sample estimation produces a variance term driven by noisy empirical scores and a bias term driven by the differences in groupwise optima. For each group, estimation additionally produces a signed component which redistributes error between groups and is of leading order. 

\subsection{Canonical GLM Class}

\begin{definition}\label{def:glm}
    Fix a one-parameter exponential family with natural parameter space $\mathcal{A} \subseteq \RR$, cumulant function $A: \mathcal{A} \to \RR$, and base term $c(y, \varphi)$. Assume that $A$ is thrice continuously differentiable and satisfies $A''(z) > 0$ for all $z \in \mathcal{A}$. Fix dispersion parameters $\varphi_g > 0$ for each $g \in \mathcal{G}$ and a radius $B > 0$.
    The canonical generalized linear model class $\mathcal{P}_{\text{GLM}}(A, \varphi_r, \varphi_b, B)$ consists of all distributions of $(X, Y, G)$ such that, for each group $g \in \mathcal{G}$ there exists a parameter $\beta_g \in \Theta_B \coloneq \{\beta \in \RR^d : \norm{\beta}_2 \leq B\}$ and $X^\top \beta_g \in \mathcal{A}$ almost surely such that the conditional distribution of $Y$ given $(X=x, G =g)$ has density or mass function
    \begin{equation*}
        p_g(y|x)=\exp\left\{\frac{y x^\top \beta_g -A(x^\top \beta_g)}{\varphi_g} + c(y, \varphi_g)\right\}.
    \end{equation*} 
    Consequently, 
    \begin{equation}\label{eqn:conditional_mean_variance_glm}
    \EE_g[Y | X] = A'(X^\top \beta_g), \quad \operatorname{Var}_g(Y | X ) = \varphi_g A''(X^\top \beta_g).
\end{equation}
\end{definition}
The class \Cref{def:glm} nests several notable examples. Gaussian-noise linear regression is recovered by taking $\varphi_g = \sigma_g^2$, $A(z) = z^2/2$, and $c(y, \sigma_g^2) = -\frac{1}{2}\left( \log(2 \pi \sigma_g^2) + y^2/\sigma_g^2 \right)$. Logistic regression for binary outcomes $Y \in \{0, 1\}$ is recovered by taking $\varphi_g = 1$, $A(z) = \log(1 + \exp(z))$, $c(y, 1) = 0$. 

\begin{assumption}\label{assumption:bounded_covariates_glm}
    For each $g \in \mathcal{G}$, $\norm{X}_2 \leq T_X$ almost surely and $[-T_X B, T_X B]\subseteq \operatorname{int}(\mathcal{A})$. 
\end{assumption} 
Under \Cref{assumption:bounded_covariates_glm}, $X^\top \beta \in [-T_X B, T_X B]$ almost surely for every $\beta \in \Theta_B$. By continuity and compactness, under \Cref{assumption:bounded_covariates_glm} we also have that the following quantities are finite: 
\begin{equation}\label{eqn:glm_abs_bounds_a}
    \underline{a} \coloneq \min_{|z|\leq T_X B} A''(z) > 0, \quad \bar{a} \coloneq \max_{|z|\leq T_X B} A''(z) , \quad L_3 \coloneq \max_{|z|\leq T_X B} |A'''(z)|. 
\end{equation}

\begin{assumption}\label{assumption:bounded_residual_glm}
    For each $g \in \mathcal{G}$, the residual $\varepsilon_g \coloneq Y - A'(X^\top \beta_g)$ is conditionally $\sigma_{\varepsilon, g}$-sub-Gaussian given $X$. 
\end{assumption}

\subsection{Population and empirical loss}
Throughout this section, fix $\lambda\in[0, 1]$. We evaluate predictors $\beta \in \Theta_B$ using the unscaled canonical GLM loss
\begin{equation}\label{eqn:glm_loss}
    \ell(\beta; x, y) \coloneq A(x^\top \beta) - y x^\top \beta. 
\end{equation}
For each $g \in \mathcal{G}$, define the population and empirical risks as before but with respect to \eqref{eqn:glm_loss}. For $\lambda \in [0,1]$, define $\mathcal{R}_\lambda$ as in \eqref{eqn:f_lambda_R_lambda_defs},  and the empirical analogs of $\widehat{\mathcal{R}}_g$ and $\widehat{\mathcal{R}}_\lambda$ as before. 

Also define the groupwise and empirical scores 
\begin{equation}
    S_g(\beta) \coloneq \nabla_\beta \mathcal{R}_g(\beta) = \EE_g\left[ (A'(X^\top \beta) - Y) X \right] , \quad \widehat{S}_g(\beta) \coloneq \nabla_{\beta} \widehat{\mathcal{R}}_g(\beta) = \frac{1}{n_g} \sum_{i = 1}^{n_g} (A'({X_i^{(g)}}^\top \beta)- Y_i^{(g)}) X_i^{(g)} 
\end{equation}
and their weighted analogs, respectively, $S_\lambda(\beta) = \lambda S_r(\beta) + (1-\lambda)S_b(\beta)$, $\widehat{S}_\lambda(\beta) = \lambda \widehat{S}_r(\beta) + (1-\lambda) \widehat{S}_b(\beta)$. 
Also define the population and empirical Hessians
\begin{equation}\label{eqn:glm_hessians}
    \mathcal{I}_g(\beta) \coloneq \nabla_\beta^2 \mathcal{R}_g(\beta) = \EE_g[A''(X^\top \beta) X X^\top], \quad \widehat{\mathcal{I}}_g(\beta) \coloneq \nabla_\beta^2 \widehat{\mathcal{R}}_g(\beta) = \frac{1}{n_g}\sum_{i=1}^{n_g}  A''({X_i^{(g)}}^\top \beta) X_i^{(g)}{X_i^{(g)}}^\top
\end{equation}
with weighted versions $\mathcal{I}_\lambda(\beta) \coloneq \lambda \mathcal{I}_r(\beta) + (1-\lambda)\mathcal{I}_b(\beta)$ and $\widehat{\mathcal{I}}_\lambda(\beta) \coloneq \lambda \widehat{\mathcal{I}}_r(\beta) + (1-\lambda)\widehat{\mathcal{I}}_b(\beta)$.

\begin{assumption}\label{assumption:cov_lambda_glm}
    For some $\rho_\lambda>0$, the $\lambda$-weighted covariance matrix satisfies $\Sigma_\lambda \succeq \rho_\lambda^2 I_d$.  
\end{assumption}
Under \Cref{assumption:bounded_covariates_glm} via \eqref{eqn:glm_abs_bounds_a}, we have that $\mathcal{R}_\lambda$ is $\underline{a} \rho_\lambda^2$-strongly convex on $\Theta_B$ (as shown in  \Cref{lemma:glm_beta_bound}). Thus, the following population target is unique:
\begin{equation}
    \beta_\lambda \coloneq \argmin_{\beta \in \Theta_B} \mathcal{R}_\lambda(\beta).  
\end{equation}
Moreover, under \Cref{assumption:invertible_cov} and group-balancedness, the risk pairs $\{(\mathcal{R}_r(\beta_\lambda), \mathcal{R}_b(\beta_\lambda)):\lambda \in [0,1]\}$ trace the entire FA frontier. As in the linear setting, if $\beta_r \neq \beta_b$, the frontier is strictly convex; see \Cref{lemma:frontier_glm_convex}. 
 
\begin{assumption}\label{assumption:interior_glm}
    There is a fixed value $\delta_B \in (0, B)$ such that $\norm{\beta_\lambda} \leq B - \delta_B$. 
\end{assumption}
Under \Cref{assumption:interior_glm}, we have $S_\lambda(\beta_\lambda) = \nabla \mathcal{R}_\lambda(\beta_\lambda) = 0$.

\begin{lemma}\label{lemma:glm_group_balanced}
    Suppose Assumptions \ref{assumption:invertible_cov} and \ref{assumption:bounded_covariates_glm} hold. If 
    \begin{equation}\label{eqn:glm_group_balanced_condition}
        \frac{-\underline{a}}{2} \norm{\beta_r - \beta_b}_{\Sigma_r}^2 \leq \mathcal{R}_r(\beta_r) - \mathcal{R}_b(\beta_b) \leq \frac{\underline{a}}{2} \norm{\beta_r - \beta_b}_{\Sigma_b}^2, 
    \end{equation}
    then the canonical GLM is group-balanced. In particular, the GLM is group-balanced whenever $\mathcal{R}_r(\beta_r) = \mathcal{R}_b(\beta_b)$. 
\end{lemma}
The result is proved in \Cref{proof:lemma:glm_group_balanced}. Condition \ref{eqn:glm_group_balanced_condition} compares the difference in the groups' risks under their respective optimal models with the cost of using the other group's model. Suppose, for example, that group $r$ has higher optimal risk than group $b$ so that $\mathcal{R}_r(\beta_r) > \mathcal{R}_b(\beta_b)$. Applying the $r$-optimal model to group $b$ increases group $b$'s risk by $\mathcal{R}_b(\beta_r) - \mathcal{R}_b(\beta_b)$. Group balance holds if this increase is large enough to offset group $b$'s initial risk advantage so that $\mathcal{R}_b(\beta_r) \geq \mathcal{R}_r(\beta_r)$. At a higher level, this condition says that neither group is so much easier to predict that it continues to have lower risk even when evaluated using the other group's optimal model.

As in the linear setting, the group-balanced condition is not required for any of the GLM risk results that follow. Group-balancedness is needed only to interpret the path $\{\beta_\lambda: \lambda \in [0, 1]\}$ as tracing the entire FA frontier. If group-balancedness does not hold, the weighted risk targets and their empirical estimates remain well-defined, but the path $\lambda\in[0,1]$ need not exhaust the FA frontier.

\subsection{Excess risk decomposition}
The squared norm identity of \Cref{proposition:excess_risk_to_matrix_norm} is exact in the linear case because the risk is exactly quadratic. For the more general canonical GLM, a second-order Taylor expansion gives the following analog. 

\begin{lemma}\label{proposition:excess_risk_to_matrix_norm_glm}
    Suppose \Cref{assumption:bounded_covariates_glm} and \Cref{assumption:interior_glm} hold.  Then, for any $\widehat{\beta} \in \Theta_B$ and any $g \in \mathcal{G}$,
    \begin{equation}\label{eqn:glm_excess_risk_to_matrix_norm_glm_group}
        \mathcal{R}_g(\widehat{\beta}) - \mathcal{R}_g(\beta_\lambda) =  \frac{1}{2} \norm{\widehat{\beta} - \beta_\lambda}_{\mathcal{I}_g(\beta_\lambda)}^2 + \nabla_\beta \mathcal{R}_g(\beta_\lambda)^\top(\widehat{\beta} - \beta_\lambda)+ r_g(\widehat{\beta}). 
    \end{equation}
    \begin{equation}\label{eqn:glm_excess_risk_to_matrix_norm_glm_lambda}
        \mathcal{R}_\lambda(\widehat{\beta}) - \mathcal{R}_\lambda(\beta_\lambda) = \frac{1}{2} \norm{\widehat{\beta} - \beta_\lambda}_{\mathcal{I}_\lambda(\beta_\lambda)}^2 + r_\lambda(\widehat{\beta})
    \end{equation}
    where the remainder terms $r_\lambda(\widehat{\beta}), r_g(\widehat{\beta})$ satisfy
    \begin{equation}
       \left| r_\lambda(\widehat{\beta})\right|\leq \frac{L_3}{6} T_X^3 \norm{\widehat{\beta} - \beta_\lambda}_2^3, \quad \left| r_g(\widehat{\beta})\right|\leq \frac{L_3}{6} T_X^3 \norm{\widehat{\beta} - \beta_\lambda}_2^3. 
    \end{equation}
\end{lemma}
This result is proved in \Cref{proof:proposition:excess_risk_to_matrix_norm_glm} and generalizes \Cref{proposition:excess_risk_to_matrix_norm}: when $A(z) = z^2/2$, we have $A'''(z) = 0$, so $r_\lambda = r_g = 0$ and $\mathcal{I}_g(\beta_\lambda) = \Sigma_g$, which recovers the optimization of the squared loss setting (up to constant factors).

Define the constrained M-estimator
\begin{equation}\label{eqn:glm_m_estimator}
    \check{\beta}_\lambda \coloneq \argmin_{\beta \in \Theta_B} \widehat{\mathcal{R}}_\lambda(\beta). 
\end{equation}
If $\widehat{\Sigma}_\lambda \succ 0$, then \Cref{eqn:glm_abs_bounds_a} yields strong convexity of $\widehat{\mathcal{R}}_\lambda$ on $\Theta_B$, and hence that $\check{\beta}_\lambda$ is unique. 
In contrast to the linear model analysis, this estimator no longer has a closed form and satisfies $\widehat{S}_\lambda(\check{\beta}_\lambda) = 0$ when $\check{\beta}_\lambda$ is interior. 

\subsection{GLM Bias--Variance Decomposition}

\begin{theorem}\label{thrm:glm_per_group}
    Suppose Assumptions \ref{assumption:invertible_cov}, \ref{assumption:bounded_covariates_glm}, \ref{assumption:bounded_residual_glm}, \ref{assumption:cov_lambda_glm}, and \ref{assumption:interior_glm} hold. Suppose that $d \geq 2$ is fixed and that $n_r \asymp n_b \asymp n$. For each group $g \in \mathcal{G}$, define  
    \begin{equation}
        \mathcal{D}_g^{\textrm{GLM}}(\lambda) \coloneq \EE\left[ S_g(\beta_\lambda)^\top (\check{\beta}_\lambda - \beta_\lambda) \right], \quad \mathcal{M}_{g}(\beta_\lambda) \coloneq {\mathcal{I}}_\lambda^{-1}(\beta_\lambda)\mathcal{I}_g(\beta_\lambda) {\mathcal{I}}_\lambda^{-1}(\beta_\lambda),  
    \end{equation}
    \begin{equation}
        \Gamma_g(\beta_\lambda) \coloneq \operatorname{Var}_g\left(  (A'(X^\top \beta_\lambda)  - A'(X^\top \beta_g))X \right).
    \end{equation}
    Then, for each group $g \in \mathcal{G}$, 
    \begin{equation}\label{eqn:glm_exact_group_g_risk_decomp_body}
        \EE\left[ \mathcal{R}_g(\check{\beta}_\lambda) - \mathcal{R}_g(\beta_\lambda)\right] = \frac{1}{2} \left\{ \mathcal{V}_g^{\textrm{GLM}}(\lambda) + \mathcal{B}_g^{\textrm{GLM}}(\lambda) \right\} + \mathcal{D}_g^{\textrm{GLM}}(\lambda) + O(n^{-3/2})
    \end{equation}
    where 
    \begin{equation}\label{eqn:glm_body_variance_g}
    \mathcal{V}_g^{\textrm{GLM}}(\lambda) \coloneq \lambda^2 \frac{\varphi_r}{n_r}\operatorname{Tr}\left\{\mathcal{M}_{g}(\beta_\lambda) \mathcal{I}_r(\beta_r) \right\} + (1-\lambda)^2\frac{\varphi_b}{n_b} \operatorname{Tr}\left(\mathcal{M}_{g}(\beta_\lambda)   \mathcal{I}_b(\beta_b) \right), 
\end{equation}
\begin{equation}\label{eqn:glm_body_bias_g}
    \mathcal{B}_g^{\textrm{GLM}}(\lambda)  \coloneq \lambda^2 \frac{1}{n_r}\operatorname{Tr}\left\{\mathcal{M}_{g}(\beta_\lambda) \Gamma_r(\beta_\lambda) \right\} + (1-\lambda)^2\frac{1}{n_b} \operatorname{Tr}\left(\mathcal{M}_{g}(\beta_\lambda)   \Gamma_b(\beta_\lambda) \right);  
\end{equation}
    the leading terms satisfy 
    \begin{equation}
        |\mathcal{D}_g^{\textrm{GLM}}(\lambda)| = O(n^{-1}), \quad \mathcal{V}_g^{\textrm{GLM}}(\lambda) + \mathcal{B}_g^{\textrm{GLM}}(\lambda) = O(n^{-1}). 
    \end{equation}
    Moreover, the redistribution terms satisfy
    \begin{equation}
        \lambda \mathcal{D}_r^{\textrm{GLM}}(\lambda) + (1-\lambda) \mathcal{D}_b^{\textrm{GLM}}(\lambda) = 0
    \end{equation}
    and the bias term satisfies
    \begin{equation}
        0 \leq \mathcal{B}_g^{\textrm{GLM}}(\lambda) \leq C_{B, g} \lambda^2 (1-\lambda)^2 \left( \frac{1}{n_r} + \frac{1}{n_b}\right) \norm{\beta_r - \beta_b}^2. 
    \end{equation}
\end{theorem}

The proof and a more detailed statement of the theorem are given in \Cref{sec:proof:thrm:bias_variance_glm}. Weighting the two groupwise risks given in \Cref{thrm:glm_per_group} eliminates the signed terms and yields the following result for the planner's objective. 


\begin{corollary}\label{cor:bias_variance_glm}
    Under the premise of \Cref{thrm:glm_per_group}, \begin{equation}\label{eqn:glm_rate_body}
        \EE\left[ \mathcal{R}_\lambda(\check{\beta}_\lambda) - \mathcal{R}_\lambda(\beta_\lambda)\right] = \frac{1}{2} \left\{ \mathcal{V}_\lambda^{\textrm{GLM}} + \mathcal{B}_\lambda^{\textrm{GLM}} \right\} + O(n^{-3/2})
    \end{equation}
    where 
    \begin{equation}\label{eqn:glm_body_variance_lambda}
    \mathcal{V}_\lambda^{\textrm{GLM}} \coloneq \lambda^2 \frac{\varphi_r}{n_r}\operatorname{Tr}\left\{\mathcal{I}_{\lambda}^{-1}(\beta_\lambda) \mathcal{I}_r(\beta_r) \right\} + (1-\lambda)^2\frac{\varphi_b}{n_b} \operatorname{Tr}\left(\mathcal{I}_{\lambda}^{-1}(\beta_\lambda)   \mathcal{I}_b(\beta_b) \right), 
\end{equation}
\begin{equation}\label{eqn:glm_body_bias_lambda}
    \mathcal{B}_\lambda^{\textrm{GLM}}  \coloneq \lambda^2 \frac{1}{n_r}\operatorname{Tr}\left\{\mathcal{I}_{\lambda}^{-1}(\beta_\lambda) \Gamma_r(\beta_\lambda) \right\} + (1-\lambda)^2\frac{1}{n_b} \operatorname{Tr}\left(\mathcal{I}_{\lambda}^{-1}(\beta_\lambda)   \Gamma_b(\beta_\lambda) \right). 
\end{equation}
Moreover, 
\begin{equation}
    \mathcal{V}_\lambda^{\textrm{GLM}} + \mathcal{B}_\lambda^{\textrm{GLM}} = O(n^{-1}),  
\end{equation}
and there exists a constant $C_B>0$ depending only on the fixed population and GLM regularity parameters such that 
\begin{equation}
    0 \leq \mathcal{B}_\lambda^{\textrm{GLM}} \leq C_{B} \lambda^2 (1-\lambda)^2 \left( \frac{1}{n_r} + \frac{1}{n_b}\right) \norm{\beta_r - \beta_b}^2. 
\end{equation}
\end{corollary}

\Cref{thrm:glm_per_group} and \Cref{cor:bias_variance_glm} describe how estimation changes each group's expected risk and the $\lambda$-weighted risk, respectively. These results parallel the unknown covariance linear analysis in \Cref{sec:unknown_covariance}. 

The variance term $\mathcal{V}_g^{\textrm{GLM}}(\lambda)$ captures outcome noise conditional on covariates, and captures the same squared welfare weight and inverse sample size dependence as the linear setting. The nonnegative bias term $\mathcal{B}_g^{\textrm{GLM}}(\lambda)$ captures estimation errors when the planner estimates the compromise between the groups' preferred predictors. As in the linear setting, this term scales quadratically with the heterogeneity $\norm{\beta_r - \beta_b}$, and disappears when $\beta_r = \beta_b$ or $\lambda \in \{0, 1\}$. 

The signed redistribution term $\mathcal{D}_g^{\textrm{GLM}}(\lambda)$ is the analog of the linear cross term in \Cref{proposition:excess_risk_to_matrix_norm}. In the planner's $\lambda$-weighted objective, these groupwise terms cancel, as in the linear setting. In the groupwise risks, the terms persist and result in a systematic asymmetry across groups that is of leading order. 

\subsection{Proofs for GLM extension}
\subsubsection{Auxiliary GLM lemmas}
\begin{lemma}[Leading linear term decomposition]\label{lemma:glm_leading_term_bound}
    Suppose \Cref{assumption:invertible_cov}, \ref{assumption:bounded_covariates_glm}, \ref{assumption:bounded_residual_glm}, \ref{assumption:cov_lambda_glm}, and  \ref{assumption:interior_glm} hold. Then, for $\mathcal{W}_\lambda$ defined as in \eqref{eqn:GLM_delta_W_Q_lambda_def}, 
    \begin{equation}\label{eqn:glm_exact_decomp}
        \EE\left[ \norm{\mathcal{W}_\lambda}_{\mathcal{I}_g(\beta_\lambda)}^2\right] = \mathcal{V}_g^{\textrm{GLM}}(\lambda) + \mathcal{B}_g^{\textrm{GLM}}(\lambda)
    \end{equation}
    where
    \begin{equation}\label{eqn:glm_variance_1}
        \mathcal{V}_g^{\textrm{GLM}}(\lambda) \coloneq \lambda^2 \frac{\varphi_r}{n_r}\operatorname{Tr}\left\{\mathcal{M}_{g}(\beta_\lambda) \mathcal{I}_r(\beta_r) \right\} + (1-\lambda)^2\frac{\varphi_b}{n_b} \operatorname{Tr}\left(\mathcal{M}_{g}(\beta_\lambda)   \mathcal{I}_b(\beta_b) \right),
    \end{equation}
    \begin{equation}\label{eqn:glm_bias_exact_1}
        \mathcal{B}_g^{\textrm{GLM}}(\lambda) \coloneq \lambda^2 \frac{1}{n_r}\operatorname{Tr}\left\{\mathcal{M}_{g}(\beta_\lambda) \Gamma_r(\beta_\lambda) \right\} + (1-\lambda)^2\frac{1}{n_b} \operatorname{Tr}\left(\mathcal{M}_{g}(\beta_\lambda)   \Gamma_b(\beta_\lambda) \right),
    \end{equation}
    \begin{equation}\label{eqn:variance_inner_term_1}
        \Gamma_g(\beta_\lambda) \coloneq \operatorname{Var}_g\left(  (A'(X^\top \beta_\lambda)  - A'(X^\top \beta_g))X \right). 
    \end{equation}
    Furthermore, the term $\mathcal{B}_g^{\textrm{GLM}}$ can be upper bounded as 
    \begin{equation}\label{eqn:glm_bias_upper}
        \mathcal{B}_g^{\textrm{GLM}}(\lambda) \leq \bar{\mathcal{B}}_g^{\textrm{GLM}}(\lambda) \eqcolon \frac{\bar{a}^4}{\underline{a}^2} \frac{T_X^2}{\rho_\lambda^4} \lambda^2 (1-\lambda)^2 \left( \frac{\norm{\Sigma_b}^2}{n_r}  \operatorname{Tr}\left\{ \mathcal{M}_g(\beta_\lambda) \Sigma_r \right\} + \frac{\norm{\Sigma_r}^2}{n_b} \operatorname{Tr}\left\{ \mathcal{M}_g(\beta_\lambda) \Sigma_b \right\} \right) \norm{\beta_r - \beta_b}^2.
    \end{equation}
\end{lemma}
\begin{proof}[Proof of \Cref{lemma:glm_leading_term_bound}]
Write $\mathcal{M}_{g}(\beta_\lambda) \coloneq {\mathcal{I}}_\lambda^{-1}(\beta_\lambda)\mathcal{I}_g(\beta_\lambda) {\mathcal{I}}_\lambda^{-1}(\beta_\lambda)$. By linearity and cyclicity of the trace,  
\begin{equation}\label{eqn:glm_linear_term_proof_1}
   \EE\left[ \norm{\mathcal{W}_\lambda}_{\mathcal{I}_g(\beta_\lambda)}^2\right] =  \operatorname{Tr}\left\{\mathcal{M}_g(\beta_\lambda) \EE\left[ \widehat{S}_\lambda(\beta_\lambda) \widehat{S}_\lambda(\beta_\lambda)^\top\right] \right\}.
\end{equation}
 For a fixed population target $\beta_\lambda$, define the per-sample score function for each group $g \in \mathcal{G}$ as $s_g(X, Y ;\beta_\lambda) \coloneq (A'(X^\top \beta_\lambda) - Y)X$ and the score mean $\bar{s}_g(\beta_\lambda) \coloneq \EE_g\left[ s_g(X, Y ;\beta_\lambda)\right]$. Note that the first order condition $S_\lambda(\beta_\lambda) = 0$ gives 
 \begin{equation}\label{eqn:glm_proof_s_bar}
    \lambda \bar{s}_r + (1-\lambda) \bar{s}_b = S_\lambda(\beta_\lambda) = 0.
\end{equation}
Hence, $\EE\left[ \widehat{S}_\lambda(\beta_\lambda)\right] = 0$.
By independence within and between groups, 
\begin{equation}\label{eqn:glm_linear_term_proof_2}
    \EE\left[ \widehat{S}_\lambda(\beta_\lambda) \widehat{S}_\lambda(\beta_\lambda)^\top\right]  = \frac{\lambda^2}{n_r} \operatorname{Var}_r\left(s_r(X, Y; \beta_\lambda) \right) + \frac{(1-\lambda)^2}{n_b} \operatorname{Var}_b\left(s_b(X, Y; \beta_\lambda) \right). 
\end{equation}
By law of total variance, for each $g \in \mathcal{G}$, 
\begin{equation*}
    \operatorname{Var}_g\left(s_g(X, Y; \beta_\lambda) \right)= \EE\left[ \operatorname{Var}_g\left(s_g(X, Y; \beta_\lambda) \mid X\right) \right] + \operatorname{Var}_g\left( \EE\left[ s_g(X, Y; \beta_\lambda) \mid X \right]  \right). 
\end{equation*}
For the first term, write $s_g(X, Y; \beta_\lambda) = A'(X^\top \beta_\lambda) X - A'(X^\top \beta_g)X - \varepsilon_g X$ and apply \eqref{eqn:conditional_mean_variance_glm} to get 
\begin{equation*}
    \EE\left[ \operatorname{Var}_g\left(s_g(X, Y; \beta_\lambda) \mid X\right) \right] = \EE\left[ \operatorname{Var}_g\left(\varepsilon_g \mid X\right) X X^\top \right]= \varphi_g \EE\left[  A''(X^\top \beta_g) X X^\top \right] = \varphi_g \mathcal{I}_g(\beta_g). 
\end{equation*}
For the second term, use
\begin{equation*}
    \operatorname{Var}_g\left( \EE\left[ s_g(X, Y; \beta_\lambda) \mid X \right]  \right) = \operatorname{Var}_g\left(  A'(X^\top \beta_\lambda) X - A'(X^\top \beta_g)X \right) \eqcolon \Gamma_g(\beta_\lambda). 
\end{equation*}
Substitution into \eqref{eqn:glm_linear_term_proof_1} yields the exact decomposition \eqref{eqn:glm_exact_decomp}. 
By \eqref{eqn:glm_abs_bounds_a} and \Cref{assumption:bounded_covariates_glm}, respectively, for each $g \in \mathcal{G}$, 
\begin{equation}\label{eqn:glm_A_X_bound_score}
    |A'(X^\top \beta_\lambda)  - A'(X^\top \beta_g)|  \leq \bar{a} |X^\top (\beta_g - \beta_\lambda)|  \leq \bar{a} T_X \norm{\beta_\lambda - \beta_g}_2. 
\end{equation}
It follows that 
\begin{equation*}
    \operatorname{Var}_g\left(  A'(X^\top \beta_\lambda) X - A'(X^\top \beta_g)X \right)\leq \bar{a}^2 T_X^2 \norm{\beta_\lambda - \beta_g}_2^2 \Sigma_g. 
\end{equation*}
Substitution into \eqref{eqn:glm_linear_term_proof_1} and applying \Cref{lemma:glm_beta_bound} yields the upper bound \eqref{eqn:glm_bias_upper}. \end{proof}
\begin{lemma}\label{lemma:glm_score_moments}
    Suppose Assumptions \ref{assumption:bounded_covariates_glm}, \ref{assumption:bounded_residual_glm}, and \ref{assumption:interior_glm} hold. Then, for every $p \geq 2$, there exists a constant $C_p>0$ such that 
    \begin{equation}\label{eqn:glm_score_moment_bound}
        \left( \EE\left[ \norm{\widehat{S}_\lambda(\beta_\lambda)}^p\right] \right)^{1/p} \leq C_p T_X \left(\frac{\lambda \tau_r}{\sqrt{n_r}} +  \frac{(1-\lambda) \tau_b}{\sqrt{n_b}}\right), \quad \tau_g \coloneq \bar{a} T_X \norm{\beta_\lambda - \beta_g} + \sigma_{\varepsilon, g}.  
    \end{equation}
\end{lemma}
\begin{proof}[Proof of \Cref{lemma:glm_score_moments}]
Set $J_i^{(g)} \coloneq s_g(X_i^{(g)}, Y_i^{(g)} ;\beta_\lambda) - \bar{s}_g(\beta_\lambda) $ for all $i \in [n_g], g \in \mathcal{G}$ where $\bar{s}_g(\beta_\lambda) \coloneq \EE_g\left[ s_g(X, Y ;\beta_\lambda)\right]$. Recalling \eqref{eqn:glm_proof_s_bar}, we have
\begin{equation}\label{eqn:glm_hat_s_lambda_j}
    \widehat{S}_\lambda(\beta_\lambda) = \frac{\lambda}{n_r} \sum_{i=1}^{n_r} J_i^{(r)}+ \frac{(1-\lambda)}{n_b} \sum_{i=1}^{n_b} J_i^{(b)}. 
\end{equation}
By the generalized Marcinkiewicz–Zygmund inequality (as in, e.g., \citet[Proposition A.1]{yang2026efficient}), for each $g \in \mathcal{G}$, there exists a constant $c_p'$ such that 
\begin{equation*}
    \begin{aligned}
        \left( \EE\left[ \norm{\sum_{i=1}^{n_g} J_i^{(g)}}^p \right]\right)^{1/p}&\leq c_p'  \left(\EE\left[\left( \sum_{i=1}^{n_g} \norm{J_i^{(g)}}^2\right)^{p/2}  \right]\right)^{1/p}
        &\leq c_p'  \left( n_g^{p/2-1} \EE\left[ \sum_{i=1}^{n_g} \norm{J_i^{(g)}}^p  \right]\right)^{1/p}. 
    \end{aligned}
\end{equation*} 
The second inequality follows by Jensen's inequality using the convexity of the map $x \mapsto x^{p/2}$ for $p \geq 2$: 
\begin{equation*}
    \left( \frac{1}{n_g} \sum_{i=1}^{n_g} \norm{J_i^{(g)}}^2\right)^{p/2} \leq \frac{1}{n_g} \sum_{i=1}^{n_g} \norm{J_i^{(g)}}^{p}.
\end{equation*}
By Minkowski's inequality, for each $i \in [n_g], g \in \mathcal{G}$, 
\begin{equation*}
    \left( \EE\left[\norm{J_i^{(g)}}^p\right]\right)^{1/p}\leq \left(\EE\left[\norm{s_g(X_i^{(g)}, Y_i^{(g)} ;\beta_\lambda) }^p\right]\right)^{1/p} + \norm{\bar{s}_g(\beta_\lambda) }\leq 2 \left(\EE\left[\norm{s_g(X_i^{(g)}, Y_i^{(g)} ;\beta_\lambda) }^p\right]\right)^{1/p}. 
\end{equation*}
Write $s_g(X, Y; \beta_\lambda) = (A'(X^\top \beta_\lambda) - A'(X^\top \beta_g))X - \varepsilon_g X$ and apply conditional Minkowski's inequality to obtain
\begin{equation*}
    \begin{aligned}
        \left(\EE\left[\norm{s_g(X_i^{(g)}, Y_i^{(g)} ;\beta_\lambda) }^p  \mid X\right]\right)^{1/p} &\leq \left( \EE\left[ \norm{(A'(X^\top \beta_\lambda) - A'(X^\top \beta_g))X}^p \mid X \right] \right)^{1/p} + \norm{X} \left( \EE\left[ |\varepsilon_g |^p \mid X \right] \right)^{1/p}\\
        &\leq \bar{a} T_X^2 \norm{\beta_\lambda- \beta_g} + \norm{X} C_1 \sqrt{p} \sigma_{\varepsilon, g}
    \end{aligned}
\end{equation*}
for a constant $C_1 > 0$ where the second inequality follows from \eqref{eqn:glm_A_X_bound_score} and \Cref{assumption:bounded_residual_glm}. Taking an expectation with respect to $X$ and using $\eqref{eqn:glm_A_X_bound_score}$ yields the bound 
\begin{equation*}
    \left(\EE\left[\norm{s_g(X_i^{(g)}, Y_i^{(g)} ;\beta_\lambda) }^p\right]\right)^{1/p} \leq C_p T_X \tau_g.
\end{equation*}
An application of Minkowski's inequality to the norm of \eqref{eqn:glm_hat_s_lambda_j} and substitution of the preceding bound yields \eqref{eqn:glm_score_moment_bound}. \end{proof}
\begin{lemma}\label{lemma:I_bar_norm_bound}
    Under \Cref{assumption:bounded_covariates_glm}, for $\bar{\mathcal{I}}_\lambda$ defined in \eqref{eqn:glm_I_bar} and $\check{\beta}_\lambda \in \Theta_B$, 
    \begin{equation}\label{eqn:I_lambda_H_bound_glm}
        \norm{ \bar{\mathcal{I}}_\lambda-\widehat{\mathcal{I}}_\lambda(\beta_\lambda)}^2 \leq \frac{L_3^2}{4} T_X^6 \norm{\check{\beta}_\lambda - \beta_\lambda}^2
    \end{equation}
\end{lemma}
\begin{proof}[Proof of \Cref{lemma:I_bar_norm_bound}]
First write 
\begin{equation*}
    \bar{\mathcal{I}}_\lambda-\widehat{\mathcal{I}}_\lambda(\beta_\lambda) = \frac{\lambda}{n_r}\sum_{i=1}^{n_r}\int_{0}^1 Z_i^{(r)}(t) X_i^{(r)} {X_i^{(r)}}^\top dt + \frac{(1-\lambda)}{n_b}\sum_{i=1}^{n_b}\int_{0}^1 Z_i^{(b)}(t) X_i^{(b)} {X_i^{(b)}}^\top dt
\end{equation*}
with 
\begin{equation*}
    Z_i^{(g)}(t) \coloneq A''\left({X_i^{(g)}}^\top (\beta_{\lambda} + t\delta_\lambda ) \right)-A''({X_i^{(g)}}^\top \beta_{\lambda})
\end{equation*}
where $\delta_\lambda$ is defined as in \eqref{eqn:GLM_delta_W_Q_lambda_def}. 
Using \eqref{eqn:glm_abs_bounds_a} and  \Cref{assumption:bounded_covariates_glm}, for each $i, g, t$,
\begin{equation*}
    |Z_i^{(g)}(t)| \leq L_3 t \left| {X_i^{(g)}}^\top \delta_\lambda  \right| \leq L_3 t T_X \norm{\delta_\lambda}_2. 
\end{equation*}
\Cref{assumption:bounded_covariates_glm} also yields for each $i, g,$ that $\norm{{X_i^{(g)}}{X_i^{(g)}}^\top} \leq T_X^2$. 

Therefore, combining these bounds with triangle inequality and evaluating the integrals, 
\begin{equation*}
    \norm{ \bar{\mathcal{I}}_\lambda-\widehat{\mathcal{I}}_\lambda(\beta_\lambda)} \leq \frac{\lambda}{2}  L_3 T_X^3 \norm{\delta_\lambda} + \frac{(1-\lambda)}{2}  L_3 T_X^3 \norm{\delta_\lambda} \leq \frac{1}{2} L_3 T_X^3 \norm{\delta_\lambda}
\end{equation*}
and squaring yields the claim. 
\end{proof}
\begin{lemma}\label{lemma:glm_high_probability_event_bound}
    Suppose Assumptions \ref{assumption:bounded_covariates_glm},  \ref{assumption:bounded_residual_glm}, \ref{assumption:cov_lambda_glm}, and \ref{assumption:interior_glm} hold. Then, the event
    \begin{equation}\label{eqn:hp_event_lambda_glm_def}
        \mathcal{A}_\lambda \coloneq \left\{ \widehat{\Sigma}_\lambda \succeq \frac{1}{2} \rho_\lambda^2 I_d\right\}\cap\left\{ \norm{\widehat{S}_\lambda(\beta_\lambda)} \leq \frac{1}{8} \underline{a} \rho_\lambda^2 \delta_B\right\} 
    \end{equation}
    satisfies 
    \begin{equation*}
        \PP(\mathcal{A}_\lambda^c) \leq \left(\frac{8 T_X\tilde{\psi}_\lambda }{\underline{a} \rho_\lambda^2 \delta_B}\right)^4 + d\exp\left( - \frac{1}{2} ( {1-\log(2)}) \frac{\rho_\lambda^2 }{T_X^2 \max\{ \lambda/n_r, (1-\lambda)/n_b\}} \right). 
    \end{equation*}
    where $\tilde{\psi}_\lambda$ is defined as in \eqref{eqn:Q_norm_decomp_2_sub3}. 
\end{lemma}
\begin{proof}[Proof of \Cref{lemma:glm_high_probability_event_bound}]
By the union bound followed by Markov's inequality and applying  \Cref{lemma:glm_score_moments} with $p = 4$, we obtain 
\begin{equation}
    \PP(\mathcal{A}_\lambda^c) \leq \PP\left(  \widehat{\Sigma}_\lambda \nsucceq \frac{1}{2} \rho_\lambda^2 I_d\right) + \PP\left(  \norm{\widehat{S}_\lambda(\beta_\lambda)} > \frac{1}{8} \underline{a} \rho_\lambda^2 \delta_B \right) \leq \PP\left(  \widehat{\Sigma}_\lambda \nsucceq \frac{1}{2} \rho_\lambda^2 I_d\right) + \left(\frac{8 T_X\tilde{\psi}_\lambda }{\underline{a} \rho_\lambda^2 \delta_B}\right)^4.  
\end{equation}
Next, write 
\begin{equation*}
    \PP\left(  \widehat{\Sigma}_\lambda \nsucceq \frac{1}{2} \rho_\lambda^2 I_d\right) = \PP\left( \lambda_{\min}(\widehat{\Sigma}_\lambda) < \frac{\rho_\lambda^2}{2}  \right). 
\end{equation*}
Under \Cref{assumption:bounded_covariates_glm}, and using $\lambda_{\min}(\EE[\widehat{\Sigma}_\lambda]) = \lambda_{\min}({\Sigma}_\lambda)\geq \rho_\lambda^2$, matrix Chernoff (e.g., \citet[Theorem 5.1.1]{tropp2015introduction}) yields 
\begin{equation*}
    \PP\left( \lambda_{\min}(\widehat{\Sigma}_\lambda) < \frac{\rho_\lambda^2}{2}  \right) \leq \PP\left( \lambda_{\min}(\widehat{\Sigma}_\lambda) < \frac{\lambda_{\min}(\Sigma_\lambda)}{2}  \right) \leq d\exp\left( - \frac{{1-\log(2)}}{2} \frac{\lambda_{\min}({\Sigma}_\lambda) }{L_\lambda} \right) \leq d\exp\left( - \frac{{1-\log(2)}}{2}  \frac{\rho_\lambda^2 }{L_\lambda} \right)
\end{equation*}
where $L_\lambda \coloneq T_X^2 \max\{ \lambda/n_r, (1-\lambda)/n_b\}$ and the third inequality follows by applying \Cref{assumption:cov_lambda_glm}. The claim follows. 
\end{proof}
\begin{lemma}\label{lemma:glm_delta_lambda_bound_linearization}
     Suppose Assumptions \ref{assumption:invertible_cov}, \ref{assumption:bounded_covariates_glm}, \ref{assumption:cov_lambda_glm}, and \ref{assumption:interior_glm} hold.  
    On the event $\mathcal{A}_\lambda$ defined in \eqref{eqn:hp_event_lambda_glm_def}, we have that $\widehat{\mathcal{R}}_\lambda$ is $\frac{1}{2} \underline{a} \rho_\lambda^2$-strongly convex on $\{\beta \in \RR^d: \norm{\beta} \leq B\}$, $\widehat{S}_\lambda(\check{\beta}_\lambda) = 0$, and, for $\delta_\lambda$ defined as in \eqref{eqn:GLM_delta_W_Q_lambda_def},
    \begin{equation}\label{eqn:delta_lambda_norm_bound}
        \norm{\delta_\lambda} \leq \frac{2}{\underline{a} \rho_\lambda^2} \norm{\widehat{S}_\lambda(\beta_\lambda)}.
    \end{equation}
\end{lemma}
\begin{proof}[Proof of \Cref{lemma:glm_delta_lambda_bound_linearization}]
By \Cref{assumption:bounded_covariates_glm}, for any feasible $\beta \in \RR^d$, $g \in \mathcal{G}$ satisfying $\norm{\beta - \beta_\lambda} \geq \delta_B$ and $i \in [n_g]$, we have $|{X_i^{(g)}}^\top \beta|\leq T_X B$. Then, by \eqref{eqn:glm_abs_bounds_a} and \Cref{assumption:cov_lambda_glm}, for any $u, \beta \in \RR^d$, $u^\top \widehat{\mathcal{I}}_\lambda(\beta) u \geq \underline{a} u^\top \widehat{\Sigma}_\lambda u$.  Therefore, on $\mathcal{A}_\lambda$, we have $u^\top \widehat{\mathcal{I}}_\lambda(\beta) u \geq \frac{1}{2} \underline{a} \rho_\lambda^2  \norm{u}^2$. 
Hence, $\widehat{\mathcal{R}}_\lambda$ is $(\frac{1}{2} \underline{a} \rho_\lambda^2)$-strongly convex on the feasible set. It follows that, for any feasible $\beta \in \RR^d$ satisfying $\norm{\beta - \beta_\lambda} \geq \delta_B$, on $\mathcal{A}_\lambda$, 
\begin{equation}\label{eqn:delta_B_contra_ineq}
    \begin{aligned}
        \widehat{\mathcal{R}}_\lambda(\beta) - \widehat{\mathcal{R}}_\lambda(\beta_\lambda) &\geq \widehat{S}_\lambda(\beta_\lambda)^\top (\beta - \beta_\lambda) + \frac{1}{4}  \rho_\lambda^2 \underline{a} \norm{\beta_\lambda - \beta}^2\\
        & \geq - \norm{\widehat{S}_\lambda(\beta_\lambda)} \norm{\beta - \beta_\lambda} + \frac{1}{4}  \rho_\lambda^2 \underline{a} \norm{\beta_\lambda - \beta}^2\\
        & \geq \underline{a} \rho_\lambda^2 \norm{\beta_\lambda - \beta} \left( - \frac{1}{8}   \delta_B  + \frac{1}{4}  \norm{\beta_\lambda - \beta}\right)\\
        & \geq \frac{1}{8} \underline{a} \rho_\lambda^2 \norm{\beta_\lambda - \beta}  \delta_B 
    \end{aligned}  
\end{equation}
where the first inequality follows from strong convexity, the second by Cauchy-Schwarz, the third uses the definition of $\mathcal{A}_\lambda$, and the last uses $\norm{\beta - \beta_\lambda} \geq \delta_B$. In contrast, $\check{\beta}_\lambda$ satisfies by construction the inequality 
\begin{equation*}
    \widehat{\mathcal{R}}_\lambda(\check{\beta}_\lambda) - \widehat{\mathcal{R}}_\lambda(\beta_\lambda) \leq 0,
\end{equation*}
and hence must satisfy $\norm{\beta_\lambda - \check{\beta}_\lambda} < \delta_B$ on $\mathcal{A}_\lambda$ by contradiction with \eqref{eqn:delta_B_contra_ineq}. Then, by triangle inequality and using $\norm{\beta_\lambda} \leq B - \delta_B$, 
\begin{equation*}
    \norm{\check{\beta}_\lambda} \leq \norm{\beta_\lambda} + \norm{\check{\beta}_\lambda - \beta_\lambda} < (B - \delta_B) + \delta_B = B. 
\end{equation*}
Hence, on $\mathcal{A}_\lambda$, $\check{\beta}_\lambda$ is interior and satisfies $\widehat{S}_\lambda(\check{\beta}_\lambda) = 0$. Thus, by a Taylor expansion, $0 = \widehat{S}_\lambda(\check{\beta}_\lambda) = \widehat{S}_\lambda(\beta_\lambda) + \bar{\mathcal{I}}_\lambda \delta_\lambda$ where $\bar{\mathcal{I}}_\lambda$ is defined in \eqref{eqn:glm_I_bar}. 
Under \Cref{assumption:invertible_emp_cov} and recalling \eqref{eqn:glm_abs_bounds_a}, the averaged Hessian $\bar{\mathcal{I}}_\lambda$ is invertible. Indeed for any $u \neq 0$,  
\begin{equation}\label{eqn:I_bar_lambda_psd_bound}
    \begin{aligned}
        u^\top \bar{\mathcal{I}}_\lambda u = \int_{0}^1 u^\top \widehat{\mathcal{I}}_\lambda(\beta_\lambda + t\delta_\lambda) u \; dt
        \geq \underline{a} u^\top \widehat{\Sigma}_\lambda u > 0.
    \end{aligned}
\end{equation}
It follows that $\delta_\lambda = -\bar{\mathcal{I}}_\lambda^{-1} \widehat{S}_\lambda(\beta_\lambda)$.  Finally, since on $\mathcal{A}_\lambda$, $\bar{\mathcal{I}}_\lambda\succeq \frac{1}{2} \underline{a} \rho_\lambda^2$, $\norm{\bar{\mathcal{I}}_\lambda^{-1}} \leq \frac{2}{\underline{a} \rho_\lambda^2}$ and \eqref{eqn:delta_lambda_norm_bound} follows.
\end{proof}
\begin{lemma}\label{lemma:glm_hessian_concentration}
    Suppose \Cref{assumption:bounded_covariates_glm} holds and that $d \geq 2$ and $\min\{n_r, n_b\} \geq \log(2ed)$. Then, for an absolute constant $C >0$, 
    \begin{equation}\label{eqn:glm_hessian_concentration}
        \left( \EE\left[ \norm{\widehat{\mathcal{I}}_\lambda(\beta_\lambda)  - \mathcal{I}_\lambda(\beta_\lambda) }^4 \right] \right)^{1/4} \leq C \bar{a} T_X^2 \sqrt{\log(2ed)}\left(  \frac{\lambda}{\sqrt{n_r}}   +  \frac{1-\lambda}{\sqrt{n_b}}\right). 
    \end{equation}
\end{lemma}
\begin{proof}[Proof of \Cref{lemma:glm_hessian_concentration}]
   Throughout the proof, we use the shorthand $\Delta_{\mathcal{I}, \lambda} \coloneq \widehat{\mathcal{I}}_\lambda(\beta_\lambda)-{\mathcal{I}}_\lambda(\beta_\lambda)$ and $\Delta_{\mathcal{I}, g} \coloneq \widehat{\mathcal{I}}_g(\beta_\lambda)-{\mathcal{I}}_g(\beta_\lambda)$. 
By Minkowski's inequality, 
\begin{equation}\label{eqn:glm_hessian_lambda_g_decomp}
    \left(\EE\left[\norm{\Delta_{\mathcal{I}, \lambda}}^4 \right]\right)^{1/4}\leq \lambda \left(\EE\left[\norm{\Delta_{\mathcal{I}, r}}^4 \right]\right)^{1/4} + (1-\lambda)\left(\EE\left[\norm{\Delta_{\mathcal{I}, b}}^4 \right]\right)^{1/4}. 
\end{equation}
For each $g \in \mathcal{G}$ and each $i \in [n_g]$, define $B_i^{(g)} \coloneq A''({X_i^{(g)}}^\top \beta_\lambda){X_i^{(g)}}{X_i^{(g)}}^\top$ and $Z_i^{(g)} \coloneq B_i^{(g)} - \mathcal{I}_g(\beta_\lambda)$ where we note $\EE_g[B_i^{(g)}] = \mathcal{I}_g(\beta_\lambda)$. By \eqref{eqn:glm_abs_bounds_a}, $\norm{\mathcal{I}_g(\beta_\lambda) }\leq \bar{a} T_X^2$. Also by \eqref{eqn:glm_abs_bounds_a}, $\norm{B_i^{(g)}} \leq \bar{a} T_X^2$, hence $\norm{Z_i^{(g)}} \leq 2 \bar{a} T_X^2$ by triangle inequality. 
Observe that
\begin{equation*}
    \EE\left[(Z_i^{(g)})^2 \right] = \EE\left[ (B_i^{(g)})^2 \right] - \mathcal{I}_g(\beta_\lambda)^2 \preceq \EE\left[ (B_i^{(g)})^2 \right] \preceq \bar{a} T_X^2 \EE\left[ B_i^{(g)} \right] = \bar{a} T_X^2 \mathcal{I}_g(\beta_\lambda).
\end{equation*}
Therefore, the matrix variance parameter satisfies
\begin{equation}\label{eqn:glm_matrix_bernstein_var_bound}
    \mathfrak{s}_g \coloneq \norm{\sum_{i=1}^{n_g} \EE\left[(Z_i^{(g)})^2 \right]} \leq n_g \bar{a} T_X^2 \norm{\mathcal{I}_g(\beta_\lambda)} \leq n_g \bar{a}^2 T_X^4. 
\end{equation}
Therefore, by matrix Bernstein (e.g., \citet[Theorem 6.1.1]{tropp2015introduction}), for every $t > 0$, 
\begin{equation*}
    \PP\left( n_g \norm{\Delta_{\mathcal{I}, g}} \geq  t \right) = \PP\left( \norm{\sum_{i = 1}^{n_g} Z_i^{(g)}} \geq t \right) \leq 2d \exp\left( \frac{-t^2}{2(\mathfrak{s}_g  + 2 \bar{a} T_X^2 t/3)}  \right). 
\end{equation*}
It follows that, for every $u > 0$, 
\begin{equation*}
    \PP\left( \norm{\sum_{i = 1}^{n_g} Z_i^{(g)}} \geq \sqrt{2 u \mathfrak{s}_g} + 
    \frac{4}{3} \bar{a} T_X^2 u \right) \leq 2d e^{-u}. 
\end{equation*}
Using \eqref{eqn:glm_matrix_bernstein_var_bound} and dividing by $n_g$ gives
\begin{equation}\label{eqn:glm_bernstein_hpb}
    \PP\left( \norm{\Delta_{\mathcal{I}, g}} \geq q_g(u)\right) \leq 2d e^{-u}, \quad q_g(u) \coloneq \bar{a} T_X^2 \left(\sqrt{\frac{2 u  }{n_g}}  +  \frac{4u}{3n_g} \right). 
\end{equation}
 We next bound the fourth moment by integrating the tail bound, splitting the integral at $q_g(\log(2ed))$. 
\begin{equation*}
    \EE\left[ \norm{\Delta_{\mathcal{I}, g}}^4 \right] = 4 \int_{0}^\infty t^3 \PP\left( \norm{\Delta_{\mathcal{I}, g}} \geq t \right) dt = 4 \int_{0}^{q_g(\log(2ed))} t^3 \PP\left( \norm{\Delta_{\mathcal{I}, g}} \geq t \right) dt + 4 \int_{q_g(\log(2ed))}^\infty t^3 \PP\left( \norm{\Delta_{\mathcal{I}, g}} \geq t \right) dt. 
\end{equation*}
For the first term, use the trivial upper bound $\PP\left( \norm{\Delta_{\mathcal{I}, g}} \geq t \right) \leq 1$ and integrate to obtain the bound 
\begin{equation*}
    \int_{0}^{q_g(\log(2ed))} t^3 \PP\left( \norm{\Delta_{\mathcal{I}, g}} \geq t \right) dt \leq q_g(\log(2ed))^4. 
\end{equation*}
For the second term, use the change of variables $q_g(u) = t$, the bound \eqref{eqn:glm_bernstein_hpb}, and integrate by parts to bound
\begin{equation*}
    \begin{aligned}
        \int_{q_g(\log(2ed))}^\infty t^3 \PP\left( \norm{\Delta_{\mathcal{I}, g}} \geq t \right) dt 
        &\leq 2d \int_{\log(2ed)}^\infty q_g(u)^3 q_g'(u) e^{-u} du\\
        &= -\frac{d}{2} q_g(\log(2ed))^4 \exp(-\log(2ed)) + \frac{d}{2} \int_{\log(2ed)}^\infty q_g(u)^4 e^{-u}du\\
        &\leq \frac{d}{2} \int_{\log(2ed)}^\infty q_g(u)^4 e^{-u}du. 
    \end{aligned}
\end{equation*}
Next, using that $(a + b)^4 \leq 8(a^4 + b^4)$, we bound
\begin{equation}\label{eqn:glm_q_g_fourth_bound}
    q_g(u)^4 \leq C_1 \bar{a}^4 T_X^8 \left( \frac{u^2}{n_g^2} + \frac{u^4}{n_g^4} \right)
\end{equation}
for some constant $C_1>0$. Moreover, since $d \geq 2$ and $2 d \exp(-\log(2ed)) = e^{-1}$,  
\begin{equation*}
    \frac{d}{2} \int_{\log(2ed)}^\infty u^2 e^{-u}du \leq C_2  \log(2ed)^2, \quad \frac{d}{2} \int_{\log(2ed)}^\infty u^4 e^{-u}du \leq C_3 \log(2ed)^4
\end{equation*}
for constants $C_2, C_3 \geq 0$. Applying \eqref{eqn:glm_q_g_fourth_bound} to $q_g(\log(2ed))^4$ and substituting the term-wise bounds, we reach 
\begin{equation*}
    \EE\left[ \norm{\Delta_{\mathcal{I}, g}}^4 \right] \leq C_4 \bar{a}^4 T_X^8 \left( \frac{\log(2ed)^2}{n_g^2} + \frac{\log(2ed)^4}{n_g^4} \right)
\end{equation*}
for a constant $C_4 > 0$. Taking fourth roots, using $n_g \geq \log(2ed)$, and  substituting each groupwise bound into \eqref{eqn:glm_hessian_lambda_g_decomp} yields \eqref{eqn:glm_hessian_concentration}. 
\end{proof}
\begin{lemma}\label{lemma:glm_beta_bound}
    Suppose that \Cref{assumption:bounded_covariates_glm},  \Cref{assumption:cov_lambda_glm}, and \Cref{assumption:interior_glm}, and \eqref{eqn:glm_abs_bounds_a} hold. Then, 
    \begin{equation}\label{eqn:glm_beta_gap_r}
        \norm{\beta_\lambda - \beta_r}^2 \leq (1-\lambda)^2 \left( \frac{\bar{a}}{\underline{a}}\right)^2  \frac{\norm{\Sigma_b}^2}{\rho_\lambda^4} \norm{\beta_r - \beta_b}^2,  
    \end{equation} 
    and a symmetric result holds for the group $b$.
\end{lemma}
\begin{proof}[Proof of \Cref{lemma:glm_beta_bound}]
By Taylor's theorem,
\begin{equation}\label{eqn:glm_beta_gap_taylor}
    \langle {S}_\lambda(\beta_r) - {S}_\lambda(\beta_\lambda), \beta_r - \beta_\lambda \rangle = \int_{0}^1 (\beta_r - \beta_\lambda)^\top \mathcal{I}_\lambda(\beta_\lambda + t (\beta_r - \beta_\lambda)) (\beta_r - \beta_\lambda) dt. 
\end{equation}
By \eqref{eqn:glm_abs_bounds_a}, $\mathcal{I}_\lambda(\beta) \succeq \underline{a} \Sigma_\lambda$. Therefore, by \Cref{assumption:cov_lambda_glm}, $\mathcal{I}_\lambda(\beta) \succeq \underline{a} \rho_\lambda^2 I_d$ for every $\beta \in \Theta_B$. Applying this bound to \eqref{eqn:glm_beta_gap_taylor} yields 
\begin{equation}
    \langle S_\lambda(\beta_r) - S_\lambda(\beta_\lambda), \beta_r - \beta_\lambda \rangle \geq \underline{a} \rho_\lambda^2 \norm{\beta_r - \beta_\lambda}^2.  
\end{equation}
By \Cref{assumption:interior_glm}, we have $S_\lambda(\beta_\lambda) = 0$ and applying Cauchy-Schwarz, we obtain 
\begin{equation}\label{eqn:glm_beta_gap_1}
     \norm{\beta_r - \beta_\lambda} \leq \frac{1}{\underline{a} \rho_\lambda^2}  \norm{S_\lambda(\beta_r)}.  
\end{equation}
By \eqref{eqn:conditional_mean_variance_glm}, $S_r(\beta_r) = 0$ and it follows that $S_\lambda(\beta_r) = (1-\lambda)S_b(\beta_r)$. It remains to bound $\norm{S_b(\beta_r)}$. By Taylor's theorem, 
\begin{equation*}
    A'(X^\top \beta_r) - A'(X^\top \beta_b) = \int_{0}^1 A''(X^\top (\beta_b + t(\beta_r - \beta_b))) X^\top (\beta_r - \beta_b) dt.   
\end{equation*}
Define 
\begin{equation*}
    \mathcal{H}_b \coloneq \int_{0}^1 \EE_b\left[ A''(X^\top (\beta_b + t(\beta_r - \beta_b))) X X^\top\right] dt
\end{equation*}
and note that, by \eqref{eqn:glm_abs_bounds_a}, $\mathcal{H}_b \preceq \bar{a} \Sigma_b$. Hence, using $S_b(\beta_r) = \mathcal{H}_b (\beta_r - \beta_b)$, we reach  
\begin{equation*}
    \norm{S_b(\beta_r)} \leq \bar{a} \norm{\Sigma_b} \norm{\beta_r - \beta_b}. 
\end{equation*}
The claim follows by substitution into \eqref{eqn:glm_beta_gap_1}. 
\end{proof}
\begin{lemma}\label{lemma:frontier_glm_convex}
    Suppose \Cref{assumption:invertible_cov} and \Cref{assumption:bounded_covariates_glm} hold. Then, in the setting of \Cref{def:glm}, for every $\lambda \in [0,1]$, the population minimizer $\beta_\lambda$ is unique. If additionally \eqref{eqn:glm_group_balanced_condition} holds so that the model is group balanced, then the risk pairs generated by $\{\beta_\lambda:\lambda \in [0,1]\}$ constitute the entire FA frontier $\mathcal{F}(\mathcal{H})$. If, moreover, $\beta_r \neq \beta_b$, then the FA frontier is strictly convex.
\end{lemma}
\begin{proof}[Proof of \Cref{lemma:frontier_glm_convex}]
    Throughout, adopt the same notation as in the proof of \Cref{lemma:frontier_regression_convex}. As shown in  \Cref{lemma:glm_beta_bound}, under \Cref{assumption:bounded_covariates_glm} via \eqref{eqn:glm_abs_bounds_a}, we have that $\mathcal{R}_\lambda$ is $\underline{a} \rho_\lambda^2$-strongly convex on $\Theta_B$. Thus, $\mathcal{R}_\lambda$ has a unique minimizer $\beta_\lambda$ on $\Theta_B$. By \eqref{eqn:conditional_mean_variance_glm}, $S_g(\beta_g) = 0$ for each $g \in \mathcal{G}$. Since each group's risk is strictly convex, $\beta_g$ uniquely minimizes $\mathcal{R}_g$ on $\Theta_B$. Hence $\beta_0 = \beta_b, \beta_1 = \beta_r$. 

    Next, for $\lambda \in (0, 1)$, the risk pair $(\mathcal{R}_r(\beta_\lambda), \mathcal{R}_b(\beta_\lambda))$ is Pareto efficient by the same argument as in the proof of \Cref{lemma:frontier_regression_convex}. For the endpoints $\lambda \in \{0, 1\}$, Pareto efficiency follows from uniqueness of $\beta_r, \beta_b$.  
    
    Conversely, fix any Pareto-efficient point $e^\star = (e_r^\star, e_b^\star) = (\mathcal{R}_r(\beta^\star), \mathcal{R}_b(\beta^\star)) \in \mathcal{E}(\mathcal{H})$. We claim that 
    \begin{equation}\label{eqn:claim_empty_intersection_glm_convexity}
        \operatorname{conv}(\mathcal{E}(\mathcal{H})) \cap (e^\star - \RR_{> 0}^2) = \emptyset
    \end{equation}
    where $\operatorname{conv}$ denotes the convex hull. For the sake of contradiction, suppose there existed a point $r = (r_r, r_b) \coloneq \sum_{i=1}^k w_i e(\beta_i) \in e^\star - \RR_{> 0}^2$ with each $w_i \geq 0$, $\sum_{i=1}^k w_i = 1$, $\beta_i \in \Theta_B$. By convexity of $\Theta_B$, $\bar{\beta} \coloneq \sum_{i=1}^k w_i \beta_i \in \Theta_B$. By the convexity of the risk, for each $g \in \mathcal{G}$, 
    \begin{equation}
        e_g(\bar{\beta}) \leq \sum_{i=1}^k w_i e_g(\beta_i) = r_g < e_g^\star.  
    \end{equation}
    In other words, $e(\bar{\beta})$ would improve both group risks relative to $e^\star$, contradicting its Pareto efficiency. This proves \eqref{eqn:claim_empty_intersection_glm_convexity}. By the separating hyperplane theorem, there exists $a = (a_r, a_b)$, $a \neq 0$ with $a_r, a_b \geq 0$ and $c \in \RR$ such that 
    \begin{equation}
        a^\top v \geq c \geq a^\top \tilde{e} 
    \end{equation}
    for all $v \in \operatorname{conv}(\mathcal{E}(\mathcal{H}))$, $\tilde{e}  \in e^\star - \RR_{> 0}^2$. Taking $v = e^\star$ and $\tilde{e} = e^\star -  (\delta, \delta)$ with $\delta > 0$ in the preceding inequality yields  
    \begin{equation}
        a^\top e^\star \geq c \geq a^\top e^\star - \delta(a_r + a_b). 
    \end{equation}
    Taking $\delta \to 0$ gives $c = a^\top e^\star$. Thus, 
    \begin{equation}
        a^\top v \geq a^\top e^\star
    \end{equation}
    for all $v \in \operatorname{conv}(\mathcal{E}(\mathcal{H}))$. By the same normalization argument as in \Cref{lemma:frontier_regression_convex} and using that $e(\beta) \in \operatorname{conv}(\mathcal{E}(\mathcal{H}))$ for all $\beta \in \Theta_B$, we reach that $\mathcal{R}_\lambda(\beta)\geq \mathcal{R}_\lambda(\beta^\star)$ for all $\beta \in \Theta_B$. Thus, $\beta^\star$ minimizes $\mathcal{R}_\lambda$ over $\Theta_B$. By uniqueness, $\beta^\star = \beta_\lambda$. Thus, the Pareto frontier is exactly $\{e(\beta_\lambda): \lambda \in [0, 1]\}$. 

    Next, we establish the monotonicity property of the groupwise risks needed to apply the same FA dominance argument from \Cref{lemma:frontier_regression_convex}. By definition of $\beta_\lambda$, for $0 \leq \lambda_1 < \lambda_2 \leq 1$, 
    \begin{equation}\label{eqn:glm_convexity_beta_lambda_1_2_a}
        \lambda_1(\mathcal{R}_r(\beta_{\lambda_2}) - \mathcal{R}_r(\beta_{\lambda_1})) + (1-\lambda_1)(\mathcal{R}_b(\beta_{\lambda_2}) - \mathcal{R}_b(\beta_{\lambda_1})) \geq 0, 
    \end{equation}
    \begin{equation}\label{eqn:glm_convexity_beta_lambda_1_2_b}
        \lambda_2(\mathcal{R}_r(\beta_{\lambda_2}) - \mathcal{R}_r(\beta_{\lambda_1})) + (1-\lambda_2)(\mathcal{R}_b(\beta_{\lambda_2}) - \mathcal{R}_b(\beta_{\lambda_1})) \leq 0.  
    \end{equation}
    Multiplying \eqref{eqn:glm_convexity_beta_lambda_1_2_a} by $(1-\lambda_2)$ and \eqref{eqn:glm_convexity_beta_lambda_1_2_b} by $1-\lambda_1$ and subtracting, we obtain $(\lambda_1-\lambda_2)(\mathcal{R}_r(\beta_{\lambda_2}) - \mathcal{R}_r(\beta_{\lambda_1}))\geq 0$ and hence that $\mathcal{R}_r(\beta_{\lambda_2}) - \mathcal{R}_r(\beta_{\lambda_1}) \leq 0$. A symmetric argument yields $(\lambda_2 - \lambda_1)(\mathcal{R}_b(\beta_{\lambda_2}) - \mathcal{R}_b(\beta_{\lambda_1}))\geq 0$ and hence that $\mathcal{R}_b(\beta_{\lambda_2}) - \mathcal{R}_b(\beta_{\lambda_1}) \geq 0$. Altogether, we obtain that $\mathcal{R}_r(\beta_\lambda)$ is nonincreasing in $\lambda$ and $\mathcal{R}_b(\beta_\lambda)$ is nondecreasing in $\lambda$ along the Pareto frontier.

    Next, by the same arguments as in \Cref{lemma:frontier_regression_convex}, every Pareto-efficient pair is FA-efficient and, under group-balancedness, we may restrict attention to $\beta_r \neq \beta_b$. 

    For $\beta_r \neq \beta_b$, under group-balancedness (by the same argument in \Cref{lemma:frontier_regression_convex}), by continuity there exists $\lambda_{\textrm{fair}} \in [0,1]$ such that $\mathcal{R}_r(\beta_{\lambda_{\textrm{fair}}}) = \mathcal{R}_b(\beta_{\lambda_{\textrm{fair}}}) \eqcolon \mathcal{R}_{\textrm{fair}}$. By the same casework (i)-(iv) from the proof of \Cref{lemma:frontier_regression_convex}, we establish that every feasible point outside of the Pareto frontier is FA-dominated. 

    It remains to establish strict convexity. For $\beta_r \neq \beta_b$, uniqueness of $\beta_r = \beta_1$ implies $\mathcal{R}_r(\beta_1) <  \mathcal{R}_r(\beta_0)$. By continuity and monotonicity, the red group risk values on the frontier cover $[\mathcal{R}_r(\beta_1),\mathcal{R}_r(\beta_0)]$. Since two Pareto-efficient pairs cannot correspond to the same red group risk but have different blue group risks, as in \Cref{lemma:frontier_regression_convex}, we can write the frontier as $\mathcal{R}_b(\beta_\lambda) = f(\mathcal{R}_r(\beta_\lambda))$ for a function $f:[\mathcal{R}_r(\beta_1),\mathcal{R}_r(\beta_0)]\to \RR$. Next, fix distinct $r_1, r_2 \in [\mathcal{R}_r(\beta_1),\mathcal{R}_r(\beta_0)]$ and choose $b_1, b_2 \in \Theta_B$ such that $(\mathcal{R}_r(b_1), \mathcal{R}_b(b_1)) = (r_1, f(r_1))$ and $(\mathcal{R}_r(b_2), \mathcal{R}_b(b_2)) = (r_2, f(r_2))$.  Fix $t \in (0, 1)$ and define 
    \begin{equation}
        \bar{b} \coloneq t b_1 + (1-t)b_2, \quad \bar{r} \coloneq t r_1 + (1-t) r_2.
    \end{equation}
    Strict convexity of the group risks implies 
    \begin{equation}
        \mathcal{R}_r(\bar{b}) < \bar{r}, \quad \mathcal{R}_b(\bar{b}) < t f(r_1) + (1-t) f(r_2). 
    \end{equation}
    Since $(\bar{r}, f(\bar{r}))$ is Pareto efficient, it must be that $f(\bar{r})<\mathcal{R}_b(\bar{b})$, otherwise $e(\bar{b})$ Pareto dominates $(\bar{r}, f(\bar{r}))$. Therefore, $f(\bar{r}) < t f(r_1) + (1-t) f(r_2)$, which establishes strict convexity of the FA frontier in the GLM setting. 
\end{proof}
\subsubsection{Proof of \Cref{lemma:glm_group_balanced}}\label{proof:lemma:glm_group_balanced}
For any $g \in \mathcal{G}$ and $\beta \in \Theta_B$, by \eqref{eqn:conditional_mean_variance_glm}, 
\begin{equation*}
    \mathcal{R}_g(\beta) - \mathcal{R}_g(\beta_g) = \EE_g\left[ A(X^\top \beta) - A(X^\top \beta_g) - A'(X^\top \beta_g) X^\top (\beta-\beta_g) \right]. 
\end{equation*}
By Taylor's theorem and \eqref{eqn:glm_abs_bounds_a}, 
\begin{equation*}
    \hspace{-1cm}A(X^\top \beta ) - A(X^\top \beta_g) - A'(X^\top \beta_g )X^\top (\beta -\beta_g) = (X^\top (\beta -\beta_g))^2 \int_{0}^1 (1-t)A''(X^\top \beta_g  + t(X^\top (\beta -\beta_g)))  dt \geq \frac{\underline{a}}{2} (X^\top(\beta-\beta_g))^2. 
\end{equation*}
Hence, taking expectations, 
\begin{equation}\label{eqn:glm_group_balance_inq_gen}
    \mathcal{R}_g(\beta) - \mathcal{R}_g(\beta_g) \geq \frac{\underline{a}}{2}  \norm{\beta-\beta_g}_{\Sigma_g}^2.
\end{equation}
Applying \eqref{eqn:glm_group_balance_inq_gen} to group $b$ with $\beta = \beta_r$ together with the upper inequality of \eqref{eqn:glm_group_balanced_condition} yields
\begin{equation}
    \mathcal{R}_r(\beta_r) - \mathcal{R}_b(\beta_b) \leq \frac{\underline{a}}{2} \norm{\beta_r - \beta_b}_{\Sigma_b}^2\leq \mathcal{R}_b(\beta_r) - \mathcal{R}_b(\beta_b);
\end{equation}
hence $\mathcal{R}_r(\beta_r) \leq \mathcal{R}_b(\beta_r)$. A symmetric argument yields $\mathcal{R}_b(\beta_b) \leq \mathcal{R}_r(\beta_b)$, and the claim follows. $\qed$. 
\subsubsection{Proof of \Cref{proposition:excess_risk_to_matrix_norm_glm}}\label{proof:proposition:excess_risk_to_matrix_norm_glm}
Throughout, we write $\delta = \widehat{\beta}-\beta_\lambda$. For $t \in [0, 1]$ define $\beta_\lambda(t) \coloneq \beta_\lambda + t \delta$ and $\mathfrak{r}_g(t) \coloneq \mathcal{R}_g(\beta_\lambda(t))$. By convexity, $\norm{\beta_\lambda(t)} \leq B$. 
By the chain rule,  
\begin{equation}
    \hspace{-1cm} \mathfrak{r}'_g(t) = \EE_g\left[\left\{ A'(X^\top \beta_\lambda(t)) - Y\right\}X^\top \delta \right], \quad \mathfrak{r}''_g(t) = \EE_g\left[ A''(X^\top \beta_\lambda(t))(X^\top \delta)^2 \right], \quad \mathfrak{r}'''_g(t) = \EE_g\left[ A'''(X^\top \beta_\lambda(t))(X^\top \delta)^3\right].
\end{equation}
In particular, 
\begin{equation*}
    \mathfrak{r}'_g(0) = \nabla\mathcal{R}_g(\beta_\lambda)^\top \delta, \quad \mathfrak{r}''_g(0) = \norm{\delta}_{\mathcal{I}_g(\beta_\lambda)}^2. 
\end{equation*}
Then, a second-order Taylor expansion with integral remainder gives 
\begin{equation}\label{eqn:glm_risk_decomp_g_taylor}
    \mathcal{R}_g(\widehat{\beta}) - \mathcal{R}_g(\beta_\lambda) = \nabla\mathcal{R}_g(\beta_\lambda)^\top \delta + \frac{1}{2} \norm{\delta}_{\mathcal{I}_g(\beta_\lambda)}^2 + \frac{1}{2} \int_{0}^1 (1-t)^2 \mathfrak{r}'''_g(t) dt.  
\end{equation}
Bound the magnitude of the third term by applying \eqref{eqn:glm_abs_bounds_a} to get $|\mathfrak{r}'''_g(t)|  \leq L_3 \EE_g\left[ |X^\top \delta|^3 \right]$ and $\int_{0}^1(1-t)^2 dt = \frac{1}{3}$ to get 
\begin{equation*}
    \mathcal{R}_g(\widehat{\beta}) - \mathcal{R}_g(\beta_\lambda) \leq \nabla\mathcal{R}_g(\beta_\lambda)^\top \delta  + \frac{1}{2} \norm{\delta}_{\mathcal{I}_g(\beta_\lambda)}^2 + \frac{L_3}{6}  \EE_g\left[ |X^\top \delta|^3 \right].
\end{equation*}
Applying \Cref{assumption:bounded_covariates_glm} establishes \eqref{eqn:glm_excess_risk_to_matrix_norm_glm_group} and the corresponding remainder bound.  Next, for $\mathcal{R}_\lambda$, applying \eqref{eqn:glm_risk_decomp_g_taylor} groupwise and combining yields 
\begin{equation*}
    \mathcal{R}_\lambda(\widehat{\beta}) - \mathcal{R}_\lambda(\beta_\lambda) = \nabla \mathcal{R}_\lambda(\beta_\lambda)^\top \delta + \frac{1}{2} \norm{\delta}_{\mathcal{I}_\lambda(\beta_\lambda)}^2 + \frac{1}{2} \int_{0}^1 (1-t)^2 \left( \lambda \mathfrak{r}'''_r(t)  + (1-\lambda) \mathfrak{r}'''_b(t)\right) dt. 
\end{equation*}
Bound the third integral term as in the groupwise expansion above and recall that, under $\norm{\beta_\lambda} \leq B - \delta_B$ for some $\delta_B > 0$, we have $\nabla \mathcal{R}_\lambda(\beta_\lambda) =0$. This yields \eqref{eqn:glm_excess_risk_to_matrix_norm_glm_lambda}. $\qed$
\subsubsection{Full Statement and Proof of \Cref{thrm:glm_per_group}}\label{sec:proof:thrm:bias_variance_glm}
We begin by establishing the following result pertaining to the groupwise GLM risk.

\begin{theoremprime}{thrm:glm_per_group}\label{thrm:groupwise_bias_variance_glm_full}
    Fix $\lambda \in [0,1]$. Suppose Assumptions \ref{assumption:invertible_cov}, \ref{assumption:bounded_covariates_glm}, \ref{assumption:bounded_residual_glm}, \ref{assumption:cov_lambda_glm}, and \ref{assumption:interior_glm} hold. Suppose also that $d \geq 2$ and $\min\{n_r, n_b\} \geq \log(2ed)$. Let $\check{\beta}_\lambda$ denote the minimizer defined in \eqref{eqn:glm_m_estimator}.
    For each group $g \in \mathcal{G}$, define
    \begin{equation*}
        \mathcal{M}_{g}(\beta_\lambda) \coloneq {\mathcal{I}}_\lambda^{-1}(\beta_\lambda)\mathcal{I}_g(\beta_\lambda) {\mathcal{I}}_\lambda^{-1}(\beta_\lambda), \quad \mathcal{D}_g^{\textrm{GLM}}(\lambda) \coloneq \EE\left[ S_g(\beta_\lambda)^\top (\check{\beta}_\lambda - \beta_\lambda) \right]. 
    \end{equation*}
    Then, for each group $g \in \mathcal{G}$, 
    \begin{equation}\label{eqn:glm_exact_group_g_risk_decomp}
        \EE\left[ \mathcal{R}_g(\check{\beta}_\lambda) - \mathcal{R}_g(\beta_\lambda)\right] = \frac{1}{2} \left\{ \mathcal{V}_g^{\textrm{GLM}}(\lambda) + \mathcal{B}_g^{\textrm{GLM}}(\lambda) \right\} + \mathcal{D}_g^{\textrm{GLM}}(\lambda) + \operatorname{Rem}_g^{\textrm{GLM}}(\lambda)
    \end{equation}
    where 
    \begin{equation}\label{eqn:glm_variance}
        \mathcal{V}_g^{\textrm{GLM}}(\lambda) \coloneq \lambda^2 \frac{\varphi_r}{n_r}\operatorname{Tr}\left\{\mathcal{M}_{g}(\beta_\lambda) \mathcal{I}_r(\beta_r) \right\} + (1-\lambda)^2\frac{\varphi_b}{n_b} \operatorname{Tr}\left(\mathcal{M}_{g}(\beta_\lambda)   \mathcal{I}_b(\beta_b) \right),
    \end{equation}
    \begin{equation}\label{eqn:glm_bias_exact}
        \mathcal{B}_g^{\textrm{GLM}}(\lambda) \coloneq \lambda^2 \frac{1}{n_r}\operatorname{Tr}\left\{\mathcal{M}_{g}(\beta_\lambda) \Gamma_r(\beta_\lambda) \right\} + (1-\lambda)^2\frac{1}{n_b} \operatorname{Tr}\left(\mathcal{M}_{g}(\beta_\lambda)   \Gamma_b(\beta_\lambda) \right),
    \end{equation}
    \begin{equation}\label{eqn:variance_inner_term}
        \Gamma_g(\beta_\lambda) \coloneq \operatorname{Var}_g\left(  (A'(X^\top \beta_\lambda)  - A'(X^\top \beta_g))X \right). 
    \end{equation}

    The terms $\mathcal{D}_g^{\textrm{GLM}}(\lambda)$ satisfy
    \begin{equation}\label{eqn:glm_redistribution_term}
        \lambda \mathcal{D}_r^{\textrm{GLM}}(\lambda) + (1-\lambda) \mathcal{D}_b^{\textrm{GLM}}(\lambda) = 0. 
    \end{equation}
    Moreover, 
    \begin{enumerate}
        \item There exist constants $C_g, c>0$ depending on the GLM regularity parameters but not on $n_r, n_b$ such that 
    \begin{equation}\label{eqn:glm_remainder_bound_final}
        \left| \operatorname{Rem}_g^{\textrm{GLM}}(\lambda)\right| \leq C_g\left\{ \psi_\lambda^3 \sqrt{\log(2ed)} + \psi_\lambda^4 \log(2ed) + d \exp\left( \frac{-c}{n_\lambda}\right) \right\} 
    \end{equation}
    where 
    \begin{equation}\label{eqn:n_lambda_psi_lambda_def}
        n_\lambda \coloneq \max\left\{ \frac{\lambda}{n_r}, \frac{1-\lambda}{n_b}\right\}, \psi_\lambda \coloneq \frac{\lambda}{\sqrt{n_r}}  + \frac{1-\lambda}{\sqrt{n_b}},
    \end{equation}
    \item There exist constants $C_{B, g} > 0$ such that 
    \begin{equation}\label{eqn:glm_bias_bound}
        0 \leq \mathcal{B}_g^{\textrm{GLM}}(\lambda) \leq C_{B, g} \lambda^2 (1-\lambda)^2 \left( \frac{1}{n_r} + \frac{1}{n_b}\right) \norm{\beta_r - \beta_b}^2, 
    \end{equation}
    \item If $n_r \asymp n_b \asymp n$ and $d$ is fixed, then 
    \begin{equation}\label{eqn:glm_rate}
        \EE\left[ \mathcal{R}_g(\check{\beta}_\lambda) - \mathcal{R}_g(\beta_\lambda)\right] = \frac{1}{2} \left\{ \mathcal{V}_g^{\textrm{GLM}}(\lambda) + \mathcal{B}_g^{\textrm{GLM}}(\lambda) \right\} + \mathcal{D}_g^{\textrm{GLM}}(\lambda) + O(n^{-3/2})
    \end{equation}
    where $|\mathcal{D}_g^{\textrm{GLM}}(\lambda)| = O(n^{-1})$ and $\mathcal{V}_g^{\textrm{GLM}}(\lambda) + \mathcal{B}_g^{\textrm{GLM}}(\lambda) = O(n^{-1})$. 
    \end{enumerate}
\end{theoremprime}
\subsubsection{Proof of \Cref{thrm:groupwise_bias_variance_glm_full}}
Throughout write $\delta_\lambda \coloneq \check{\beta}_\lambda - \beta_\lambda$ and let $\mathcal{A}_\lambda$ denote the event defined in \eqref{eqn:hp_event_lambda_glm_def}. 
By \Cref{proposition:excess_risk_to_matrix_norm_glm},  
\begin{equation}\label{eqn:glm_group_g_risk}
    \EE\left[ \mathcal{R}_g(\check{\beta}_\lambda) - \mathcal{R}_g(\beta_\lambda)\right] = \frac{1}{2} \EE\left[\norm{\delta_\lambda}_{\mathcal{I}_g(\beta_\lambda)}^2\right] + \EE\left[\nabla_\beta \mathcal{R}_g(\beta_\lambda)^\top(\delta_\lambda)\right]+ \EE\left[r_g(\check{\beta}_\lambda)\right]. 
\end{equation}
Define 
\begin{equation}\label{eqn:glm_I_bar}
    \bar{\mathcal{I}}_\lambda \coloneq \int_{0}^1 \widehat{\mathcal{I}}_\lambda(\beta_\lambda + t \delta_\lambda) dt. 
\end{equation}
By \Cref{lemma:glm_delta_lambda_bound_linearization}, on $\mathcal{A}_\lambda$, we have $\widehat S_\lambda(\check\beta_\lambda)=0$ and that $\bar{\mathcal{I}}_\lambda$ is invertible. This yields, on $\mathcal{A}_\lambda$,
\begin{equation}\label{eqn:glm_linearization}
    \delta_\lambda = - \bar{\mathcal{I}}_\lambda^{-1} \widehat{S}_\lambda (\beta_\lambda). 
\end{equation}
We further decompose the preceding linearization on $\mathcal{A}_\lambda$ as
\begin{equation}\label{eqn:GLM_delta_W_Q_lambda_def}
    \delta_\lambda  = \mathcal{W}_\lambda + \mathcal{Q}_\lambda, \quad \mathcal{W}_\lambda \coloneq -{\mathcal{I}}_\lambda^{-1}(\beta_\lambda) \widehat{S}_\lambda (\beta_\lambda), \quad \mathcal{Q}_\lambda \coloneq \left( {\mathcal{I}}_\lambda(\beta_\lambda)^{-1}  - \bar{\mathcal{I}}_\lambda^{-1} \right) \widehat{S}_\lambda (\beta_\lambda).
\end{equation}

\paragraph{Leading term:} To evaluate the first term of \eqref{eqn:glm_group_g_risk}, expand the $\mathcal{I}_g(\beta_\lambda)$-norm as
\begin{equation}\label{eqn:glm_group_g_beta_gap}
    \EE\left[\norm{\delta_\lambda}_{\mathcal{I}_g(\beta_\lambda)}^2 \mathbbm{1}\{\mathcal{A}_\lambda\}\right] = \EE\left[ \norm{\mathcal{W}_\lambda}_{\mathcal{I}_g(\beta_\lambda)}^2\mathbbm{1}\{\mathcal{A}_\lambda\}\right] + \EE\left[ \norm{\mathcal{Q}_\lambda}_{\mathcal{I}_g(\beta_\lambda)}^2\mathbbm{1}\{\mathcal{A}_\lambda\}\right] + 2 \EE\left[\langle \mathcal{W}_\lambda,\mathcal{Q}_\lambda \rangle_{\mathcal{I}_g(\beta_\lambda)} \mathbbm{1}\{\mathcal{A}_\lambda\}\right]. 
\end{equation}
We bound these three terms of \eqref{eqn:glm_group_g_beta_gap} separately. By \Cref{lemma:glm_leading_term_bound}
\begin{equation}\label{eqn:glm_w_A_bound}
    \EE\left[ \norm{\mathcal{W}_\lambda}_{\mathcal{I}_g(\beta_\lambda)}^2\mathbbm{1}\{\mathcal{A}_\lambda\}\right]  = \mathcal{V}_g^{\textrm{GLM}}(\lambda) + \mathcal{B}_g^{\textrm{GLM}}(\lambda) - \EE\left[ \norm{\mathcal{W}_\lambda}_{\mathcal{I}_g(\beta_\lambda)}^2\mathbbm{1}\{\mathcal{A}_\lambda^c\}\right]. 
\end{equation}

Next we bound the second term of \eqref{eqn:glm_group_g_beta_gap}. By the resolvent identity $A^{-1} - B^{-1} = A^{-1} (B-A) B^{-1}$ and \eqref{eqn:glm_linearization}, we have 
\begin{equation}\label{eqn:glm_risk_q_decomp1}
    \begin{aligned}
        \mathcal{Q}_\lambda &= -  {\mathcal{I}}_\lambda(\beta_\lambda)^{-1}\left( \bar{\mathcal{I}}_\lambda - {\mathcal{I}}_\lambda(\beta_\lambda)  \right) \delta_\lambda
        &=-{\mathcal{I}}_\lambda(\beta_\lambda)^{-1} \left[ (\widehat{\mathcal{I}}_\lambda(\beta_\lambda)-{\mathcal{I}}_\lambda(\beta_\lambda))  + (\bar{\mathcal{I}}_\lambda-\widehat{\mathcal{I}}_\lambda(\beta_\lambda)) \right] \delta_\lambda
    \end{aligned}
\end{equation}
where the equality uses $\bar{\mathcal{I}}_\lambda - {\mathcal{I}}_\lambda(\beta_\lambda) = (\widehat{\mathcal{I}}_\lambda(\beta_\lambda)-{\mathcal{I}}_\lambda(\beta_\lambda)) + (\bar{\mathcal{I}}_\lambda-\widehat{\mathcal{I}}_\lambda(\beta_\lambda))$. Hence, triangle inequality and submultiplicativity give 
\begin{equation}\label{eqn:Q_norm_decomp}
     \norm{\mathcal{Q}_\lambda}_{\mathcal{I}_g(\beta_\lambda)}^2 \leq 2  K_{g, \lambda} \left( \norm{\widehat{\mathcal{I}}_\lambda(\beta_\lambda)-{\mathcal{I}}_\lambda(\beta_\lambda)}^2 + \norm{\bar{\mathcal{I}}_\lambda-\widehat{\mathcal{I}}_\lambda(\beta_\lambda)}^2 \right)\norm{\delta_\lambda}^2. 
\end{equation}
where $K_{g, \lambda} \coloneq \norm{\mathcal{I}_g(\beta_\lambda)}\norm{{\mathcal{I}}_\lambda(\beta_\lambda)^{-1}}^2$. 

Taking expectations and applying \Cref{lemma:I_bar_norm_bound} with the shorthand $\Delta_{\mathcal{I}, \lambda} \coloneq \widehat{\mathcal{I}}_\lambda(\beta_\lambda)-{\mathcal{I}}_\lambda(\beta_\lambda)$, 
\begin{equation}\label{eqn:Q_norm_decomp_2}
     \hspace{-1cm}\EE\left[ \norm{\mathcal{Q}_\lambda}_{\mathcal{I}_g(\beta_\lambda)}^2 \mathbbm{1}\{\mathcal{A}_\lambda\} \right] \leq 2  K_{g, \lambda} \left[ \EE\left[ \mathbbm{1}\{\mathcal{A}_\lambda\} \norm{\Delta_{\mathcal{I}, \lambda}}^2 \norm{\delta_\lambda}^2\right] + \frac{L_3^2}{4} T_X^6 \EE\left[ \mathbbm{1}\{\mathcal{A}_\lambda\} \norm{\delta_\lambda}^4\right] \right]. 
\end{equation}

By Cauchy-Schwarz, 
\begin{equation}\label{eqn:Q_norm_decomp_2_sub1}
    \EE\left[ \mathbbm{1}\{\mathcal{A}_\lambda\} \norm{\Delta_{\mathcal{I}, \lambda}}^2 \norm{\delta_\lambda}^2\right] \leq \left( \EE\left[\norm{\Delta_{\mathcal{I}, \lambda}}^4 \right]\right)^{1/2}\left(\EE\left[\mathbbm{1}\{\mathcal{A}_\lambda\} \norm{\delta_\lambda}^4 \right] \right)^{1/2}.
\end{equation}
By \Cref{lemma:glm_delta_lambda_bound_linearization},
\begin{equation}\label{eqn:Q_norm_decomp_2_sub2}
    \left(\EE\left[\mathbbm{1}\{\mathcal{A}_\lambda\} \norm{\delta_\lambda}^4 \right] \right)^{1/2} \leq \frac{4}{\underline{a}^2 \rho_\lambda^4} \left(\EE\left[ \norm{\widehat{S}_\lambda(\beta_\lambda)}^4 \right] \right)^{1/2}. 
\end{equation}
By \Cref{lemma:glm_score_moments} with $p = 4$, 
\begin{equation}\label{eqn:Q_norm_decomp_2_sub3}
    \EE\left[ \norm{\widehat{S}_\lambda(\beta_\lambda)}^4\right] \leq C_1 T_X^4 \tilde{\psi}_\lambda^4, \quad \tilde{\psi}_\lambda \coloneq \frac{\lambda \tau_r}{\sqrt{n_r}} +  \frac{(1-\lambda) \tau_b}{\sqrt{n_b}}
\end{equation}
for an absolute constant $C_1>0$ where we recall $\tau_g \coloneq \bar{a} T_X \norm{\beta_\lambda - \beta_g} + \sigma_{\varepsilon, g}$. 

By \Cref{lemma:glm_hessian_concentration}, for some absolute constant $C_3  >0$, 
\begin{equation}\label{eqn:Q_norm_decomp_2_sub4}
    \left( \EE\left[ \norm{\Delta_{\mathcal{I}, \lambda}}^4 \right] \right)^{1/2} \leq C_3 \bar{a}^2 T_X^4 \log(2ed)\left(  \frac{\lambda^2}{n_r}   +  \frac{(1-\lambda)^2}{n_b}\right). 
\end{equation}
Combining \eqref{eqn:Q_norm_decomp_2}-\eqref{eqn:Q_norm_decomp_2_sub4} yields 
\begin{equation}\label{eqn:Q_norm_final}
    \EE\left[ \norm{\mathcal{Q}_\lambda}_{\mathcal{I}_g(\beta_\lambda)}^2 \mathbbm{1}\{\mathcal{A}_\lambda\} \right] \leq \mathcal{Z}_g(\lambda)
\end{equation}
where
\begin{equation*}
    \mathcal{Z}_g(\lambda) \coloneq C_2  K_{g, \lambda} \left[ \frac{\bar{a}^2 T_X^6 \log(2ed)}{\underline{a}^2 \rho_\lambda^4}  \tilde{\psi}_\lambda^2  q_\lambda  + \frac{L_3^2 T_X^{10}}{\underline{a}^4 \rho_\lambda^8}   \tilde{\psi}_\lambda^4 \right], \quad q_\lambda \coloneq   \frac{\lambda^2}{n_r}   +  \frac{(1-\lambda)^2}{n_b}
\end{equation*}
for an absolute constant $C_2>0$.

For the cross term in \eqref{eqn:glm_group_g_beta_gap},  by Cauchy Schwarz,
\begin{equation}\label{eqn:Q_W_cross_term}
    \begin{aligned}
        \left| \EE\left[\langle \mathcal{W}_\lambda,\mathcal{Q}_\lambda \rangle_{\mathcal{I}_g(\beta_\lambda)}\mathbbm{1}\{\mathcal{A}_\lambda\}  \right]\right|  &\leq \sqrt{\EE\left[ \norm{\mathcal{W}_\lambda}_{\mathcal{I}_g(\beta_\lambda)}^2 \mathbbm{1}\{\mathcal{A}_\lambda\}\right]}\sqrt{\EE\left[ \norm{\mathcal{Q}_\lambda}_{\mathcal{I}_g(\beta_\lambda)}^2 \mathbbm{1}\{\mathcal{A}_\lambda\}\right]}\\
        &\leq \sqrt{\mathcal{V}_g^{\textrm{GLM}}(\lambda) + \mathcal{B}_g^{\textrm{GLM}}(\lambda)}\sqrt{\mathcal{Z}_g(\lambda)} . 
    \end{aligned} 
\end{equation}

On the complementary event $\mathcal{A}_\lambda^c$, we use the feasibility bound $\norm{\delta_\lambda} \leq 2B$. Consequently, 
\begin{equation}\label{eqn:glm_group_g_beta_gap_3}
    \EE\left[\norm{\delta_\lambda}_{\mathcal{I}_g(\beta_\lambda)}^2 \mathbbm{1}\{\mathcal{A}_\lambda^c\}\right] \leq 4B^2 \norm{\mathcal{I}_g(\beta_\lambda)} \PP(\mathcal{A}_\lambda^c) \leq 4B^2 \norm{\mathcal{I}_g(\beta_\lambda)} \Pi_\lambda
\end{equation}
where the second inequality applies \Cref{lemma:glm_high_probability_event_bound}: 
\begin{equation*}
    \PP(\mathcal{A}_\lambda^c) \leq \left(\frac{8 T_X\tilde{\psi}_\lambda }{\underline{a} \rho_\lambda^2 \delta_B}\right)^4 + d\exp\left( -\frac{\rho_\lambda^2 ( {1-\log(2)})}{2T_X^2 n_\lambda } \right) \eqqcolon \Pi_\lambda. 
\end{equation*}

Since 
\begin{equation*}
    \norm{\mathcal{W}_\lambda}_{\mathcal{I}_g(\beta_\lambda)}^2 = \widehat{S}_\lambda(\beta_\lambda)^\top \mathcal{M}_g(\beta_\lambda) \widehat{S}_\lambda(\beta_\lambda) \leq K_{g, \lambda} \norm{\widehat{S}_\lambda(\beta_\lambda)}^2,
\end{equation*}
we apply \Cref{lemma:glm_score_moments} with $p = 4$ to obtain that, for some constant $C_3>0$
\begin{equation}
    \left( \EE\left[ \norm{\mathcal{W}_\lambda}_{\mathcal{I}_g(\beta_\lambda)}^4\right]\right)^{1/2}  \leq K_{g, \lambda} \left( \EE\left[ \norm{\widehat{S}_\lambda(\beta_\lambda)}^4 \right]\right)^{1/2} \leq C_3 K_{g, \lambda} T_X^2 \tilde{\psi}_\lambda^2. 
\end{equation}
Thus, by Cauchy Schwarz and \Cref{lemma:glm_high_probability_event_bound}, 
\begin{equation}\label{eqn:glm_w_Ac_bound}
    \EE\left[ \norm{\mathcal{W}_\lambda}_{\mathcal{I}_g(\beta_\lambda)}^2 \mathbbm{1}\{\mathcal{A}_\lambda^c\}\right] \leq \left( \EE\left[ \norm{\mathcal{W}_\lambda}_{\mathcal{I}_g(\beta_\lambda)}^4\right]\right)^{1/2} \left( \PP(\mathcal{A}_\lambda^c)\right)^{1/2} \leq C_3 K_{g, \lambda} T_X^2 \tilde{\psi}_\lambda^2 \sqrt{\Pi_\lambda}. 
\end{equation}

Next, recalling \eqref{eqn:glm_group_g_beta_gap} and \eqref{eqn:glm_w_A_bound}, we obtain the exact identity
\begin{equation}\label{eqn:delta_lambda_A_lambda_decomp_exact}
    \begin{aligned}
        \EE\left[\norm{\delta_\lambda}_{\mathcal{I}_g(\beta_\lambda)}^2 \right] &= \EE\left[\norm{\delta_\lambda}_{\mathcal{I}_g(\beta_\lambda)}^2 \mathbbm{1}\{\mathcal{A}_\lambda\} \right] + \EE\left[\norm{\delta_\lambda}_{\mathcal{I}_g(\beta_\lambda)}^2 \mathbbm{1}\{\mathcal{A}_\lambda^c\} \right]\\
        &= \mathcal{V}_g^{\textrm{GLM}}(\lambda) + \mathcal{B}_g^{\textrm{GLM}}(\lambda) + \EE\left[ \norm{\mathcal{Q}_\lambda}_{\mathcal{I}_g(\beta_\lambda)}^2\mathbbm{1}\{\mathcal{A}_\lambda\}\right] + 2 \EE\left[\langle \mathcal{W}_\lambda,\mathcal{Q}_\lambda \rangle_{\mathcal{I}_g(\beta_\lambda)} \mathbbm{1}\{\mathcal{A}_\lambda\}\right]\\
         \quad &+\EE\left[\left\{ \norm{\delta_\lambda}_{\mathcal{I}_g(\beta_\lambda)}^2 - \norm{\mathcal{W}_\lambda}_{\mathcal{I}_g(\beta_\lambda)}^2\right\} \mathbbm{1}\{\mathcal{A}_\lambda^c\} \right]. 
    \end{aligned}
\end{equation}
By triangle inequality, \eqref{eqn:glm_group_g_beta_gap_3}, and \eqref{eqn:glm_w_Ac_bound},  
\begin{equation}\label{eqn:glm_Ac_gap_triangle}
    \begin{aligned}
        \left| \EE\left[\left\{ \norm{\delta_\lambda}_{\mathcal{I}_g(\beta_\lambda)}^2 - \norm{\mathcal{W}_\lambda}_{\mathcal{I}_g(\beta_\lambda)}^2\right\} \mathbbm{1}\{\mathcal{A}_\lambda^c\} \right] \right| &\leq \EE\left[ \norm{\delta_\lambda}_{\mathcal{I}_g(\beta_\lambda)}^2\mathbbm{1}\{\mathcal{A}_\lambda^c\}\right]  + \EE\left[ \norm{\mathcal{W}_\lambda}_{\mathcal{I}_g(\beta_\lambda)}^2\mathbbm{1}\{\mathcal{A}_\lambda^c\}\right]\\
        &\leq 4B^2 \norm{\mathcal{I}_g(\beta_\lambda)} \Pi_\lambda + C_3 K_{g, \lambda} T_X^2 \tilde{\psi}_\lambda^2 \sqrt{\Pi_\lambda}. 
    \end{aligned}
\end{equation}

Therefore, by \eqref{eqn:Q_norm_final}, \eqref{eqn:Q_W_cross_term}, and \eqref{eqn:glm_Ac_gap_triangle}, 
\begin{equation}
    \begin{aligned}
        \EE\left[\norm{\delta_\lambda}_{\mathcal{I}_g(\beta_\lambda)}^2 \right] &\leq \mathcal{V}_g^{\textrm{GLM}}(\lambda) + \mathcal{B}_g^{\textrm{GLM}}(\lambda) + \mathcal{Z}_g(\lambda) + 2 \sqrt{\mathcal{V}_g^{\textrm{GLM}}(\lambda) + \mathcal{B}_g^{\textrm{GLM}}(\lambda)}\sqrt{\mathcal{Z}_g(\lambda)}\\
         \quad &+4B^2 \norm{\mathcal{I}_g(\beta_\lambda)} \Pi_\lambda + C_3 K_{g, \lambda} T_X^2 \tilde{\psi}_\lambda^2 \sqrt{\Pi_\lambda}.
    \end{aligned}
\end{equation}
This concludes the bound of the first term of \eqref{eqn:glm_group_g_risk}.

\paragraph{Redistributive term:} The second term of \eqref{eqn:glm_group_g_risk} is exactly 
\begin{equation}
    \EE\left[\nabla_\beta \mathcal{R}_g(\beta_\lambda)^\top(\delta_\lambda)\right] = \mathcal{D}_g^{\textrm{GLM}}(\lambda). 
\end{equation}
Since $S_\lambda(\beta_\lambda) = 0$ under \Cref{assumption:interior_glm},  it follows that 
\begin{equation}
    \lambda \mathcal{D}_r^{\textrm{GLM}}(\lambda) + (1-\lambda) \mathcal{D}_b^{\textrm{GLM}}(\lambda)= \EE\left[ S_\lambda(\beta_\lambda)^\top \delta_\lambda \right] = 0,
\end{equation}
establishing \eqref{eqn:glm_redistribution_term}.

To analyze the order of the second term, use the shorthand $w_g \coloneq \nabla_\beta \mathcal{R}_g(\beta_\lambda)$. Using $\EE[\mathcal{W}_\lambda] = 0$ and decomposing on $\mathcal{A}_\lambda$, 
\begin{equation*}
    \EE\left[w_g^\top \delta_\lambda \right] = \EE\left[w_g^\top \mathcal{Q}_\lambda \mathbbm{1}\{\mathcal{A}_\lambda\} \right] + \EE\left[w_g^\top (\delta_\lambda-\mathcal{W}_\lambda) \mathbbm{1}\{\mathcal{A}_\lambda^c\} \right] 
\end{equation*}
where we also used the previously established observation that, on $\mathcal{A}_\lambda$ we have $\delta_\lambda - \mathcal{W}_\lambda = \mathcal{Q}_\lambda$. For the first term, by two applications of Cauchy-Schwarz, 
\begin{equation*}
    \EE\left[w_g^\top \mathcal{Q}_\lambda \mathbbm{1}\{\mathcal{A}_\lambda\} \right] \leq \norm{w_g}_{\mathcal{I}_g^{-1}(\beta_\lambda)} \EE\left[ \norm{\mathcal{Q}_\lambda}_{\mathcal{I}_g(\beta_\lambda)} \mathbbm{1}\{\mathcal{A}_\lambda\} \right] \leq \norm{w_g}_{\mathcal{I}_g^{-1}(\beta_\lambda)} \left(\EE\left[ \norm{\mathcal{Q}_\lambda}_{\mathcal{I}_g(\beta_\lambda)}^2 \mathbbm{1}\{\mathcal{A}_\lambda\} \right] \right)^{1/2} \PP(\mathcal{A}_\lambda)^{1/2}. 
\end{equation*}
By \eqref{eqn:Q_norm_final} and using $\PP(\mathcal{A}_\lambda)^{1/2} \leq 1$, we obtain the bound 
\begin{equation*}
    \EE\left[w_g^\top \mathcal{Q}_\lambda \mathbbm{1}\{\mathcal{A}_\lambda\} \right] \leq \norm{w_g}_{\mathcal{I}_g^{-1}(\beta_\lambda)}\sqrt{\mathcal{Z}_g(\lambda)}. 
\end{equation*}
For the second term, first bound via Cauchy-Schwarz
\begin{equation*}
    \EE\left[w_g^\top \delta_\lambda \mathbbm{1}\{\mathcal{A}_\lambda^c\} \right] \leq \norm{w_g} \EE\left[ \norm{\delta_\lambda} \mathbbm{1}\{\mathcal{A}_\lambda^c\}\right] \leq 2 B \norm{w_g} \Pi_\lambda 
\end{equation*}
where the second inequality uses the feasibility condition $\norm{\delta_\lambda} \leq 2B$ and \Cref{lemma:glm_high_probability_event_bound}. Next, by Cauchy-Schwarz, 
\begin{equation*}
    \EE\left[w_g^\top \mathcal{W}_\lambda \mathbbm{1}\{\mathcal{A}_\lambda^c\} \right] \leq \norm{w_g}_{\mathcal{I}_g^{-1}(\beta_\lambda)} \EE\left[ \norm{\mathcal{W}_\lambda}_{\mathcal{I}_g(\beta_\lambda)} \mathbbm{1}\{\mathcal{A}_\lambda^c\} \right] \leq \norm{w_g}_{\mathcal{I}_g^{-1}(\beta_\lambda)} \EE\left[ \norm{\mathcal{W}_\lambda}^2_{\mathcal{I}_g(\beta_\lambda)} \mathbbm{1}\{\mathcal{A}_\lambda^c\} \right] ^{1/2}\PP(\mathcal{A}_\lambda^c)^{1/2}. 
\end{equation*}
Thus, by  \Cref{lemma:glm_high_probability_event_bound},  
\begin{equation*}
    \left| \EE\left[w_g^\top \mathcal{W}_\lambda \mathbbm{1}\{\mathcal{A}_\lambda^c\} \right] \right| \leq \norm{w_g}_{\mathcal{I}_g^{-1}(\beta_\lambda)}  \sqrt{\mathcal{V}_g^{\textrm{GLM}}(\lambda) + \mathcal{B}_g^{\textrm{GLM}}(\lambda)} \sqrt{\Pi_\lambda}.
\end{equation*}
By \eqref{eqn:glm_abs_bounds_a}, $\norm{w_g}_{\mathcal{I}_g^{-1}(\beta_\lambda)}^2 \leq \underline{a}^{-1}\norm{w_g}_{\Sigma_g^{-1}}^2$ thus, by Cauchy-Schwarz and \eqref{eqn:glm_abs_bounds_a},  
\begin{equation*}
    \norm{w_g}_{\mathcal{I}_g^{-1}(\beta_\lambda)} \leq \frac{\bar{a}}{\sqrt{\underline{a}}} \norm{\beta_\lambda - \beta_g}_{\Sigma_g}.
\end{equation*}
Similarly by \eqref{eqn:glm_abs_bounds_a}, Cauchy-Schwarz, and \Cref{assumption:bounded_covariates_glm}, 
\begin{equation*}
    \norm{w_g}\leq \bar{a} \sqrt{\norm{\Sigma_g}}\norm{\beta_\lambda - \beta_g}_{\Sigma_g} \leq \bar{a} T_X \norm{\beta_\lambda - \beta_g}_{\Sigma_g}.
\end{equation*}

Collecting bounds for each term, we reach 
\begin{equation}\label{eqn:redistributive_final}
    \left| \mathcal{D}_g^{\textrm{GLM}}(\lambda) \right| \leq \frac{\bar{a}}{\sqrt{\underline{a}}} \norm{\beta_\lambda - \beta_g}_{\Sigma_g}\left( \sqrt{\mathcal{Z}_g(\lambda)} + \sqrt{\mathcal{V}_g^{\textrm{GLM}}(\lambda) + \mathcal{B}_g^{\textrm{GLM}}(\lambda)} \sqrt{\Pi_\lambda}\right) + 2 B  \bar{a} T_X \norm{\beta_\lambda - \beta_g}_{\Sigma_g} \Pi_\lambda.   
\end{equation}

\paragraph{Remainder term:} For the third term of \eqref{eqn:glm_group_g_risk}, on the event $\mathcal{A}_\lambda$, by \Cref{proposition:excess_risk_to_matrix_norm_glm} and \Cref{lemma:glm_delta_lambda_bound_linearization},  
\begin{equation}
    \EE\left[|r_g(\check{\beta}_\lambda) |  \mathbbm{1}\{\mathcal{A}_\lambda \} \right]\leq \frac{4 L_3 T_X^3}{3 \underline{a}^3 \rho_\lambda^6 }   \EE\left[ \norm{\widehat{S}_\lambda(\beta_\lambda)}^3\right] \leq C_4 \frac{L_3 T_X^6 }{\underline{a}^3 \rho_\lambda^6 } \tilde{\psi}_\lambda^3
\end{equation}
for a constant $C_4 > 0$ where the second inequality follows from \Cref{lemma:glm_score_moments} with $p = 3$. 
Likewise, by \Cref{proposition:excess_risk_to_matrix_norm_glm},  on $\mathcal{A}_\lambda^c$, using the feasibility constraint $\norm{\delta_\lambda} \leq 2B$,
\begin{equation}
    \EE\left[|r_g(\check{\beta}_\lambda) |  \mathbbm{1}\{\mathcal{A}_\lambda^c \} \right]\leq  \frac{4 }{3}  L_3 T_X^3 B^3 \Pi_\lambda 
\end{equation}
Hence, 
\begin{equation}\label{eqn:glm_r_g_final}
    \left| \EE\left[r_g(\check{\beta}_\lambda)\right]\right| \leq C_4 \frac{L_3 T_X^6 }{\underline{a}^3 \rho_\lambda^6 }  \tilde{\psi}_\lambda^3 + \frac{4 }{3}  L_3 T_X^3 B^3 \Pi_\lambda. 
\end{equation}
Define the remainder term 
\begin{equation}
    \begin{aligned}
        \operatorname{Rem}_g^{\textrm{GLM}}(\lambda) &\coloneq \frac{1}{2} \EE\left[ \norm{\mathcal{Q}_\lambda}_{\mathcal{I}_g(\beta_\lambda)}^2\mathbbm{1}\{\mathcal{A}_\lambda\}\right] + \EE\left[\langle \mathcal{W}_\lambda,\mathcal{Q}_\lambda \rangle_{\mathcal{I}_g(\beta_\lambda)} \mathbbm{1}\{\mathcal{A}_\lambda\}\right]\\
        &\quad + \frac{1}{2} \EE\left[\left\{ \norm{\delta_\lambda}_{\mathcal{I}_g(\beta_\lambda)}^2 - \norm{\mathcal{W}_\lambda}_{\mathcal{I}_g(\beta_\lambda)}^2\right\} \mathbbm{1}\{\mathcal{A}_\lambda^c\} \right] + \EE\left[r_g(\check{\beta}_\lambda)\right]
    \end{aligned}
\end{equation}
so that, combining \eqref{eqn:glm_group_g_risk} and \eqref{eqn:delta_lambda_A_lambda_decomp_exact}, we have
\begin{equation}
    \EE\left[ \mathcal{R}_g(\check{\beta}_\lambda) - \mathcal{R}_g(\beta_\lambda)\right] = \frac{1}{2} \left\{ \mathcal{V}_g^{\textrm{GLM}}(\lambda) + \mathcal{B}_g^{\textrm{GLM}}(\lambda) \right\}  + \operatorname{Rem}_g^{\textrm{GLM}}(\lambda) + \mathcal{D}_g^{\textrm{GLM}}(\lambda), 
\end{equation}
which establishes \eqref{eqn:glm_exact_group_g_risk_decomp}. 

Apply \eqref{eqn:Q_norm_final}, \eqref{eqn:Q_W_cross_term}, \eqref{eqn:glm_Ac_gap_triangle}, and \eqref{eqn:glm_r_g_final} to bound 
\begin{equation}\label{eqn:glm_remainder_full}
    \begin{aligned}
        \left| \operatorname{Rem}_g^{\textrm{GLM}}(\lambda) \right| &\leq \frac{1}{2} \mathcal{Z}_g(\lambda) + \sqrt{\mathcal{V}_g^{\textrm{GLM}}(\lambda) + \mathcal{B}_g^{\textrm{GLM}}(\lambda)}\sqrt{\mathcal{Z}_g(\lambda)}\\
        &\quad + 2B^2 \norm{\mathcal{I}_g(\beta_\lambda)} \Pi_\lambda + \frac{1}{2} C_3 K_{g, \lambda} T_X^2 \tilde{\psi}_\lambda^2 \sqrt{\Pi_\lambda}\\
        &\quad + C_4 \frac{L_3 T_X^6 }{\underline{a}^3 \rho_\lambda^6 }  \tilde{\psi}_\lambda^3 + \frac{4}{3}  L_3 T_X^3 B^3 \Pi_\lambda. 
    \end{aligned}
\end{equation}
To simplify the preceding bound, observe $\tau_g \leq 2 \bar{a}T_X B + \max_{h} \sigma_{\varepsilon, h}$, so that $\tilde{\psi}_\lambda \leq 2 (\bar{a}T_X B + \max_{h} \sigma_{\varepsilon, h}) \psi_\lambda$. Similarly, $q_\lambda \leq \psi_\lambda^2$. 

Then, by the definition of the leading terms, $\mathcal{V}_g^{\textrm{GLM}}(\lambda) + \mathcal{B}_g^{\textrm{GLM}}(\lambda) \leq C_5 \psi_\lambda^2, $ for an absolute constant $C_5>0$  that does not depend on $n_r, n_b$. Similarly, $\mathcal{Z}_g(\lambda) \leq C_6 \log(2ed)  {\psi}_\lambda^4,$
for a constant $C_6>0$ and $\Pi_\lambda \leq C_7 \psi_\lambda^4  + d\exp\left( -\frac{ C_8}{n_\lambda } \right)$
for constants $C_7, C_8 > 0$. By Young's inequality, 
\begin{equation*}
    \tilde{\psi}_\lambda^2 \sqrt{\Pi_\lambda}\leq C_9 \psi_\lambda^2 \left(  \psi_\lambda^2  + \sqrt{d} \exp\left( -\frac{ C_8}{2 n_\lambda } \right)  \right)  \leq C_{10} \left( \psi_\lambda^4 +  d \exp\left( -\frac{ C_8}{2 n_\lambda } \right)
    \right)
\end{equation*}
for a constant $C_9> 0$.  Hence, combining constant factors into $C_g, c > 0$, and noting that $\log(2ed)>1$ for $d \geq 2$, we obtain 
\begin{equation}
    \begin{aligned}
        \left| \operatorname{Rem}_g^{\textrm{GLM}}(\lambda) \right| &\leq C_g\left\{ \psi_\lambda^3 \sqrt{ \log(2ed)} + {\psi}_\lambda^4 \log(2ed)    
          + d\exp\left( -\frac{ c}{n_\lambda } \right)  \right\} 
    \end{aligned}
\end{equation}
which establishes the claim \eqref{eqn:glm_remainder_bound_final}. 

Since $\Gamma_g(\beta_\lambda), \mathcal{I}_g(\beta_\lambda), \mathcal{M}_g(\beta_\lambda) \succeq 0$ by \eqref{eqn:glm_abs_bounds_a}, it follows that $\mathcal{B}_g^{\textrm{GLM}}(\lambda) \geq 0$ which establishes the lower bound of \eqref{eqn:glm_bias_bound}. By \eqref{eqn:glm_bias_upper} of \Cref{lemma:glm_leading_term_bound}, 
\begin{equation}
    \mathcal{B}_g^{\textrm{GLM}}(\lambda) \leq C_B \lambda^2 (1-\lambda)^2 \left( \frac{1}{n_r} + \frac{1}{n_b} \right) \norm{\beta_r - \beta_b}^2
\end{equation}
for a constant $C_B > 0$ independent of $n_r, n_b$ which establishes the upper bound of \eqref{eqn:glm_bias_bound}. 

Finally, for fixed $d$ and when $n_r \asymp n_b \asymp n$, we have 
\begin{equation*}
    \psi_\lambda = O(n^{-1/2}),  \quad \mathcal{V}_g^{\textrm{GLM}}(\lambda) = O(n^{-1}), \quad \mathcal{B}_g^{\textrm{GLM}}(\lambda)= O(n^{-1}), \quad \mathcal{Z}_g(\lambda) = O(n^{-2}), \quad \Pi_\lambda = O(n^{-2}) + O(e^{-cn}). 
\end{equation*}
Thus, by \eqref{eqn:glm_remainder_full} and \eqref{eqn:redistributive_final}, we obtain 
\begin{equation}\label{eqn:glm_remainder_rate}
    \begin{aligned}
        \left| \operatorname{Rem}_g^{\textrm{GLM}}(\lambda) \right|  = O(n^{-3/2}), \quad \mathcal{D}_g^{\textrm{GLM}}(\lambda) =  O(n^{-1}), \quad \mathcal{V}_g^{\textrm{GLM}}(\lambda) +  \mathcal{B}_g^{\textrm{GLM}}(\lambda) =O(n^{-1}), 
    \end{aligned}
\end{equation}
thereby establishing the final claim \eqref{eqn:glm_rate}. $\qed$

\end{document}